\documentclass{article}
\PassOptionsToPackage{numbers, compress}{natbib}
\usepackage[main,final]{neurips_2026}
\renewcommand{\acksection}{\section*{Acknowledgment}}

\usepackage[utf8]{inputenc} 
\usepackage[T1]{fontenc}    
\usepackage[
    colorlinks=true,   
    linkcolor=red,     
    citecolor=blue,    
    urlcolor=blue      
]{hyperref}       
\usepackage{url}            
\usepackage{booktabs}       
\usepackage{amsfonts}       
\usepackage{mathtools,amssymb}
\usepackage{amsthm}         
\newtheorem{proposition}{Proposition}
\usepackage{nicefrac}       
\usepackage{microtype}      
\usepackage{xcolor}         
\usepackage{marvosym}       
\usepackage{graphicx}       
\usepackage{booktabs}
\usepackage{multirow}
\usepackage{graphicx}
\usepackage{colortbl}
\usepackage{subcaption}
\usepackage{wrapfig}
\usepackage{needspace}
\usepackage{tikz}
\usepackage{pgfplots}
\pgfplotsset{compat=1.18}
\pgfplotsset{
  colormap={magma}{
    rgb255=(0,0,4) rgb255=(29,17,71) rgb255=(81,18,124)
    rgb255=(131,38,129) rgb255=(183,55,121) rgb255=(229,80,100)
    rgb255=(251,135,97) rgb255=(254,194,135) rgb255=(252,253,191)
  }
}

\definecolor{zoomcolor}{RGB}{255,200,0}

\definecolor{ourgreen}{RGB}{218,242,208}    
\definecolor{catband}{RGB}{240,240,246}     
\definecolor{deltapink}{RGB}{198,68,136}    
\newcommand{\tablegroup}[2]{\rowcolor{catband}\multicolumn{#1}{c}{\textit{#2}} \\}
\newcommand{\tabgain}[1]{\textcolor{deltapink}{\scriptsize(#1)}}
\newcommand{\muted}[1]{\textcolor{black!60}{#1}}

\newlength{\zcellW}

\newcommand{\zcell}[6]{%
  \begin{tikzpicture}[x=\zcellW,y=0.5625\zcellW,inner sep=0pt,outer sep=0pt]
    \node[anchor=south west,inner sep=0pt] at (0,0)
      {\includegraphics[width=\zcellW]{#1}};
    \draw[zoomcolor,line width=0.35pt] (#3,#4) rectangle (#5,#6);
    \node[anchor=south east,draw=zoomcolor,line width=0.4pt,inner sep=0pt]
      at (1,0) {\includegraphics[width=0.7\zcellW,trim={#2},clip]{#1}};
  \end{tikzpicture}%
}

\newcommand{\diffcell}[1]{%
  \begin{tikzpicture}[x=\zcellW,y=0.5625\zcellW,inner sep=0pt,outer sep=0pt]
    \useasboundingbox (0,0) rectangle (1,1);
    \node[anchor=south west,inner sep=0pt] at (0,0)
      {\includegraphics[width=\zcellW]{#1}};
    \draw[zoomcolor,line width=0.45pt] (0.01,0.70) rectangle (0.55,0.985);
    \draw[zoomcolor,line width=0.45pt] (0.77,0.27) rectangle (0.94,0.70);
  \end{tikzpicture}%
}

\definecolor{roired}{RGB}{231,76,60}

\newcommand{\fullroi}[5]{%
  \begin{tikzpicture}[x=\zcellW,y=0.5625\zcellW,inner sep=0pt,outer sep=0pt]
    \useasboundingbox (0,0) rectangle (1,1);
    \node[anchor=south west,inner sep=0pt] at (0,0)
      {\includegraphics[width=\zcellW]{#1}};
    \draw[roired,line width=0.7pt] (#2,#3) rectangle (#4,#5);
  \end{tikzpicture}%
}

\newcommand{\zoomcrop}[2]{%
  \includegraphics[width=\zcellW,trim={#2},clip]{#1}%
}

\newcommand{\cbarpanel}{%
  \begin{tikzpicture}[inner sep=0pt,outer sep=0pt]
    \useasboundingbox (0mm,0mm) rectangle (15mm,50mm);
    \node[anchor=south west,inner sep=0pt] (cb) at (0,0)
      {\includegraphics[height=47mm]{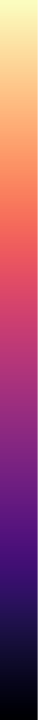}};
    \draw[line width=0.3pt] (cb.south east) -- ++(0.9mm,0)
      node[anchor=west,font=\scriptsize,inner sep=1pt] {0.0};
    \draw[line width=0.3pt] ([yshift=23.5mm]cb.south east) -- ++(0.9mm,0)
      node[anchor=west,font=\scriptsize,inner sep=1pt] {0.5};
    \draw[line width=0.3pt] (cb.north east) -- ++(0.9mm,0)
      node[anchor=west,font=\scriptsize,inner sep=1pt] {1.0};
    \node[anchor=south,inner sep=0pt]
      at ([xshift=6mm,yshift=15mm]cb.south east)
      {\rotatebox{90}{\scriptsize Pixel Difference}};
  \end{tikzpicture}%
}

\title{LiBrA-Net: Lie-Algebraic Bilateral Affine Fields for Real-Time 4K Video Dehazing}

\renewcommand{\thefootnote}{\fnsymbol{footnote}}
\author{%
  \bfseries
  Yongcong Wang$^{1}$\quad
  Chengchao Shen\textsuperscript{1,\Letter}\quad
  Guangwei Gao$^{2}$\quad
  Wei Wang$^{3}$\quad
  Pengwen Dai$^{3}$\\
  \bfseries
  Dianjie Lu$^{4}$\quad
  Guijuan Zhang$^{4}$\quad
  Zhuoran Zheng\textsuperscript{5,\Letter}\\[5pt]
  \mdseries\small
  $^{1}$\,Central South University\qquad
  $^{2}$\,Nanjing University of Science and Technology\\[1pt]
  \mdseries\small
  $^{3}$\,Sun Yat-sen University\qquad
  $^{4}$\,Shandong Normal University\\[1pt]
  \mdseries\small
  $^{5}$\,National University of Defense Technology\\[3pt]
  \mdseries\small
  \texttt{ycwang1031@gmail.com, scc.cs@csu.edu.cn, gwgao@njust.edu.cn}\\
  \mdseries\small
  \texttt{wangwei29@mail.sysu.edu.cn, daipw@mail.sysu.edu.cn}\\
  \mdseries\small
  \texttt{ludianjie@sdnu.edu.cn, zhangguijuan@sdnu.edu.cn, zhengzr@njust.edu.cn}
}

\hypersetup{
  pdftitle={LiBrA-Net: Lie-Algebraic Bilateral Affine Fields for Real-Time 4K Video Dehazing},
  pdfauthor={Yongcong Wang, Chengchao Shen, Guangwei Gao, Wei Wang, Pengwen Dai, Dianjie Lu, Guijuan Zhang, Zhuoran Zheng}
}

\begin{document}

\maketitle
\raggedbottom
\renewcommand{\thefootnote}{\Letter}
\footnotetext[1]{Corresponding authors.}
\renewcommand{\thefootnote}{\arabic{footnote}}
\setcounter{footnote}{0}

\begin{abstract}
Currently, there is a gap in the field of ultra-high-definition (UHD) video dehazing due to the lack of a benchmark for evaluation.
Furthermore, existing video dehazing methods cannot run on consumer-grade GPUs when processing continuous UHD sequences of 3--5 frames at a time.
In this paper, we address both issues with a new benchmark and an efficient method.
Our key observation is that atmospheric dehazing reduces to a per-pixel affine transform governed by the low-frequency depth field, which can be compactly encoded in bilateral grids whose prediction cost is decoupled from the output resolution.
Building on this, we propose LiBrA-Net, which factorizes the spatiotemporal affine field into a spatial--color and a temporal bilateral sub-grid predicted at a fixed low resolution, fuses their coefficients in the $\mathfrak{gl}(3)$ Lie algebra under group-theoretic regularization, maps the result to invertible $GL(3)$ transforms via a Cayley parameterization, and restores high-frequency detail through a lightweight input-guided branch.
We further release UHV-4K, the first paired 4K video dehazing benchmark with depth, transmission, and optical-flow annotations on every frame.
Across UHV-4K, REVIDE, and HazeWorld, LiBrA-Net sets a new state of the art among compared video dehazing methods while running native 4K at 25\,FPS on a single GPU with only 6.12\,M parameters. Code is available at \url{https://github.com/ycwang31/LiBrA-Net}.
\end{abstract}

\vspace{-6pt}
\section{Introduction}
\vspace{-4pt}

\begin{figure*}[!t]
  \centering
  \includegraphics[width=\textwidth]{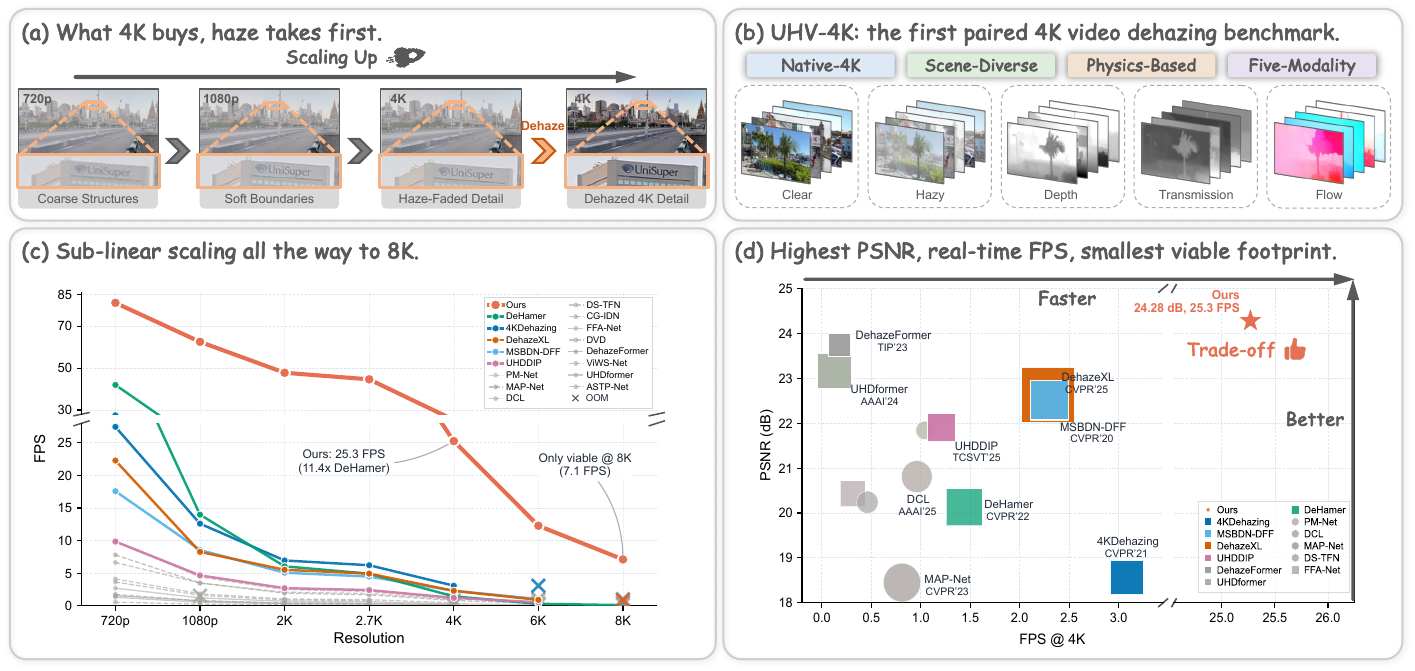}
  \caption{\textbf{(a)}~Higher resolution reveals finer scene structure that haze attenuates first.
  \textbf{(b)}~UHV-4K benchmark with five aligned modalities.
  \textbf{(c)}~Bilateral-grid prediction is resolution-decoupled: our throughput drops sub-linearly with pixel count from 720p to 8K, while dense baselines slow down roughly linearly.
  \textbf{(d)}~Quality--speed trade-off on UHV-4K; LiBrA-Net reaches 24.28\,dB at real-time 4K throughput.}
  \label{fig:teaser}
  \vspace{-4mm}
\end{figure*}

Ultra-high-definition (UHD) video at $3840{\times}2160$ is now standard for autonomous driving and outdoor surveillance, where fine textures enable small-object detection at range. Atmospheric haze degrades this high-frequency detail (Figure~\ref{fig:teaser}a): under the scattering model~\citep{narasimhan2002vision, he2010single}, it suppresses contrast and blurs the structures that justify a 4K sensor, causing downstream perception to fail in adverse weather~\citep{xie2024synfog}. Restoring these details on continuous video streams---in real time---remains an open problem. Drawing on other fields~\citep{huang2026tools, jiang2026cofie, qi2026cacr, zeng2026dta, wang2026visual}, we turn to bilateral grids.

Bilateral-grid-based dehazing offers a promising direction. 4KDehazing~\citep{4KDehazing} showed that inverting the atmospheric scattering model (ASM) reduces to a per-pixel affine color transform governed by the low-frequency depth field; encoding this transform in a bilateral grid~\citep{chen2007real,gharbi2017deep} decouples prediction cost from output resolution, enabling real-time UHD single-image dehazing. However, the approach operates frame by frame and introduces no temporal constraint, so applying it to video produces visible flicker between consecutive frames~\citep{lei2020blind}. Meanwhile, dedicated video dehazing methods~\citep{CG-IDN_REVIDE_dataset,MAP-Net_HazeWorld_dataset,DVD_GoProHazy_and_DrivingHazy_dataset,DCL,PM-Net,DS-TFN,ASTP-Net} aggregate temporal cues through optical flow or dense attention, but their spatiotemporal backbones are designed for $1080$p or below. Scaling them to 4K is prohibitive: dense attention grows quadratically with pixel count, flow-based warping at least linearly, and common workarounds---spatial patching or bilinear downsampling---either destroy the global atmospheric context that the ASM requires or discard the high-frequency detail that 4K sensors exist to capture. The shared root cause is that conventional dense-prediction architectures couple spatiotemporal modeling cost to the output resolution.

We address this coupling with a single design principle: \emph{model spatiotemporal structure at low resolution and apply the learned transform at the target resolution.} Native 4K video dehazing is not a single-image method scaled up in time---it requires reconciling spatial bandwidth and temporal coherence under the same compute budget. We propose the \textbf{Li}e-algebraic \textbf{B}ilateral-g\textbf{r}id \textbf{A}ffine \textbf{Net}work (LiBrA-Net), which factorizes the full spatiotemporal affine field into a spatial--color and a temporal sub-grid, fuses them through a Lie-algebraic operator that enforces frame-to-frame coherence and channel symmetry, and slices the result back to native 4K, leaving only a lightweight detail branch to recover the high-frequency residual. Because both grid predictors run at a fixed $256{\times}256$ regardless of output size, grid prediction cost is constant from 720p to 8K; only the trilinear slice and the lightweight detail branch scale with the target resolution (Figure~\ref{fig:teaser}c). We make three contributions:
\begin{itemize}
\item \textbf{New pipeline.} We show that the spatiotemporal affine field of a video dehazer can be compactly represented as a pair of low-resolution bilateral sub-grids whose 12-dimensional coefficients lie in $\mathfrak{aff}(3)$ and combine additively; a Cayley map sends the fused result back to $GL(3)$ with one closed-form $3{\times}3$ solve per pixel, which is faster and trains to higher quality than composing the two transforms sequentially. The resulting architecture, LiBrA-Net, has 6.12\,M parameters and costs 282\,GFLOPs at native 4K.

\item \textbf{New benchmark.} We release UHV-4K, the first paired 4K video dehazing benchmark with depth, transmission, and optical-flow annotations on every frame (Figure~\ref{fig:teaser}b), filling the data void that has prevented systematic study of UHD video dehazing.

\item \textbf{State-of-the-art results.} Across UHV-4K, REVIDE, and HazeWorld, LiBrA-Net is the only method that exceeds 23\,dB PSNR on the 4K benchmark while sustaining 25\,FPS at native $3840{\times}2160$ on a single GPU (Figure~\ref{fig:teaser}d). Its output drives off-the-shelf YOLOv8-X and SegFormer-B5 close to clean-image detection and segmentation quality.
\end{itemize}

\vspace{-6pt}
\section{Related Work}
\label{sec:related}
\vspace{-6pt}

\subsection{Dehazing Datasets and Methods}
\label{subsec:related_dehazing}
\vspace{-4pt}

Single-image dehazing has progressed from handcrafted priors~\citep{he2010single,berman2016nonlocal} through deep CNNs and attention-based architectures~\citep{FFA-Net,MSBDN-DFF,DeHamer,DehazeFormer} to UHD-specialized methods~\citep{4KDehazing,UHDformer,wang2025ultra,liu2025uhd,Li2023EmbeddingFF} that extract features at a reduced resolution and recover the full image via guided upsampling or frequency decomposition, reaching 4K input at manageable cost. Image benchmarks---synthetic RESIDE~\citep{Li2017BenchmarkingSD} and recent 4K paired sets~\citep{4KDehazing,liu2025uhd,wang2025ultra}---have driven this line, but all these methods treat each frame in isolation; applying them to video introduces visible flicker~\citep{lei2020blind} that no per-frame architecture can suppress.
Video dehazing methods address temporal coherence through optical-flow warping~\citep{Ren2019DeepVD,MAP-Net_HazeWorld_dataset}, deformable-convolution alignment~\citep{CG-IDN_REVIDE_dataset,DVD_GoProHazy_and_DrivingHazy_dataset,DCL}, or motion-free propagation such as phase-based memory and windowed spatiotemporal attention~\citep{PM-Net,ASTP-Net,DS-TFN,ViWS-Net}. Paired video corpora have advanced from real indoor scenes~\citep{CG-IDN_REVIDE_dataset} to large-scale outdoor synthesis~\citep{MAP-Net_HazeWorld_dataset} and 1080p driving sequences~\citep{DVD_GoProHazy_and_DrivingHazy_dataset}, yet none provides native 4K resolution with aligned depth, transmission, and optical-flow annotations (Appendix~Table~\ref{tab:dataset_comparison}). Their temporal modules operate on feature maps whose spatial extent scales at least linearly with pixel count, capping practical training at $1080$p. Native 4K video dehazing therefore remains at the unattended intersection of these two lines: the UHD branch offers spatial scalability without temporal coherence, and the video branch offers temporal coherence without resolution scalability.

\vspace{-4pt}
\subsection{Bilateral Grids and Locally Affine Color Transforms}
\label{subsec:related_bilateral}
\vspace{-4pt}

The bilateral filter~\citep{tomasi1998bilateral} and the guided filter~\citep{he2013guided} introduced edge-aware smoothing; the bilateral grid~\citep{chen2007real} cast the same operation as a 3D splat--blur--slice pipeline, and Bilateral Guided Upsampling~\citep{Chen2016BilateralGU} showed that locally affine transforms fitted in this space can approximate a wide class of image operators. HDRNet~\citep{gharbi2017deep} made the fitting step learnable via differentiable slicing, and subsequent work extended the paradigm to image-adaptive 3D lookup tables (LUTs)~\citep{Zeng2020LearningI3,Yang2022AdaIntLA,Yang2022SepLUTSI}. A shared property is that grid coefficients are predicted at a fixed low resolution and applied at full resolution by trilinear interpolation. Within dehazing, the closest prior is 4KDehazing~\citep{4KDehazing}, which fits a single-frame spatial--color bilateral grid with unconstrained affine coefficients and carries no temporal axis. Adding increments in a Lie algebra and mapping the sum back to the group is standard in geometric estimation~\citep{sola2018micro}. Cayley parameterizations have been used for orthogonal RNN weights~\citep{helfrich2018orthogonal,lezcano2019cheap} with $M$ constrained to be skew-symmetric, and as a low-cost retraction on the Stiefel manifold~\citep{li2020efficient}; we apply the Cayley map---to our knowledge for the first time---to unrestricted $M\in\mathfrak{gl}(3)$ for invertible affine color transforms in video restoration, where its closed-form $3{\times}3$ inverse keeps per-pixel evaluation affordable at 4K.

\vspace{-6pt}
\section{Methodology}
\label{sec:method}
\vspace{-4pt}

Given $T$ consecutive hazy frames $\{\mathbf{I}_\tau\}_{\tau=1}^{T}$ at resolution $H \times W$, with the full-resolution center frame denoted $\mathbf{I}_{\mathrm{hr}} := \mathbf{I}_{\lceil T/2 \rceil}$, LiBrA-Net restores the clean center frame $\hat{\mathbf{J}}_{\lceil T/2 \rceil}$ at the same resolution. The ASM reduces dehazing to a per-pixel affine color transform whose spatial variation is inherited from the low-frequency depth field, a structure the bilateral grid encodes compactly. Two lightweight branches---the \textbf{Chromatic Affine Field} and the \textbf{Temporal Affine Field}---predict 4D bilateral grids on a fixed $G_h \times G_w$ spatial lattice from inputs downsampled to $S \times S$, and a lightweight \textbf{High-Frequency Refiner (HF-Refiner)} restores fine textures directly from $\mathbf{I}_{\mathrm{hr}}$; Figure~\ref{fig:overview} previews the data flow.
\begin{figure}[t]
  \centering
  \includegraphics[width=\linewidth]{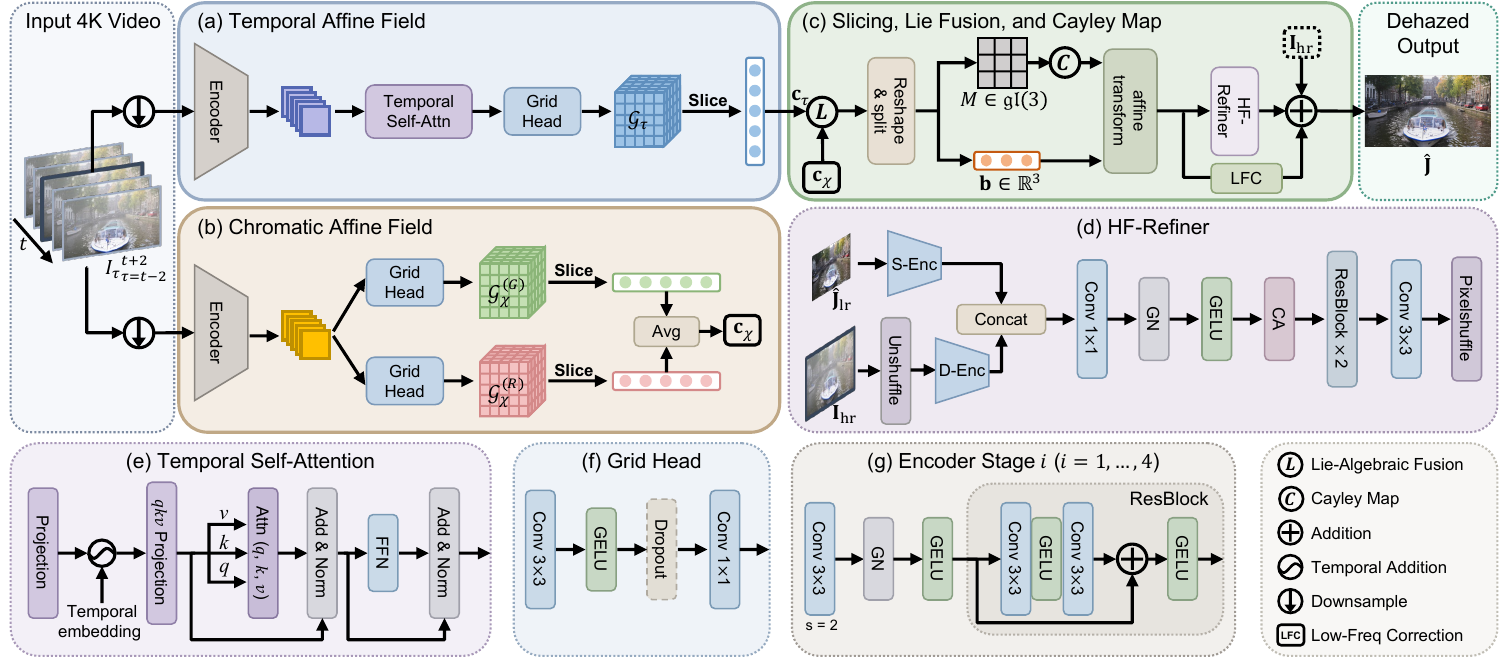}
  \caption{\textbf{Overview of LiBrA-Net.}
  Two encoder branches predict bilateral grid coefficients at a fixed low resolution: the Chromatic Affine Field consumes the center frame, and the Temporal Affine Field consumes all $T$ frames.
  Their coefficients are fused in $\mathfrak{gl}(3)$ and mapped to per-pixel $GL(3)$ affine transforms via the Cayley map.
  A lightweight HF-Refiner restores high-frequency detail.}
  \label{fig:overview}
  \vspace{-4mm}
\end{figure}
\vspace{-4pt}
\subsection{From Physics to a Bilateral Grid}
\label{sec:bilateral_grid}
\vspace{-4pt}
Our pipeline starts from the physics. Bold symbols denote per-pixel color quantities, either RGB vectors in $\mathbb{R}^3$ ($\mathbf{I}, \mathbf{J}, \mathbf{A}_\infty, \mathbf{b}$) or $3{\times}3$ matrices ($\mathbf{A}$); non-bold italic symbols ($t, d, \beta, a$) are scalars; $M$ denotes a $3{\times}3$ Lie-algebra element, $I$ the identity, and $x$ a pixel location (Appendix~\ref{app:notation} lists all symbols). The ASM writes the hazy radiance as
\begin{equation}
  \mathbf{I}(x) = \mathbf{J}(x)\,t(x) + \mathbf{A}_\infty\bigl(1 - t(x)\bigr), \qquad
  t(x) = \exp\!\bigl(-\beta\, d(x)\bigr),
  \label{eq:asm}
\end{equation}
with $\mathbf{J}$ the clean radiance, $d$ the normalized depth, $\beta$ the scattering coefficient, and $\mathbf{A}_\infty$ the global atmospheric light. Inverting it gives a per-pixel affine map
\begin{equation}
  \mathbf{J}(x) \;=\; a(x)\,\mathbf{I}(x) + \mathbf{b}(x), \qquad
  a(x) = \tfrac{1}{t(x)}, \quad \mathbf{b}(x) = \mathbf{A}_\infty\!\left(1 - \tfrac{1}{t(x)}\right),
  \label{eq:dehaze_affine}
\end{equation}
whose gain and offset are determined solely by the local depth.

\vspace{-0.5em}
\paragraph{From scalar to $GL(3)$.}
In the ideal single-channel ASM the gain reduces to a scaled identity $t^{-1}I$. Rayleigh/Mie wavelength dependence breaks the identity to a non-uniform diagonal, and image signal processing (ISP) pipelines (demosaicking, white balance, color correction) couple the three channels. We absorb these effects by promoting the scalar gain $a(x)$ to a matrix $\mathbf{A}(x) \in GL(3)$, the group of real invertible $3{\times}3$ matrices:
\begin{equation}
  \mathbf{J}(x) \;=\; \mathbf{A}(x)\,\mathbf{I}(x) + \mathbf{b}(x), \qquad
  \mathbf{A}(x) \in \mathbb{R}^{3\times 3},\;\; \mathbf{b}(x) \in \mathbb{R}^{3},
  \label{eq:matrix_affine}
\end{equation}
which contributes 12 affine coefficients per pixel, stacked as $\mathbf{c}(x) \in \mathbb{R}^{12}$.

\vspace{-0.5em}
\paragraph{Low-frequency coefficients meet bilateral grids.}
The spatial bandwidth of $\mathbf{c}(x)$ is inherited entirely from the depth field, since $a(x) = \exp(\beta\,d(x))$ varies only where $d(x)$ does. Natural-scene depth is dominated by large-scale geometry and stays smooth except at object boundaries, exactly the low-frequency edge-aware structure bilateral grids were designed to encode. A grid of spatial extent $G_h \times G_w = 16 \times 16$ suffices whether the output is 720p or $3840 \times 2160$: enlarging the output enlarges only the slicing cost, not the grid itself, a property already exploited at single-image scale by 4KDehazing~\citep{4KDehazing}.

\vspace{-0.5em}
\paragraph{6D to two 4D grids.}
A full spatiotemporal bilateral grid $\mathcal{G}[\tau, x, y, r, g, b] \in \mathbb{R}^{12}$ would carry six grid axes and become impractical to learn or store. We approximate it by two sub-grids, each a tensor with three grid axes carrying a 12-dim coefficient per cell: a \emph{spatial--color grid} $\mathcal{G}_\chi \in \mathbb{R}^{12 \times G_c \times G_h \times G_w}$ (the subscript $\chi$ indexes its chromatic axis) and a \emph{temporal grid} $\mathcal{G}_\tau \in \mathbb{R}^{12 \times T \times G_h \times G_w}$ ($T{=}5$). Their coefficients are added component-wise in $\mathfrak{aff}(3) \cong \mathfrak{gl}(3) \oplus \mathbb{R}^3$ (as a vector space), the Lie algebra of the affine group, where $\mathfrak{gl}(3)$, the vector space of all real $3{\times}3$ matrices, is the Lie algebra of $GL(3)$ and $\mathbb{R}^3$ holds the translation; the validity of this additive composition is established below. The full 3D color coordinate $(r,g,b)$ is projected onto a scalar color guide (instantiated as the R and G channels in \S\ref{sec:grid_prediction}), reducing the color axis to $G_c{=}8$ bins. This projection is justified physically: all three channels share the same scalar transmission $t(x)$, so their affine coefficients vary much more with depth than across channels. We fix $G_c{=}8$ and $G_h{=}G_w{=}16$ throughout.
\vspace{-4pt}
\subsection{Predicting the Two Bilateral Affine Fields}
\label{sec:grid_prediction}
\vspace{-4pt}
We predict $\mathcal{G}_\chi$ and $\mathcal{G}_\tau$ with two lightweight branches that share a stride-16 encoder structure but use independent weights and consume different inputs, all at a fixed resolution $S \times S = 256 \times 256$ regardless of the input size.

\vspace{-0.5em}
\paragraph{Chromatic Affine Field.}
The center frame $\mathbf{I}_{\lceil T/2 \rceil}$ is bilinearly downsampled to $S \times S$ and processed by a stride-16 encoder with channel widths $[24, 48, 96, 192]$. Each stage applies a stride-2 convolution, GroupNorm, GELU, and a residual block with stochastic depth. The encoder output, a $192$-channel $16 \times 16$ feature map, is passed through a grid head (two convolutional layers interleaved with GELU and Dropout2d) that produces a \emph{pair} of spatial--color grids $\mathcal{G}_\chi^{(R)}, \mathcal{G}_\chi^{(G)} \in \mathbb{R}^{12 \times G_c \times G_h \times G_w}$, guided by the R and G channels respectively.

Because the R and G channels have different spectral sensitivities, each induces a complementary edge-aware partition of the color axis. After slicing, the two predictions are averaged into a single coefficient vector
\begin{equation*}
  \mathbf{c}_\chi \;=\; \tfrac{1}{2}\bigl(\mathbf{c}_\chi^{(R)} + \mathbf{c}_\chi^{(G)}\bigr),
\end{equation*}
and a guide-consistency penalty in $\mathcal{L}_{\mathrm{lie}}$ (\S\ref{sec:lie_cayley}) keeps the two grids aligned.

\vspace{-0.5em}
\paragraph{Temporal Affine Field.}
Temporal affine coefficients vary across frames in ways that a frame-averaged feature would erase, so we model them with per-position self-attention rather than temporal convolution.
All $T$ frames are downsampled to $S \times S$ and encoded in a single batch of $B \cdot T$ samples, sharing the Chromatic Field's architecture but with independent weights. Per-frame features are reshaped to $[B, T, 192, 16, 16]$ and augmented with a learnable temporal positional embedding $\mathbf{p} \in \mathbb{R}^{1 \times T \times 192 \times 1 \times 1}$. A pre-norm multi-head self-attention block (8 heads) with a two-layer feed-forward network (FFN) then operates along the temporal axis at each spatial location independently, treating the tensor as $T$ tokens of dimension $192$ at each of the $B{\cdot}16{\cdot}16$ spatial positions. With only $T{=}5$ tokens per position the attention cost is negligible, yet unlike a $1{\times}1$ temporal convolution that fixes the cross-frame weights, it adapts to each input and preserves per-frame variation. A grid head mirroring the Chromatic Field's produces per-frame coefficients, yielding $\mathcal{G}_\tau \in \mathbb{R}^{12 \times T \times G_h \times G_w}$. Appendix~\ref{app:temporal_role} shows that the branch pools the window as repeated observations of a slowly varying transmission field rather than as a motion model.

Both grid predictors operate entirely at $16 \times 16$, so their cost is constant in $H$ and $W$; the following section shows how trilinear slicing and Lie-algebraic fusion bridge the gap to the target resolution while keeping the resulting affine transform invertible.
\vspace{-4pt}
\subsection{Slicing, Lie Fusion, and the Cayley Map}
\label{sec:lie_cayley}
\vspace{-4pt}
With both grids predicted, three operations remain: querying them at the output resolution, fusing them coherently across the spatial--color and temporal axes, and turning their coefficients into invertible affine matrices.

\vspace{-0.5em}
\paragraph{Resolution-adaptive trilinear slicing.}
Both grids are queried at the target resolution by trilinear interpolation, the standard bilateral-grid slice. Writing $\tilde{x}, \tilde{y} \in [-1,1]$ for the normalized spatial coordinates of pixel $x$ and $\tilde{c}_R := 2I_R(x) - 1$ for the R-channel guide, the R-guided spatial--color slice reads
\begin{equation}
  \mathbf{c}_\chi^{(R)}(x) \;=\; \operatorname{trilinear}\!\bigl(\mathcal{G}_\chi^{(R)};\,\tilde{x},\,\tilde{y},\,\tilde{c}_R\bigr).
  \label{eq:sc_slice}
\end{equation}
The G-guided grid is sliced analogously with $\tilde{c}_G$, and $\mathcal{G}_\tau$ reuses the same spatial coordinates with $\tilde{\tau} = 0$ (the center frame). Because self-attention has already mixed all $T$ frames before the grid head, $\tilde{\tau}$ selects which frame's field is applied, not how much temporal context enters. The slicing resolution is set to $H' \times W' = H/p \times W/p$ with $p{=}4$, matching the PixelUnshuffle factor of the detail recovery stage, rather than the fixed encoder size $S \times S$. This ties training crops and 4K inference to the same resolution path, removing the domain shift a fixed-resolution slice would create.

\vspace{-0.5em}
\paragraph{Residual fusion in the Lie algebra.}
With $\mathbf{c}_\chi$ already obtained from the dual-guide average above, the two sliced coefficient vectors are combined by component-wise addition:
\begin{equation}
  \mathbf{c}_{\mathrm{fused}} \;=\; \mathbf{c}_\chi + \mathbf{c}_\tau \;\in\; \mathbb{R}^{12}.
  \label{eq:lie_fusion}
\end{equation}
The first nine components form a matrix $M \in \mathfrak{gl}(3)$ and the remaining three a translation; the translation adds trivially as vectors, and the argument below concerns the matrix part. Adding in $\mathfrak{gl}(3)$ stays inside a vector space and preserves invertibility downstream, whereas the analogous matrix sum of two $GL(3)$ elements is not even guaranteed to lie in $GL(3)$. The exact alternative, applying the two branch transforms one after the other, needs two Cayley solves and a matrix product per pixel, and it couples the branches during training: the gradient reaching each grid head is multiplied by the other branch's current transform, whereas under additive fusion both heads share one Jacobian. Trained from scratch, additive fusion reaches 2.3\,dB higher PSNR on UHV-4K and runs at 25.3 instead of 22.1\,FPS (Appendix~\ref{app:param_choice}). Proposition~\ref{prop:lie_composition} bounds the difference between the two routes, with $\operatorname{Cay}(M) := (I - M/2)^{-1}(I + M/2)$ the Cayley map of Eq.~\eqref{eq:cayley}.

\begin{proposition}[Lie-algebraic fusion approximates group composition]
\label{prop:lie_composition}
For $M_\chi, M_\tau \in \mathfrak{gl}(3)$ with $\rho := \max(\|M_\chi\|_F, \|M_\tau\|_F) < 1$, as $\rho \to 0$,
\begin{equation}
  \operatorname{Cay}(M_\chi + M_\tau) \;=\; \operatorname{Cay}(M_\chi)\,\operatorname{Cay}(M_\tau) \;-\; \tfrac{1}{2}[M_\chi, M_\tau] \;+\; O(\rho^3),
  \label{eq:cayley_expand}
\end{equation}
where the commutator $[M_\chi, M_\tau] := M_\chi M_\tau - M_\tau M_\chi$ satisfies $\|\tfrac{1}{2}[M_\chi, M_\tau]\|_F \le \rho^2$. Hence the two sides differ by $O(\rho^2)$, for either order of the product.
\end{proposition}

\noindent\textit{Intuition.} Cayley is locally an exponential at the identity, so adding two near-zero matrices in $\mathfrak{gl}(3)$ approximates multiplying the corresponding $GL(3)$ matrices; the leading residual is proportional to the commutator $[M_\chi, M_\tau]$, kept small by the identity prior below. The proof is given in Appendix~\ref{app:proof_lie}.
\vspace{-0.5em}
\paragraph{Cayley affine parameterization.}
The matrix part $M$ is mapped to a per-pixel invertible affine by the Cayley map,
\begin{equation}
  \mathbf{A} = \operatorname{Cay}(M) = (I - M/2)^{-1}(I + M/2),
  \label{eq:cayley}
\end{equation}
the $[1/1]$ Pad\'e approximant of $\exp(M)$, computed by one closed-form $3 \times 3$ solve per pixel. Three properties make this parameterization useful here. First, $M = 0$ implies $\mathbf{A} = I$, so zero-initializing the last layer of every grid head sets the initial output to the input and lets the network learn corrections incrementally. Second, $\det \mathbf{A} = \det(I + M/2)/\det(I - M/2)$ is non-zero whenever $\pm 2$ are not eigenvalues of $M$, a condition the identity prior maintains in practice. Third, it is cheap: since the grids predict $M$ directly in $\mathfrak{gl}(3)$, no matrix logarithm is ever needed, and among the maps we tested from $\mathfrak{gl}(3)$ to invertible matrices, Cayley is the fastest; the matrix exponential lowers 4K throughput to 20.4\,FPS, and projecting $I{+}M$ onto invertible matrices by an SVD lowers it to 3.8\,FPS, both at lower PSNR (Appendix~\ref{app:param_choice}). Applying $(\mathbf{A}(x), \mathbf{b}(x))$ pixel-wise via Eq.~\eqref{eq:matrix_affine} at $H/p \times W/p$ produces the low-resolution dehazed image $\hat{\mathbf{J}}_{\mathrm{lr}} \in [0,1]^{3 \times H/p \times W/p}$.
\vspace{-0.5em}
\paragraph{Lie-algebraic regularization as grid compactification.}
Without constraints the grid carries far more degrees of freedom than the underlying physics warrants, leading to temporal flicker, hue drift between the two guide-conditioned grids, and departure from the near-identity regime of Proposition~\ref{prop:lie_composition}. Four parameter-free penalties applied \emph{before} the Cayley map address each axis:
\emph{identity prior} $\|M\|^2$ ($\lambda_{\mathrm{id}}{=}0.01$) keeps the $\mathfrak{gl}(3)$ coefficients small;
\emph{spatial smoothness} $\|\nabla_{xy}\mathcal{G}\|^2$ ($\lambda_{\mathrm{sp}}{=}0.05$) band-limits each grid spatially;
\emph{temporal smoothness} $\|\partial_t \mathcal{G}_\tau\|^2$ ($\lambda_{\mathrm{tm}}{=}0.10$) enforces continuity across frames;
\emph{guide consistency} $\|\mathcal{G}_\chi^{(R)} - \mathcal{G}_\chi^{(G)}\|_1$ ($\lambda_{\mathrm{g}}{=}0.02$) ties the two chromatic grids together.
Together they confine the grid to a smooth, compact subset on which the first-order Cayley composition remains valid; \S\ref{sec:analysis_chapter} verifies this empirically.
\vspace{-4pt}
\subsection{High-Frequency Refinement}
\label{sec:drn}
\vspace{-4pt}
The bilateral path delivers low-frequency color correction; high-frequency texture, suppressed by transmission, has yet to come back.

\vspace{-0.5em}
\paragraph{High-frequency preservation under the ASM.}
Differentiating Eq.~\eqref{eq:asm} spatially gives $\nabla \mathbf{I} = t\,\nabla \mathbf{J} + (\mathbf{J} - \mathbf{A}_\infty)\,\nabla t \approx t\,\nabla \mathbf{J}$, where the approximation holds because $t = \exp(-\beta\,d)$ inherits the low-frequency character of the depth field, so $\nabla t \approx 0$ at most pixels. Equivalently,
\begin{equation}
  \nabla \mathbf{I}(x) \;\approx\; t(x)\,\nabla \mathbf{J}(x), \qquad t(x) \in (0,1],
  \label{eq:asm_gradient}
\end{equation}
so the hazy input already carries a scaled copy of the clean gradients. The HF-Refiner therefore only needs to extract these high-frequency features from $\mathbf{I}_{\mathrm{hr}}$ and undo the transmission-induced attenuation, an easier task than super-resolution.

Concretely, the HF-Refiner operates at $H/p \times W/p$ on two streams: the hazy frame split via PixelUnshuffle and the low-resolution dehazed estimate $\hat{\mathbf{J}}_{\mathrm{lr}}$ from the bilateral path. Both share the same spatial dimensions by construction, requiring no resampling. We denote the encoded features $\mathbf{f}_{\mathrm{detail}}$ and $\mathbf{f}_{\mathrm{struct}}$; a learned fusion path combines them and produces a high-frequency residual via PixelShuffle. GroupNorm is used throughout so that normalization statistics remain independent of spatial extent, avoiding mismatch between training crops and 4K inference. The final output is
\begin{equation}
  \hat{\mathbf{J}} \;=\; \operatorname{clamp}\!\Bigl(\;
    \underbrace{\mathbf{I}_{\mathrm{hr}}}_{\mathclap{\text{high-freq carrier}}}
    \;+\; \underbrace{\operatorname{Up}\bigl(\hat{\mathbf{J}}_{\mathrm{lr}} - \operatorname{Down}(\mathbf{I}_{\mathrm{hr}})\bigr)}_{\mathclap{\text{low-freq color correction}}}
    \;+\; \underbrace{\operatorname{PixelShuffle}\bigl(\operatorname{fuse}(\mathbf{f}_{\mathrm{detail}}, \mathbf{f}_{\mathrm{struct}})\bigr)}_{\mathclap{\text{learned high-freq residual}}}
  \;,\; 0,\,1\Bigr),
  \label{eq:drn_output}
\end{equation}
with $\operatorname{Up}$ and $\operatorname{Down}$ denoting bilinear resampling. The fusion path's last convolution is zero-initialized, so at step zero $\hat{\mathbf{J}}$ reduces to $\mathbf{I}_{\mathrm{hr}}$ plus the bilateral-grid correction; together with the zero-initialized grid heads (\S\ref{sec:lie_cayley}), the entire forward path starts as an identity mapping.
\vspace{-4pt}
\subsection{Supervision}
\label{sec:supervision}
\vspace{-4pt}
LiBrA-Net is trained end-to-end by minimizing
\begin{equation}
  \mathcal{L} = \mathcal{L}_{\mathrm{pixel}} + \lambda_{\mathrm{perc}}\,\mathcal{L}_{\mathrm{perc}} + \lambda_{\mathrm{lie}}\,\mathcal{L}_{\mathrm{lie}}, \qquad \lambda_{\mathrm{perc}}{=}0.04,\; \lambda_{\mathrm{lie}}{=}0.2,
  \label{eq:total_loss}
\end{equation}
with three complementary terms.
\emph{Pixel.} $\mathcal{L}_{\mathrm{pixel}} = \tfrac{1}{N}\sum_i \sqrt{(\hat{J}_i - J_i^{\mathrm{gt}})^2 + \epsilon^2}$ is the Charbonnier penalty ($\epsilon{=}10^{-3}$) on the full-resolution HF-Refiner output.
\emph{Perceptual.} $\mathcal{L}_{\mathrm{perc}}$ is the layer-averaged $\ell_1$ distance between DINOv2 ViT-S/14~\citep{oquab2024dinov2} features of the prediction and the ground truth.
\emph{Lie regularizer.} $\mathcal{L}_{\mathrm{lie}}$ is the four-term penalty of \S\ref{sec:lie_cayley}.
Temporal consistency is not enforced by an explicit flow-warped loss; as our ablation in \S\ref{sec:ablation} confirms, it emerges from the Temporal Field's self-attention together with the temporal-smoothness component of $\mathcal{L}_{\mathrm{lie}}$; we revisit the per-frame consequence in \S\ref{sec:analysis}. We next describe the 4K paired video corpus on which LiBrA-Net is trained and evaluated.

\vspace{-6pt}
\section{UHV-4K Dataset}
\label{sec:dataset}
\vspace{-4pt}

No existing video dehazing dataset provides native 4K paired data with aligned auxiliary annotations. We construct \textbf{UHV-4K} to fill this gap: 100 videos at $3840{\times}2160$ (2{,}500 frames, 80/20 train-test split) drawn from Inter4K~\citep{inter4k_dataset} and UVG~\citep{uvg_dataset}, each providing five per-frame modalities---hazy input, clean ground truth, monocular depth, transmission map, and optical flow. We sub-sample 60\,fps sources at 5\,fps to yield 25 consecutive frames per video, preserving camera motion and depth parallax while bounding storage. Figure~\ref{fig:uhv4k_composition}(a) summarizes the scene and geographic coverage.
\Needspace{4\baselineskip}
\vspace{-4pt}
\subsection{Haze Synthesis}
\label{sec:haze_synthesis}
\vspace{-4pt}

Each video is assigned to one of nine $(\beta,A_\infty)$ regimes (Figure~\ref{fig:uhv4k_composition}b), with haze parameters fixed across frames for temporal coherence. We estimate depth and flow at proxy resolution, upsample depth to native 4K, and render hazy frames with Eq.~\eqref{eq:asm}. Pipeline details, quality checks, and the perceptual study are in Appendices~\ref{app:synthesis_pipeline} and~\ref{app:user_study}.

\Needspace{21\baselineskip}
\vspace{-4pt}
\subsection{Why Native 4K Supervision Matters}
\label{sec:why_4k}
\vspace{-4pt}

\begin{wrapfigure}{r}{0.58\linewidth}
  \vspace{-12pt}
  \centering
  \begin{subfigure}[b]{0.5\linewidth}
    \centering
    \includegraphics[width=\linewidth]{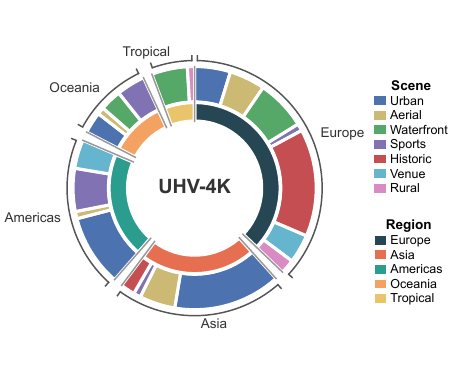}
    \caption{Scene and region diversity.}
    \label{fig:uhv4k_scene_geo}
  \end{subfigure}\hfill
  \begin{subfigure}[b]{0.5\linewidth}
    \centering
    \includegraphics[width=\linewidth]{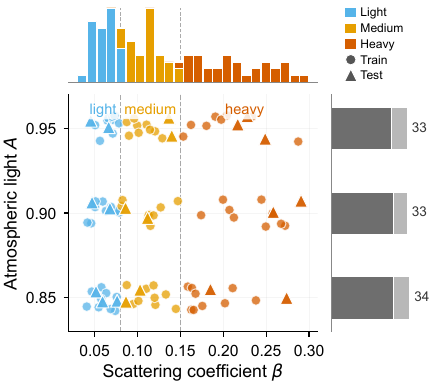}
    \caption{Haze parameters $(\beta, A_\infty)$.}
    \label{fig:uhv4k_beta_A}
  \end{subfigure}
  \caption{\textbf{UHV-4K composition.} (a)~Joint distribution of source scenes and geographic regions across the 100 videos. (b)~Per-video scattering coefficient $\beta$ and atmospheric light $A_\infty$, grouped by haze tier (color) and train/test split (marker shape).}
  \label{fig:uhv4k_composition}
  \vspace{-10pt}
\end{wrapfigure}

The ASM attenuates image gradients by the local transmission: $\nabla \mathbf{I} \approx t\,\nabla \mathbf{J}$ (Eq.~\ref{eq:asm_gradient}). Fine structures near the 4K Nyquist limit---thin cables, distant railings, signage strokes---are the first casualties, and bilinear downsampling to 720p discards them entirely since the 720p Nyquist sits at roughly one-third of the 4K limit. A dehazer trained only on 720p data therefore never sees the low-contrast, high-frequency signal it must recover at 4K. The ablation in \S\ref{sec:ablation} confirms this gap: removing the HF-Refiner costs 3.41\,dB on UHV-4K versus 2.59\,dB on REVIDE (Table~\ref{tab:ablation}), reflecting the richer high-frequency content that native 4K supervision provides. UHV-4K fills this role as both the training source and the evaluation testbed under matched native-4K conditions.

\begin{table*}[!t]
  \centering
  \caption{Quantitative comparison on three benchmarks. Efficiency is measured at native 4K on a single V100 32\,GB GPU. Best in \textbf{bold}, second \underline{underlined} within the video methods.}
  \label{tab:overall_comparison}
  \resizebox{\textwidth}{!}{
  \begin{tabular}{ll|ccc|ccc|ccc|ccc}
    \toprule
    \multirow{2}{*}{Methods} & \multirow{2}{*}{Venue} & \multicolumn{3}{c|}{UHV-4K} & \multicolumn{3}{c|}{REVIDE} & \multicolumn{3}{c|}{HazeWorld} & \multicolumn{3}{c}{Efficiency at 4K} \\
    \cmidrule(lr){3-5} \cmidrule(lr){6-8} \cmidrule(lr){9-11} \cmidrule(l){12-14}
    & & PSNR$\uparrow$ & SSIM$\uparrow$ & LPIPS$\downarrow$ & PSNR$\uparrow$ & SSIM$\uparrow$ & LPIPS$\downarrow$ & PSNR$\uparrow$ & SSIM$\uparrow$ & LPIPS$\downarrow$ & FPS$\uparrow$ & GFLOPs & Params\,(M) \\
    \midrule
    \textcolor{black!55}{No Dehazing} & \textcolor{black!55}{--} & \textcolor{black!55}{13.19} & \textcolor{black!55}{0.7690} & \textcolor{black!55}{0.2256} & \textcolor{black!55}{14.98} & \textcolor{black!55}{0.7820} & \textcolor{black!55}{0.3476} & \textcolor{black!55}{14.37} & \textcolor{black!55}{0.8032} & \textcolor{black!55}{0.1605} & \textcolor{black!55}{--} & \textcolor{black!55}{--} & \textcolor{black!55}{--} \\
    \midrule
    \rowcolor{catband} \multicolumn{14}{c}{\textit{(a) Image dehazing methods}} \\
    \midrule
    MSBDN-DFF~\citep{MSBDN-DFF} & CVPR'20 & 22.51 & 0.9257 & 0.0748 & 21.67 & 0.8618 & 0.2932 & 28.99 & 0.9702 & 0.0223 & 2.31 & 5.74K & 30.31 \\
    FFA-Net~\citep{FFA-Net} & AAAI'20 & 20.43 & 0.9042 & 0.0731 & 18.16 & 0.8334 & 0.2903 & 26.78 & 0.9581 & 0.0352 & 0.31 & 75.81K & 4.70 \\
    4KDehazing~\citep{4KDehazing} & CVPR'21 & 18.56 & 0.8608 & 0.1011 & 17.68 & 0.8325 & 0.2996 & 27.15 & 0.9649 & 0.0291 & 3.11 & 10.20K & 17.26 \\
    DeHamer~\citep{DeHamer} & CVPR'22 & 20.13 & 0.8589 & 0.1583 & 18.13 & 0.8181 & 0.3305 & 25.84 & 0.9403 & 0.0696 & 1.50 & 2.61K & 23.89 \\
    DehazeFormer~\citep{DehazeFormer} & TIP'23 & 23.74 & 0.9313 & 0.0468 & 20.76 & 0.8397 & 0.2603 & 28.38 & 0.9663 & 0.0270 & 0.18 & 6.00K & 2.54 \\
    UHDformer~\citep{UHDformer} & AAAI'24 & 23.16 & 0.9381 & 0.0746 & 18.22 & 0.8135 & 0.3141 & 22.32 & 0.9288 & 0.0605 & 0.13 & 33.30K & 20.33 \\
    UHDDIP~\citep{wang2025ultra} & TCSVT'25 & 21.90 & 0.9344 & 0.0644 & 18.20 & 0.8366 & 0.2734 & 28.02 & 0.9692 & 0.0244 & 1.25 & 18.35K & 7.58 \\
    DehazeXL~\citep{DeHazeXL} & CVPR'25 & 22.63 & 0.9235 & 0.0656 & 19.47 & 0.8355 & 0.2856 & 27.45 & 0.9659 & 0.0257 & 2.27 & 6.50K & 125.50 \\
    \midrule
    \rowcolor{catband} \multicolumn{14}{c}{\textit{(b) Video dehazing methods}} \\
    \midrule
    CG-IDN~\citep{CG-IDN_REVIDE_dataset} & CVPR'21 & 20.13 & 0.8037 & 0.2756 & 17.25 & 0.8060 & 0.4624 & 25.62 & 0.9224 & 0.1258 & \underline{6.97} & 2.19K & 1.09 \\
    PM-Net~\citep{PM-Net} & MM'22 & 21.84 & 0.9113 & 0.0911 & 18.16 & 0.8234 & 0.2984 & 24.22 & 0.9309 & 0.0654 & 1.06 & 17.26K & \underline{0.91} \\
    MAP-Net~\citep{MAP-Net_HazeWorld_dataset} & CVPR'23 & 18.45 & 0.8147 & 0.2055 & 18.42 & 0.8070 & 0.3641 & 20.31 & 0.7692 & 0.2130 & 0.81 & 8.62K & 28.33 \\
    ViWS-Net~\citep{ViWS-Net} & ICCV'23 & \underline{22.06} & 0.7742 & 0.3586 & \underline{20.74} & 0.8188 & 0.4643 & \textbf{27.66} & 0.9320 & 0.1202 & 4.21 & 1.20K & 57.67 \\
    DVD~\citep{DVD_GoProHazy_and_DrivingHazy_dataset} & CVPR'24 & 21.02 & 0.8184 & 0.2409 & 20.63 & \underline{0.8383} & 0.4064 & 27.05 & 0.9310 & 0.1153 & 3.35 & 4.93K & 3.26 \\
    DCL~\citep{DCL} & AAAI'25 & 20.81 & \underline{0.9167} & \underline{0.0818} & 18.29 & 0.8353 & \underline{0.2939} & 27.45 & \textbf{0.9617} & \underline{0.0370} & 0.98 & 20.55K & 12.31 \\
    DS-TFN~\citep{DS-TFN} & VC'25 & 20.24 & 0.8975 & 0.0850 & 15.79 & 0.8143 & 0.3484 & 25.79 & 0.9512 & 0.0428 & 0.46 & 51.99K & 2.47 \\
    ASTP-Net~\citep{ASTP-Net} & VC'25 & 19.00 & 0.7051 & 0.4904 & 17.33 & 0.7835 & 0.5120 & 22.42 & 0.8408 & 0.2517 & 3.56 & \underline{0.41K} & \textbf{0.63} \\
    \midrule
    \rowcolor{ourgreen} \textbf{Ours} & -- & \textbf{24.28} & \textbf{0.9437} & \textbf{0.0401} & \textbf{21.43} & \textbf{0.8678} & \textbf{0.2397} & \underline{27.46} & \underline{0.9593} & \textbf{0.0336} & \textbf{25.27} & \textbf{0.28K} & 6.12 \\
    \bottomrule
  \end{tabular}
  }
  \vspace{-2mm}
\end{table*}

\begin{figure*}[t]
  \centering
  \newcommand{\cimg}[1]{\raisebox{-0.5\height}{\includegraphics[width=0.118\textwidth]{#1}}}
  \setlength{\zcellW}{0.118\textwidth}%
  \newcommand{\zcimgBR}[6]{\raisebox{-0.5\height}{%
    \begin{tikzpicture}[x=\zcellW,y=0.5625\zcellW,inner sep=0pt,outer sep=0pt]
      \useasboundingbox (0,0) rectangle (1,1);
      \node[anchor=south west,inner sep=0pt] at (0,0)
        {\includegraphics[width=\zcellW]{#1}};
      \draw[roired,line width=0.6pt] (#3,#4) rectangle (#5,#6);
      \node[anchor=south east,draw=roired,line width=0.7pt,inner sep=0pt]
        at (1,0) {\includegraphics[width=0.5\zcellW,trim={#2},clip]{#1}};
    \end{tikzpicture}}}%
  \newcommand{\zcimgBL}[6]{\raisebox{-0.5\height}{%
    \begin{tikzpicture}[x=\zcellW,y=0.5625\zcellW,inner sep=0pt,outer sep=0pt]
      \useasboundingbox (0,0) rectangle (1,1);
      \node[anchor=south west,inner sep=0pt] at (0,0)
        {\includegraphics[width=\zcellW]{#1}};
      \draw[roired,line width=0.6pt] (#3,#4) rectangle (#5,#6);
      \node[anchor=south west,draw=roired,line width=0.7pt,inner sep=0pt]
        at (0,0) {\includegraphics[width=0.5\zcellW,trim={#2},clip]{#1}};
    \end{tikzpicture}}}%
  \newcommand{\zcimgTRr}[6]{\raisebox{-0.5\height}{%
    \begin{tikzpicture}[x=\zcellW,y=0.6645\zcellW,inner sep=0pt,outer sep=0pt]
      \useasboundingbox (0,0) rectangle (1,1);
      \node[anchor=south west,inner sep=0pt] at (0,0)
        {\includegraphics[width=\zcellW]{#1}};
      \draw[roired,line width=0.6pt] (#3,#4) rectangle (#5,#6);
      \node[anchor=north east,draw=roired,line width=0.7pt,inner sep=0pt]
        at (1,1) {\includegraphics[width=0.5\zcellW,trim={#2},clip]{#1}};
    \end{tikzpicture}}}%
  \newcommand{\zoomA}[1]{\zcimgBR{#1}{100 1620 3320 80}{0.03}{0.75}{0.14}{0.96}}
  \newcommand{\zoomB}[1]{\zcimgTRr{#1}{1000 620 1208 800}{0.37}{0.34}{0.55}{0.56}}
  \newcommand{\zoomC}[1]{\zcimgBL{#1}{1024 216 0 302}{0.8}{0.3}{1}{0.58}}
  \setlength{\tabcolsep}{0.8pt}
  \renewcommand{\arraystretch}{0}
  \begin{tabular}{@{}c@{\hspace{1pt}}ccccccc@{\hspace{2pt}}c@{}}
    &
    \scriptsize Input(Hazy) &
    \scriptsize CG-IDN &
    \scriptsize MAP-Net &
    \scriptsize PM-Net &
    \scriptsize ViWS-Net &
    \scriptsize DVD &
    \scriptsize\textbf{Ours} &
    \scriptsize\textit{GT}
    \\[3pt]
    \rotatebox[origin=c]{0}{\scriptsize(a)} &
      \zoomA{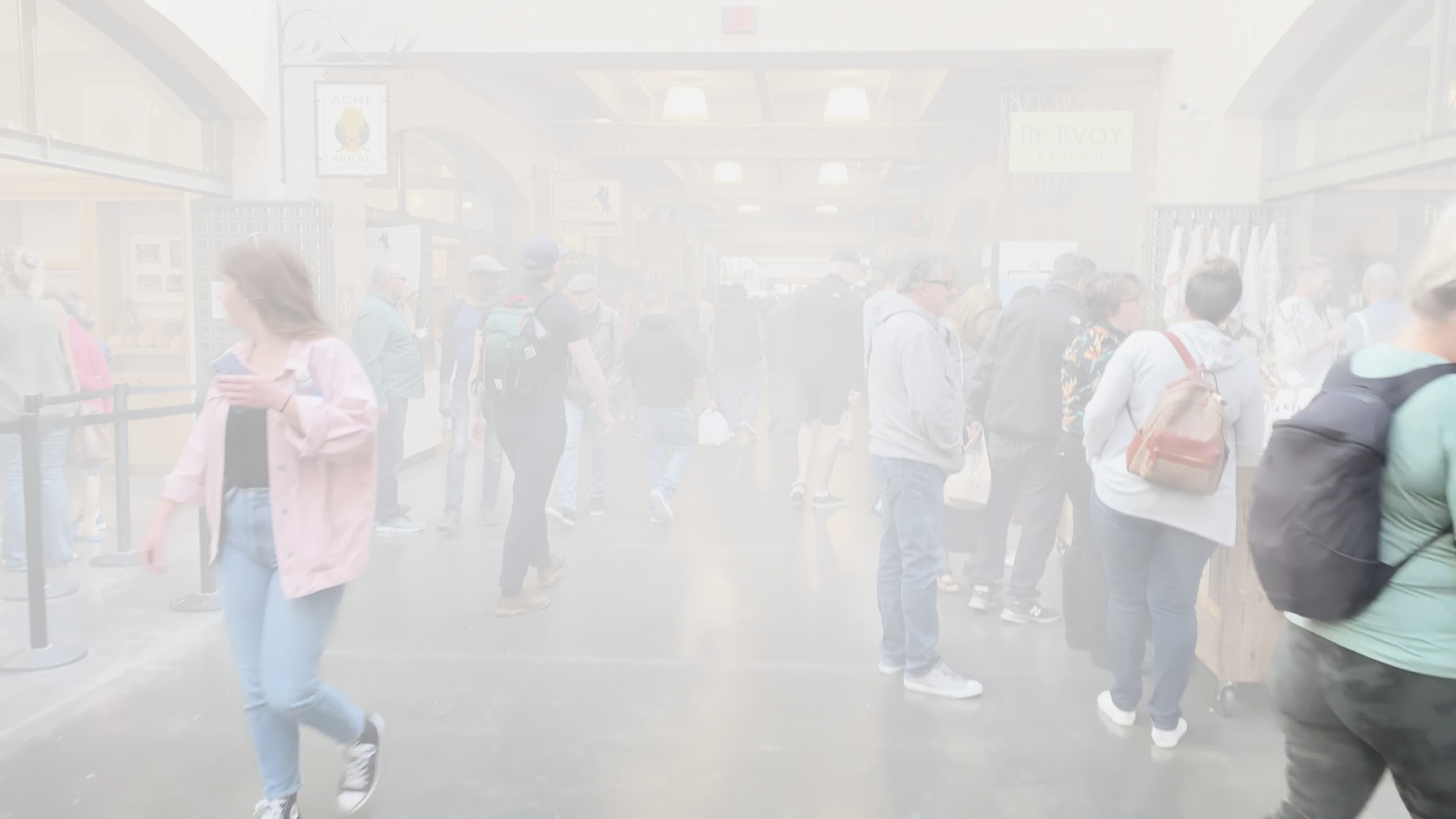} &
      \zoomA{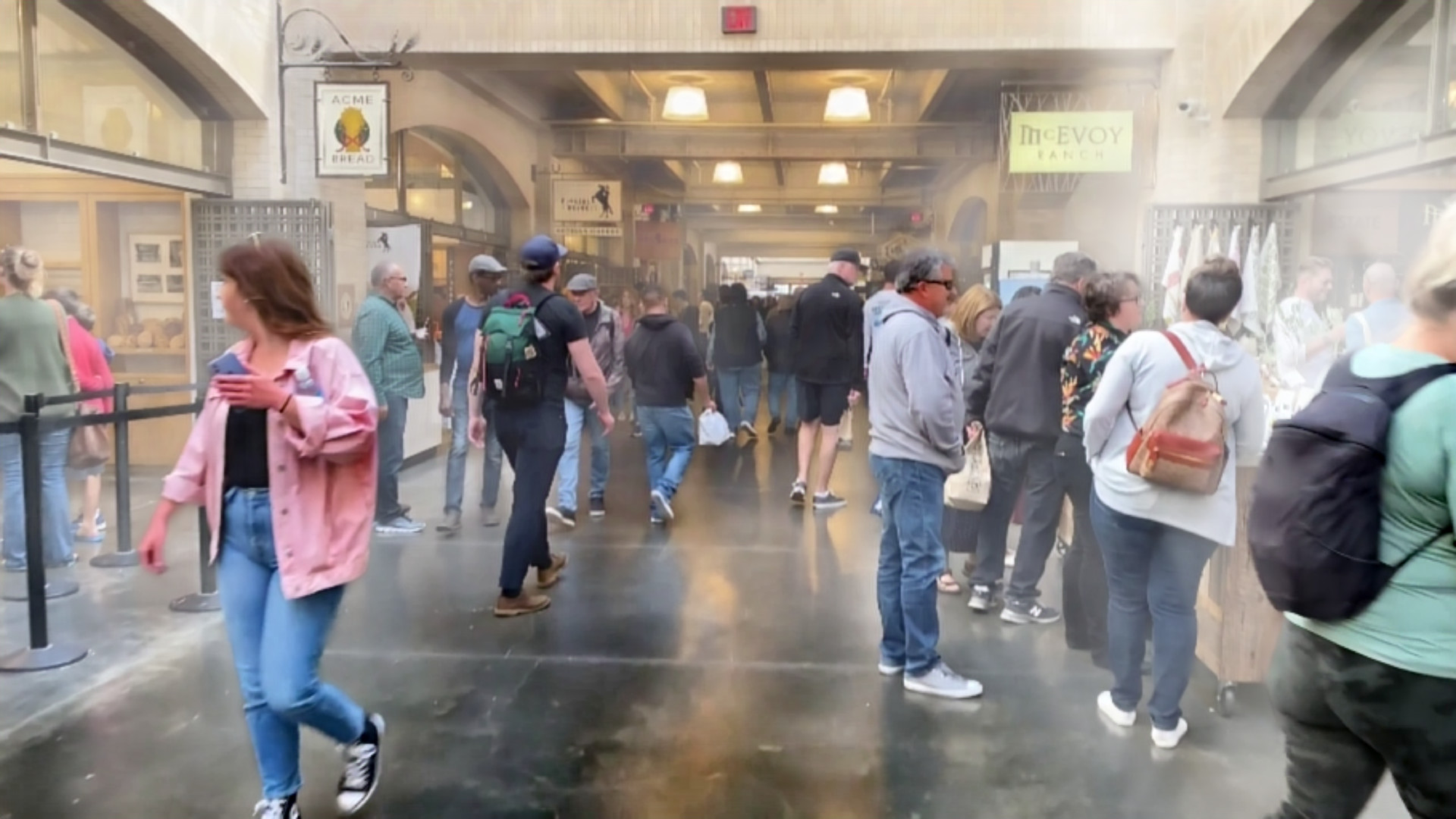} &
      \zoomA{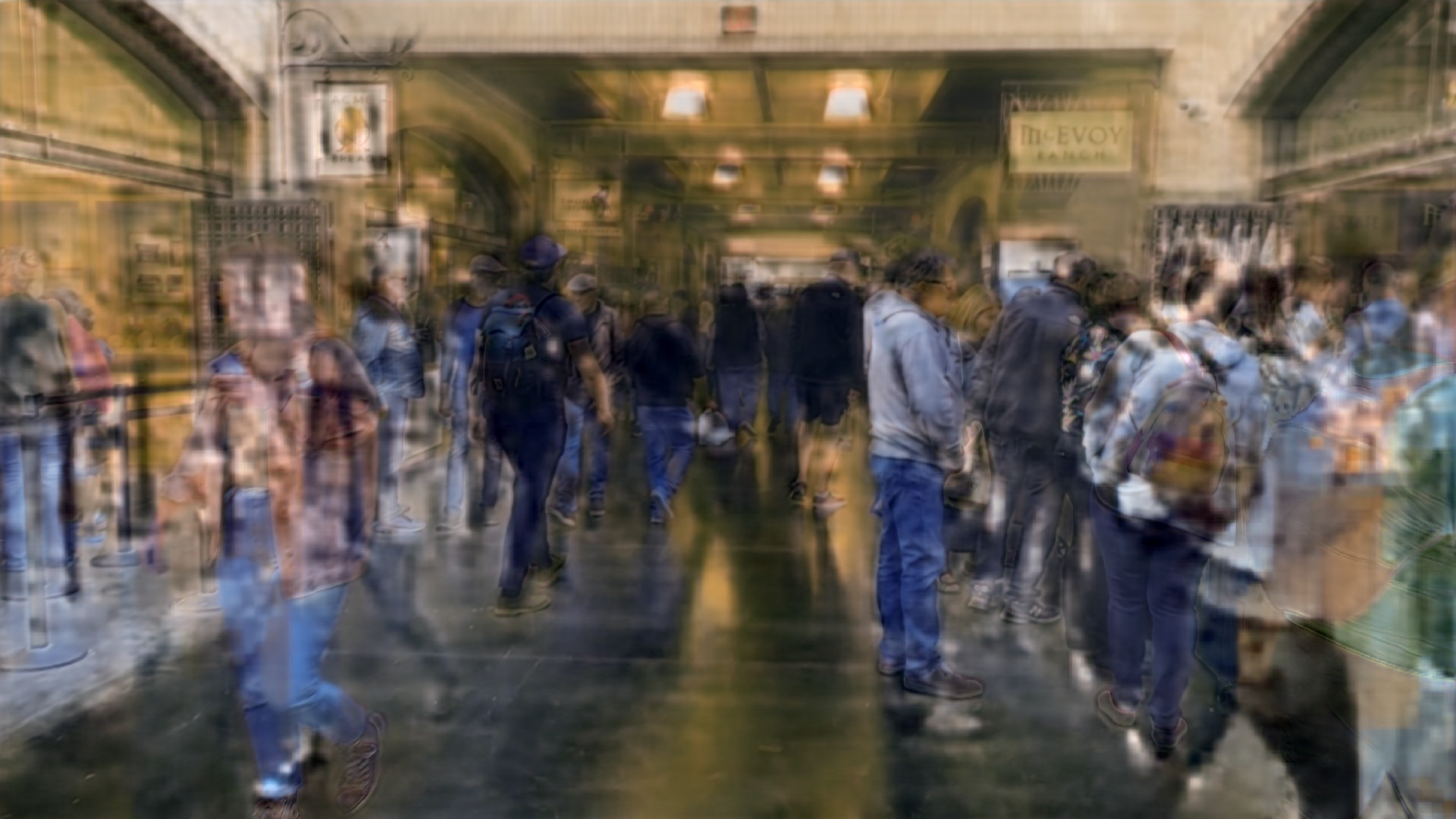} &
      \zoomA{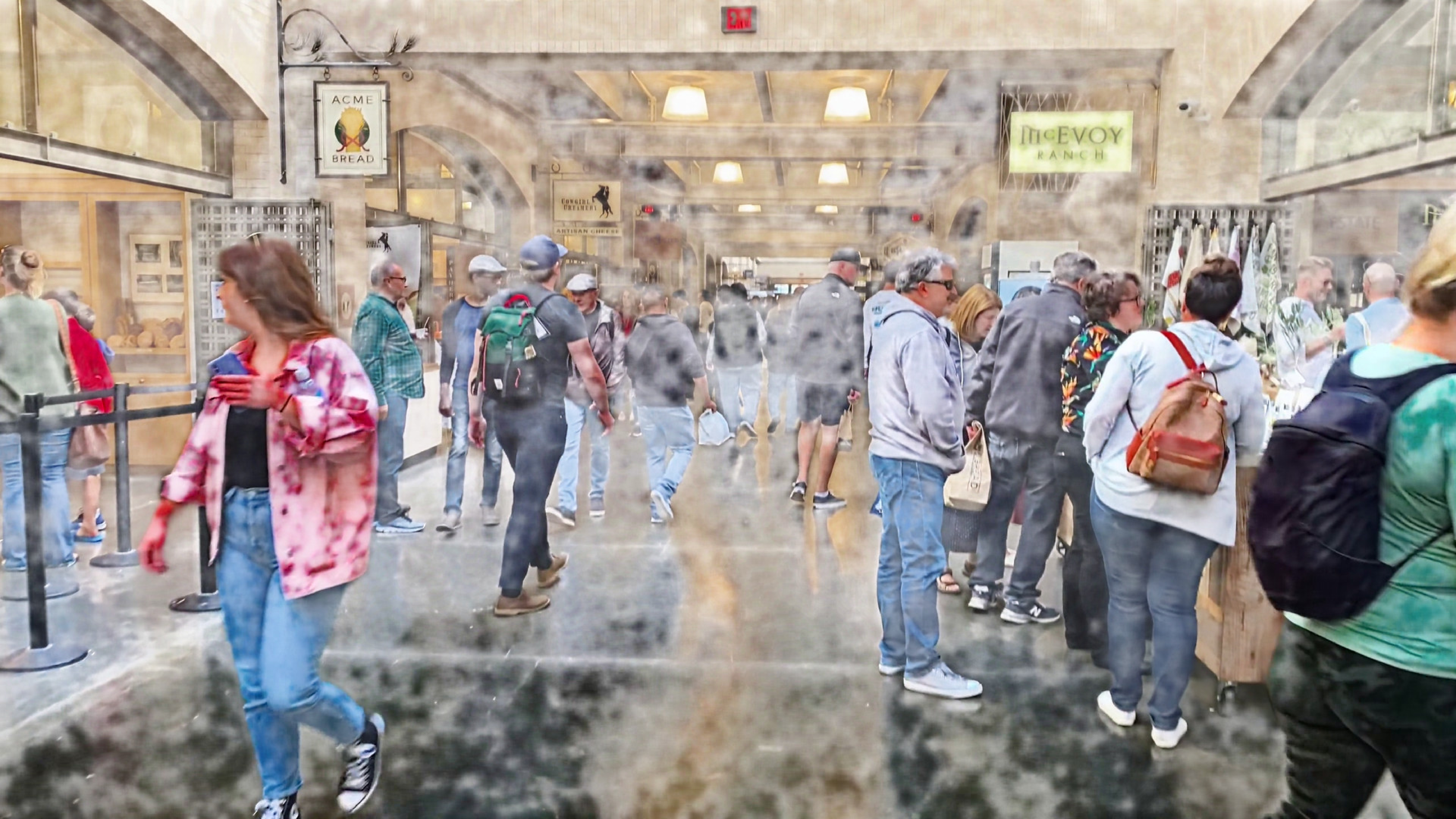} &
      \zoomA{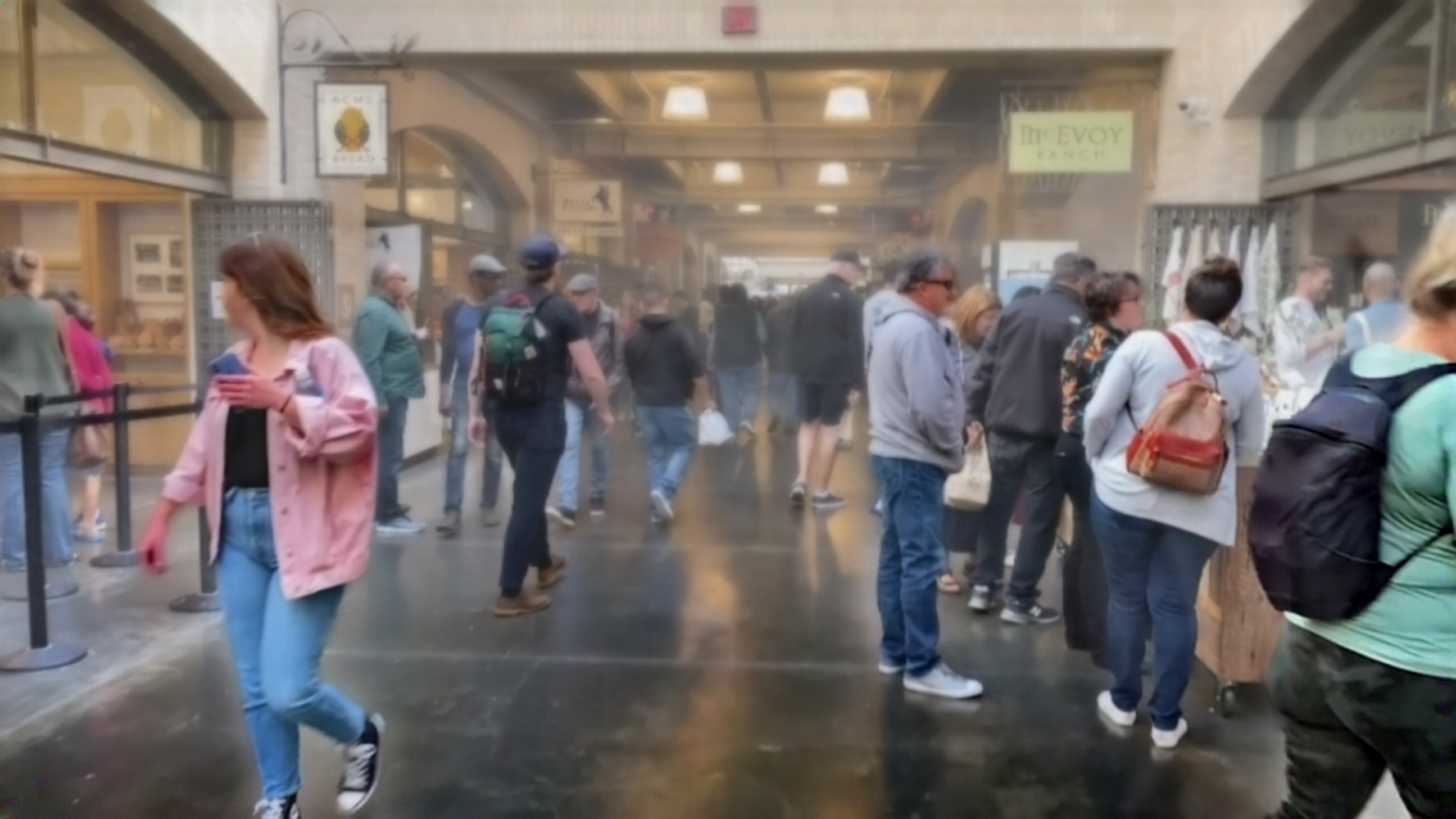} &
      \zoomA{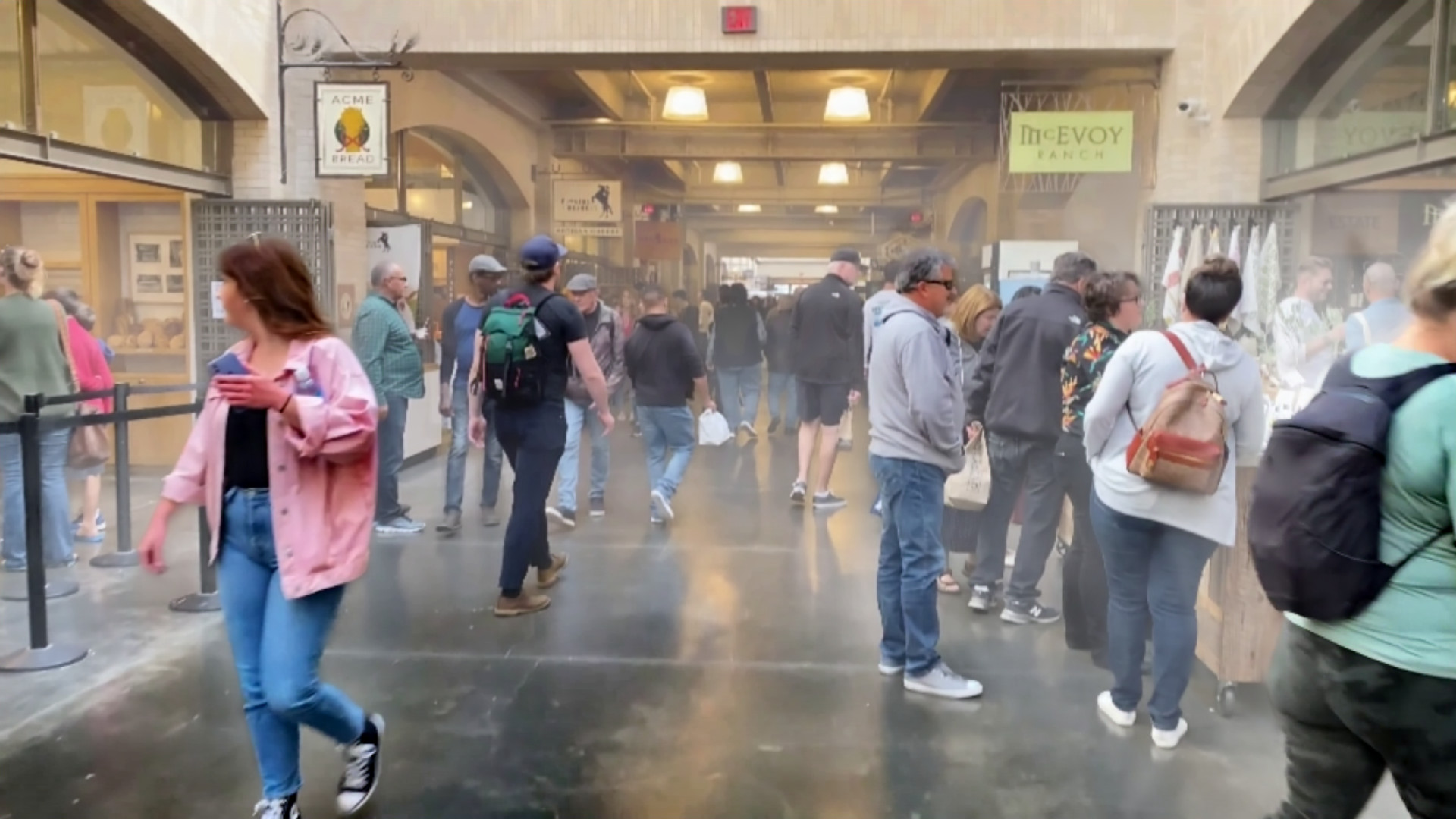} &
      \zoomA{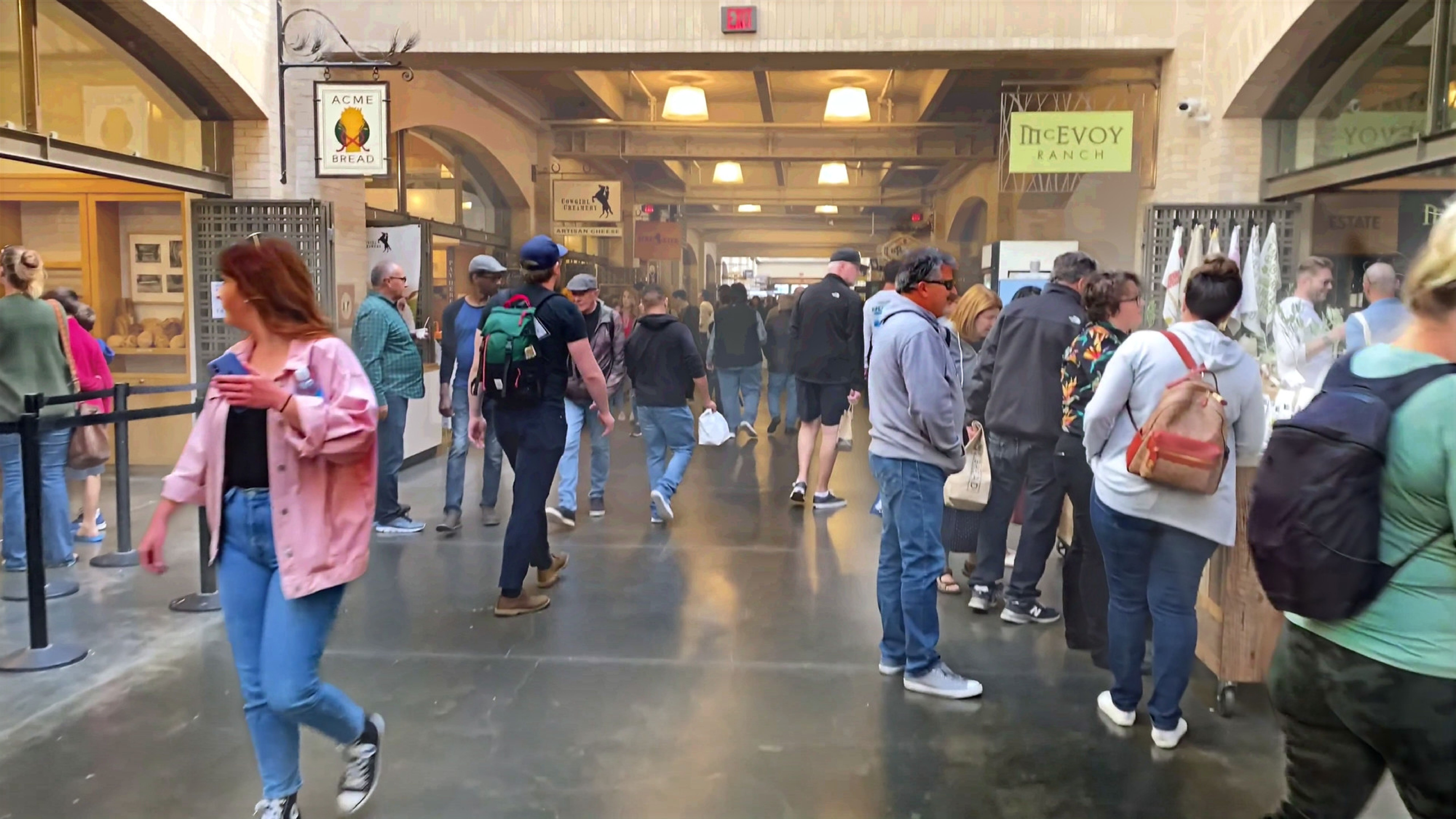} &
      \zoomA{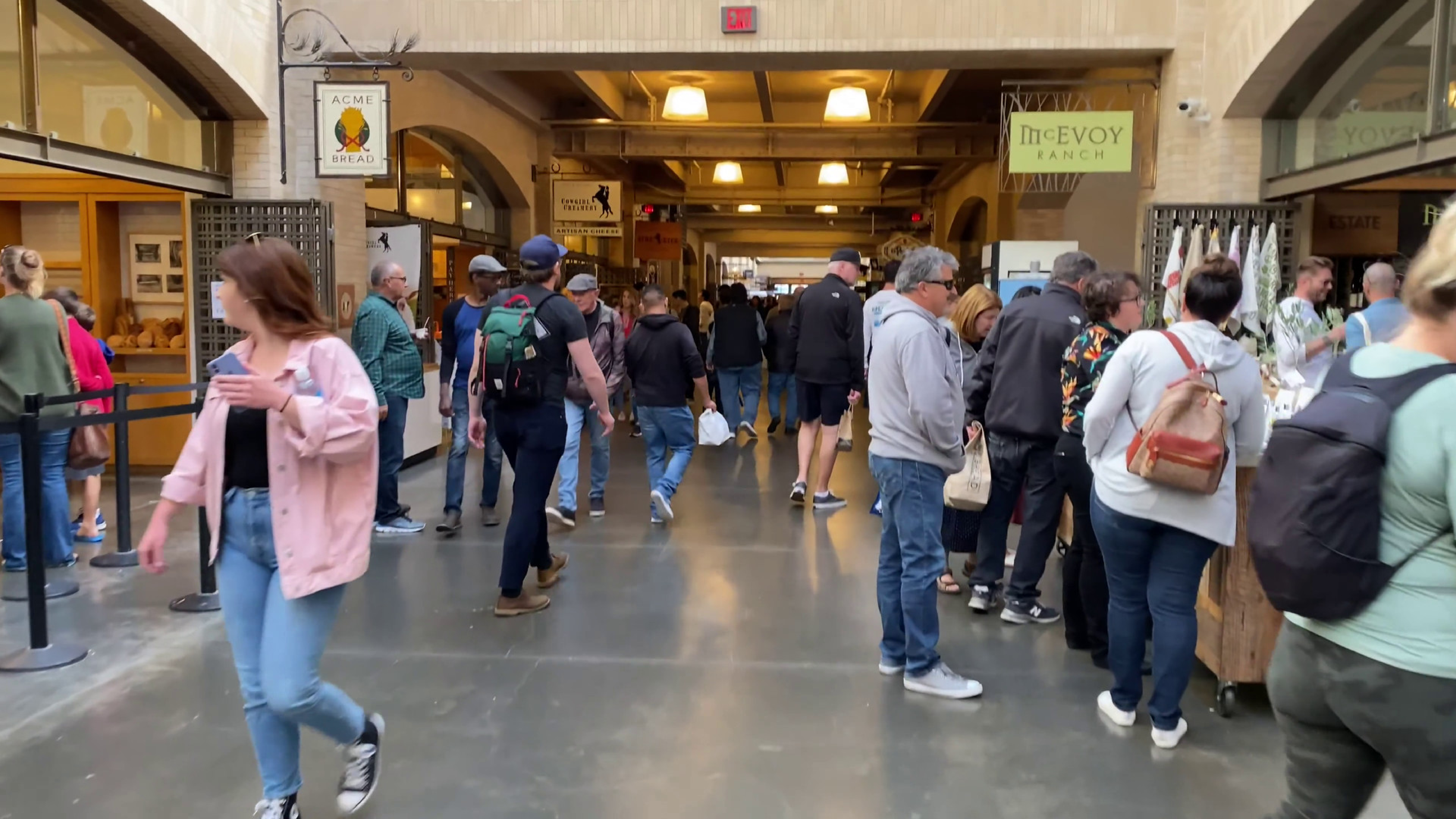}
    \\[14.3pt]
    \rotatebox[origin=c]{0}{\scriptsize} &
    \cimg{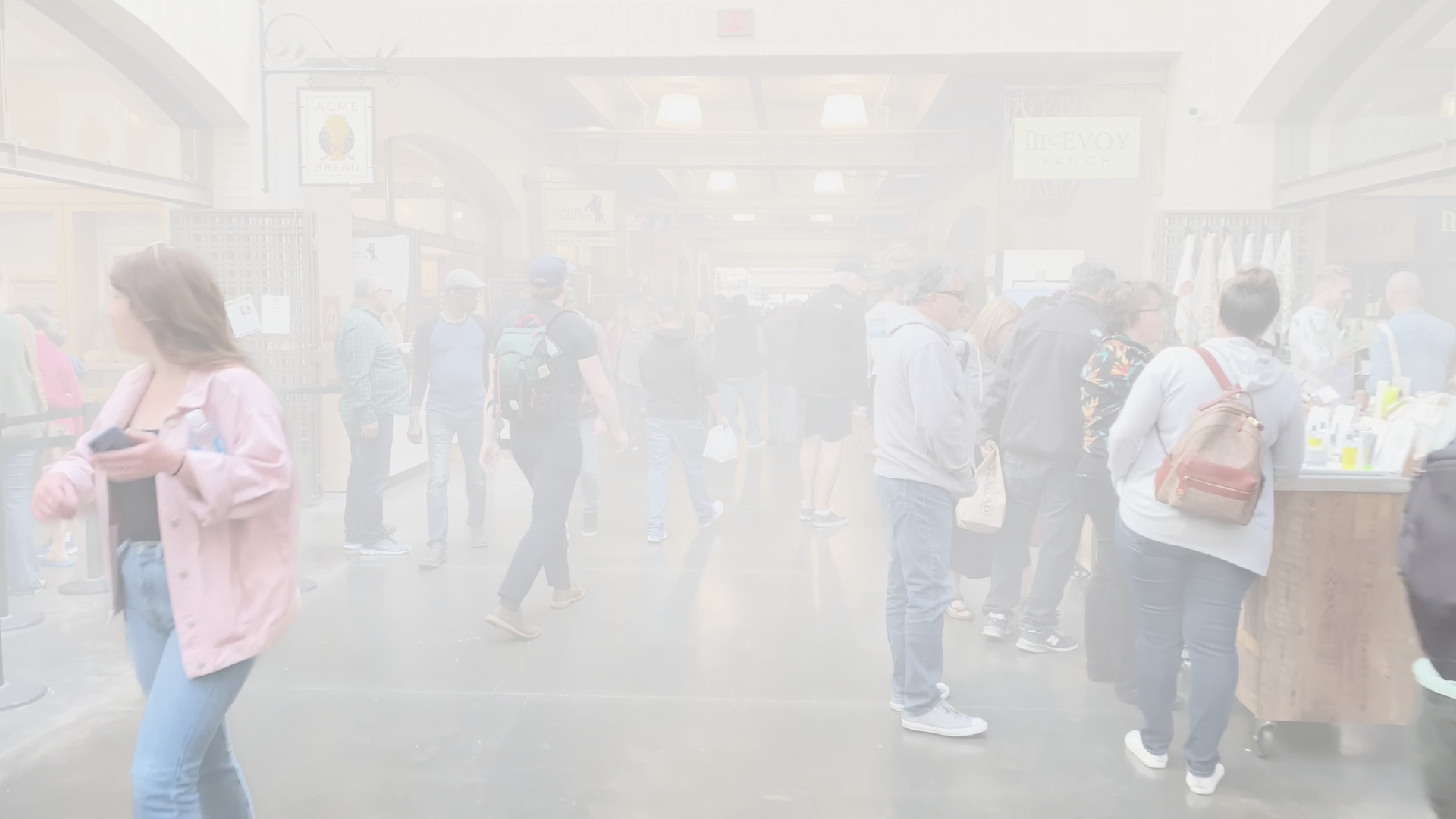} &
    \cimg{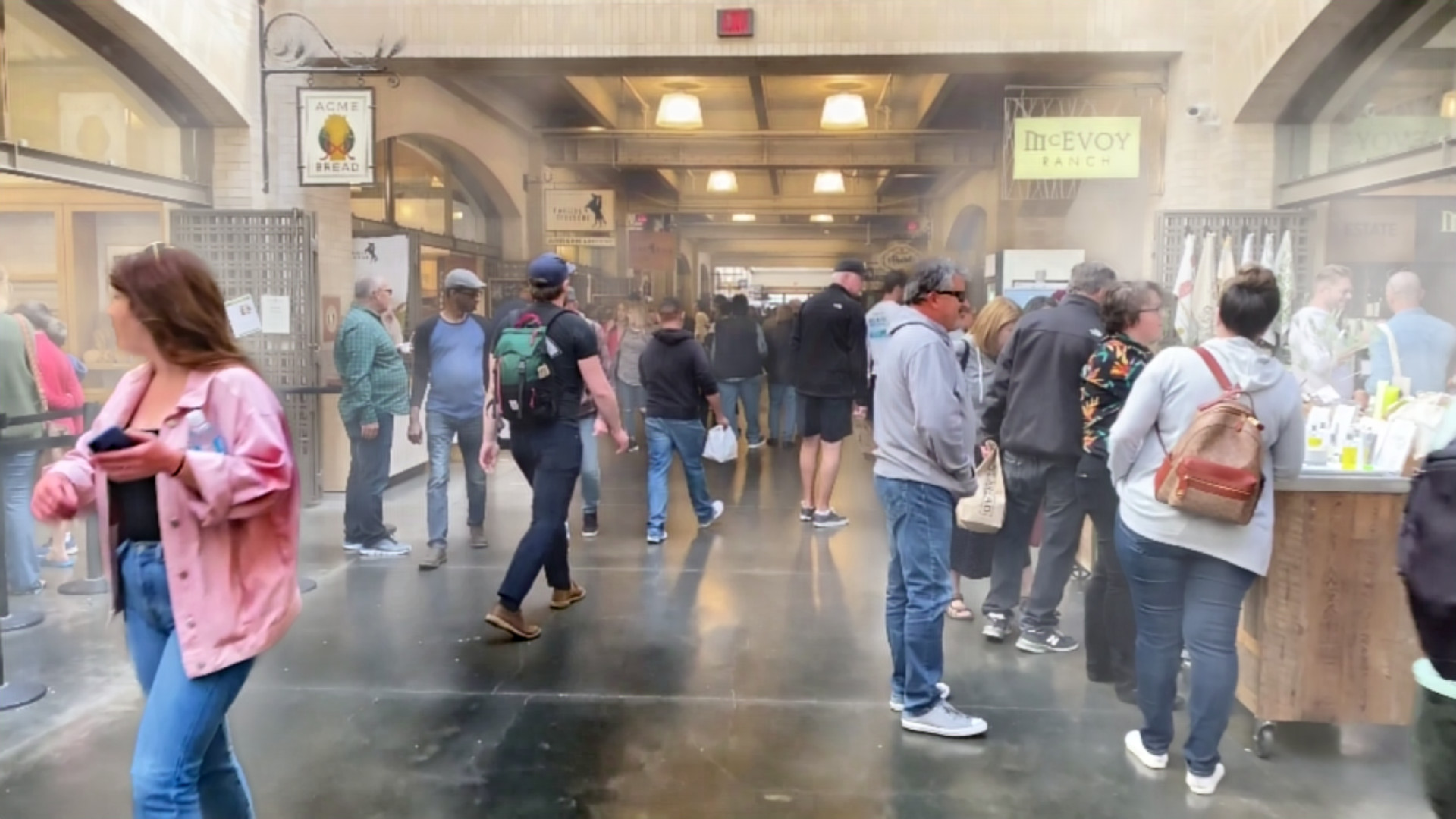} &
    \cimg{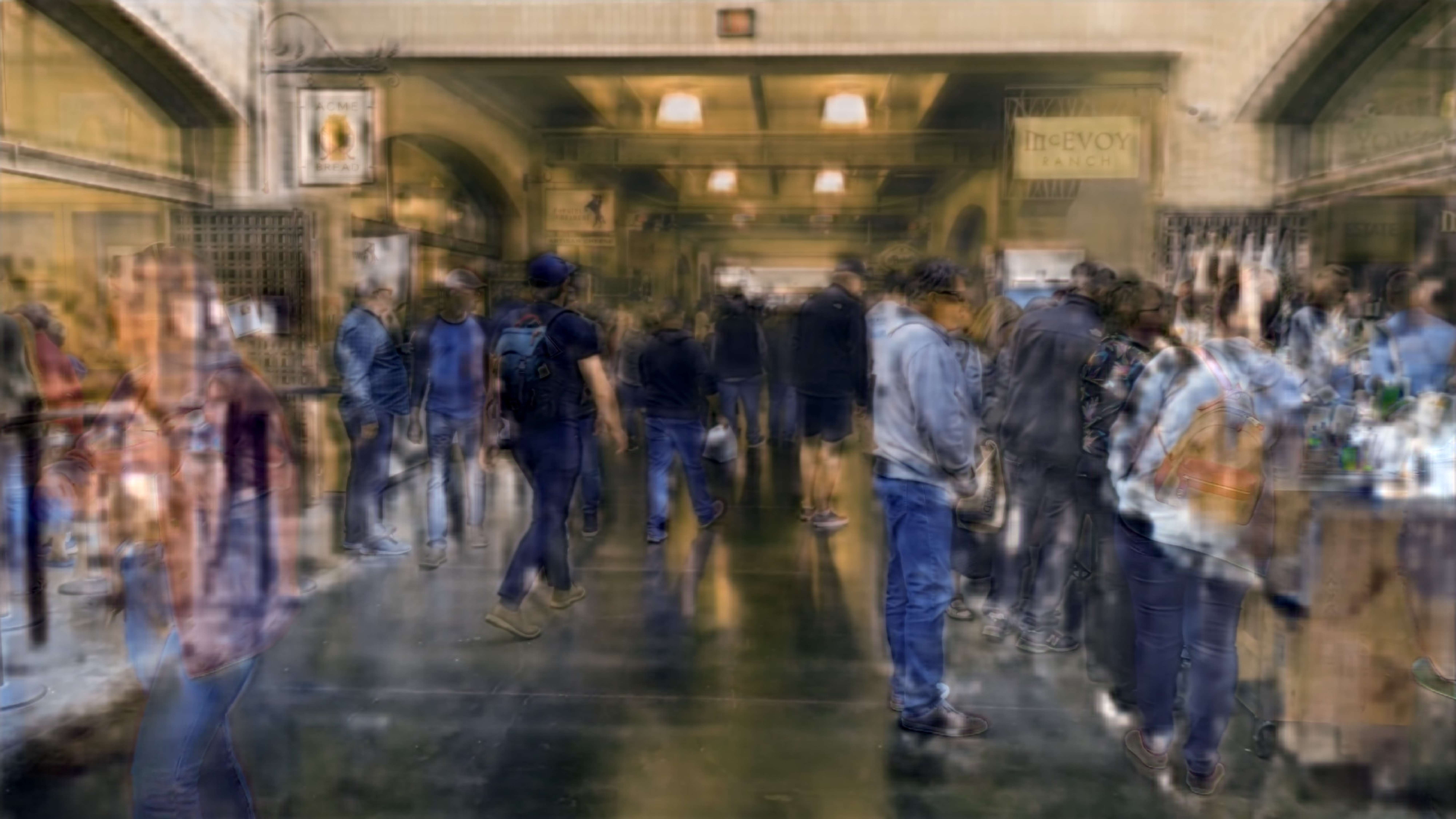} &
    \cimg{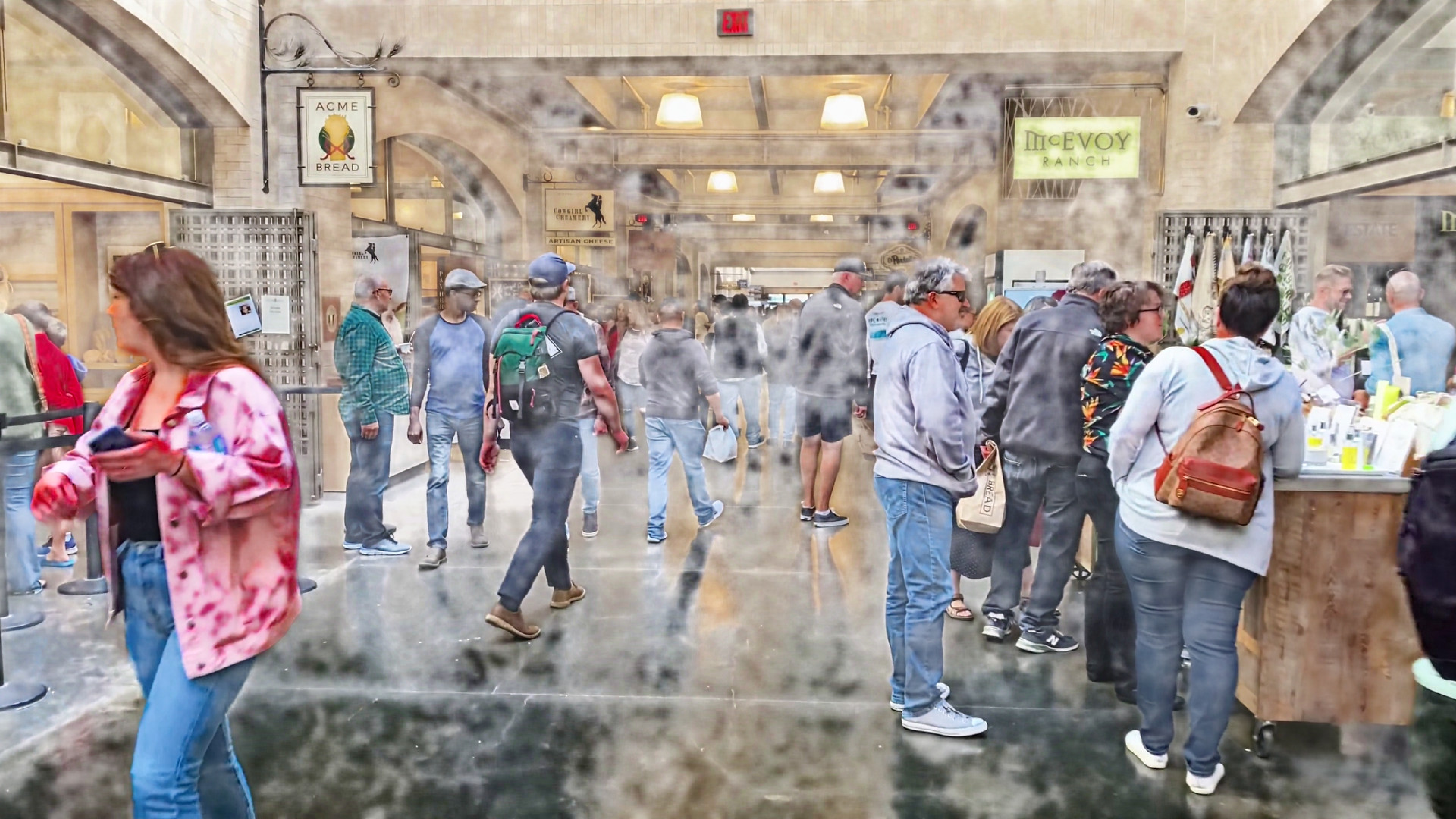} &
    \cimg{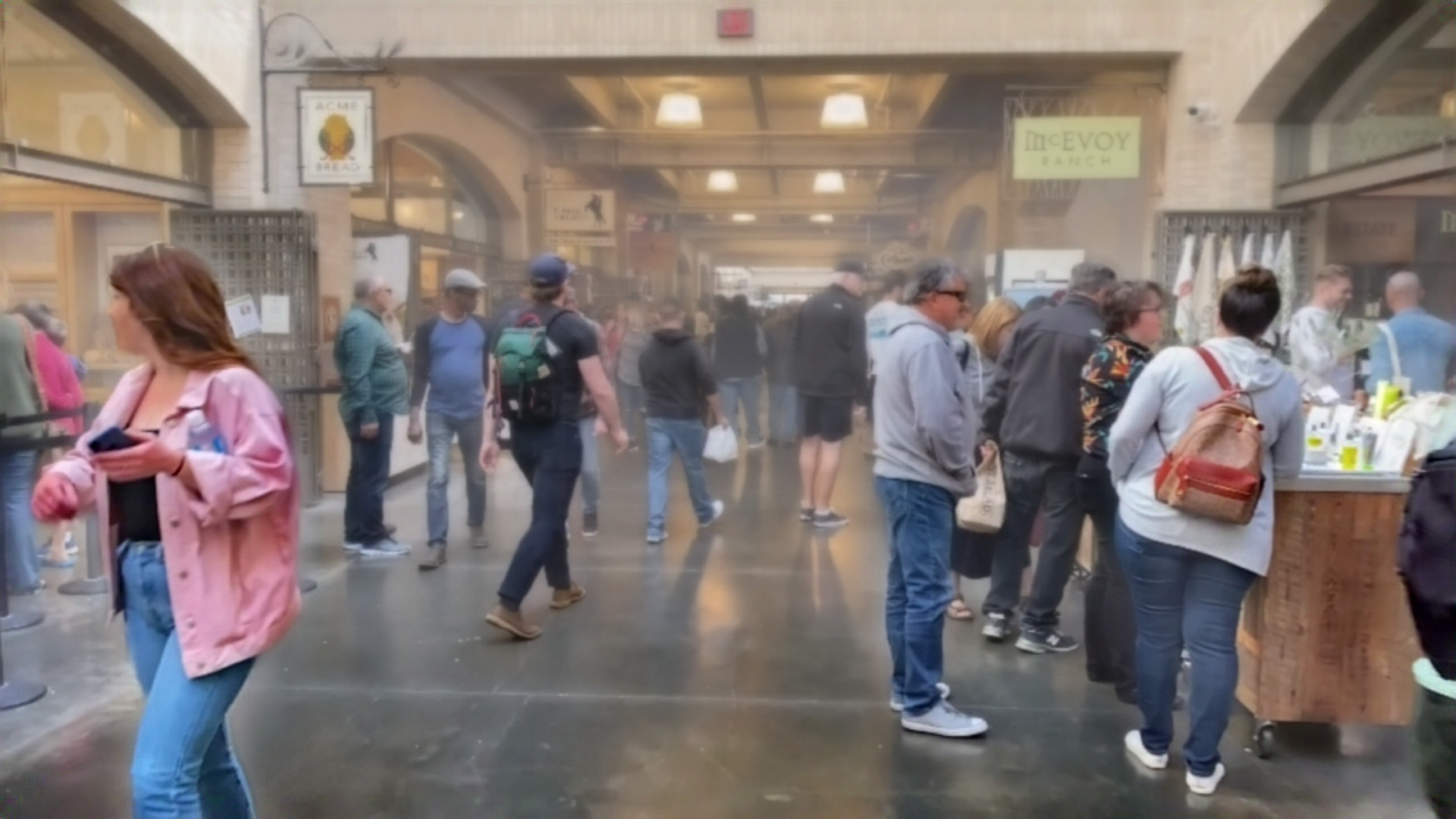} &
    \cimg{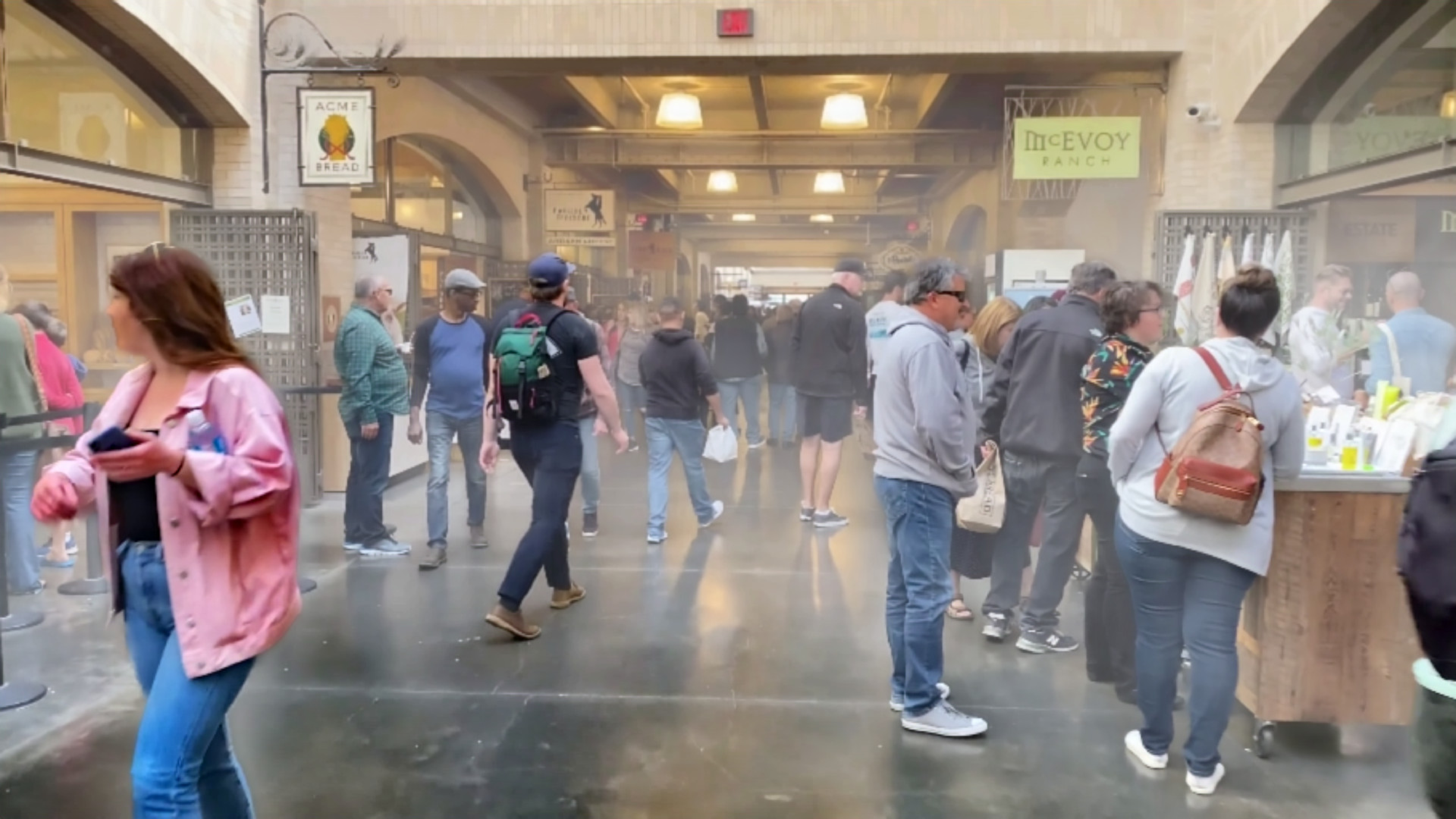} &
    \cimg{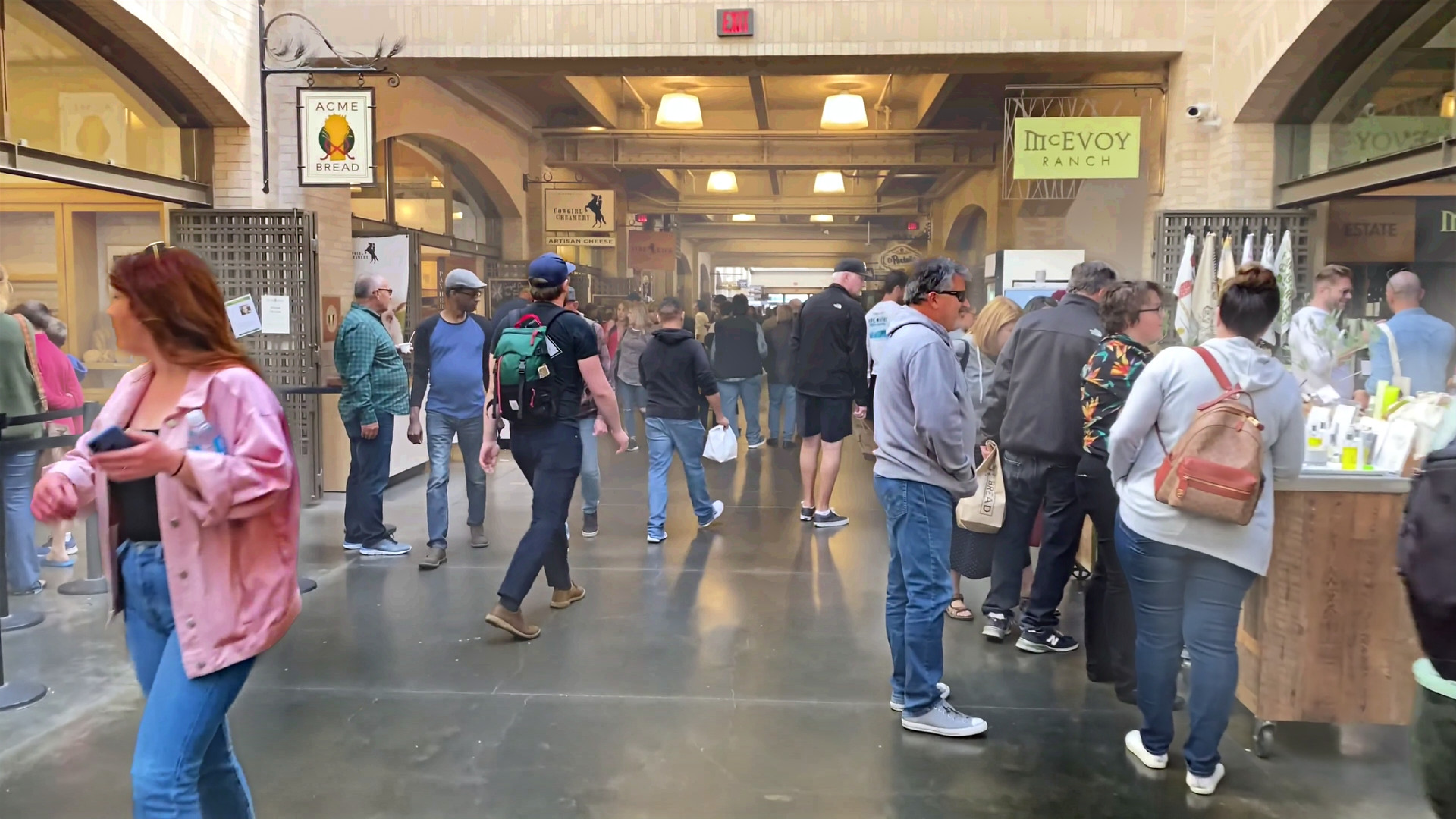} &
    \cimg{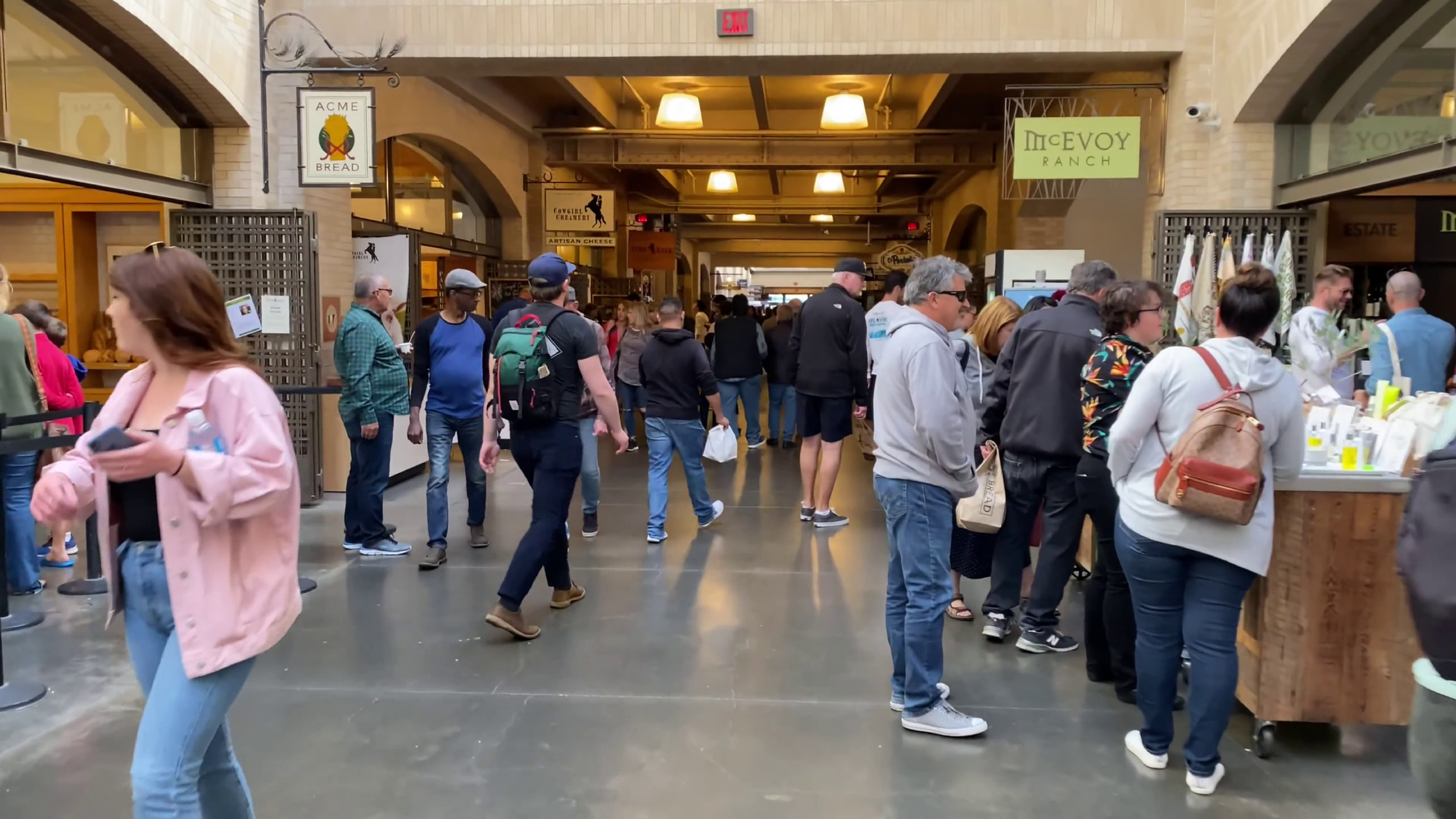}
  \\[14.3pt]
    \rotatebox[origin=c]{0}{\scriptsize(b)} &
      \zoomB{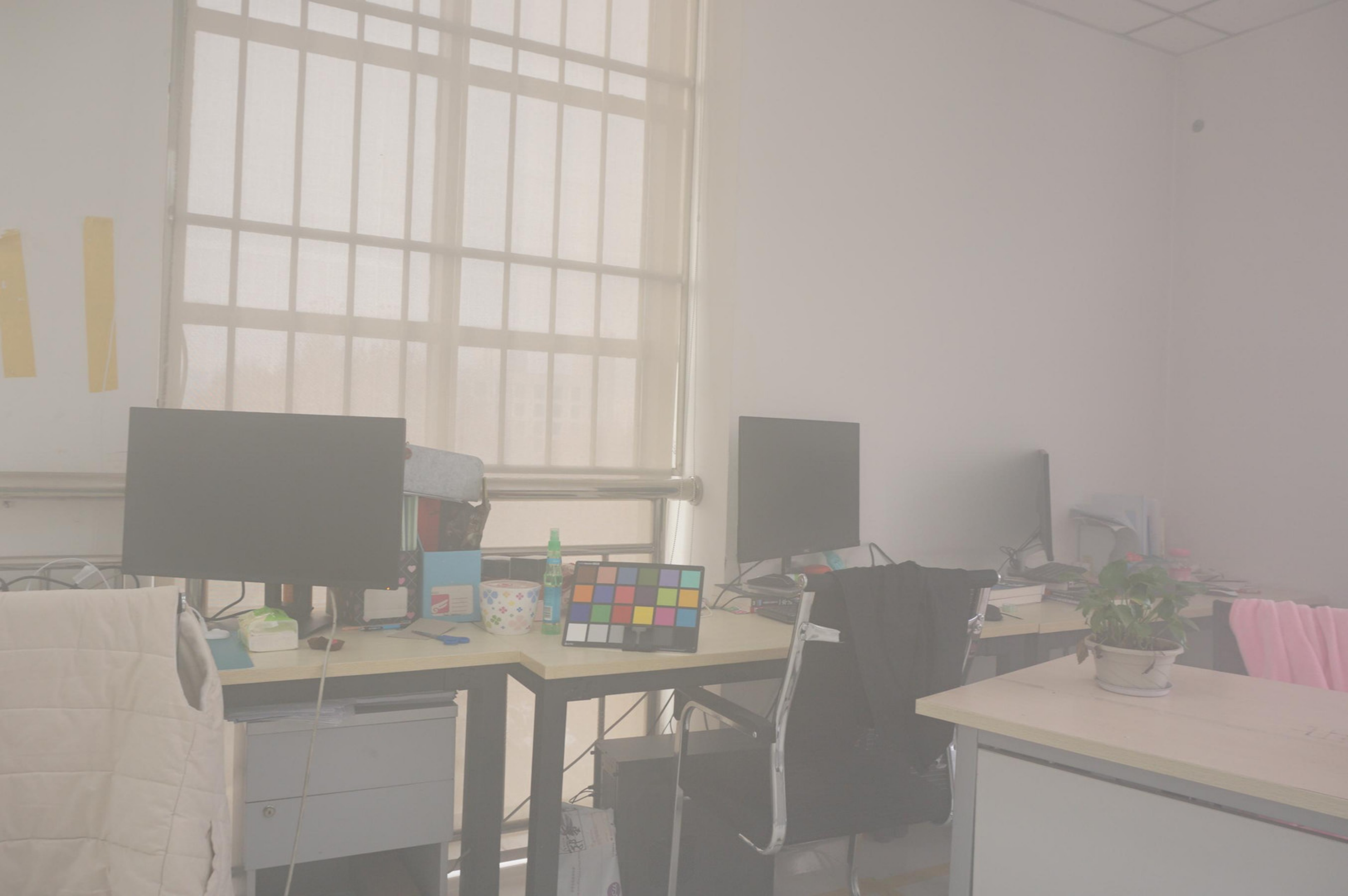} &
      \zoomB{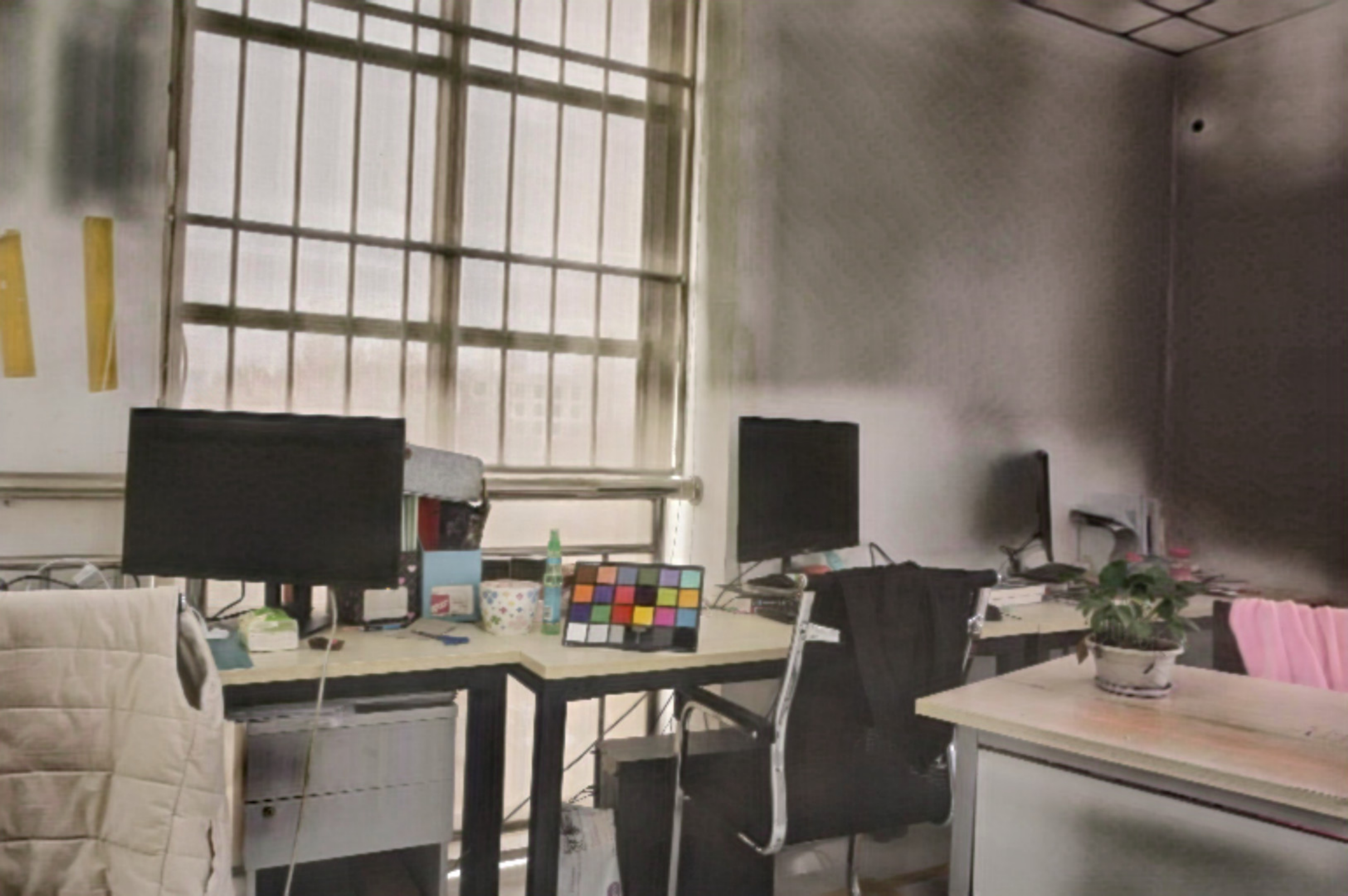} &
      \zoomB{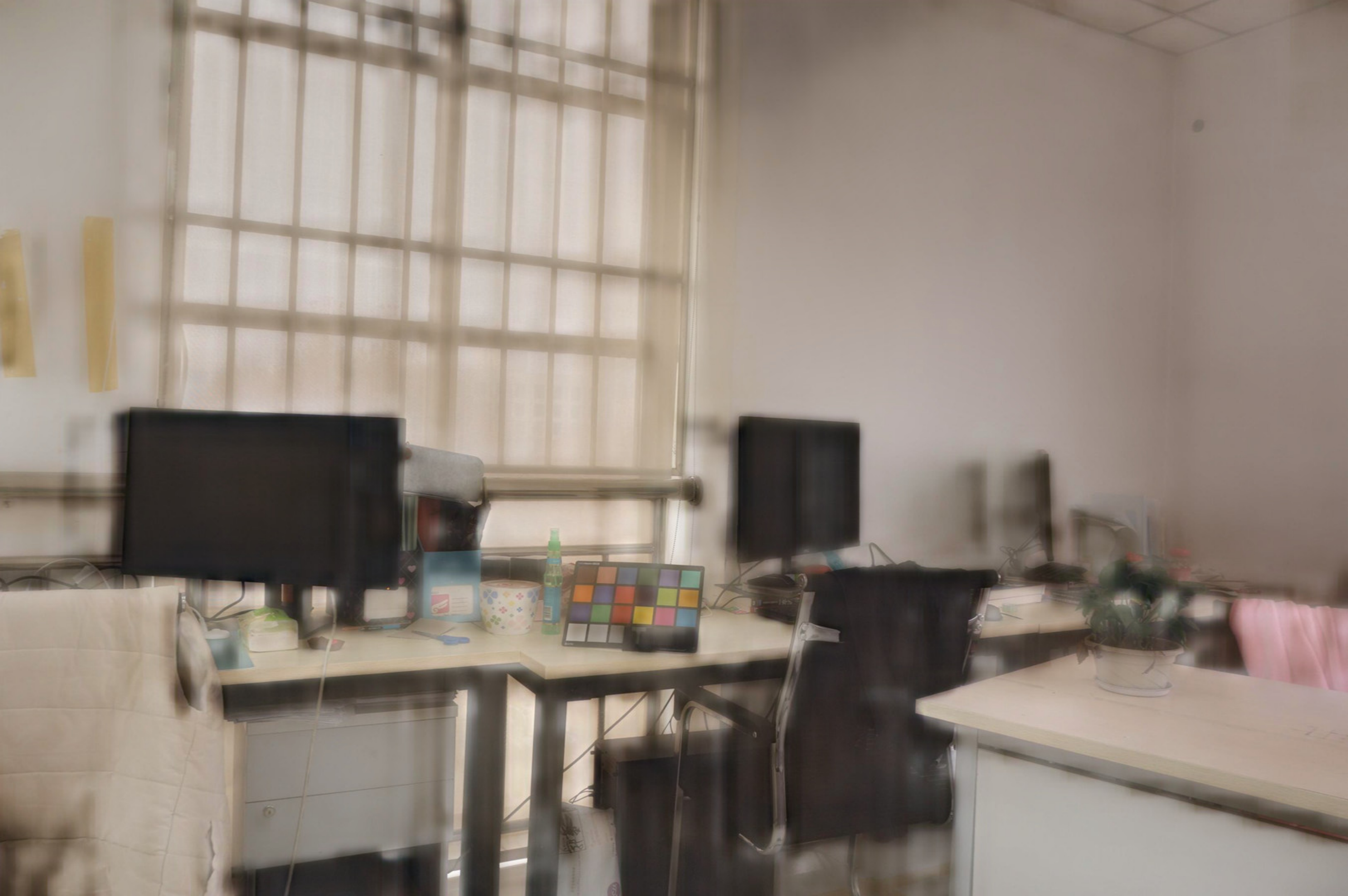} &
      \zoomB{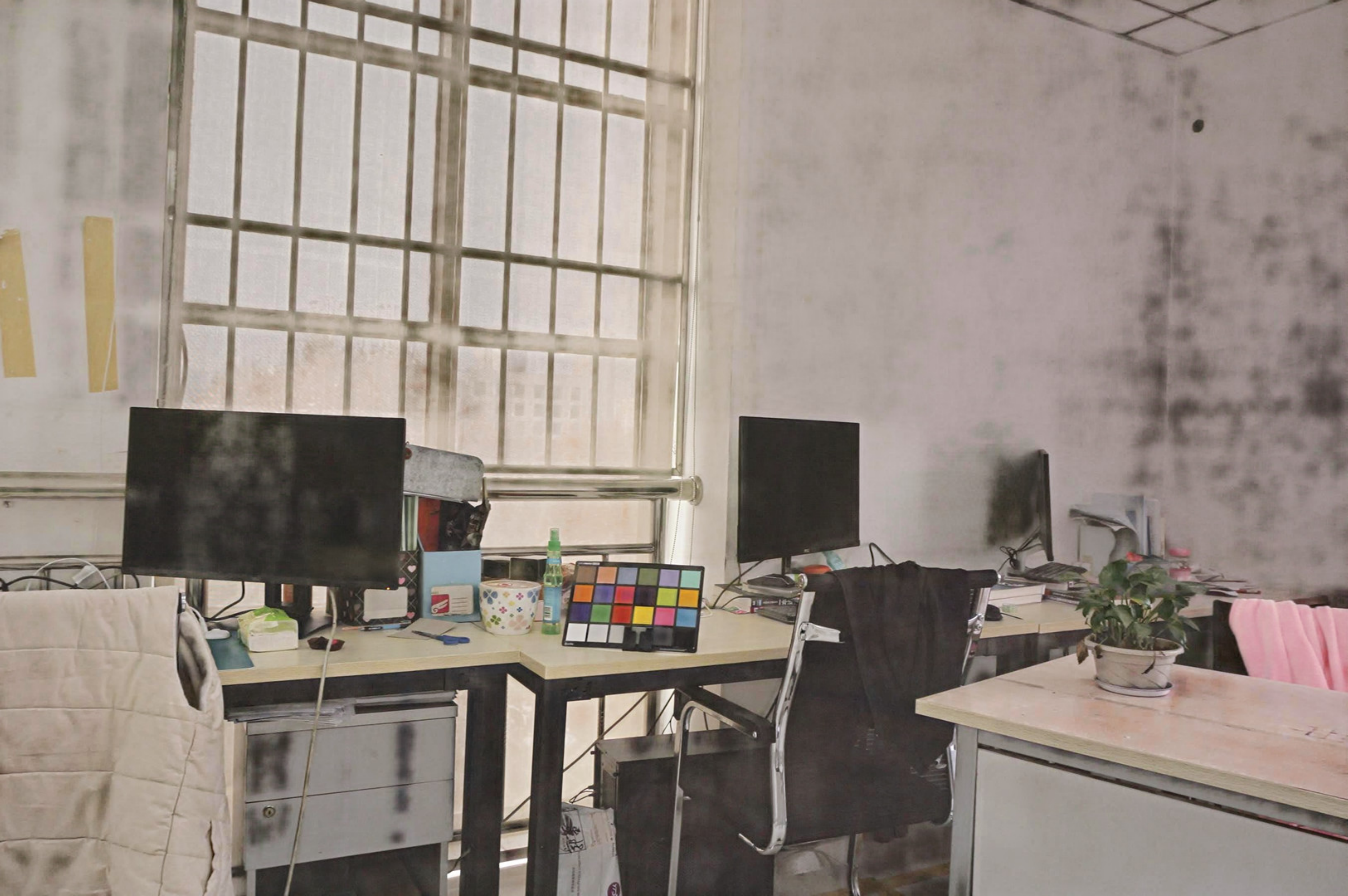} &
      \zoomB{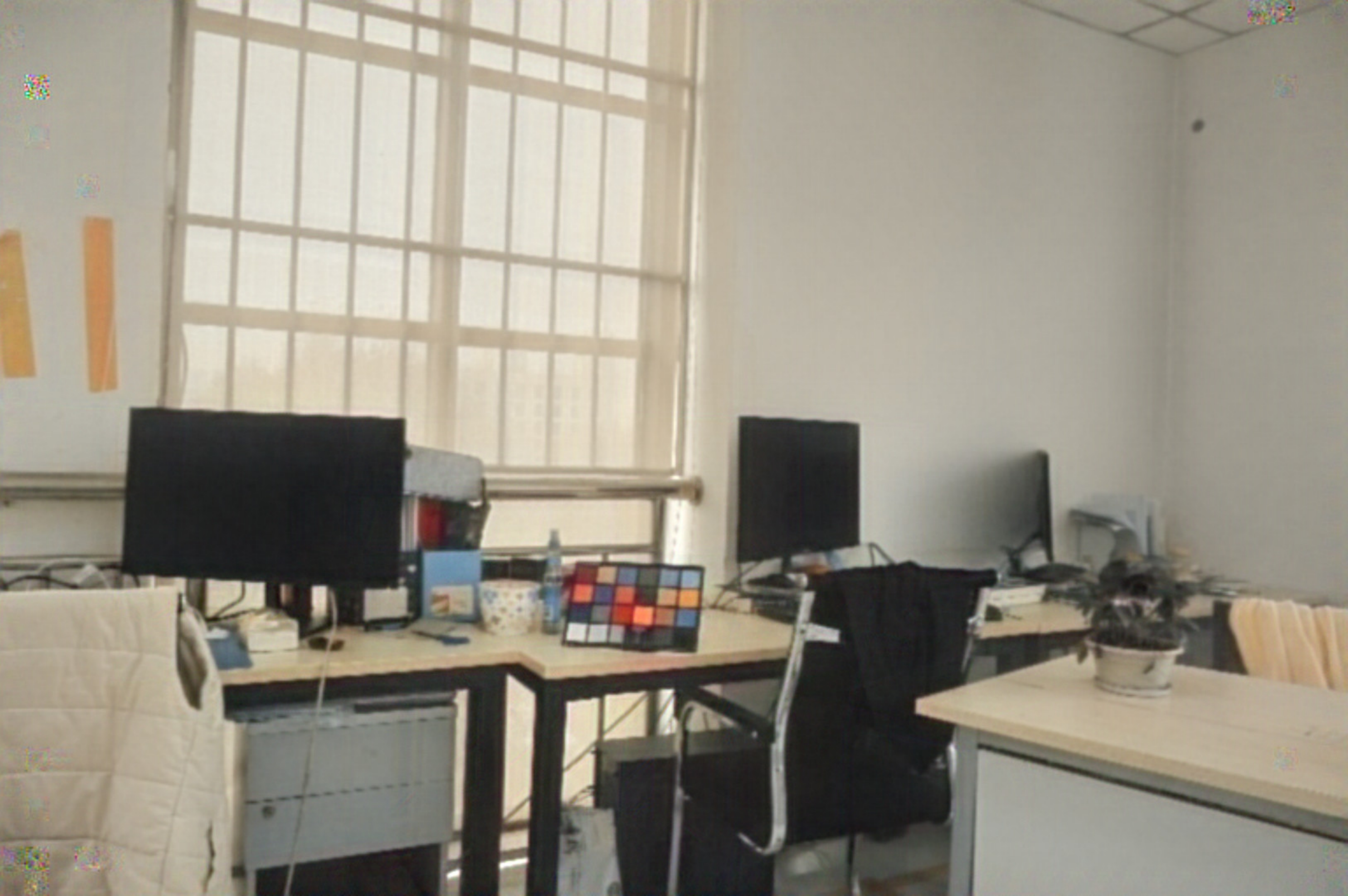} &
      \zoomB{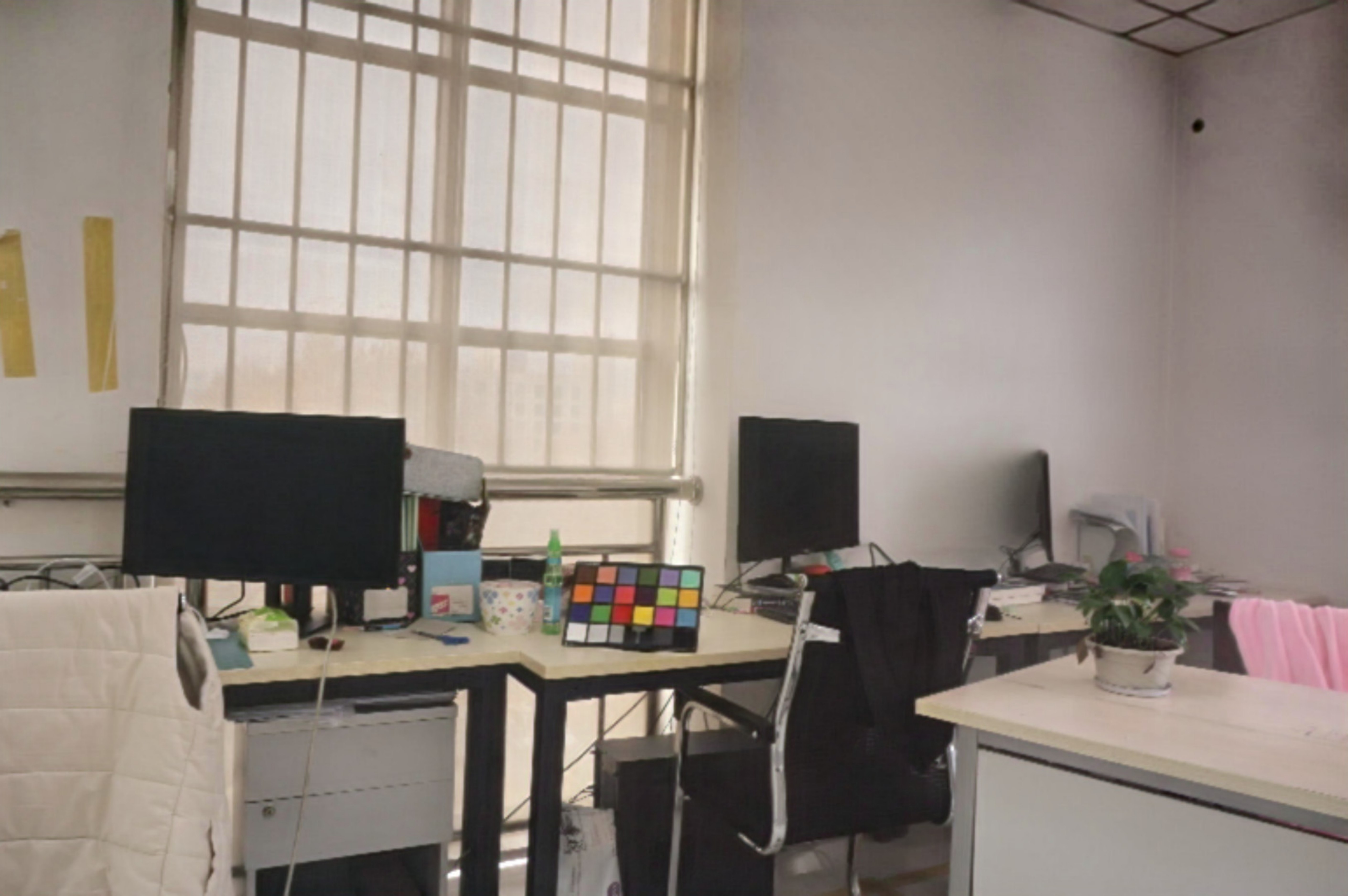} &
      \zoomB{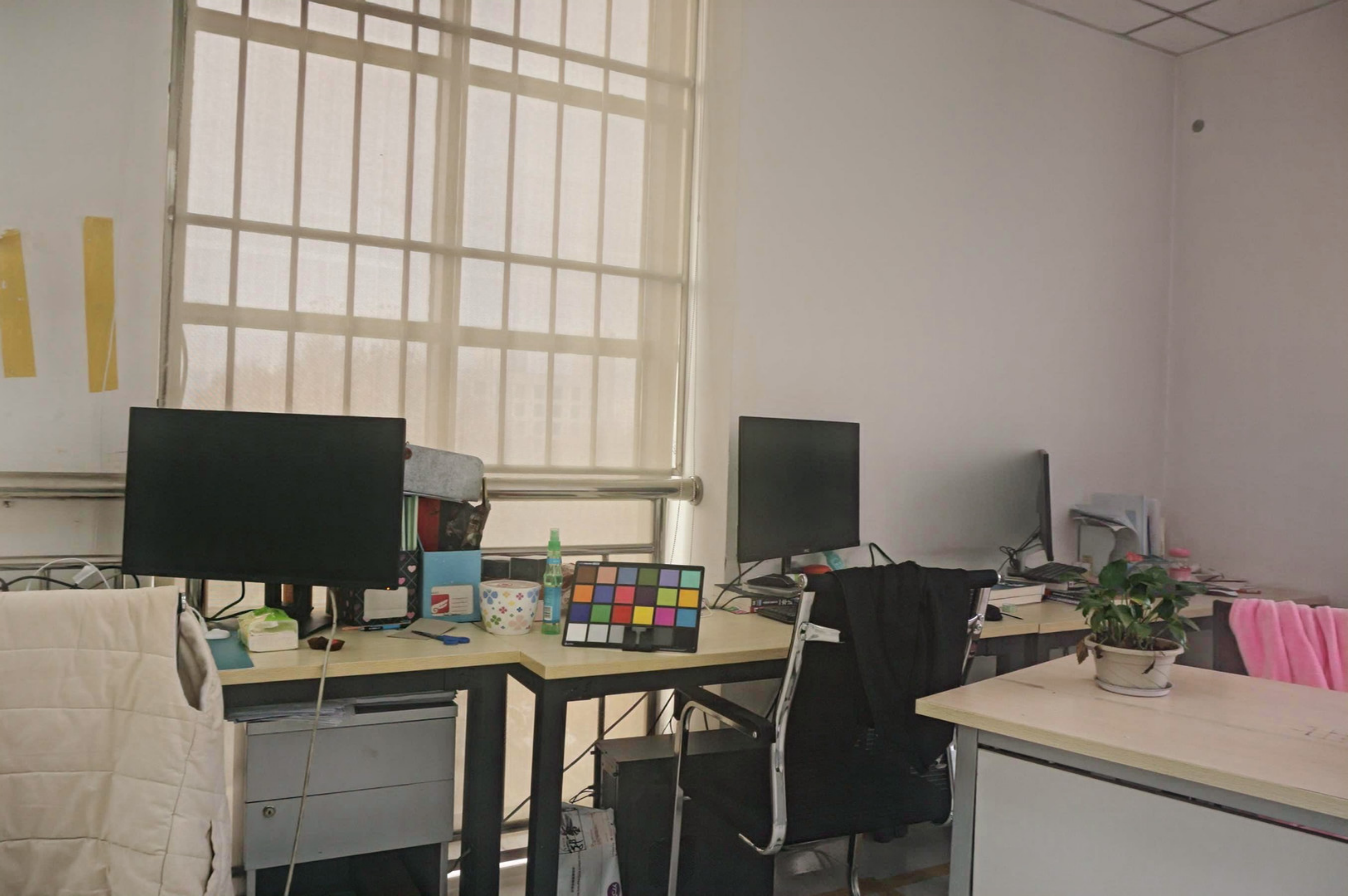} &
      \zoomB{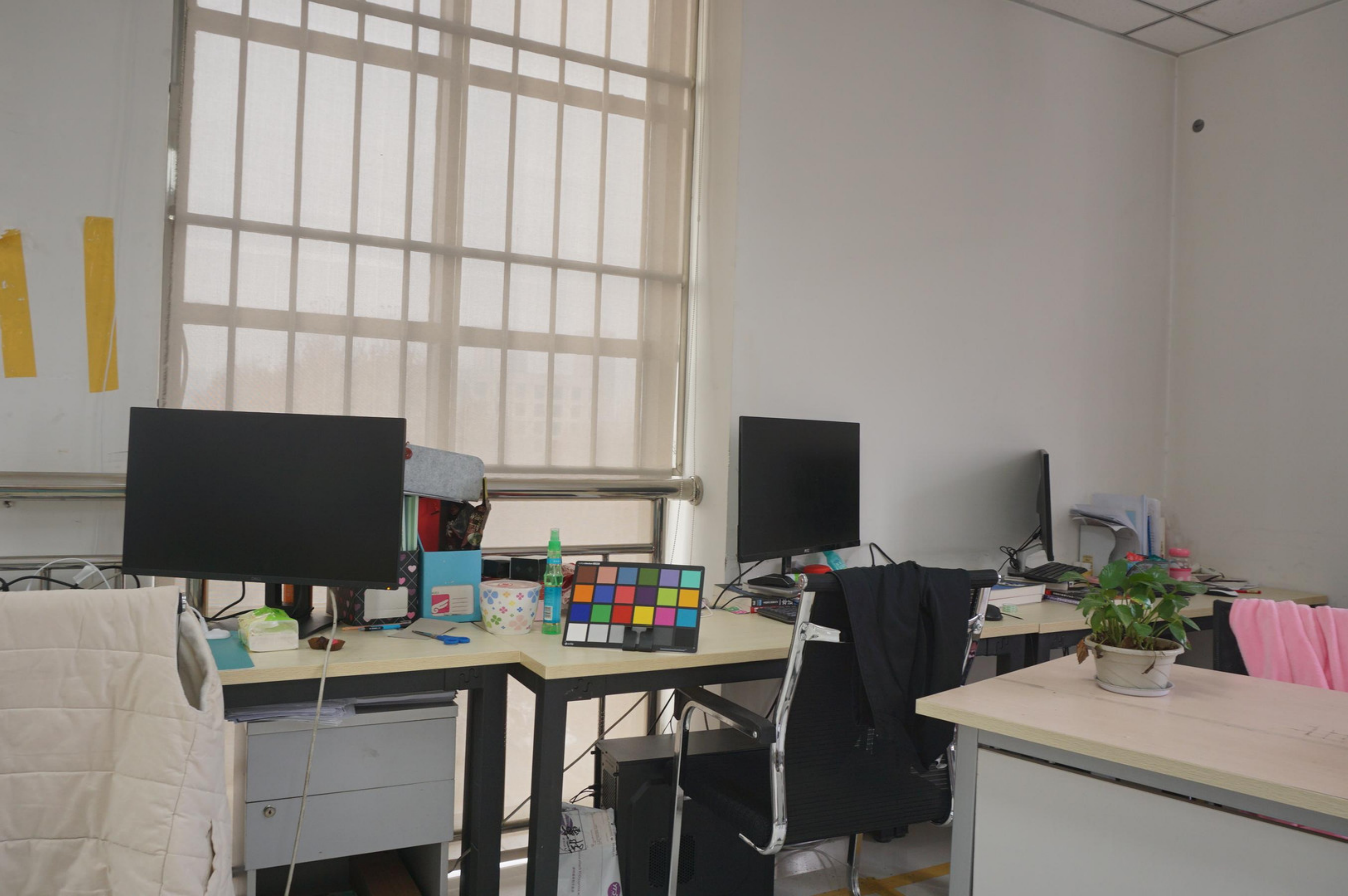}
    \\[16.4pt]
    \rotatebox[origin=c]{0}{\scriptsize} &
      \cimg{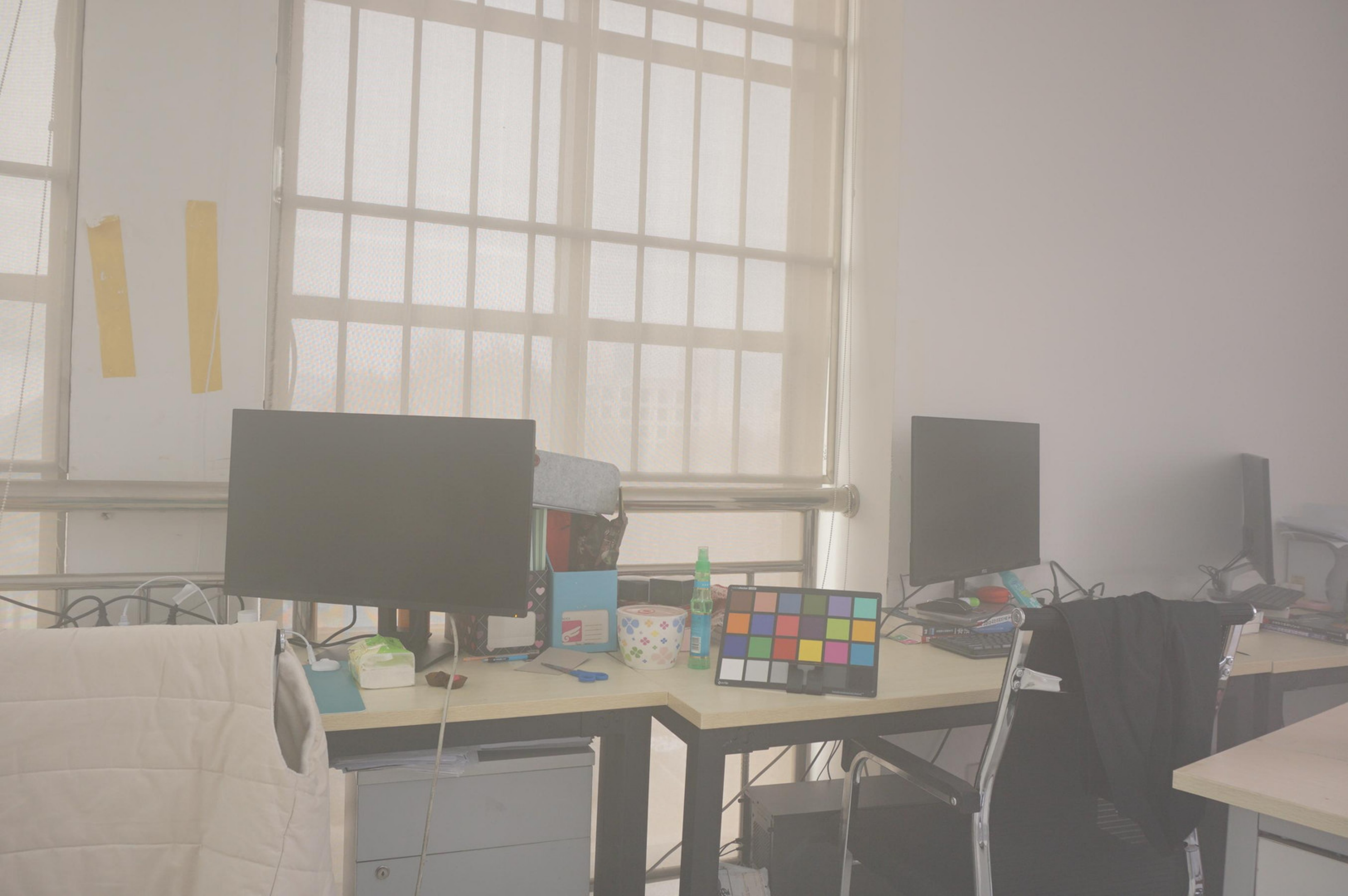} &
      \cimg{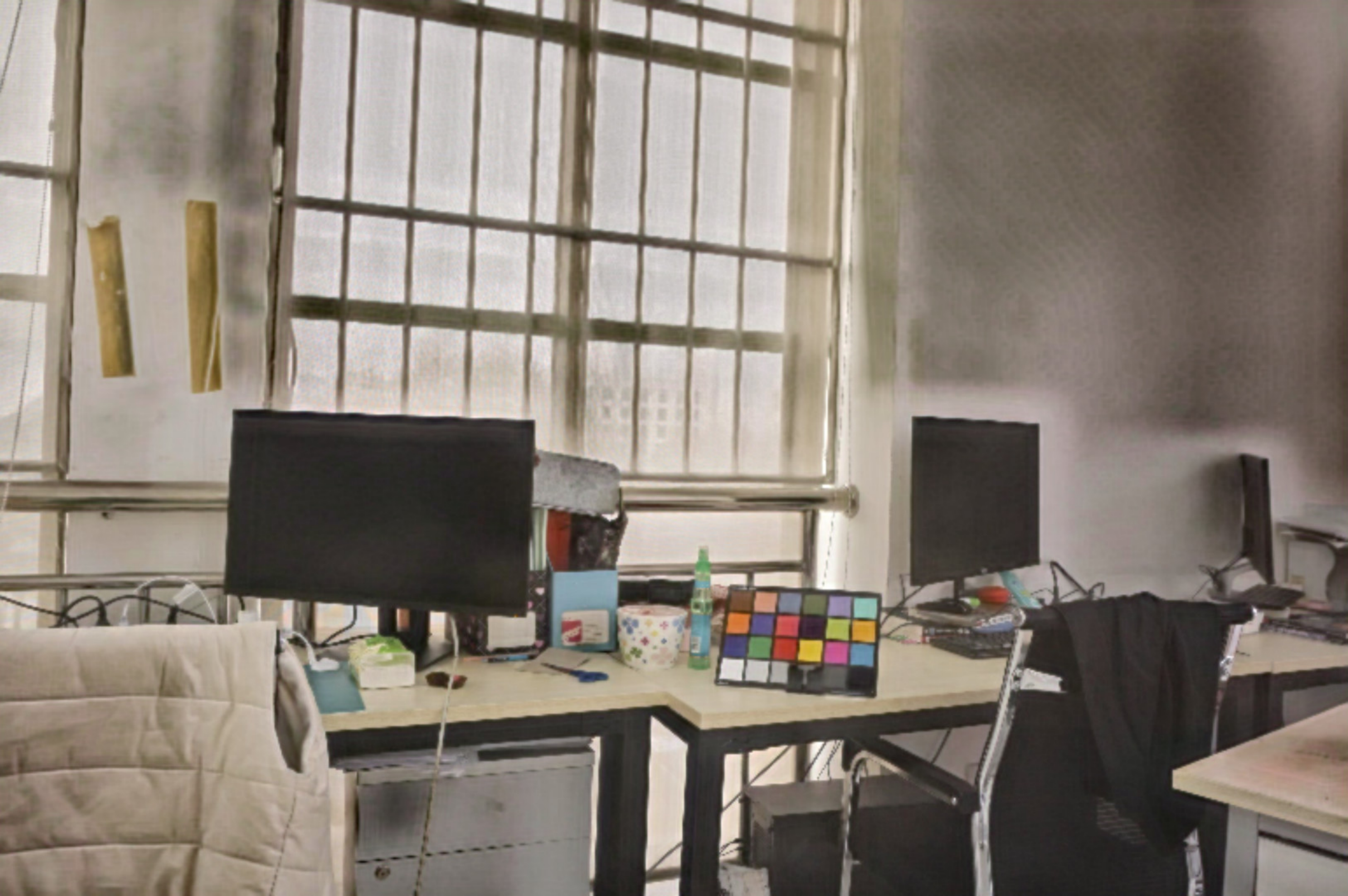} &
      \cimg{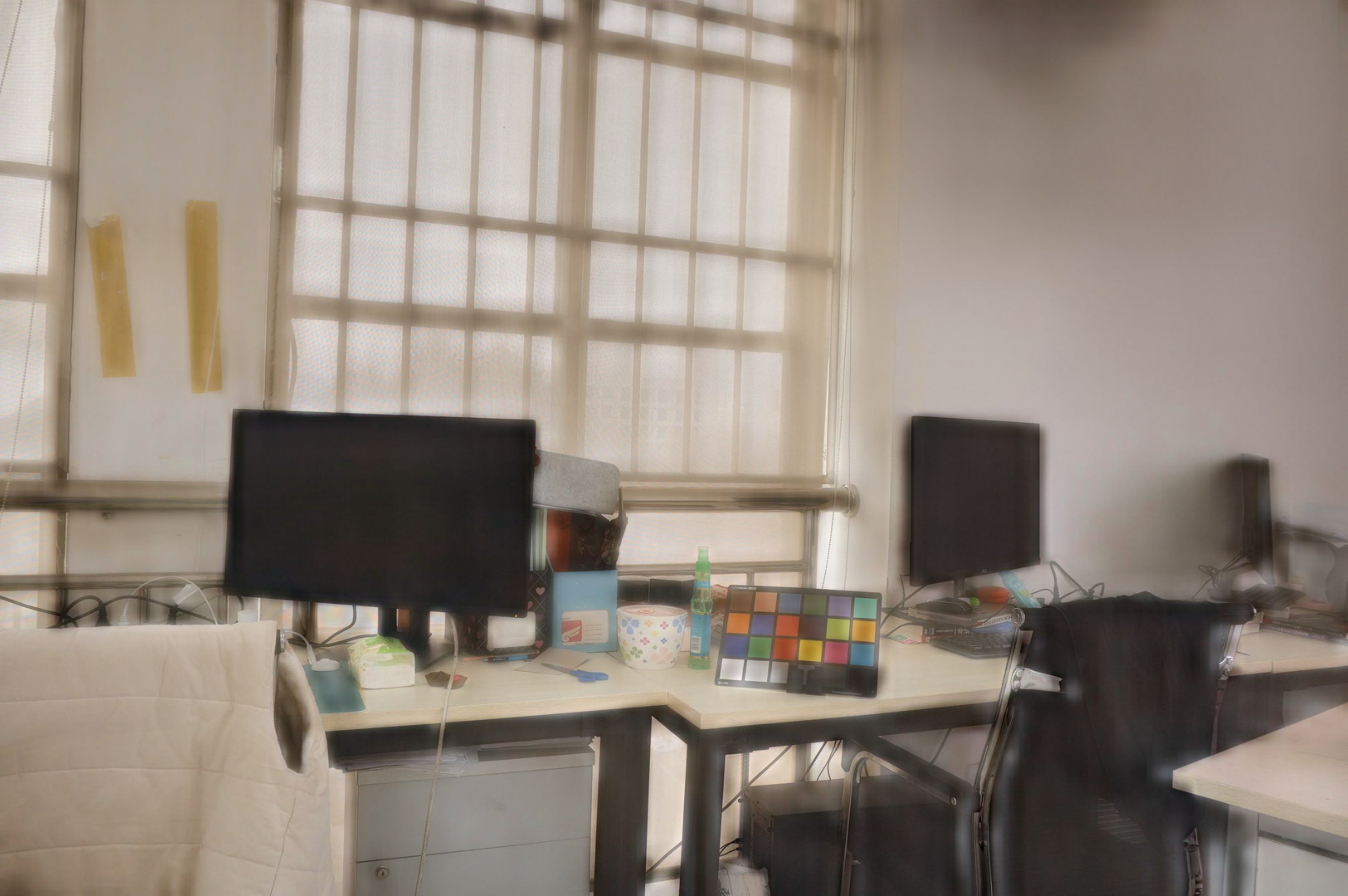} &
      \cimg{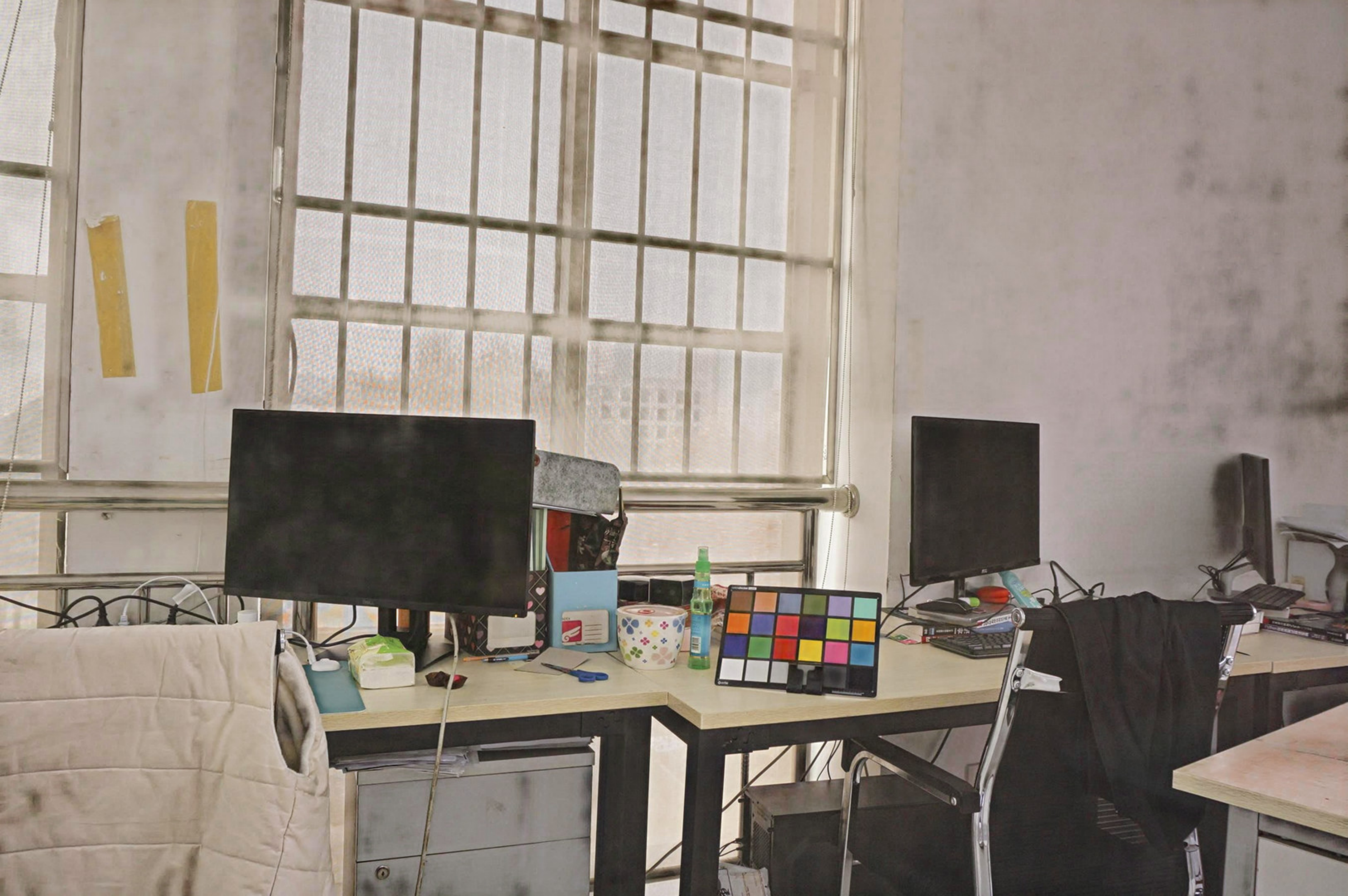} &
      \cimg{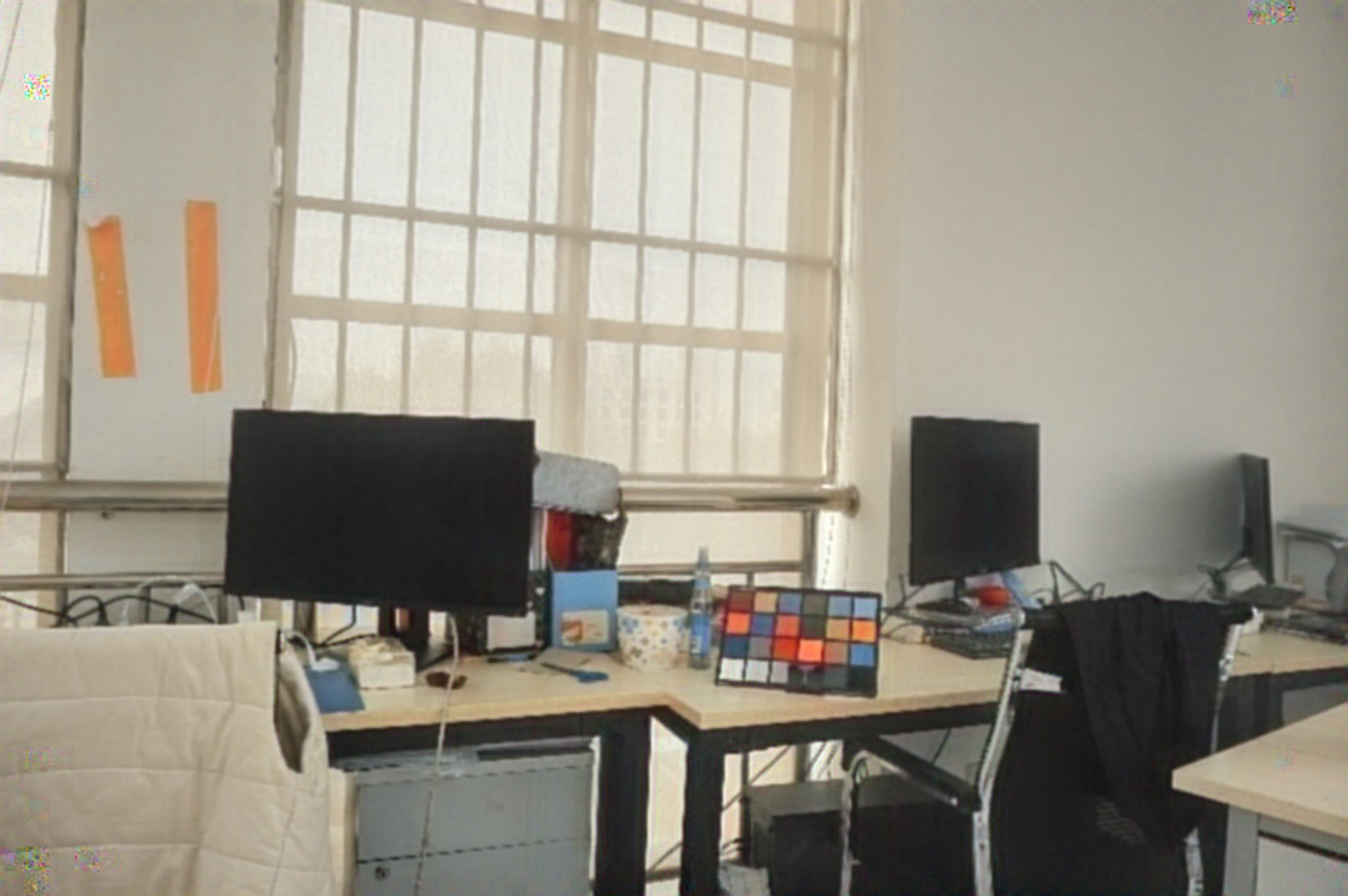} &
      \cimg{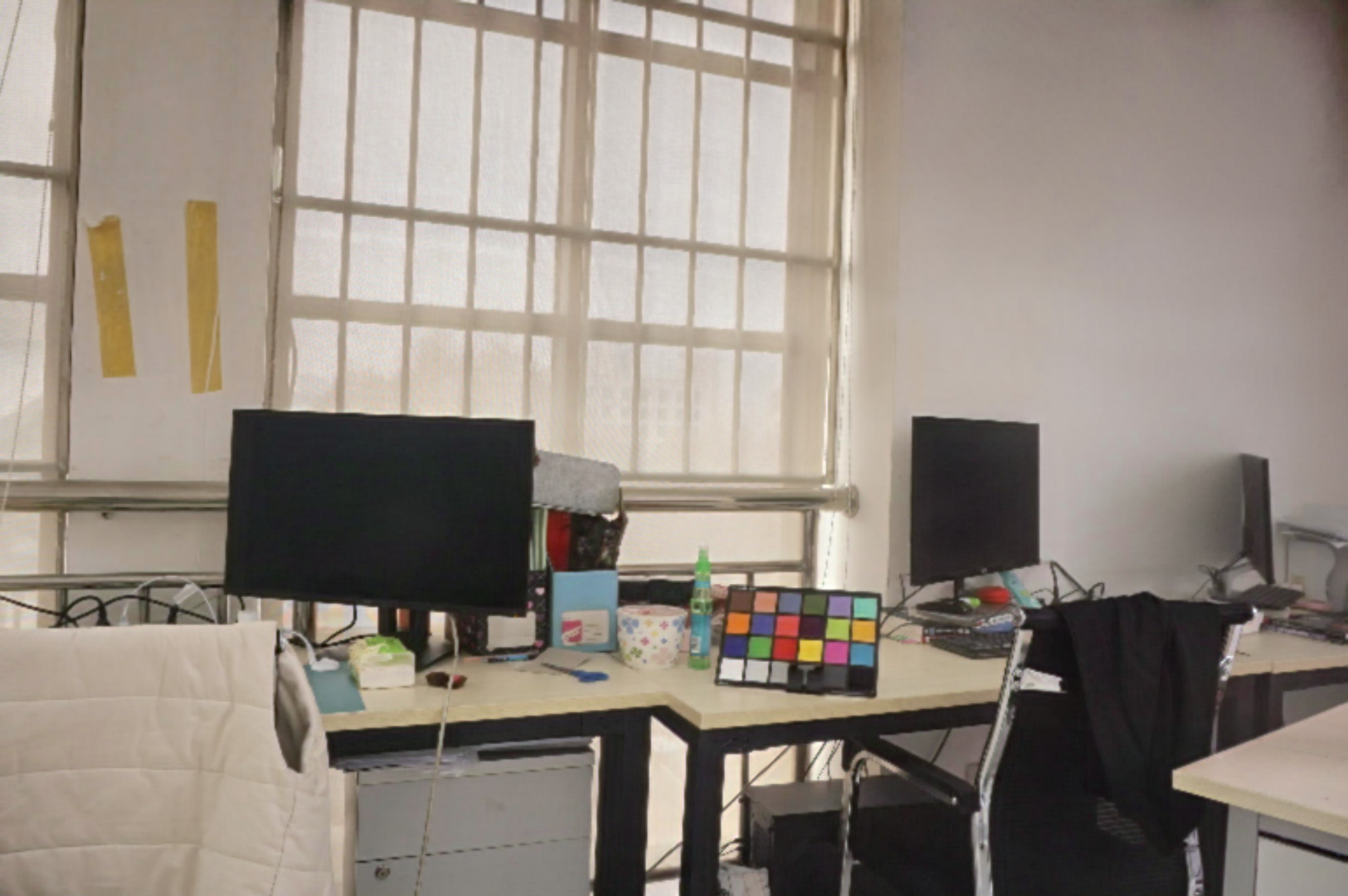} &
      \cimg{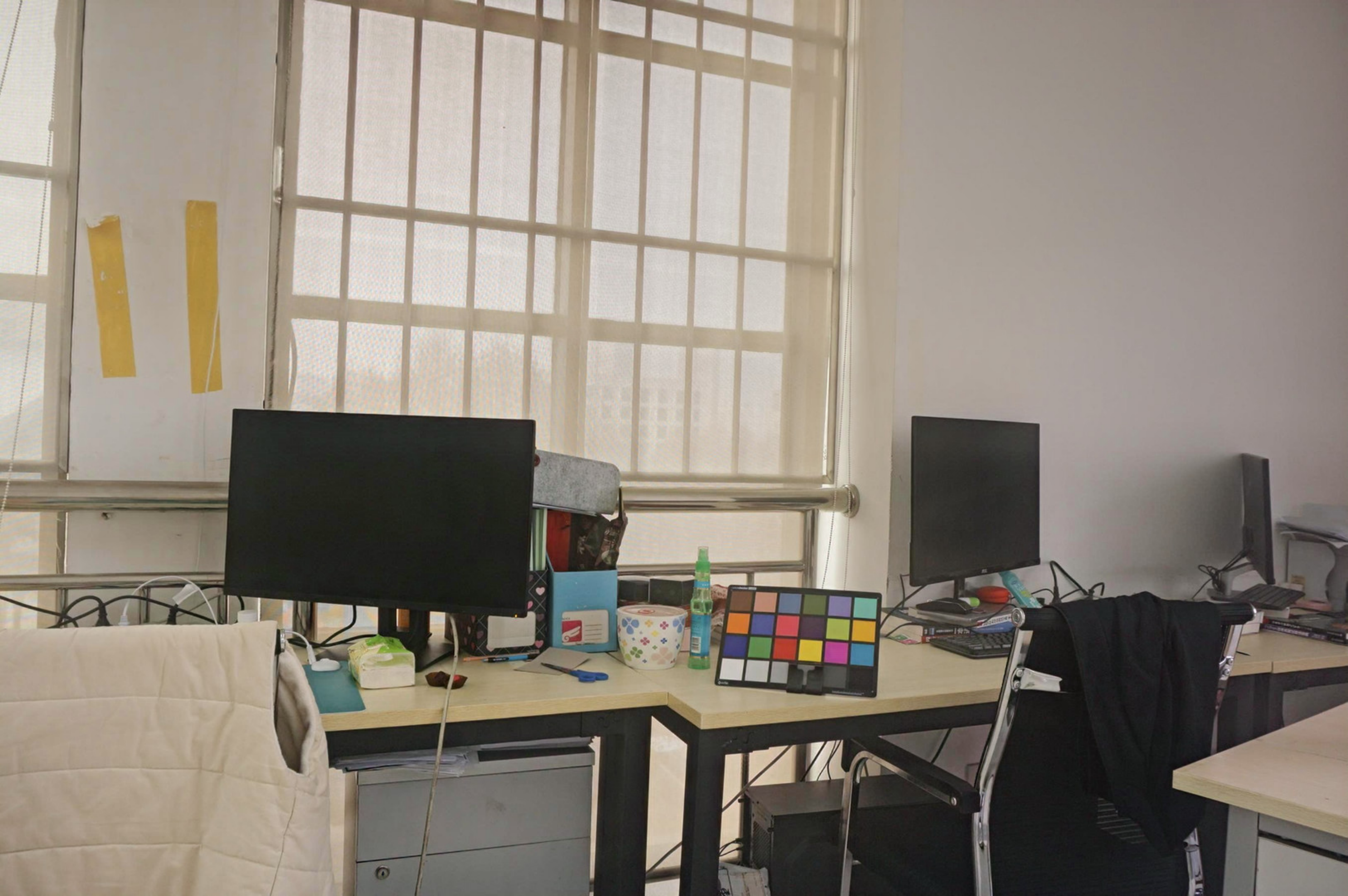} &
      \cimg{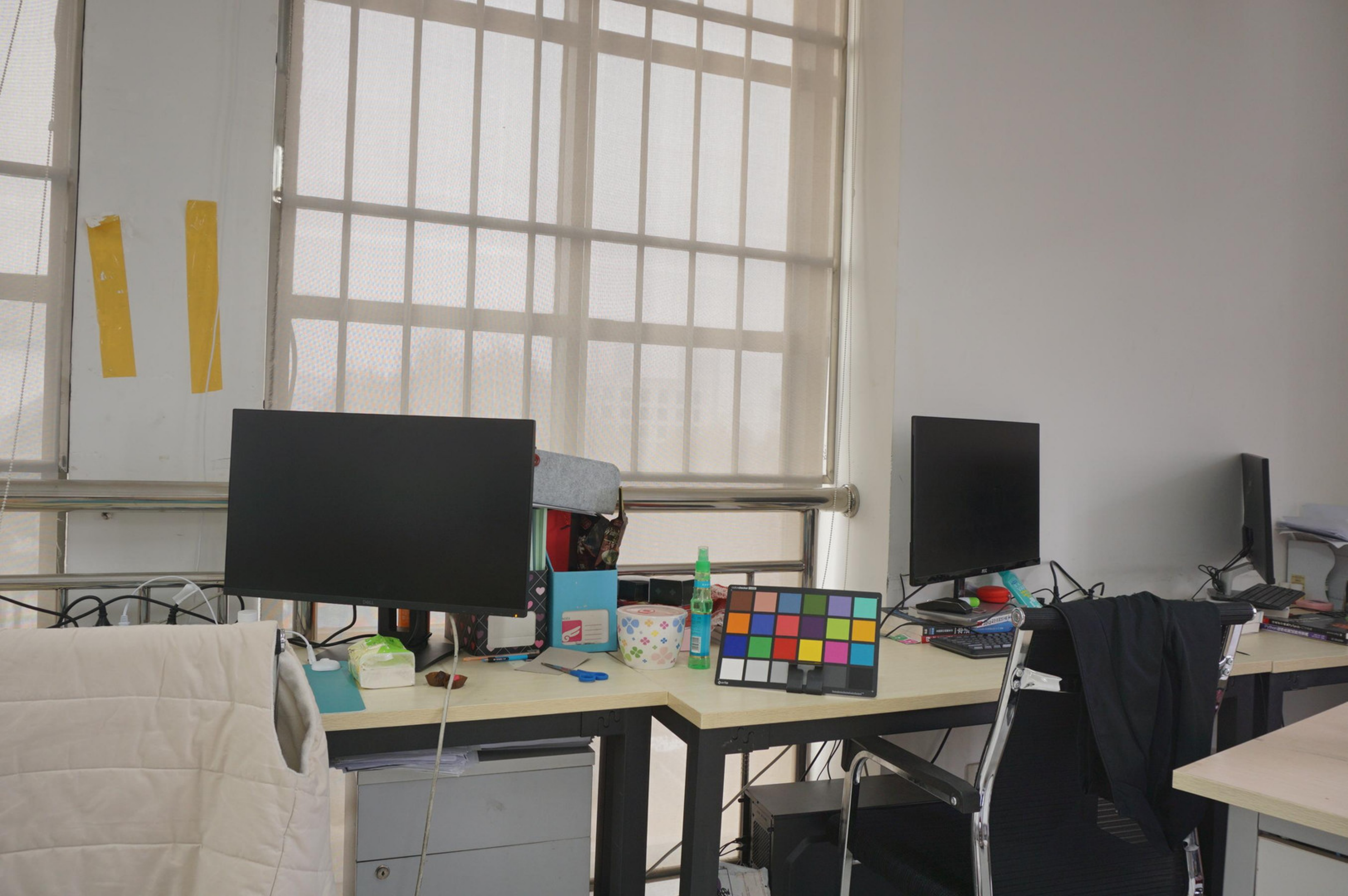}
    \\[16.4pt]
    \rotatebox[origin=c]{0}{\scriptsize(c)} &
      \zoomC{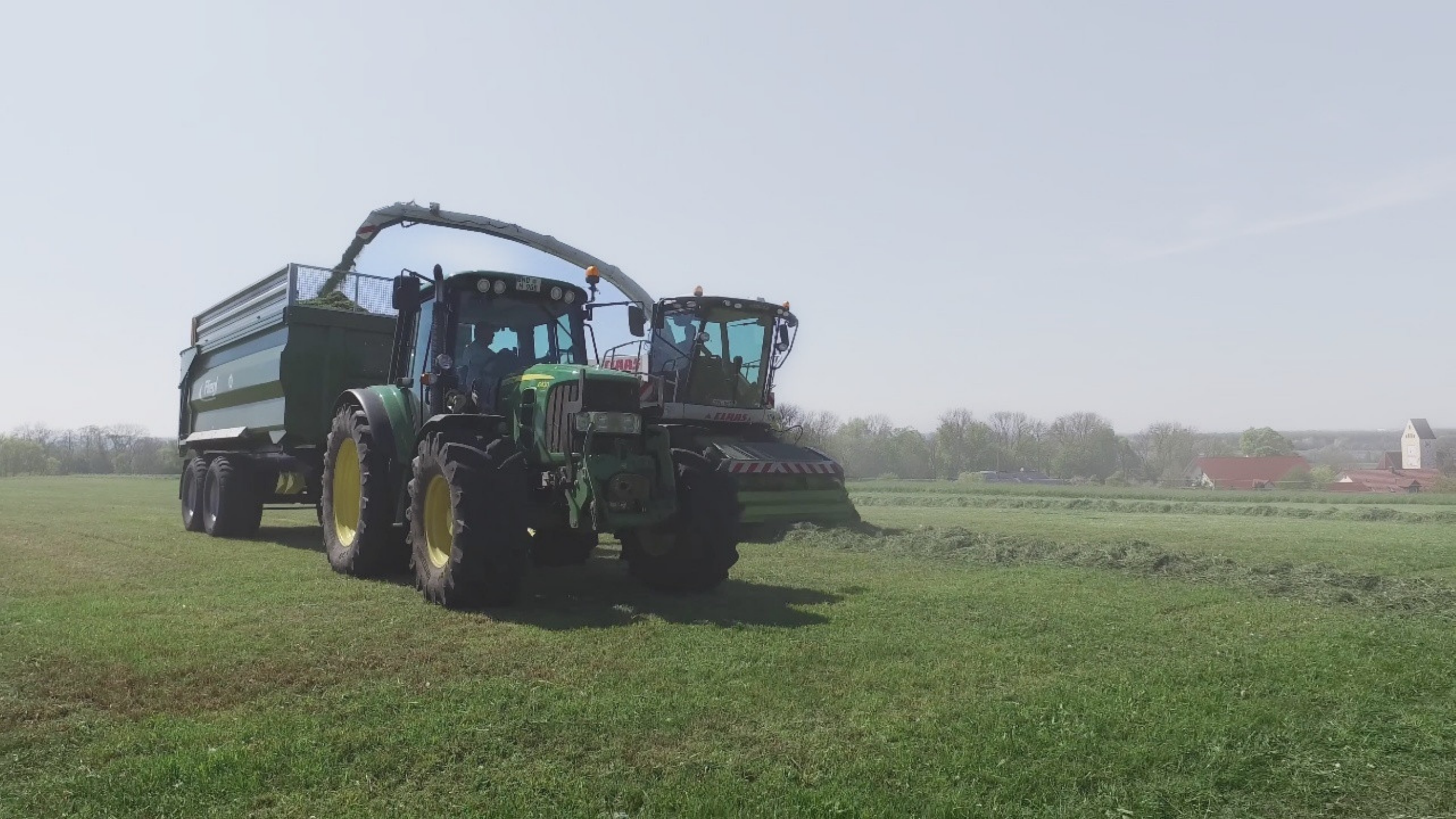} &
      \zoomC{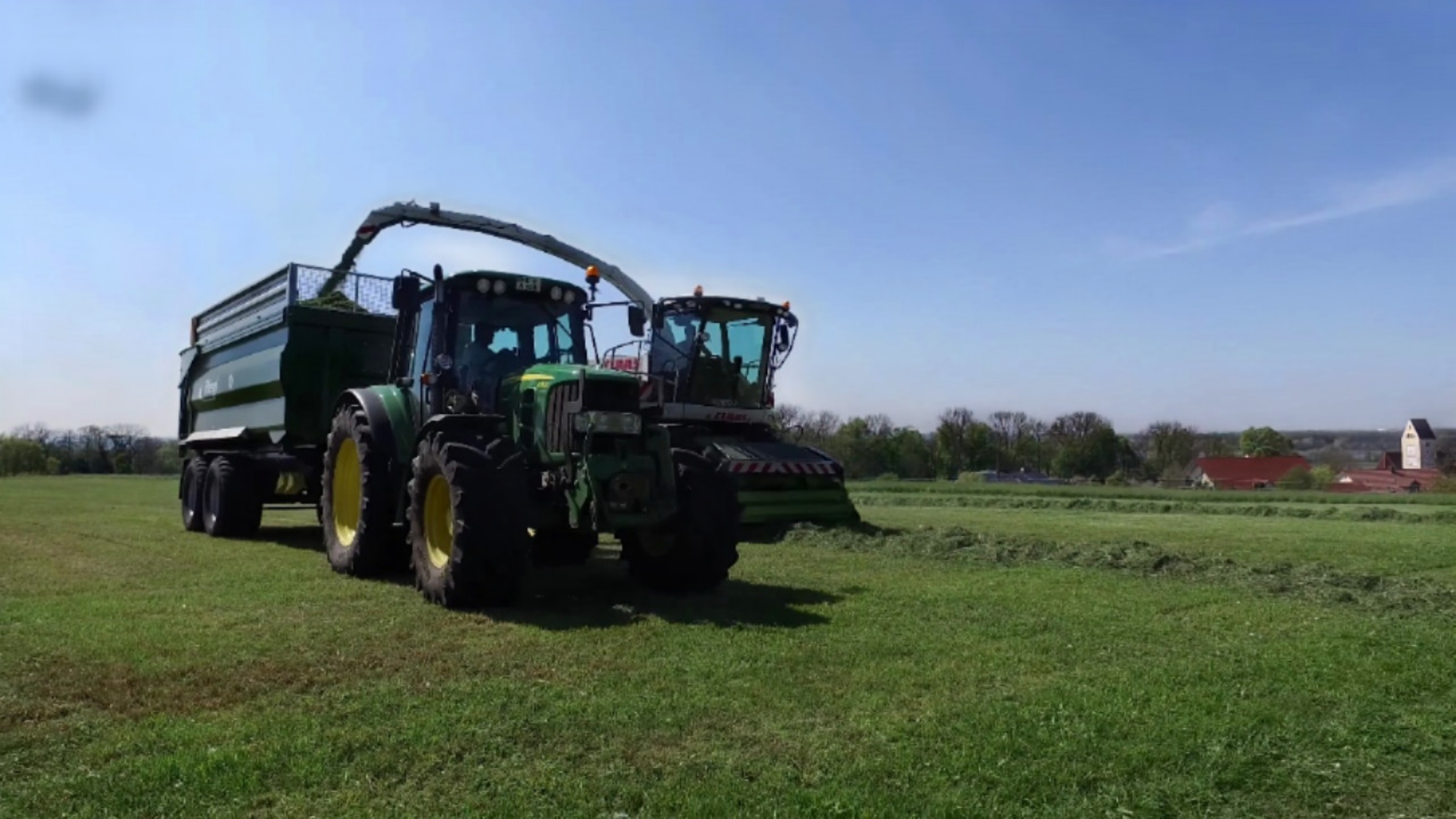} &
      \zoomC{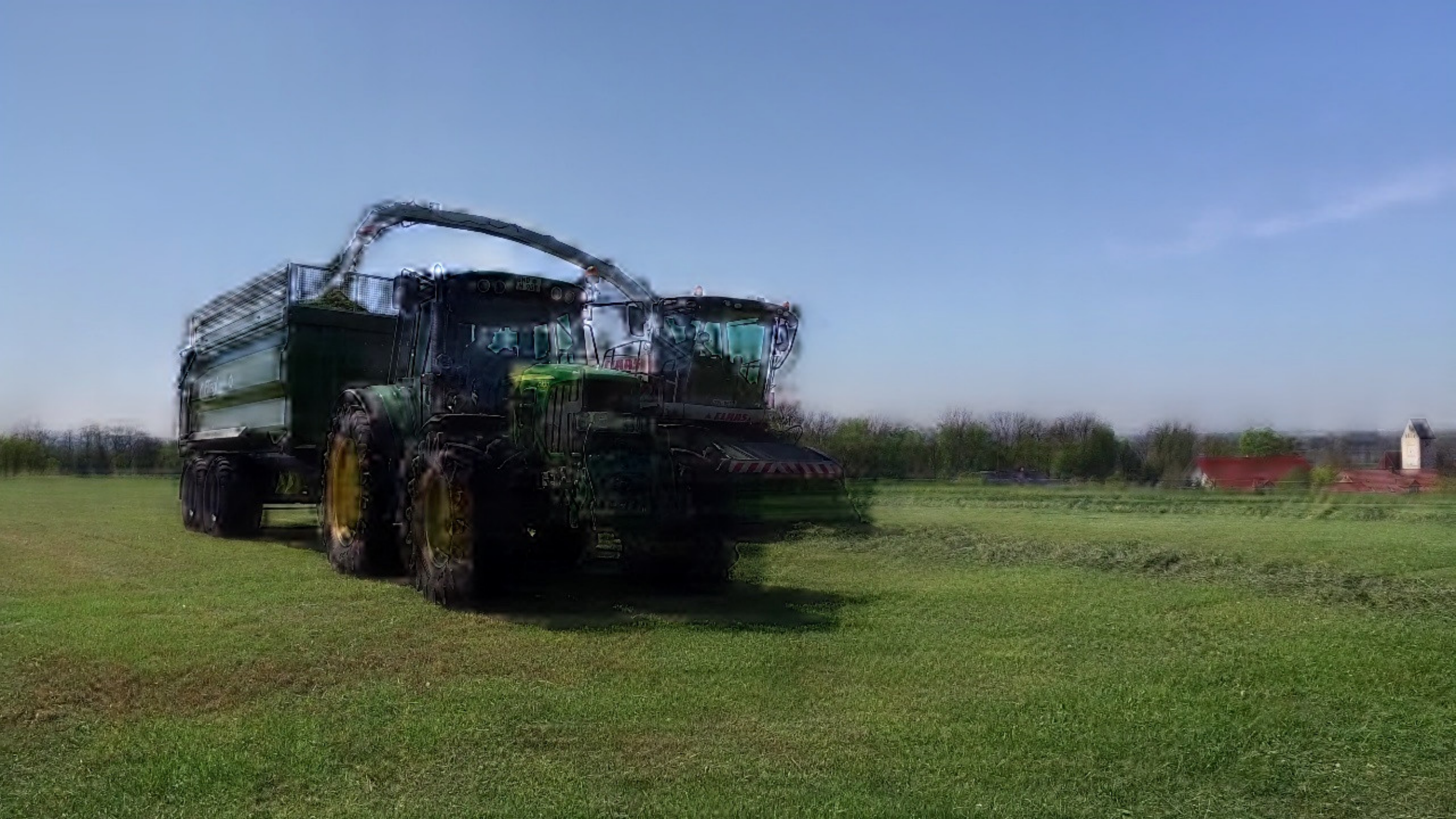} &
      \zoomC{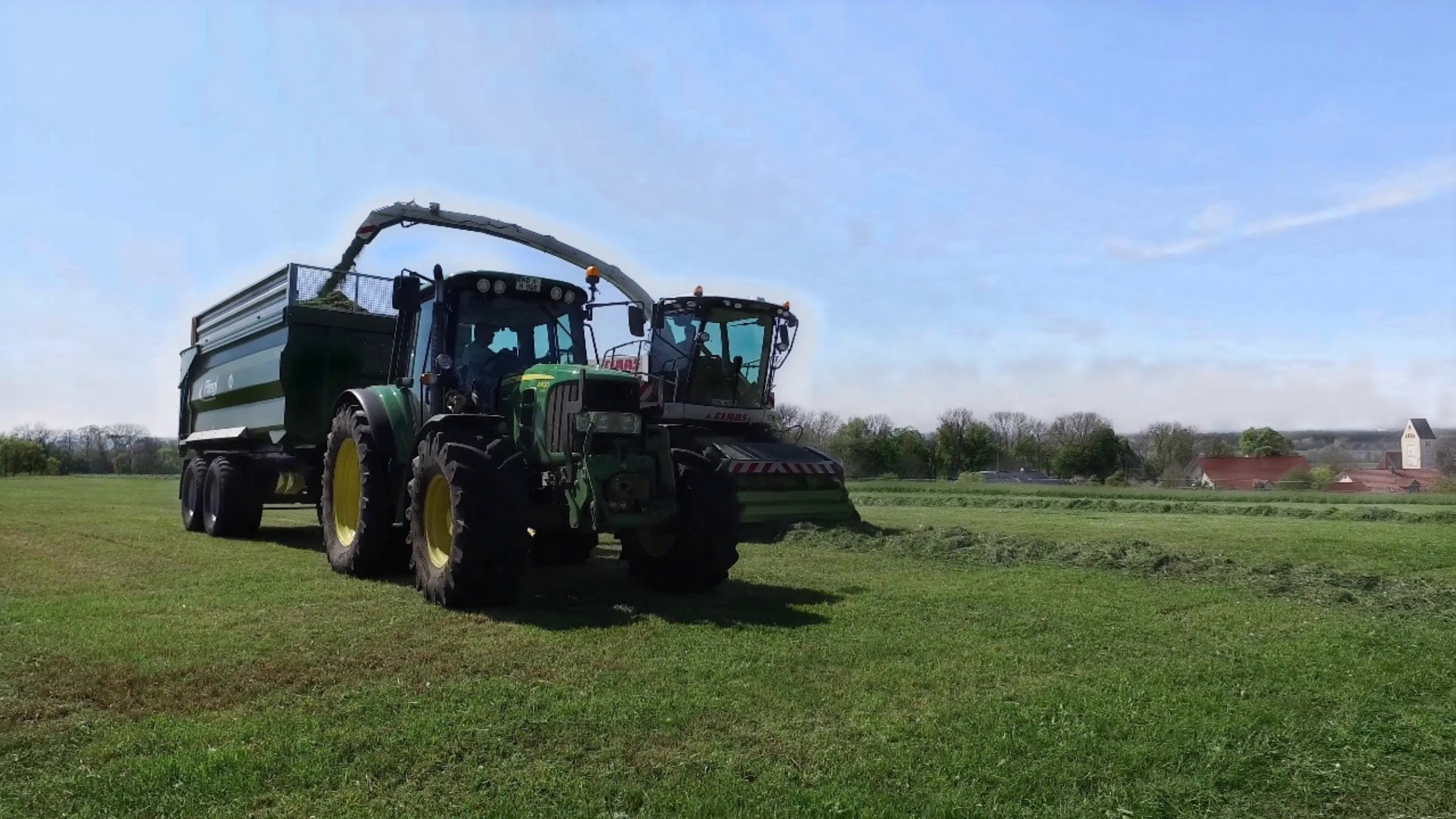} &
      \zoomC{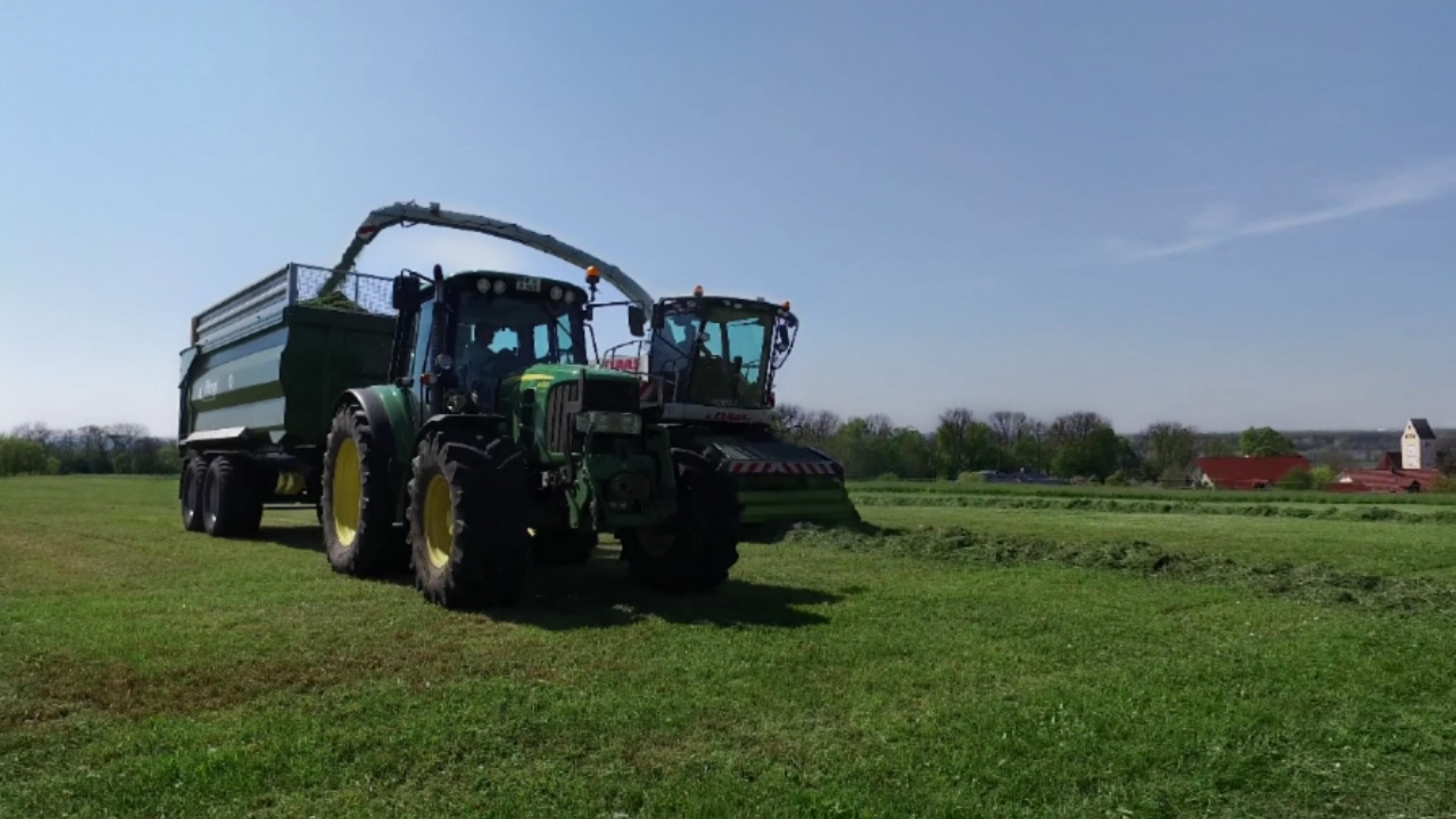} &
      \zoomC{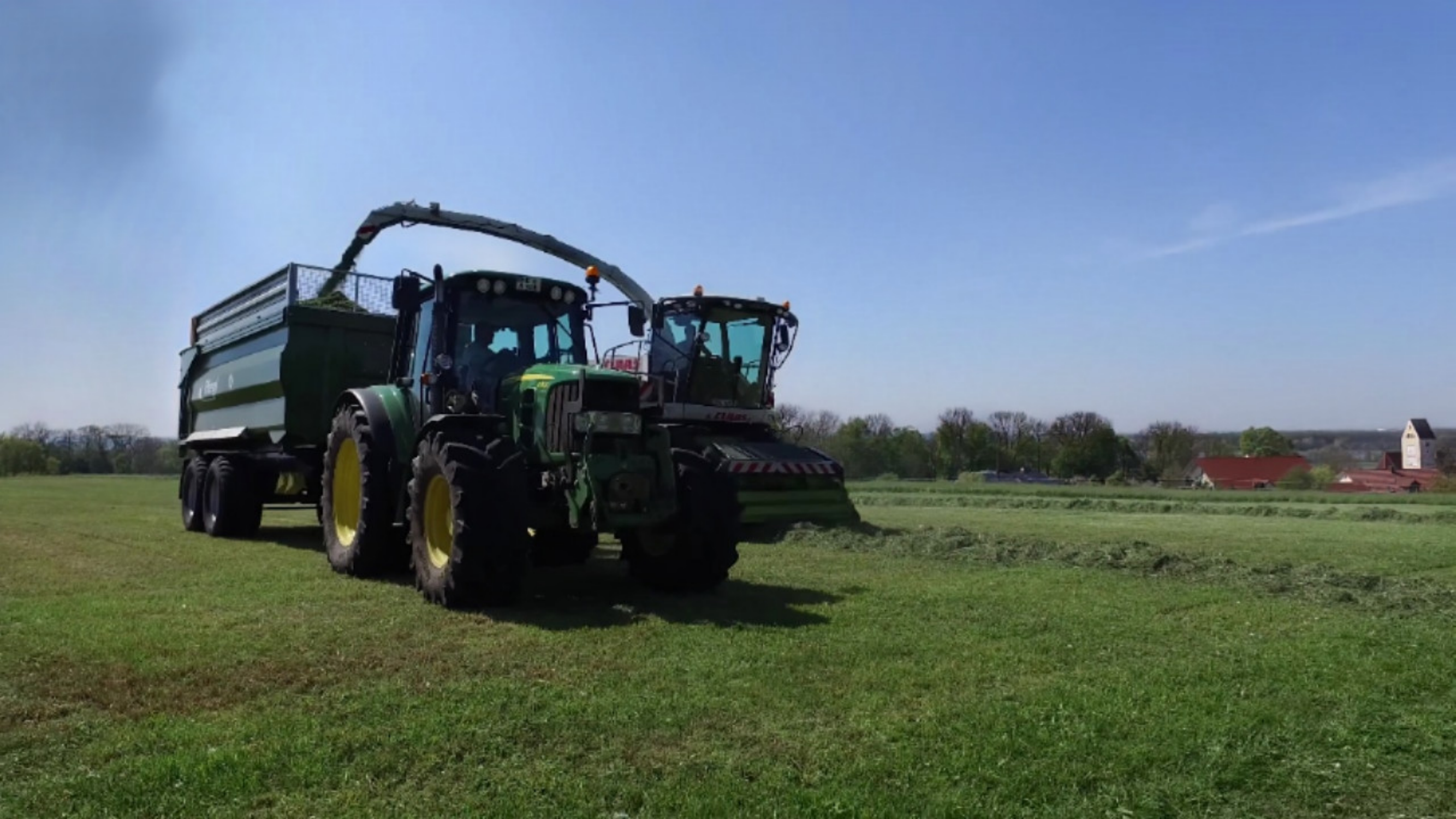} &
      \zoomC{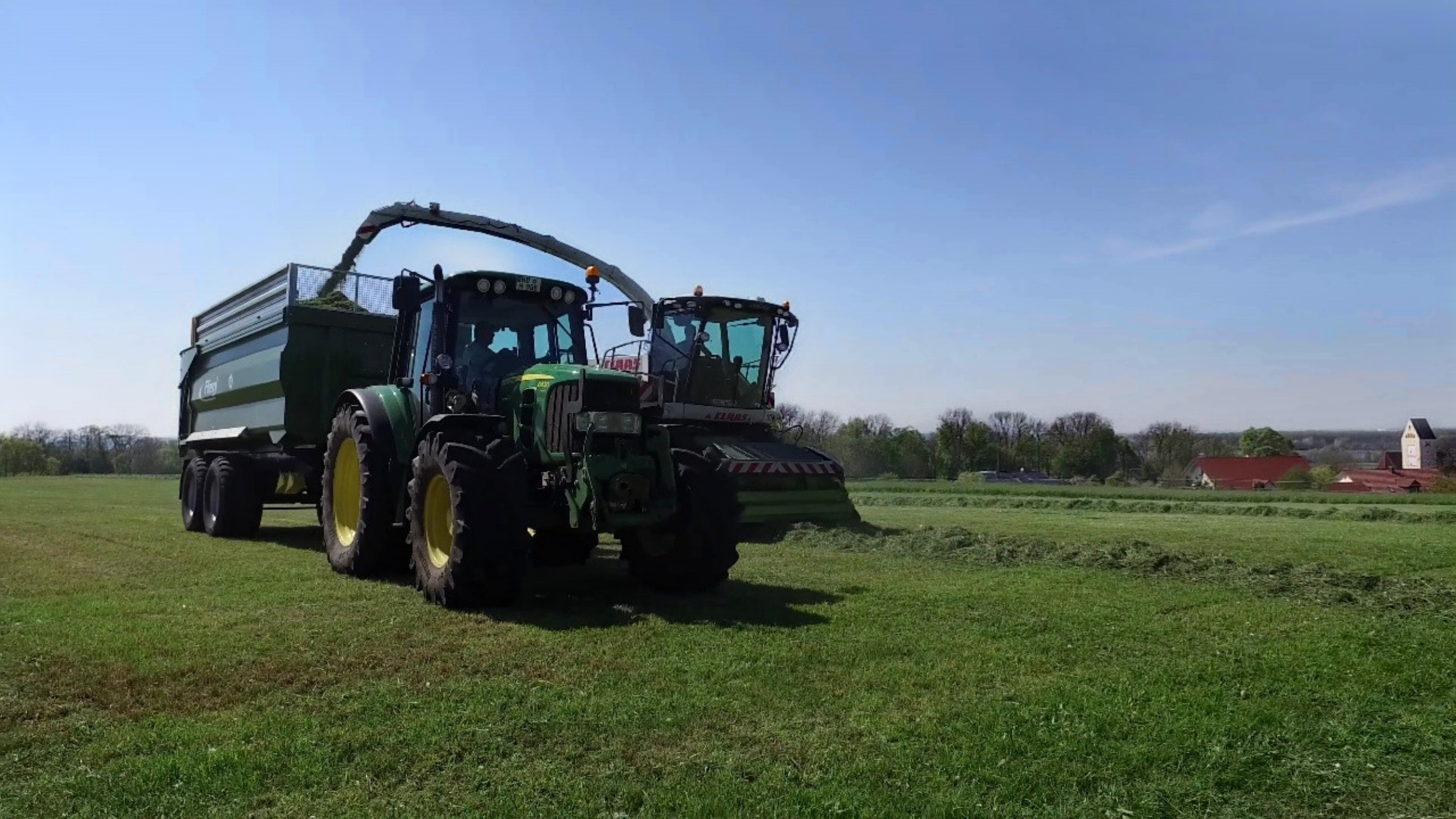} &
      \zoomC{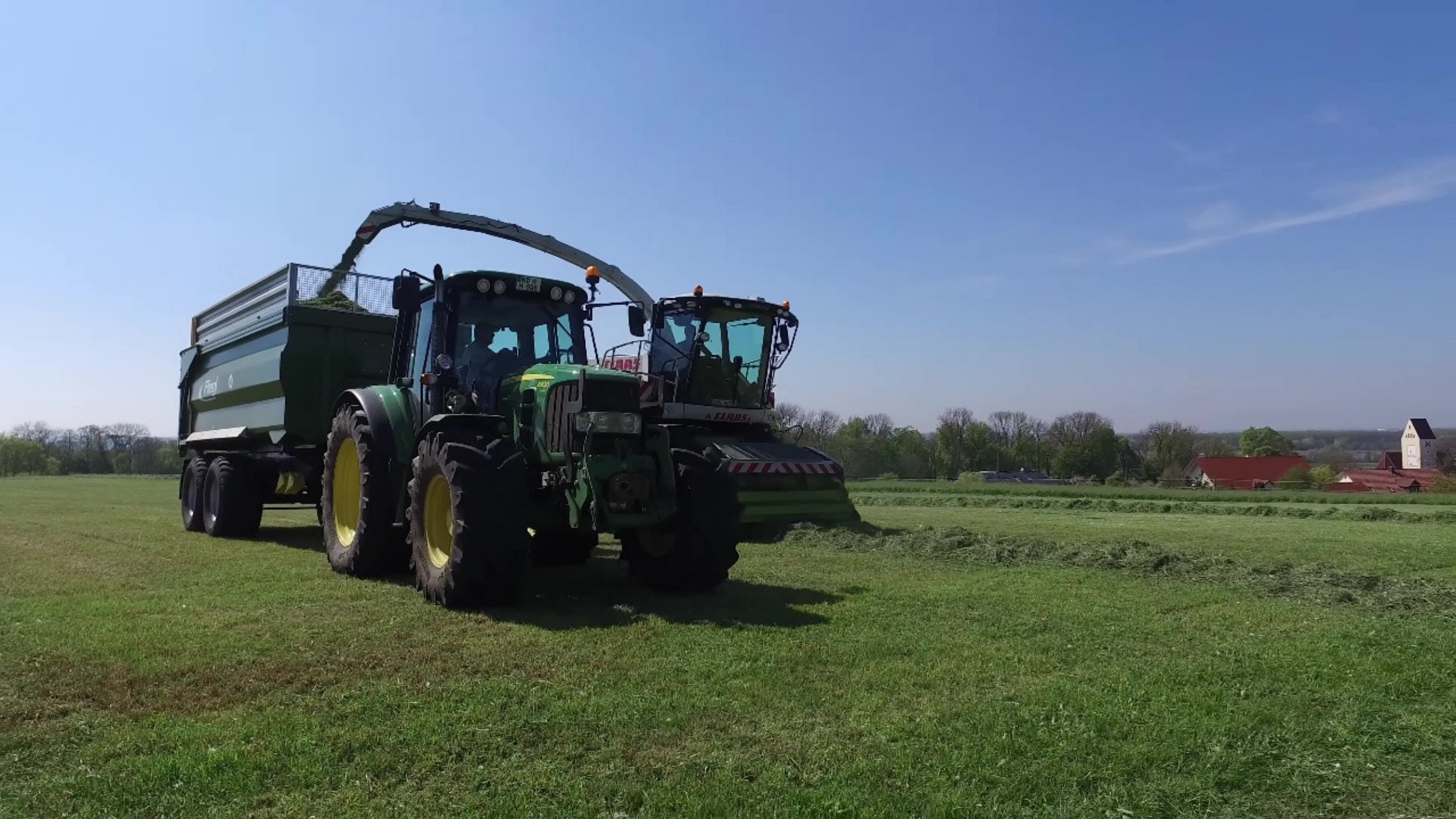}
      \\[14.2pt]
    \rotatebox[origin=c]{0}{\scriptsize} &
      \cimg{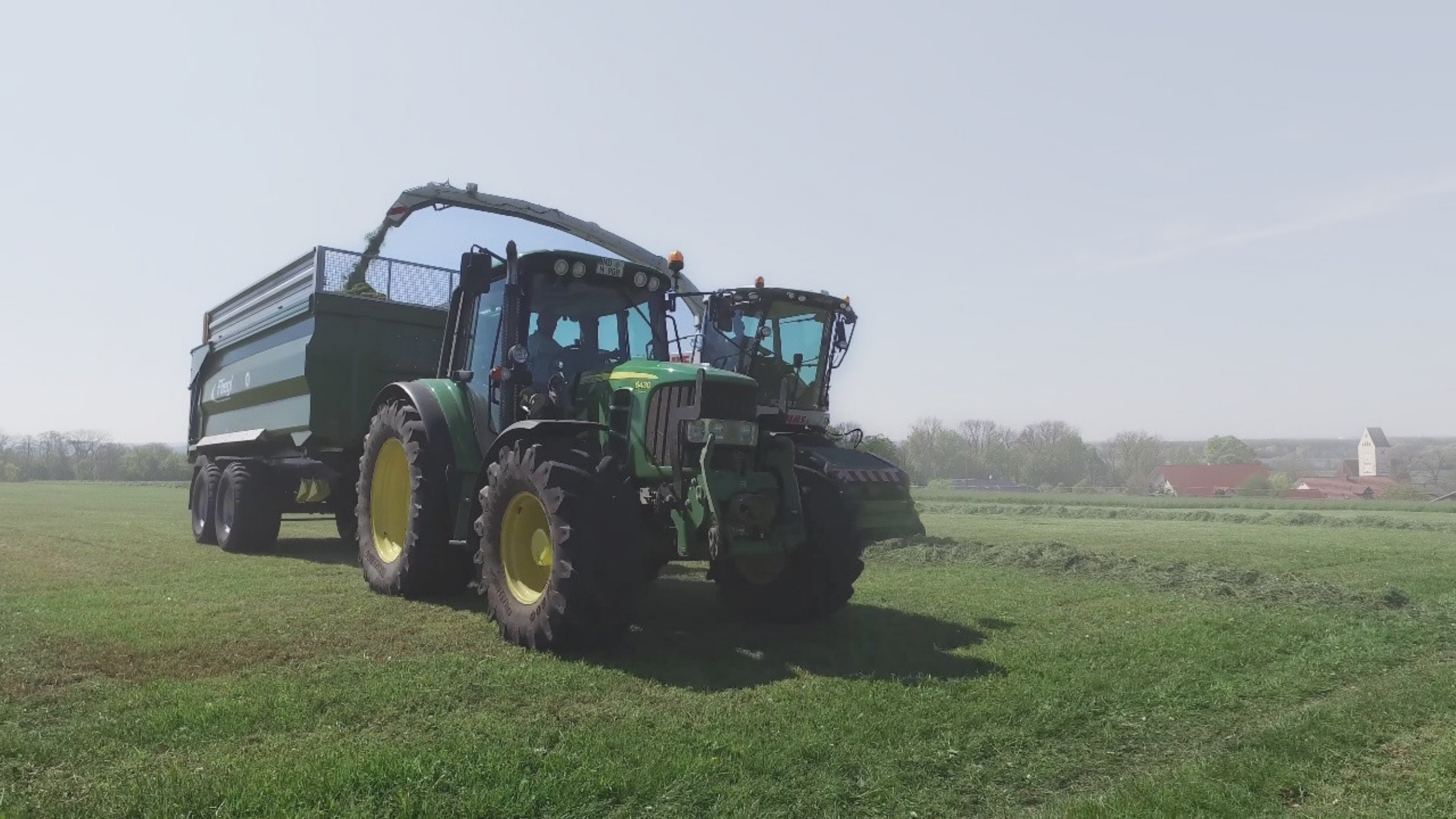} &
      \cimg{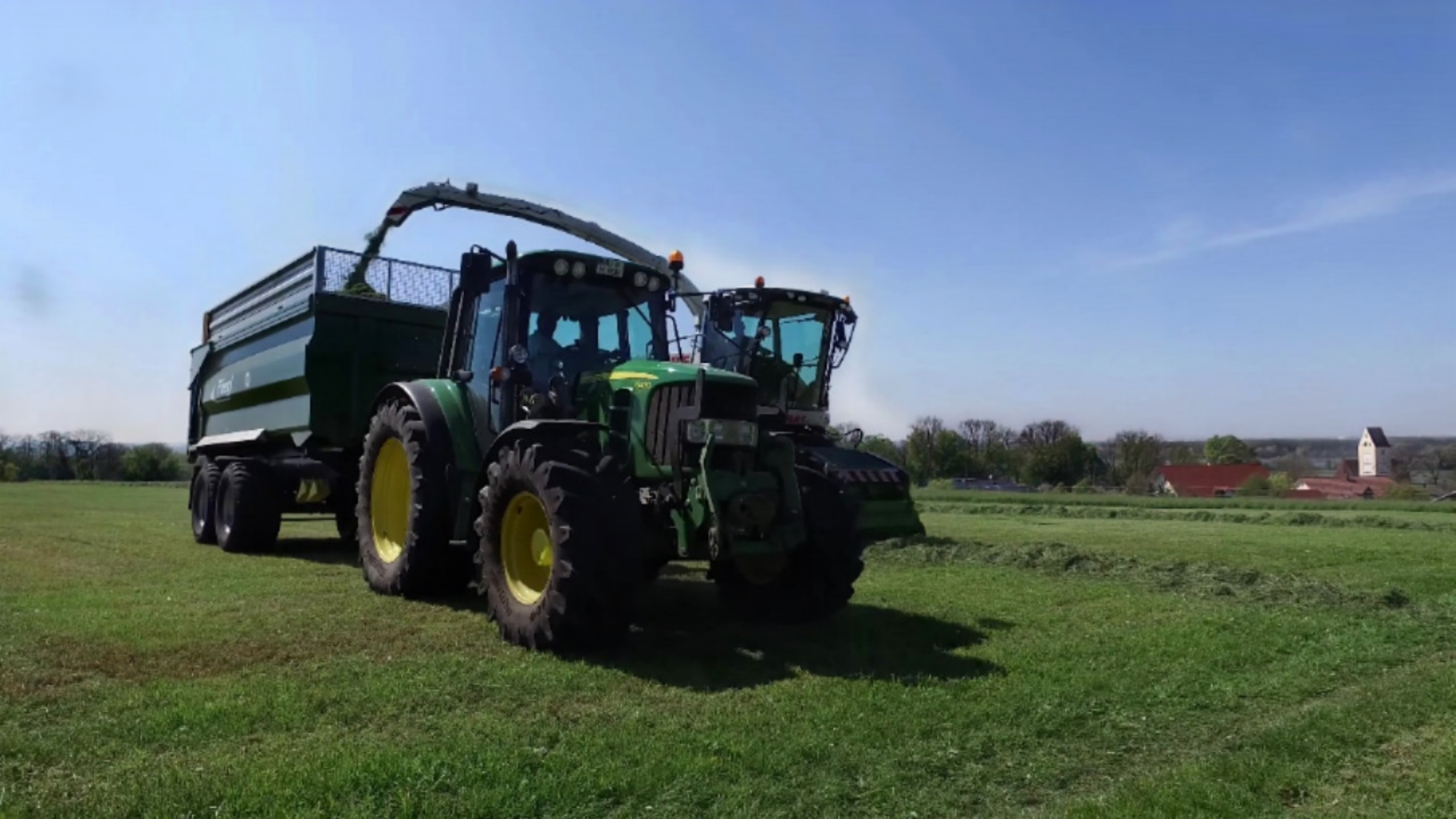} &
      \cimg{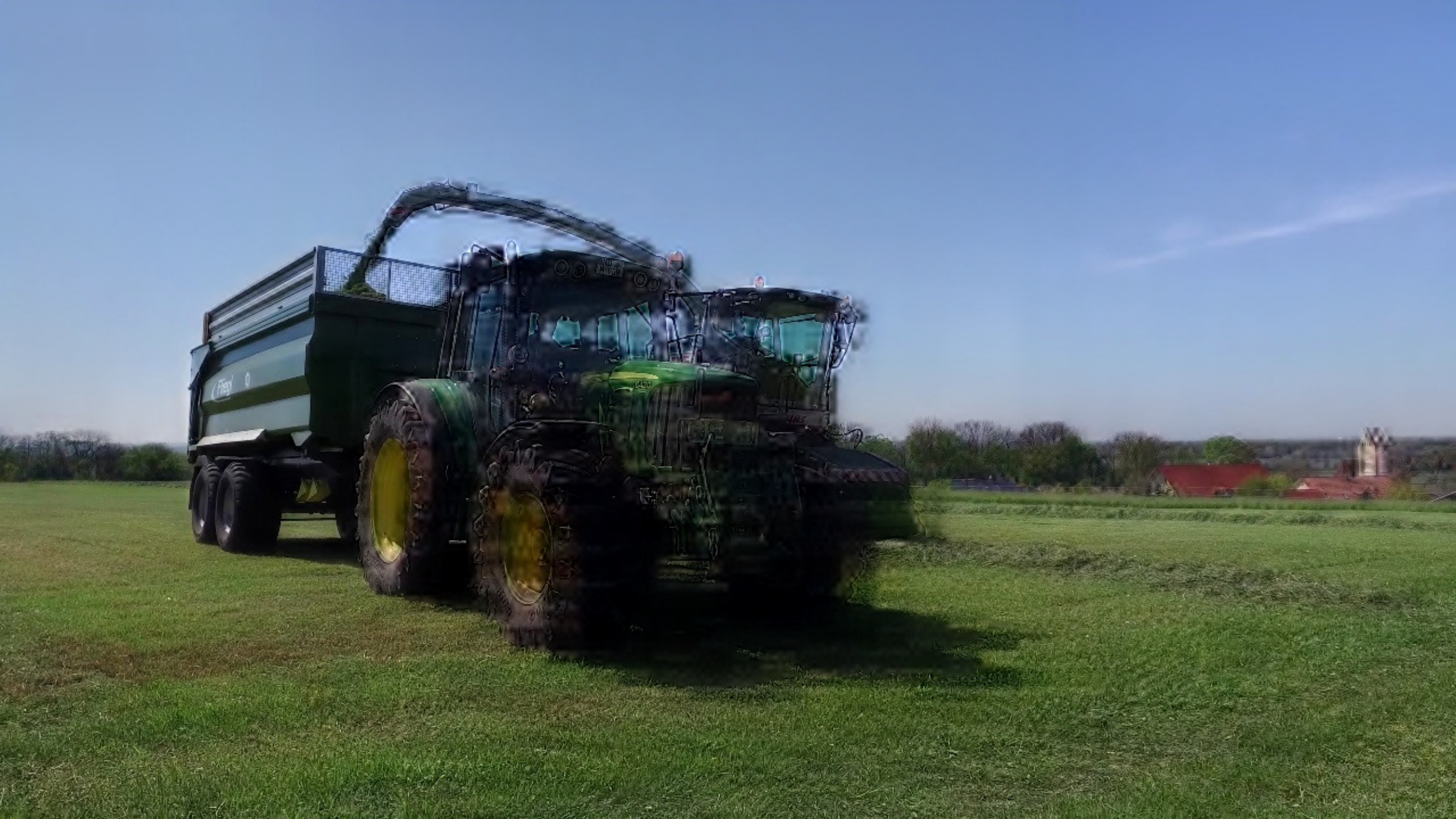} &
      \cimg{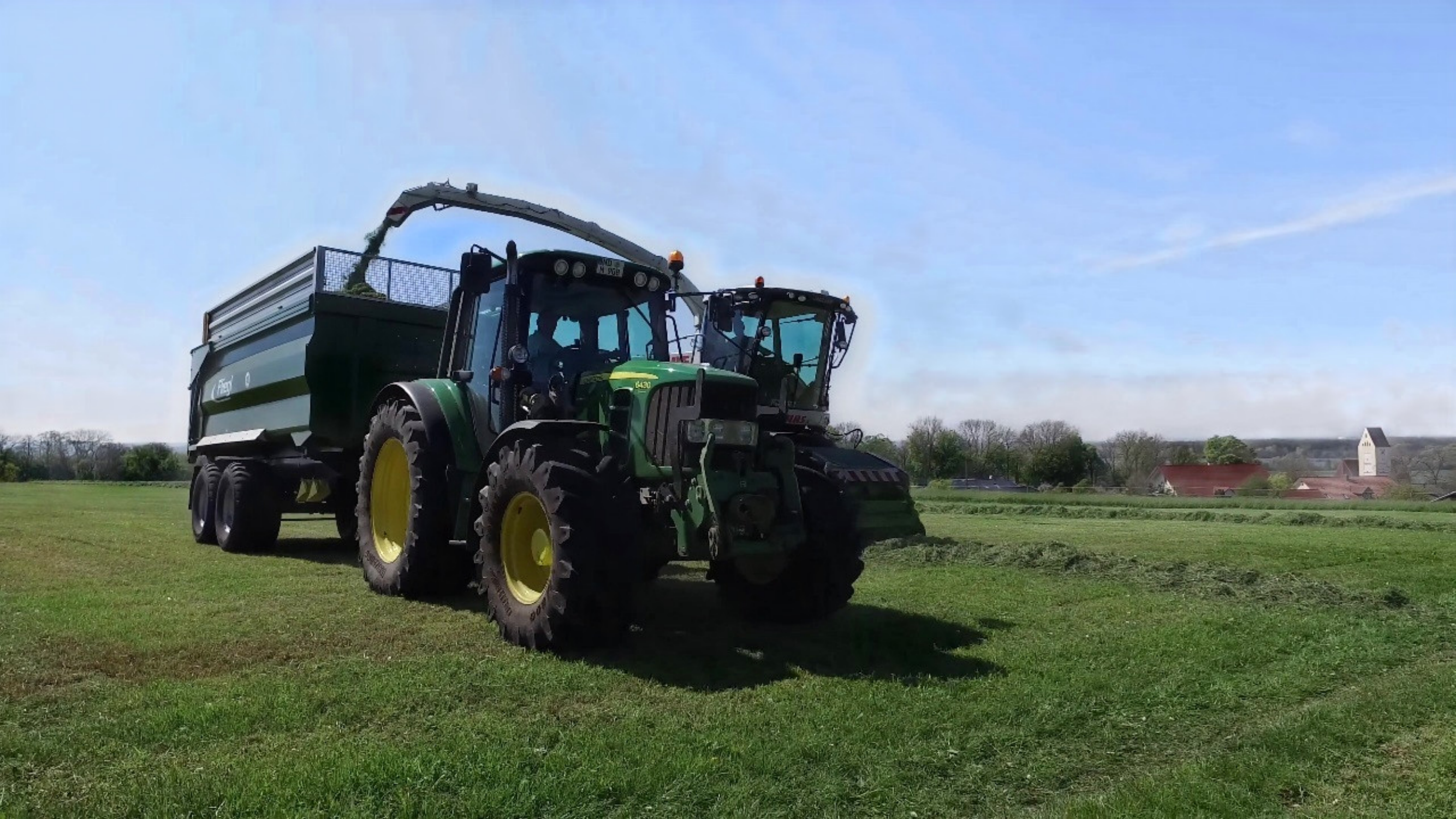} &
      \cimg{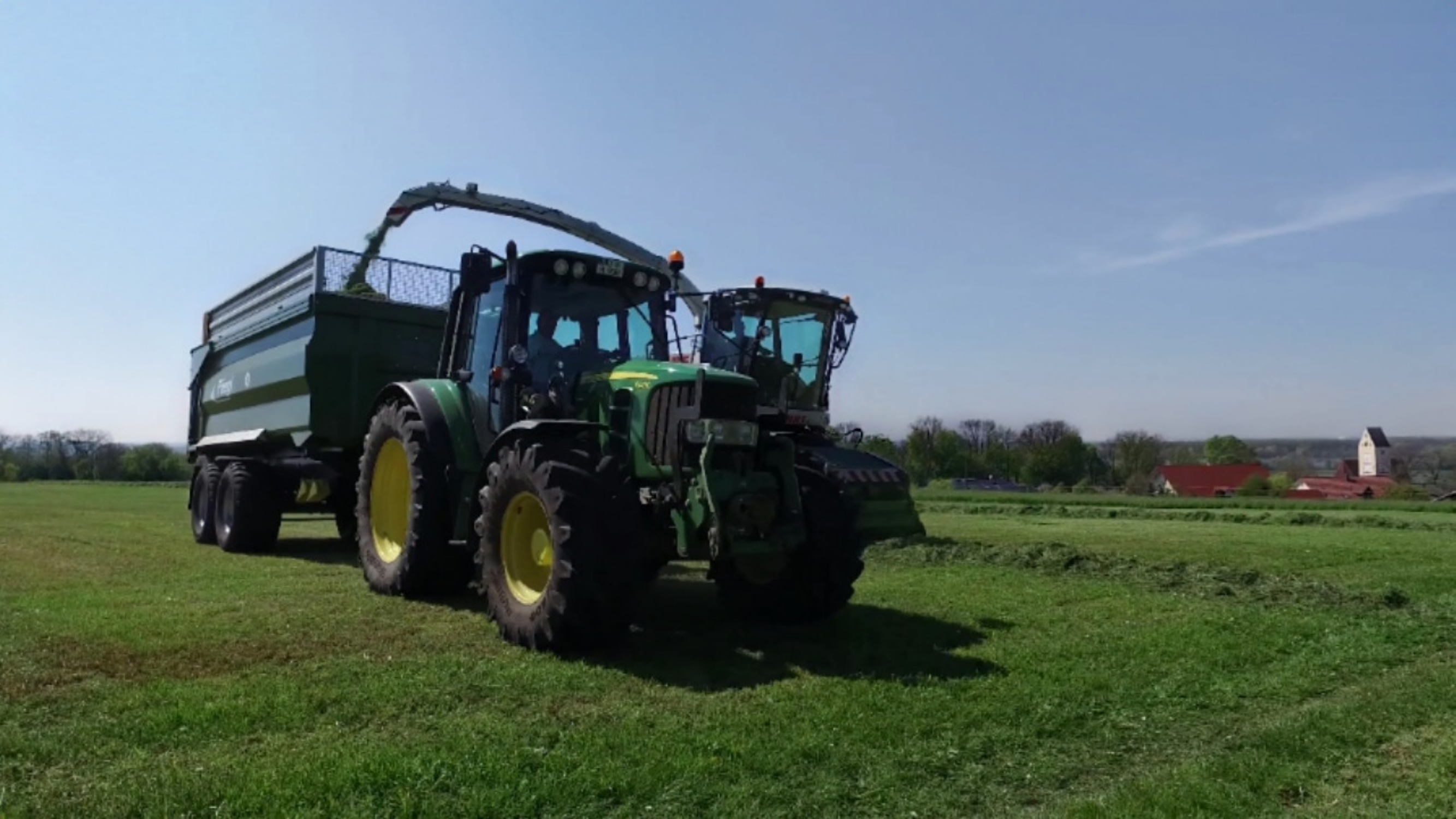} &
      \cimg{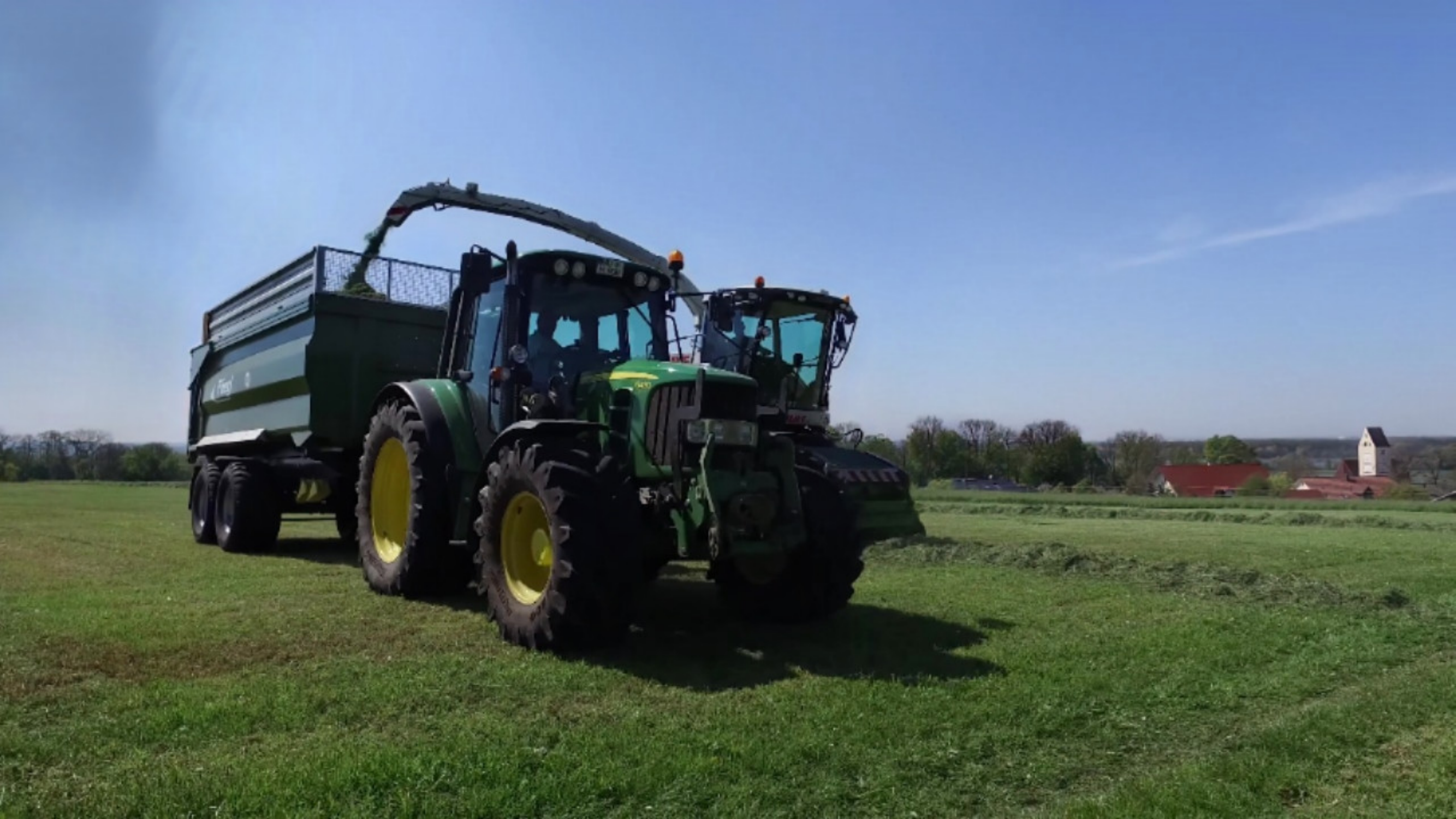} &
      \cimg{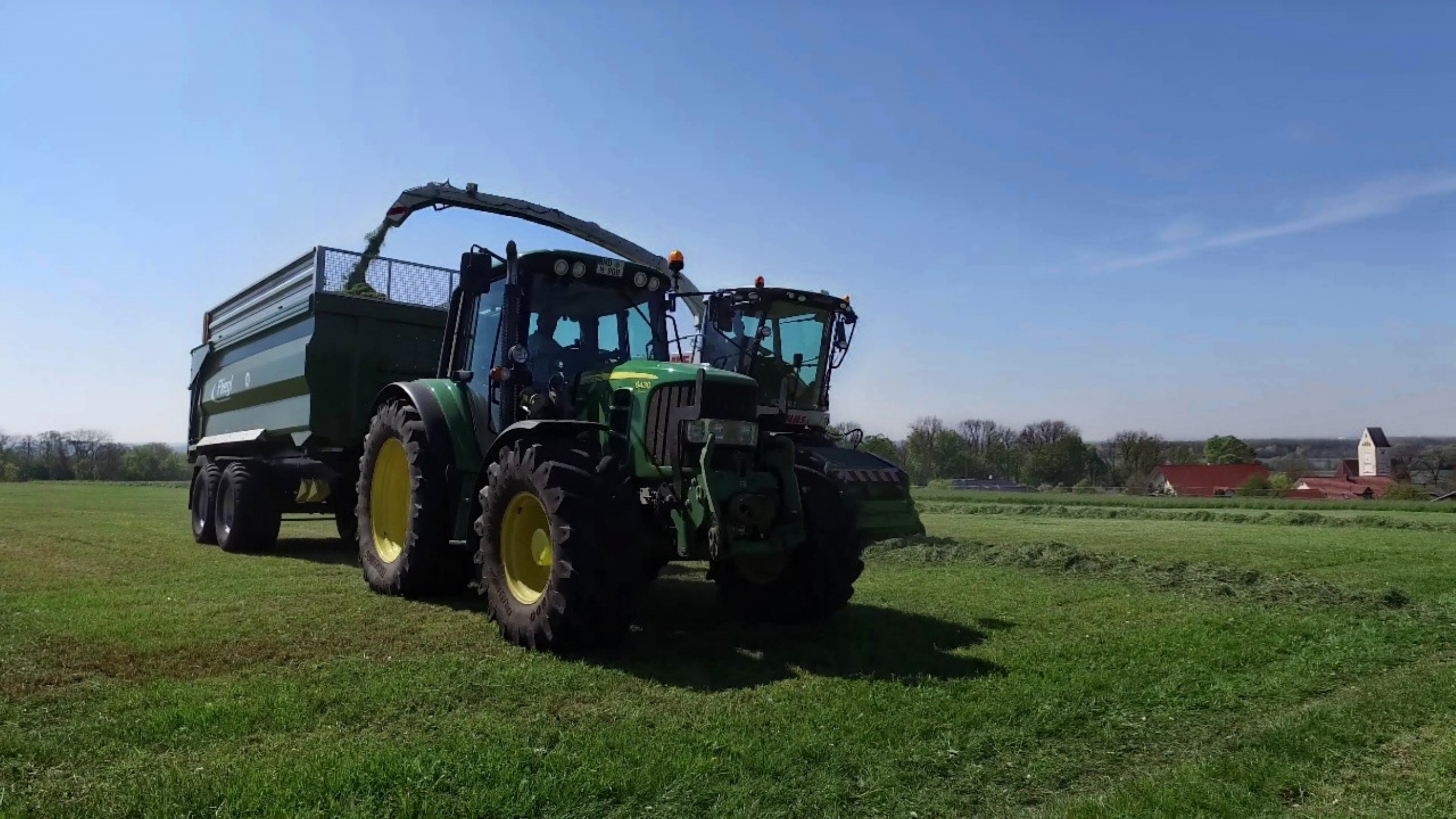} &
      \cimg{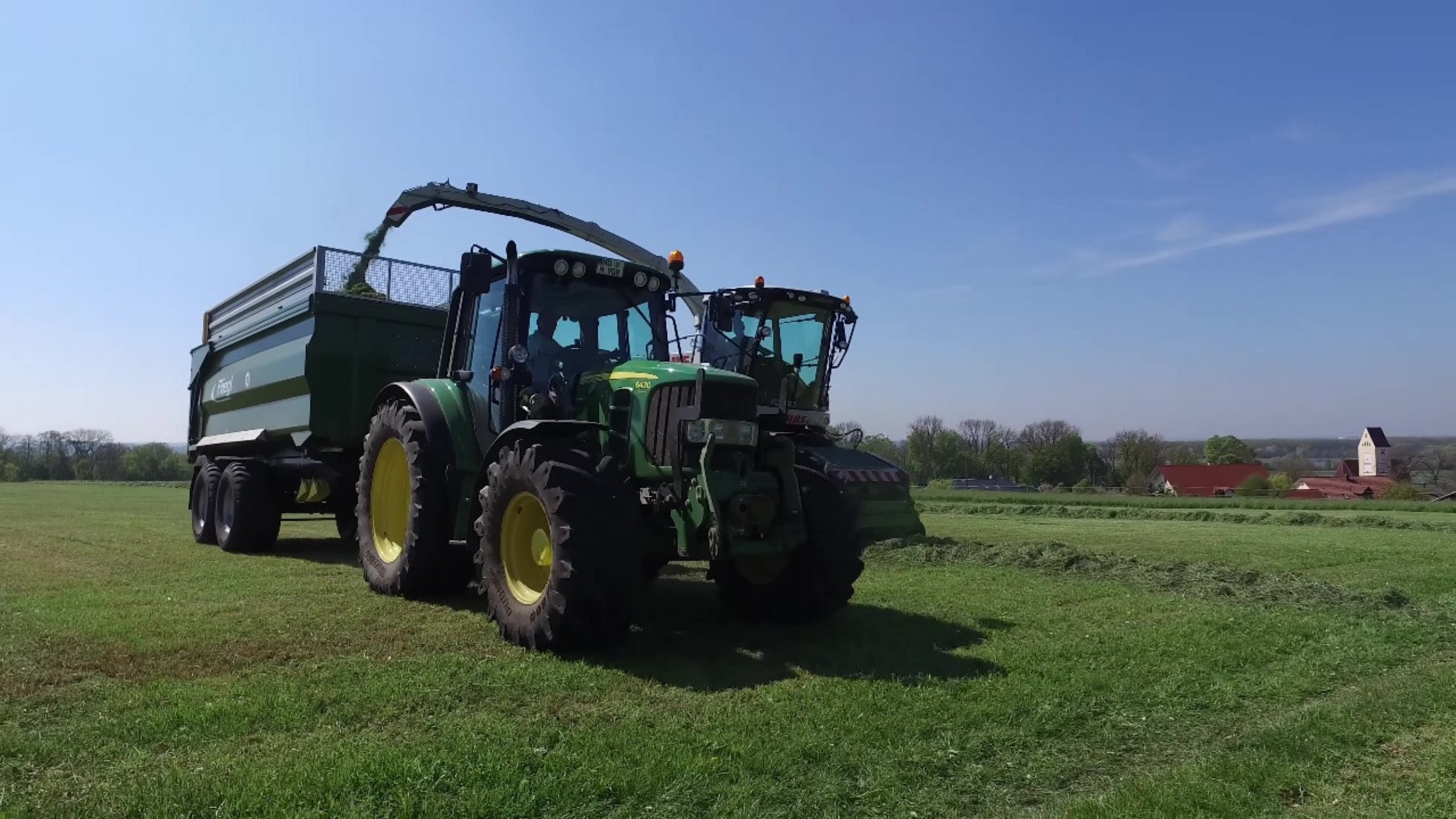}
    \\
  \end{tabular}
  \caption{%
    \textbf{Qualitative comparison with video dehazing methods.}
    Rows: (a)~UHV-4K ($3840{\times}2160$), (b)~REVIDE ($2708{\times}1800$), (c)~HazeWorld (${\leq}720$p); each row pair shows two frames from the same video. In the first frame of each pair, the red box marks the region enlarged in the inset.%
  }
  \label{fig:visual_comparison}
  \vspace{-3mm}
\end{figure*}

\vspace{-6pt}
\section{Experiments}
\vspace{-6pt}

\subsection{Experimental Setup}
\label{sec:exp_settings}
\vspace{-4pt}
\paragraph{Datasets and evaluation metrics.}
We evaluate on three paired video dehazing benchmarks. REVIDE~\citep{CG-IDN_REVIDE_dataset} is a real-world indoor dataset with 42 training and 6 testing videos in its official split. From the synthetic HazeWorld corpus~\citep{MAP-Net_HazeWorld_dataset} we retain three outdoor subsets---DAVIS~\citep{DAVIS_dataset}, DDAD~\citep{DDAD_dataset}, and UA-DETRAC~\citep{UA-DETRAC_dataset}---totaling 1{,}016 training and 584 testing videos. UHV-4K, introduced in \S\ref{sec:dataset}, follows its standard split. Restoration quality is measured by PSNR~\citep{PSNR}, SSIM~\citep{SSIM}, and LPIPS~\citep{LPIPS}. Efficiency is reported as parameter count, GFLOPs, and FPS. Temporal consistency is evaluated with tOF~\citep{lai2018learning} and within-video $\sigma$-PSNR; all tOF scores use the RAFT~\citep{teed2020raft} estimator.

\vspace{-0.5em}
\paragraph{Implementation details.}
All models are trained on $8{\times}$ V100 GPUs with AdamW, cosine annealing after a 5-epoch warmup, and FP16 mixed precision on random $512{\times}512$ crops of $T{=}5$ consecutive frames. At inference, 4K videos are processed at native resolution without patching. All sixteen baselines share the same optimizer and schedule; baselines are trained without $\mathcal{L}_{\mathrm{lie}}$, since its penalties act on bilateral-grid tensors exclusive to our architecture. Full hyperparameters are in Appendix~Table~\ref{tab:hyperparams}.

\vspace{-4pt}
\subsection{Comparison with State-of-the-Art Methods}
\label{sec:sota}
\vspace{-4pt}
\paragraph{Compared methods.}
We compare LiBrA-Net against sixteen published methods: eight single-image dehazers, of which 4KDehazing, UHDformer, UHDDIP, and DehazeXL are UHD-specialized, and eight video dehazers. Among all baselines, only 4KDehazing also adopts bilateral grids and is thus our closest architectural reference.

\vspace{-0.5em}
\paragraph{Dehaze quality.}
Table~\ref{tab:overall_comparison} reports results across three benchmarks. On UHV-4K, LiBrA-Net leads all sixteen compared methods in PSNR, SSIM, and LPIPS. On REVIDE, it ranks first among video dehazers and matches MSBDN-DFF overall. Both datasets present the scenario the bilateral grid is built for: high output resolution where resolution-decoupled prediction pays off.
HazeWorld caps at 720p, where the resolution-decoupling advantage is muted---dense feature maps fit in memory. LiBrA-Net remains competitive, ranking second in PSNR among video methods and first in LPIPS.
Figure~\ref{fig:visual_comparison} confirms that LiBrA-Net recovers faithful color and sharp edges across all three benchmarks.

\vspace{-0.5em}
\paragraph{Efficiency.}
At native 4K, LiBrA-Net reaches 25\,FPS---an $8{\times}$ speedup over 4KDehazing at $36{\times}$ fewer GFLOPs. No other method simultaneously exceeds 23\,dB PSNR on UHV-4K and sustains real-time 4K throughput. From 720p to 4K the pixel count grows $9{\times}$, yet our FPS drops only $3.2{\times}$ (dense baselines drop ${\ge}7.6{\times}$), because grid prediction is bounded by the fixed-resolution encoder and only the trilinear slice and the HF-Refiner scale with the output (Figure~\ref{fig:teaser}c).

\vspace{-4pt}
\subsection{Ablation Studies}
\label{sec:ablation}
\vspace{-4pt}

\begin{table*}[t]
  \centering
  \renewcommand{\arraystretch}{0.9}
  \small
  \caption{Ablation study on UHV-4K and REVIDE. (a)~Architectural ablation: one component removed at a time. (b)~Supervision ablation: one loss term removed or substituted.}
  \label{tab:ablation}
  \resizebox{\textwidth}{!}{%
  \begin{tabular}{l|cccc|cccc}
  \toprule
  \multirow{2}{*}{Variant} & \multicolumn{4}{c|}{\textbf{UHV-4K} ($3840{\times}2160$)} & \multicolumn{4}{c}{\textbf{REVIDE} ($2708{\times}1800$ indoor)} \\
  \cmidrule(lr){2-5} \cmidrule(lr){6-9}
   & PSNR\,$\uparrow$ & SSIM\,$\uparrow$ & LPIPS\,$\downarrow$ & $\Delta$PSNR & PSNR\,$\uparrow$ & SSIM\,$\uparrow$ & LPIPS\,$\downarrow$ & $\Delta$PSNR \\
  \midrule
  \rowcolor{catband} \multicolumn{9}{c}{\textit{(a) Architectural ablation}} \\
  \midrule
  \rowcolor{ourgreen} Full model              & \textbf{24.28} & \textbf{0.9437} & \textbf{0.0401} & --                                   & \textbf{21.43} & \textbf{0.8678} & \textbf{0.2397} & -- \\
  w/o Chromatic Field                       & 22.02 & 0.9074 & 0.0792 & \textcolor{red!70!black}{$-2.26$}              & 20.43 & 0.8594 & 0.2639 & \textcolor{red!70!black}{$-1.00$} \\
  w/o Temporal Field                        & 22.09 & 0.9111 & 0.0786 & \textcolor{red!70!black}{$-2.19$}              & 20.42 & 0.8579 & 0.2617 & \textcolor{red!70!black}{$-1.01$} \\
  w/o HF-Refiner                             & 20.87 & 0.8119 & 0.2456 & \textcolor{red!70!black}{$-3.41$}              & 18.84 & 0.8147 & 0.4182 & \textcolor{red!70!black}{$-2.59$} \\
  w/o Cayley map                             & 21.69 & 0.9035 & 0.0823 & \textcolor{red!70!black}{$-2.59$}              & 20.81 & 0.8623 & 0.2567 & \textcolor{red!70!black}{$-0.62$} \\
  w/o Res-Adaptive Slicing                   & 21.01 & 0.8708 & 0.1264 & \textcolor{red!70!black}{$-3.27$}              & 20.61 & 0.8557 & 0.2695 & \textcolor{red!70!black}{$-0.82$} \\
  \midrule
  \rowcolor{catband} \multicolumn{9}{c}{\textit{(b) Supervision ablation}} \\
  \midrule
  \rowcolor{ourgreen} Full loss               & \textbf{24.28} & \textbf{0.9437} & \textbf{0.0401} & --                                   & \textbf{21.43} & \textbf{0.8678} & \textbf{0.2397} & -- \\
  w/o $\mathcal{L}_{\mathrm{perc}}$           & 21.66 & 0.9075 & 0.0838 & \textcolor{red!70!black}{$-2.62$}              & 20.85 & 0.8596 & 0.2765 & \textcolor{red!70!black}{$-0.58$} \\
  w/o $\mathcal{L}_{\mathrm{lie}}$            & 21.87 & 0.9117 & 0.0748 & \textcolor{red!70!black}{$-2.41$}              & 20.58 & 0.8588 & 0.2600 & \textcolor{red!70!black}{$-0.85$} \\
  $\mathcal{L}_{\mathrm{pixel}}$ only         & 22.12 & 0.9119 & 0.0805 & \textcolor{red!70!black}{$-2.16$}              & 20.57 & 0.8574 & 0.2811 & \textcolor{red!70!black}{$-0.86$} \\
  VGG-19 $\mathcal{L}_{\mathrm{perc}}$        & 22.73 & 0.9228 & 0.0660 & \textcolor{red!70!black}{$-1.55$}              & 20.78 & 0.8631 & 0.2674 & \textcolor{red!70!black}{$-0.65$} \\
  \bottomrule
  \end{tabular}%
  }
  \vspace{-4mm}
  \end{table*}

\vspace{-0.0em}
\paragraph{Architectural ablation.}
Table~\ref{tab:ablation}(a) ablates each component on both UHV-4K and REVIDE. Two patterns stand out. The Chromatic and Temporal Fields contribute nearly equally on both datasets, validating the additive factorization in which each branch carries roughly half of the transform; a capacity-matched control in Appendix~\ref{app:temporal_role} attributes the temporal gain to cross-frame attention. The remaining components---HF-Refiner, resolution-adaptive slicing, and Cayley map---all show amplified drops on UHV-4K relative to REVIDE, confirming that the resolution-decoupled design becomes increasingly critical as the output resolution grows.

\vspace{-0.5em}
\paragraph{Supervision ablation.}
Table~\ref{tab:ablation}(b) reveals that $\mathcal{L}_{\mathrm{perc}}$ and $\mathcal{L}_{\mathrm{lie}}$ are complementary: removing either alone costs more than dropping both together, indicating each constrains a distinct failure mode. Replacing DINOv2 with VGG-19 at matched weight also degrades quality, confirming the benefit of self-supervised perceptual features.

\vspace{-4pt}
\subsection{Temporal Consistency}
\label{sec:analysis}
\vspace{-4pt}

On UHV-4K, LiBrA-Net keeps tOF below every image dehazer while attaining the highest PSNR (Table~\ref{tab:temporal}; Appendix~\ref{app:temporal}). The few video methods with lower tOF are substantially worse in PSNR, indicating steadiness from over-smoothing rather than faithful recovery. Ablations show that the Temporal Field stabilizes per-frame quality, while the Lie temporal term regularizes grid trajectories; together they improve consistency without an explicit flow-warped loss.

\vspace{-4pt}
\subsection{External Validity}
\label{sec:external_validity}
\vspace{-4pt}

\paragraph{Downstream perception.}
We feed each method's 4K output into frozen detection and segmentation models. LiBrA-Net ranks first on 8 of 10 no-reference metrics, with detection confidence approaching the clean-frame oracle and segmentation boundaries closest to the ground truth (Appendix~\ref{app:external}, Table~\ref{tab:downstream_nr}).

\vspace{-0.5em}
\paragraph{Real-world generalization.}
On eight hand-collected 4K hazy videos without ground truth, LiBrA-Net is the only video method that improves all three no-reference image quality assessment (NR-IQA) scores over raw input; competing methods degrade when atmospheric statistics deviate from training. The bilateral grid's low-frequency transform---governed by smooth atmospheric structure rather than scene-specific texture---facilitates transfer from synthetic to real outdoor haze; qualitative results are in Appendix~\ref{app:external}.

\Needspace{19\baselineskip}
\vspace{-6pt}
\section{Why Does the Bilateral Grid Work at 4K?}
\label{sec:analysis_chapter}
\vspace{-4pt}

\begin{wrapfigure}{r}{0.56\linewidth}
  \vspace{-14pt}
  \centering
  \begin{subfigure}[b]{0.5\linewidth}
    \centering
    \includegraphics[width=\linewidth]{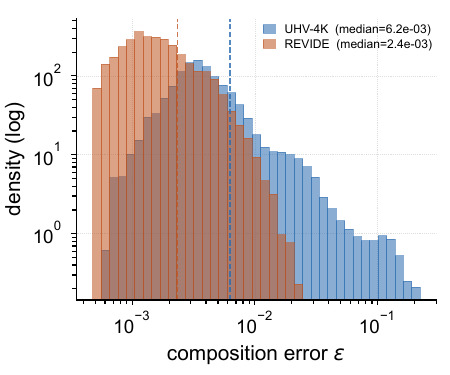}
    \caption{}
    \label{fig:comp_hist}
  \end{subfigure}\hfill
  \begin{subfigure}[b]{0.5\linewidth}
    \centering
    \includegraphics[width=\linewidth]{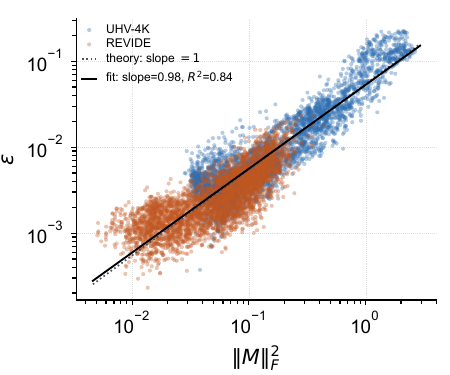}
    \caption{}
    \label{fig:comp_fit}
  \end{subfigure}
  \caption{\textbf{Lie-algebraic composition at the trained operating point.}
    (a)~Distribution of relative composition error $\epsilon$ on UHV-4K and REVIDE.
    (b)~$\epsilon$ vs.\ $\|M\|_F^2$ on log--log axes; dashed line shows slope ${\approx}1$.}
  \label{fig:grid_geometry}
  \vspace{-10pt}
\end{wrapfigure}

\label{sec:analysis_grid_geometry}%
We probe the trained grid to understand \emph{what} it represents beyond which components matter. Three findings emerge. First, the spatial--color grid exploits the full chromatic axis across all eight color bins rather than collapsing to a depth-only template; t-SNE visualization confirms that grid cells cluster by color bin, indicating a structured affine descriptor whose Chromatic Field alone accounts for ${\sim}$2.3\,dB (Appendix~Figure~\ref{fig:grid_anatomy_appendix}). Second, the Lie-algebraic composition error remains on the order of $10^{-3}$ with a log--log slope of $0.98$ against $\|M\|_F^2$ (Figure~\ref{fig:grid_geometry}), validating Proposition~\ref{prop:lie_composition} at the trained operating point. Third, the grid is near-invariant under resolution change: Pearson correlation between grids predicted from downsampled and native-resolution inputs remains ${\geq}0.995$ across a ${\geq}10{\times}$ range (Appendix~Table~\ref{tab:resolution_invariance}), explaining why a fixed $256{\times}256$ encoder serves 360p to 4K at constant cost.

\Needspace{17\baselineskip}
\vspace{-6pt}
\section{Conclusion, Limitations and Future Works}
\vspace{-4pt}

We have presented LiBrA-Net, a bilateral-grid framework that decouples spatiotemporal modeling from output resolution for real-time 4K video dehazing. The method predicts affine color transforms on compact sub-grids at fixed resolution and applies them at native 4K through trilinear slicing, achieving 25\,FPS with 6.12\,M parameters---the only video dehazer exceeding 23\,dB PSNR on UHV-4K at real-time 4K throughput. We also release UHV-4K, the first paired 4K video dehazing benchmark. The resolution-decoupled design principle is not specific to dehazing and may benefit other UHD video tasks governed by spatially smooth degradation fields. Its advantage grows with output resolution: on 720p HazeWorld, where dense feature maps fit in memory, LiBrA-Net ranks second in PSNR among video methods. Our current temporal window covers roughly one second; modeling minute-scale coherence at 4K remains open. The scattering model assumes a single $\beta$ and $\mathbf{A}_\infty$ per video, so mixed-weather conditions fall outside UHV-4K's scope, and paired real 4K hazy/clean data does not yet exist.

\begin{ack}

This work was supported by the National Natural Science Foundation of China (No. 62302523) and the Natural Science Foundation of Hunan Province (No. 2026JJ50231). This work was also supported in part by the National Natural Science Foundation of China under Grant 62602324, and in part by the Shandong Provincial Natural Science Foundation under Grant ZR2026MS1003, ZR2025MS1025. This work was carried out in part using computing resources at the HPC Center of Central South University.

\end{ack}

\bibliographystyle{plainnat}
\bibliography{references}


\clearpage
\appendix

\section{Proof of Proposition~\ref{prop:lie_composition}}
\label{app:proof_lie}

We restate the result for convenience and then give a self-contained proof.

\medskip
\noindent\textbf{Proposition~\ref{prop:lie_composition}} (Lie-algebraic fusion approximates group composition).
\textit{For any $M_\chi, M_\tau \in \mathfrak{gl}(3)$ with $\rho := \max(\|M_\chi\|_F, \|M_\tau\|_F)$,}
\begin{equation*}
  \operatorname{Cay}(M_\chi + M_\tau) \;=\;
  \operatorname{Cay}(M_\chi)\,\operatorname{Cay}(M_\tau)
  \;+\; O(\rho^{2}).
\end{equation*}

\begin{proof}
Recall the Cayley map $\operatorname{Cay}(M) = (I - M/2)^{-1}(I + M/2)$.

\paragraph{Step 1: Neumann expansion of the Cayley map.}
Whenever $\|M\|/2 < 1$ (i.e.\ $\|M\| < 2$, satisfied throughout training by the identity prior), the resolvent $(I - M/2)^{-1}$ admits the absolutely convergent Neumann series
\begin{equation}
  (I - M/2)^{-1}
  \;=\; \sum_{k=0}^{\infty} (M/2)^{k}
  \;=\; I + \tfrac{1}{2}M + \tfrac{1}{4}M^{2} + O(\|M\|^{3}).
  \label{eq:neumann}
\end{equation}
Multiplying by $(I + M/2)$ and collecting powers of $M$ gives
\begin{align}
  \operatorname{Cay}(M)
  &= \bigl(I + \tfrac{1}{2}M + \tfrac{1}{4}M^{2} + O(\|M\|^{3})\bigr)
     \bigl(I + \tfrac{1}{2}M\bigr) \notag\\
  &= I + \tfrac{1}{2}M + \tfrac{1}{2}M + \tfrac{1}{4}M^{2} + \tfrac{1}{4}M^{2}
     + O(\|M\|^{3}) \notag\\
  &= I + M + \tfrac{1}{2}M^{2} + O(\|M\|^{3}).
  \label{eq:cay_expand}
\end{align}

\paragraph{Step 2: Expand the left-hand side.}
Substituting $M_\chi + M_\tau$ into Eq.~\eqref{eq:cay_expand}:
\begin{align}
  \operatorname{Cay}(M_\chi {+} M_\tau)
  &= I + (M_\chi + M_\tau) + \tfrac{1}{2}(M_\chi + M_\tau)^{2} + O(\rho^{3}) \notag\\
  &= I + M_\chi + M_\tau
     + \tfrac{1}{2}\bigl(M_\chi^{2} + M_\chi M_\tau + M_\tau M_\chi + M_\tau^{2}\bigr)
     + O(\rho^{3}).
  \label{eq:lhs}
\end{align}

\paragraph{Step 3: Expand the right-hand side.}
Applying Eq.~\eqref{eq:cay_expand} to each factor and multiplying:
\begin{align}
  &\operatorname{Cay}(M_\chi)\,\operatorname{Cay}(M_\tau) \notag\\
  &\quad= \bigl(I + M_\chi + \tfrac{1}{2}M_\chi^{2} + O(\rho^{3})\bigr)
          \bigl(I + M_\tau + \tfrac{1}{2}M_\tau^{2} + O(\rho^{3})\bigr) \notag\\
  &\quad= I + M_\chi + M_\tau
     + \tfrac{1}{2}M_\chi^{2} + M_\chi M_\tau + \tfrac{1}{2}M_\tau^{2}
     + O(\rho^{3}).
  \label{eq:rhs}
\end{align}
In the last line, the $O(\rho^3)$ remainder absorbs the cross-terms $M_\chi \cdot \tfrac{1}{2}M_\tau^2$, $\tfrac{1}{2}M_\chi^2 \cdot M_\tau$, and $\tfrac{1}{2}M_\chi^2 \cdot \tfrac{1}{2}M_\tau^2$, all of which are cubic or higher in $\rho$.

\paragraph{Step 4: Compute the residual.}
Subtracting Eq.~\eqref{eq:rhs} from Eq.~\eqref{eq:lhs}:
\begin{align}
  &\operatorname{Cay}(M_\chi {+} M_\tau)
   - \operatorname{Cay}(M_\chi)\,\operatorname{Cay}(M_\tau) \notag\\
  &\quad= \tfrac{1}{2}\bigl(M_\chi M_\tau + M_\tau M_\chi\bigr)
          - M_\chi M_\tau + O(\rho^{3}) \notag\\
  &\quad= \tfrac{1}{2}\bigl(M_\tau M_\chi - M_\chi M_\tau\bigr) + O(\rho^{3}) \notag\\
  &\quad= -\,\tfrac{1}{2}\,[M_\chi,\, M_\tau] + O(\rho^{3}),
  \label{eq:residual}
\end{align}
where $[M_\chi, M_\tau] := M_\chi M_\tau - M_\tau M_\chi$ is the matrix commutator.

\paragraph{Step 5: Frobenius-norm bound.}
By the triangle inequality and submultiplicativity of $\|\cdot\|_F$,
\begin{equation}
  \|[M_\chi, M_\tau]\|_F
  \;\le\; \|M_\chi M_\tau\|_F + \|M_\tau M_\chi\|_F
  \;\le\; 2\,\|M_\chi\|_F\,\|M_\tau\|_F
  \;\le\; 2\rho^{2}.
  \label{eq:commutator_bound}
\end{equation}
Hence the leading residual satisfies $\bigl\|\tfrac{1}{2}[M_\chi, M_\tau]\bigr\|_F \le \rho^{2}$, which is $O(\rho^2)$ as claimed. When $M_\chi$ and $M_\tau$ commute (e.g.\ both diagonal), the commutator vanishes and the approximation is exact to second order.
\end{proof}

\noindent\textit{Remark.}
The $O(\rho^2)$ regime is maintained in practice by the identity prior $\lambda_{\mathrm{id}}\|M\|^2$ in $\mathcal{L}_{\mathrm{lie}}$, which keeps $\|M\|_F$ small throughout training. The log--log analysis in Figure~\ref{fig:grid_geometry}(b) confirms a fitted slope of ${\approx}\,0.98$, matching the predicted quadratic scaling of the composition error with $\|M\|_F^2$.

\section{UHV-4K Synthesis Pipeline}
\label{app:synthesis_pipeline}

UHV-4K is built with a two-pass pipeline from clean 4K source videos.

\paragraph{Pass~1 (analysis).}
Each video is sub-sampled at 5\,fps and processed at a reduced proxy resolution.
Depth Anything V2 (ViT-L)~\citep{yang2024depthanythingv2} produces per-frame monocular depth maps, and RAFT~\citep{teed2020raft} produces forward optical flow between consecutive frames.
Because the depth estimator outputs disparity (high values at near surfaces), we apply a polarity flip ($d = 1 - \hat{d}_{\mathrm{norm}}$) so that larger depth values correspond to farther objects---the convention the ASM requires for distant-heavy haze.

\paragraph{Pass~2 (synthesis).}
Depth maps are per-clip percentile-normalized, temporally smoothed via flow-guided blending across adjacent frames, and upsampled to native $3840{\times}2160$ with an edge-aware guided filter~\citep{he2013guided} that preserves object boundaries.
The ASM of Eq.~\eqref{eq:asm} then generates the hazy frames, with scattering coefficient $\beta$ and atmospheric light $A_\infty$ held fixed per video.
All five modalities (hazy, ground truth, depth, transmission, optical flow) are exported per frame.

\paragraph{Quality assurance.}
Five automated checks verify every video before inclusion: (i)~modality alignment across all five channels, (ii)~depth temporal stability between consecutive frames, (iii)~inter-frame haze consistency via adjacent-frame intensity difference, (iv)~atmospheric-light round-trip verification against the synthesized images, and (v)~pixel-range validation.
Videos failing any check are flagged for exclusion or regeneration.

\section{Dataset Comparison}
\label{app:dataset_comparison}

Table~\ref{tab:dataset_comparison} compares representative dehazing benchmarks across resolution, modality, and auxiliary annotations. UHV-4K is the only video dataset that pairs native 4K resolution with aligned depth, transmission, and optical-flow annotations.

\begin{table}[h]
  \centering
  \small
  \setlength{\tabcolsep}{3pt}
  \renewcommand{\arraystretch}{1}
  \caption{Comparison of representative dehazing benchmarks across resolution, modality, and auxiliary annotations.}
  \label{tab:dataset_comparison}
  \resizebox{\textwidth}{!}{%
  \begin{tabular}{@{}lccccccc@{}}
    \toprule
    Dataset & Type & Real/Syn & Annot. & Resolution & \#Vid. & \#Frames & Aux. Mod. \\
    \midrule

    I-HAZE~\citep{Ancuti2018IHaze}
      & image & real & paired & 2K--4K                                    & --    & 35          & $-$ \\
    O-HAZE~\citep{Ancuti2018OHaze}
      & image & real & paired & ${\leq}5436{\times}3612$                  & --    & 45          & $-$ \\

    4KID~\citep{4KDehazing}
      & image & syn. & paired & $3840{\times}2160$                        & --    & 10{,}000    & $-$ \\
    8KDehaze~\citep{DeHazeXL}
      & image & syn. & paired & $8192{\times}8192$                        & --    & 10{,}000    & $-$ \\

    NYU-Haze~\citep{Ren2019DeepVD,Li2017EndtoEndUV}
      & video & syn. & paired   & $640{\times}480$                        & --    & --          & depth \\
    REVIDE~\citep{CG-IDN_REVIDE_dataset}
      & video & real & paired   & $2708{\times}1800$                      & 48    & 1,982       & $-$ \\
    HazeWorld~\citep{MAP-Net_HazeWorld_dataset}
      & video & syn. & paired   & ${\leq}720$p                            & 5,084 & $\sim$326k  & trans. \\
    GoProHazy~\citep{DVD_GoProHazy_and_DrivingHazy_dataset}
      & video & real & unaligned & $1920{\times}1080$                     & 27    & 4,256       & $-$ \\
    DrivingHazy~\citep{DVD_GoProHazy_and_DrivingHazy_dataset}
      & video & real & no-ref    & $1920{\times}1080$                     & 20    & 1,807       & $-$ \\
    \midrule
    \rowcolor{ourgreen} \textbf{UHV-4K (Ours)}
      & video & syn. & paired & $3840{\times}2160$                        & 100   & 2,500       & depth, trans., flow \\
    \bottomrule
  \end{tabular}%
  }
\end{table}

\section{Perceptual Realism and Diversity Study}
\label{app:user_study}

Because UHV-4K is synthetic, we run a perceptual study to validate its realism, fidelity, and diversity against existing real-capture and synthetic benchmarks. Table~\ref{tab:user_study} reports mean scores with 95\,\% confidence intervals on a 5-point Likert scale ($N{=}20$ raters, $K{=}10$ frames per dataset). UHV-4K leads in visual fidelity and scene diversity while trailing real-capture datasets by less than one Likert step on haze realism---expected, since ASM synthesis cannot reproduce the photographic character of real fog. The 4K corpus therefore complements real-capture data rather than replacing it.

\begin{table}[h]
  \centering
  \small
  \setlength{\tabcolsep}{6.5pt}
  \renewcommand{\arraystretch}{1}
  \caption{\textbf{Perceptual realism and diversity study.} Mean $\pm$ 95\,\% CI on a 5-point Likert scale ($N{=}20$ raters, $K{=}10$ frames per dataset). R = haze realism, F = visual fidelity, D = scene diversity.}
  \label{tab:user_study}
  \begin{tabular}{@{}lcccc@{}}
    \toprule
    Dataset & Realism (R) & Fidelity (F) & Diversity (D) & Overall \\
    \midrule
    REVIDE~\citep{CG-IDN_REVIDE_dataset} & \textbf{4.28\,$\pm$\,0.08} & \underline{3.79\,$\pm$\,0.09} & 2.55\,$\pm$\,0.24 & 3.54 \\
    HazeWorld~\citep{MAP-Net_HazeWorld_dataset} & 3.14\,$\pm$\,0.08 & 3.28\,$\pm$\,0.09 & \underline{4.10\,$\pm$\,0.21} & 3.51 \\
    GoProHazy~\citep{DVD_GoProHazy_and_DrivingHazy_dataset} & \underline{4.22\,$\pm$\,0.08} & 3.51\,$\pm$\,0.09 & 3.00\,$\pm$\,0.26 & \underline{3.58} \\
    \rowcolor{ourgreen} \textbf{UHV-4K (Ours)} & 3.81\,$\pm$\,0.09 & \textbf{4.58\,$\pm$\,0.07} & \textbf{4.60\,$\pm$\,0.24} & \textbf{4.33} \\
    \bottomrule
  \end{tabular}
\end{table}

\section{Training Configuration}
\label{app:hyperparams}

Table~\ref{tab:hyperparams} lists all training hyperparameters. Loss weights are set once on UHV-4K and reused on REVIDE and HazeWorld without per-dataset tuning.

\begin{table}[h]
  \centering
  \caption{Training hyperparameters for all experiments.}
  \label{tab:hyperparams}
  \small
  \setlength{\tabcolsep}{8pt}
  \begin{tabular}{@{}ll@{}}
    \toprule
    Parameter & Value \\
    \midrule
    Optimizer & AdamW ($\text{lr}{=}1.5{\times}10^{-4}$, $\text{wd}{=}10^{-4}$) \\
    LR schedule & Cosine annealing, 5-epoch linear warmup \\
    Batch size & 16 ($2 \times 8$ GPUs) \\
    Crop size & $512 \times 512$ \\
    Input frames & $T{=}5$ consecutive \\
    Precision & FP16 mixed precision \\
    Epochs & 200 (REVIDE, UHV-4K); 100 (HazeWorld) \\
    Gradient clipping & max-norm 1.0 \\
    \midrule
    $\lambda_{\mathrm{pixel}}$ & 1.0 (Charbonnier, $\epsilon{=}10^{-3}$) \\
    $\lambda_{\mathrm{perc}}$ & 0.04 (DINOv2 ViT-S/14) \\
    $\lambda_{\mathrm{lie}}$ & 0.2 \\
    \quad $\lambda_{\mathrm{id}}$ (identity prior) & 0.01 \\
    \quad $\lambda_{\mathrm{sp}}$ (spatial smoothness) & 0.05 \\
    \quad $\lambda_{\mathrm{tm}}$ (temporal smoothness) & 0.10 \\
    \quad $\lambda_{\mathrm{g}}$ (guide consistency) & 0.02 \\
    \bottomrule
  \end{tabular}
\end{table}

\section{Temporal Consistency: Full Results}
\label{app:temporal}

Table~\ref{tab:temporal} provides the complete temporal consistency evaluation on UHV-4K, extending the summary in \S\ref{sec:analysis}. We report tOF (RAFT-warped inter-frame $\ell_1$ error) and $\sigma$-PSNR (within-video standard deviation of per-frame PSNR) alongside PSNR to jointly assess steadiness and fidelity.

\begin{table}[h]
  \centering
  \caption{\textbf{Temporal consistency on UHV-4K.} tOF = RAFT-warped inter-frame error; $\sigma$-PSNR = within-video standard deviation of per-frame PSNR. Lower is steadier; ablations excluded from ranking.}
  \label{tab:temporal}
  \small
  \setlength{\tabcolsep}{6pt}
  \renewcommand{\arraystretch}{1}
  \begin{tabular}{lccc}
    \toprule
    \multirow{2}{*}{Method} & \multicolumn{2}{c}{Temporal consistency} & \multicolumn{1}{c}{Restoration quality} \\
    \cmidrule(lr){2-3} \cmidrule(l){4-4}
    & tOF\,$\downarrow$ & $\sigma$-PSNR\,$\downarrow$ & PSNR\,$\uparrow$ \\
    \midrule
    \tablegroup{4}{(a) Image dehazing methods}
    \midrule
    MSBDN-DFF                & 0.005139 & 0.70 & 22.51 \\
    DehazeFormer             & 0.005299 & \underline{0.51} & \underline{23.74} \\
    DehazeXL                 & 0.005403 & 0.82 & 22.63 \\
    UHDDIP                   & 0.005630 & 0.58 & 21.90 \\
    UHDformer                & 0.006129 & 0.57 & 23.16 \\
    \midrule
    \tablegroup{4}{(b) Video dehazing methods}
    \midrule
    DCL                      & \textbf{0.004154} & 0.95 & 20.81 \\
    ViWS-Net                 & \underline{0.004161} & 0.55 & 22.06 \\
    MAP-Net                  & 0.004405 & 0.54 & 18.45 \\
    CG-IDN                   & 0.004547 & \textbf{0.44} & 20.13 \\
    DVD                      & 0.004551 & \underline{0.51} & 21.02 \\
    \midrule
    \rowcolor{ourgreen} \textbf{Ours}  & 0.004769 & 0.66 & \textbf{24.28}\,\tabgain{+0.54} \\
    \midrule
    \tablegroup{4}{(c) Targeted ablations of LiBrA-Net}
    \midrule
    w/o Temporal Field                   & 0.004800\,\muted{\scriptsize(+0.6\%)} & 0.99\,\textcolor{red!70!black}{\scriptsize(+50\%)} & 22.09\,\textcolor{red!70!black}{\scriptsize(-2.19)} \\
    w/o $\mathcal{L}_{\mathrm{lie}}$ temporal & 0.004396\,\muted{\scriptsize(-7.8\%)} & 0.95\,\textcolor{red!70!black}{\scriptsize(+44\%)} & 22.31\,\textcolor{red!70!black}{\scriptsize(-1.97)} \\
    \bottomrule
  \end{tabular}
\end{table}

\paragraph{Reading the table.}
Because tOF rewards any form of frame-to-frame smoothness---including spatial blurring that destroys detail---it must be read jointly with PSNR. The two lowest-tOF video methods (DCL at 0.004154, ViWS-Net at 0.004161) trail LiBrA-Net in PSNR by 3.47\,dB and 2.22\,dB respectively, suggesting that their low tOF comes at the cost of fine-detail loss rather than faithful coherence. LiBrA-Net achieves the highest PSNR among all sixteen compared methods while keeping tOF below every image dehazer---the combination of high per-frame fidelity with stable inter-frame behavior.

\paragraph{Ablation analysis.}
The two targeted ablations in part~(c) isolate complementary stabilization mechanisms:
\begin{itemize}
  \item \textbf{w/o Temporal Field}: tOF barely changes (+0.6\%) but $\sigma$-PSNR inflates by 50\% and PSNR drops 2.19\,dB. The temporal branch drives per-frame quality stability through self-attention across the $T{=}5$ frame window, not pixel-level warping.
  \item \textbf{w/o $\mathcal{L}_{\mathrm{lie}}$ temporal}: tOF actually \emph{decreases} ($-7.8\%$) yet $\sigma$-PSNR still inflates by 44\%. Without the temporal smoothness penalty on the grid trajectory, the network over-fits each frame independently; the lower tOF reflects accidental inter-frame similarity rather than structured temporal modeling.
\end{itemize}
Together, the Temporal Field stabilizes per-frame quality and the Lie temporal term smooths grid coefficients across frames; their joint action decouples consistency from over-smoothing.

\paragraph{Visual confirmation.}
Figure~\ref{fig:temporal_heatmap} shows inter-frame difference heatmaps for a representative UHV-4K test scene. Single-image methods (UHDformer, DehazeFormer) produce frame-to-frame halos around the moving pedestrian, visible as bright patches in the difference map. LiBrA-Net's difference map stays uniformly dark, consistent with the grid's low-frequency, smoothly varying affine field that avoids introducing high-frequency temporal flicker.

\begin{figure}[h]
  \centering
  \setlength{\zcellW}{0.19\linewidth}%
  \setlength{\tabcolsep}{1pt}%
  \renewcommand{\arraystretch}{0}%
  \begin{minipage}[c]{0.88\linewidth}
    \centering
    \begin{tabular}{@{}c@{\hspace{2pt}}cccc@{}}
      & \scriptsize\textbf{Frame $t{-}1$} & \scriptsize\textbf{Frame $t$}
      & \scriptsize\textbf{Frame $t{+}1$}
      & \scriptsize\textbf{$|\Delta(t,t{-}1)|$} \\[1pt]
      \rotatebox{90}{\scriptsize\textbf{Hazy}} &
        \includegraphics[width=\zcellW]{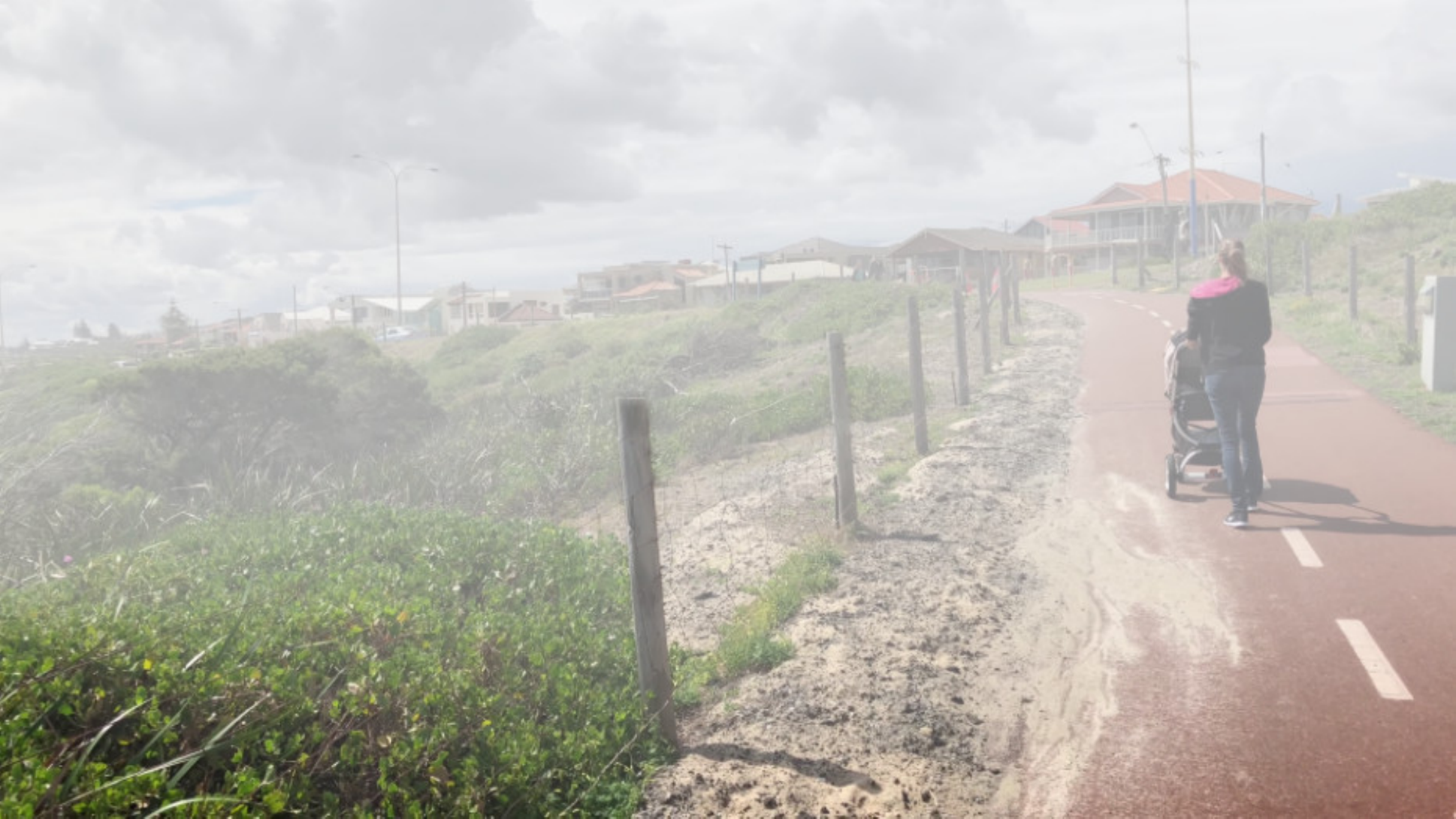} &
        \includegraphics[width=\zcellW]{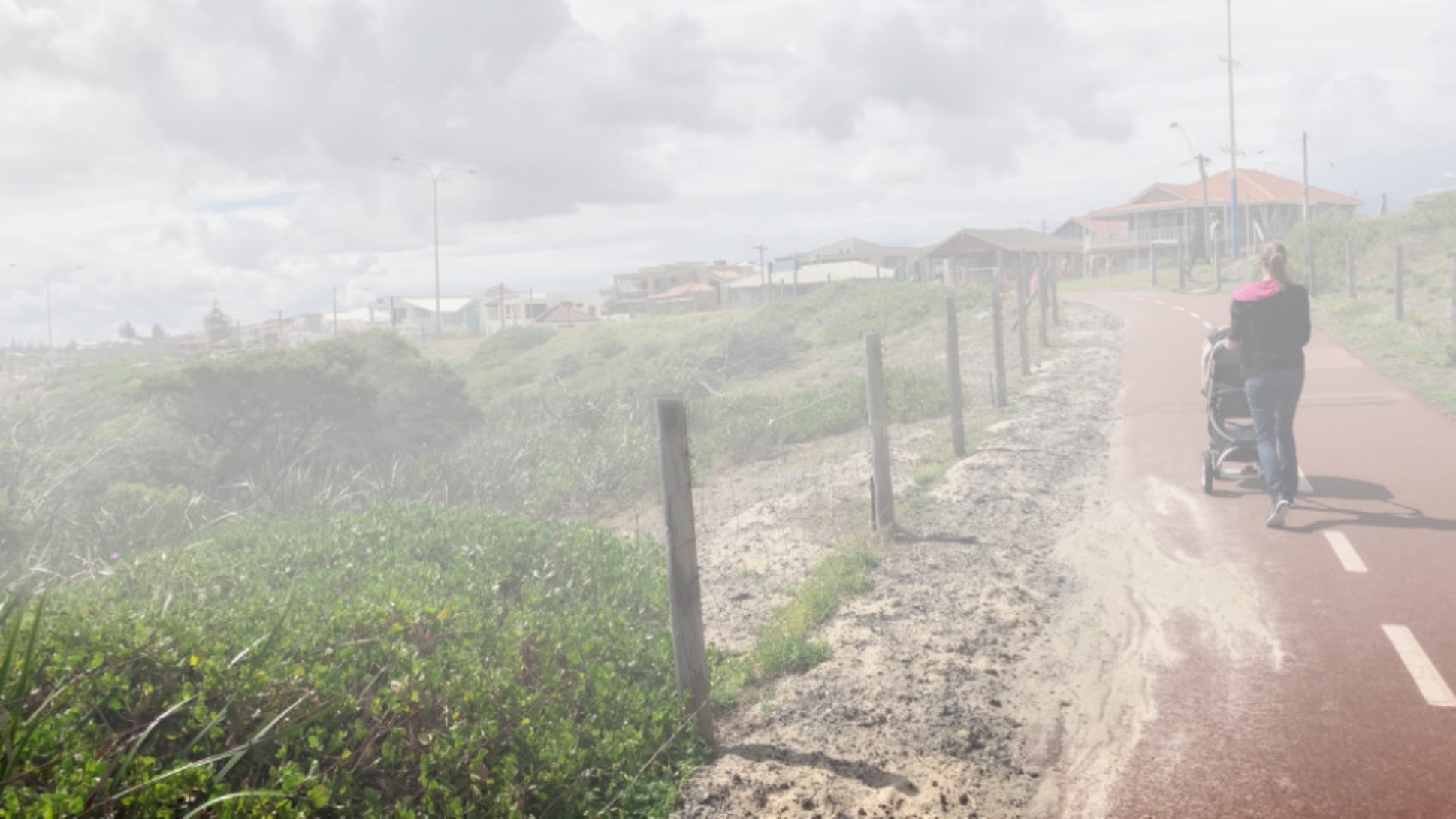} &
        \includegraphics[width=\zcellW]{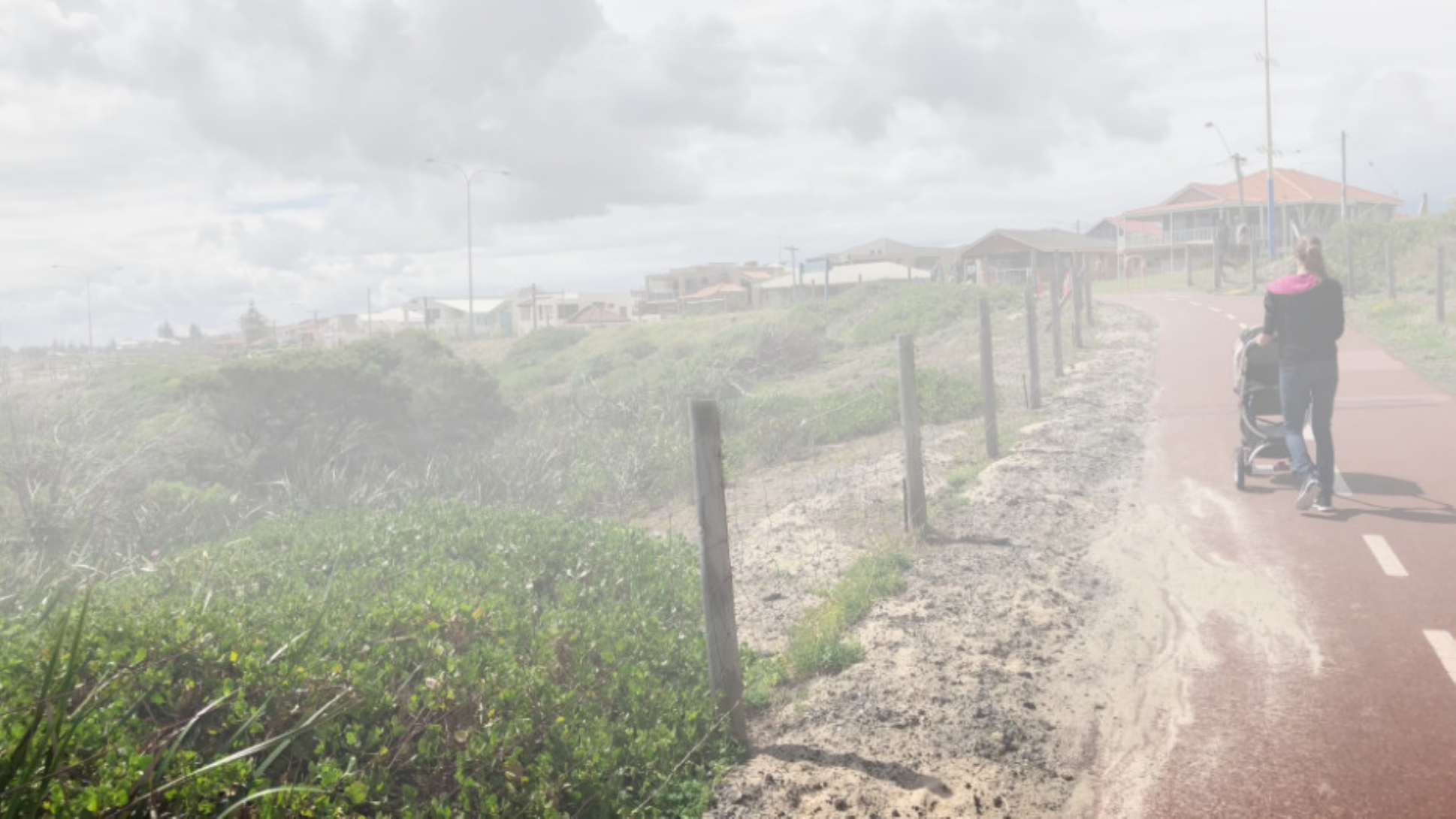} &
        \diffcell{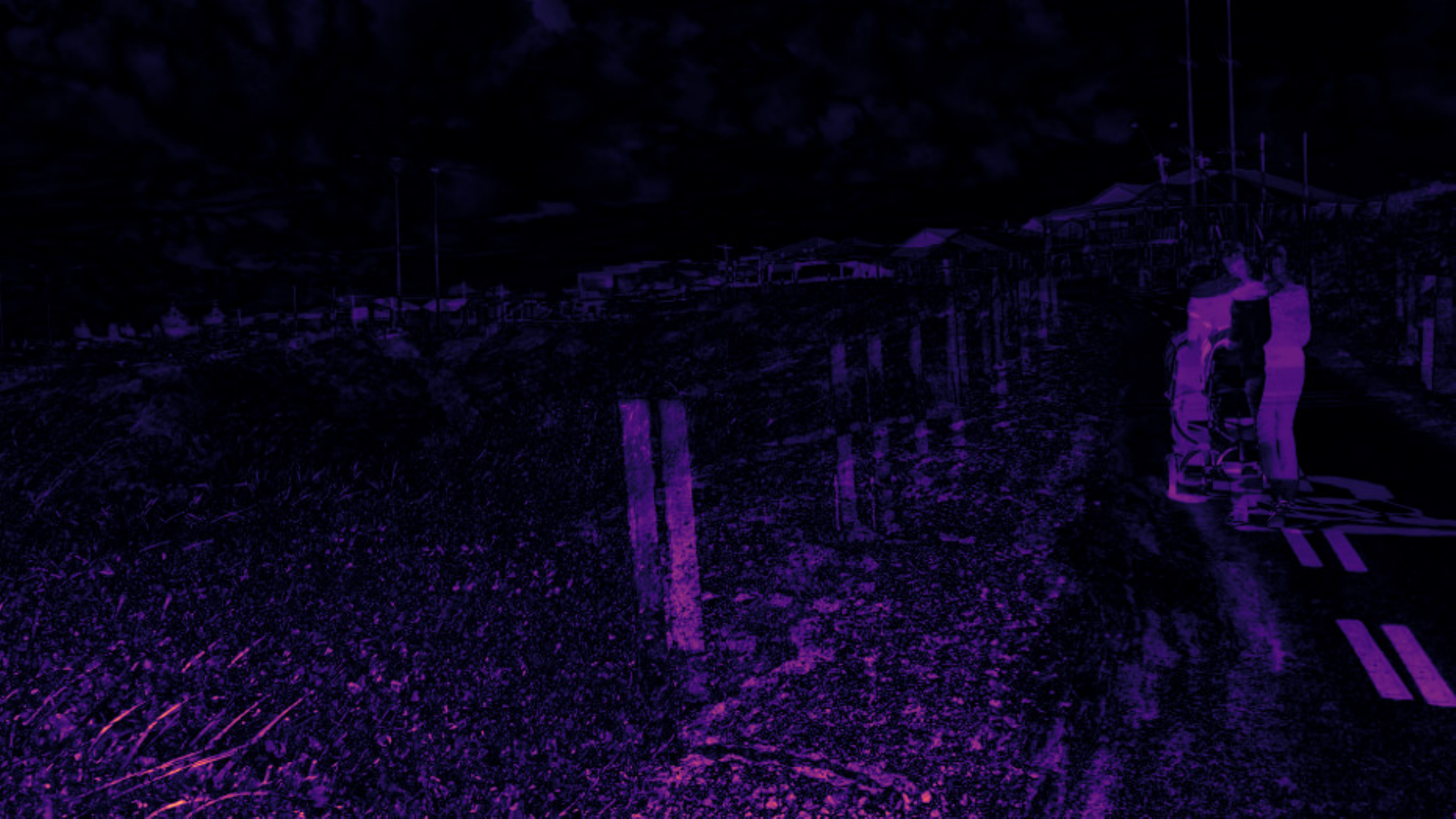}
      \\[0.5pt]
      \rotatebox{90}{\scriptsize\textbf{UHDformer}} &
        \includegraphics[width=\zcellW]{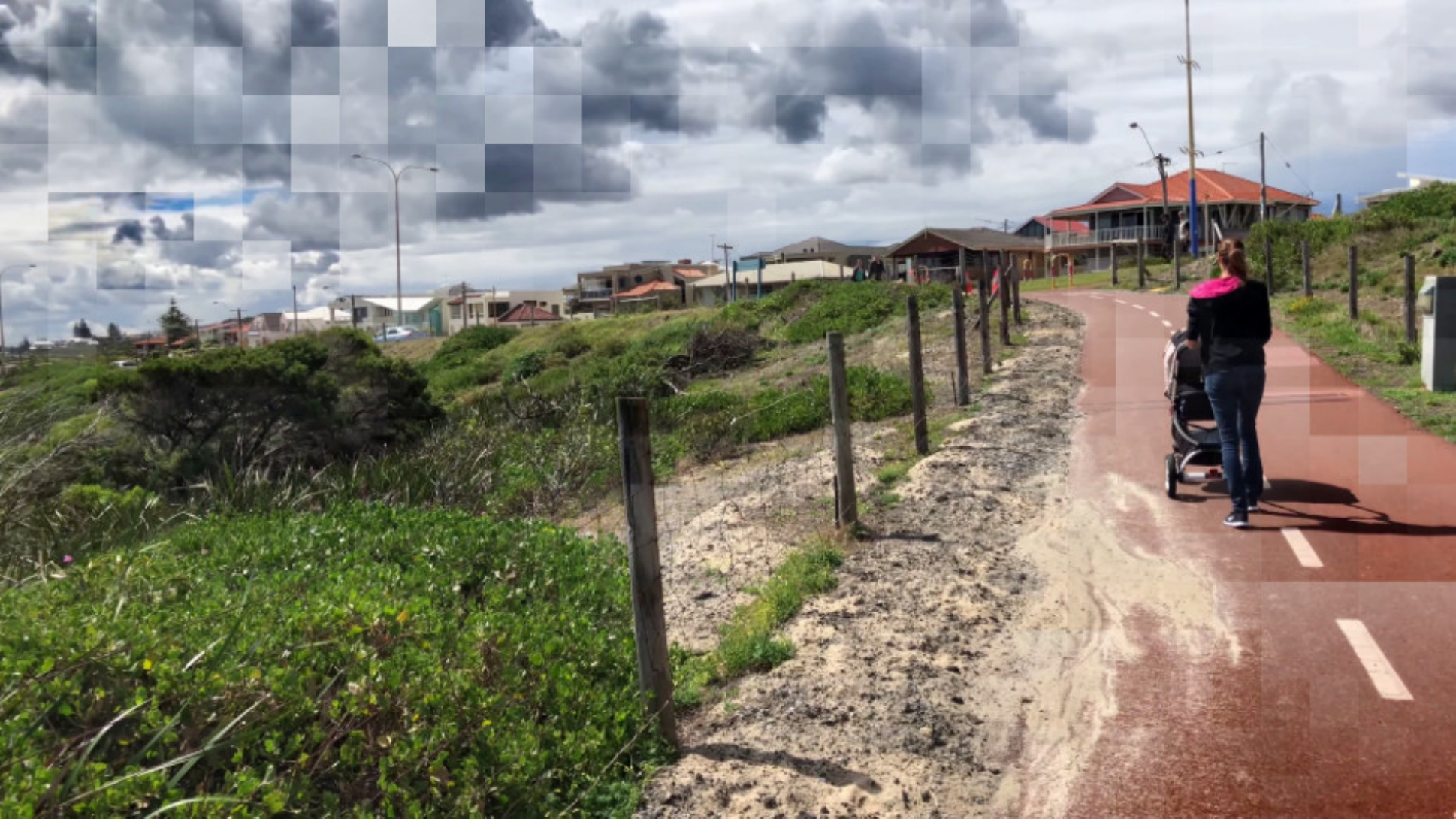} &
        \includegraphics[width=\zcellW]{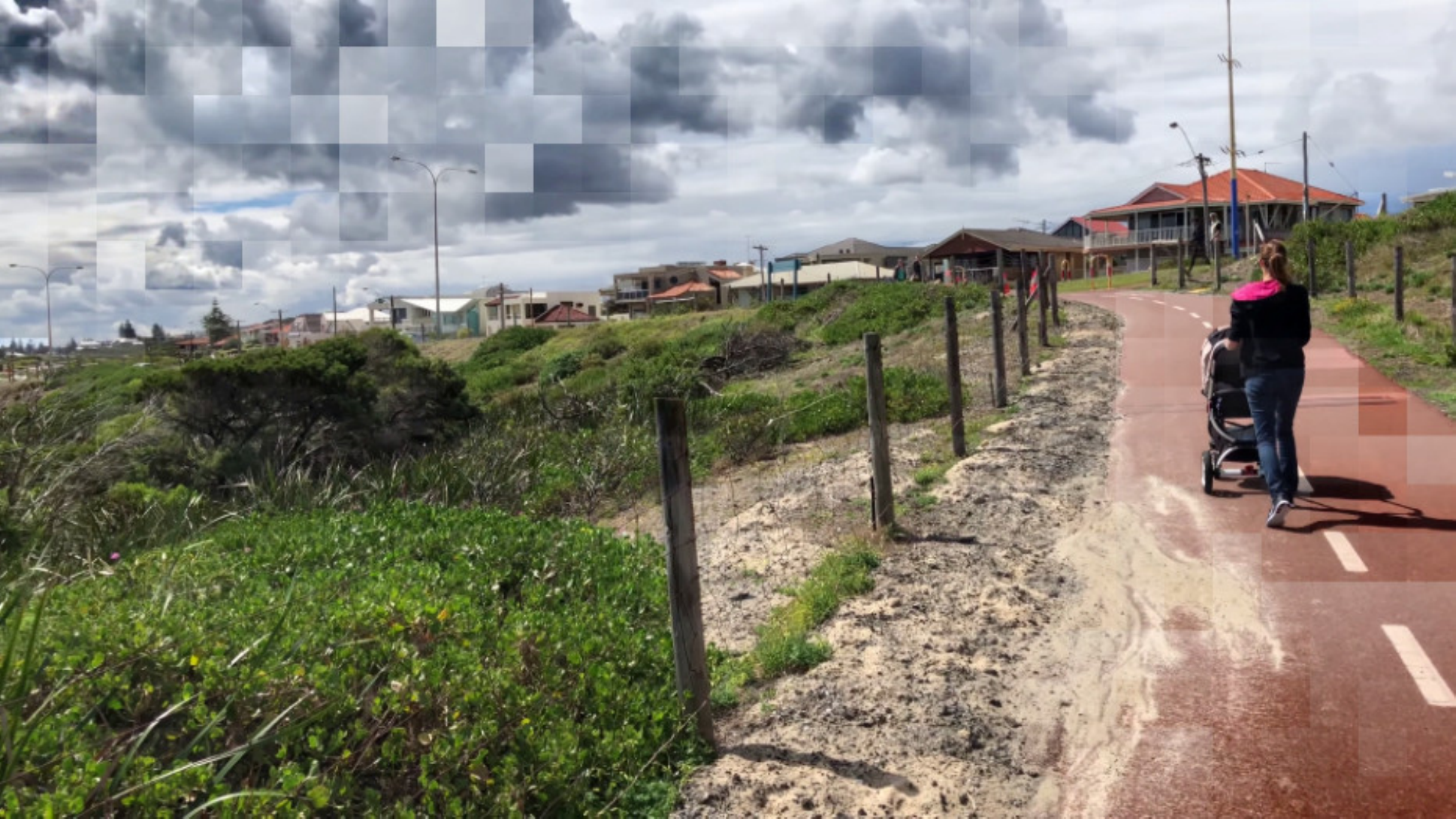} &
        \includegraphics[width=\zcellW]{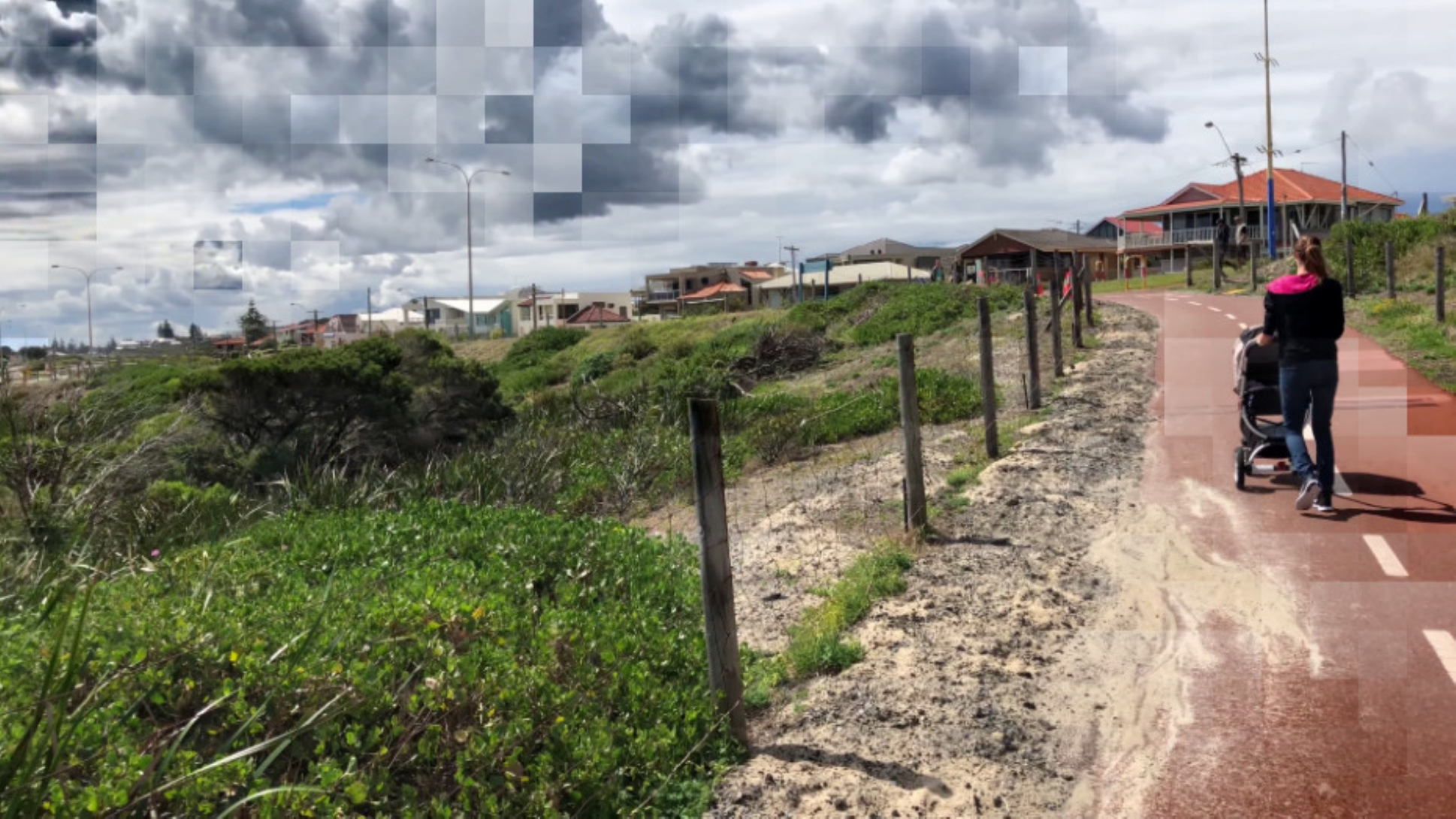} &
        \diffcell{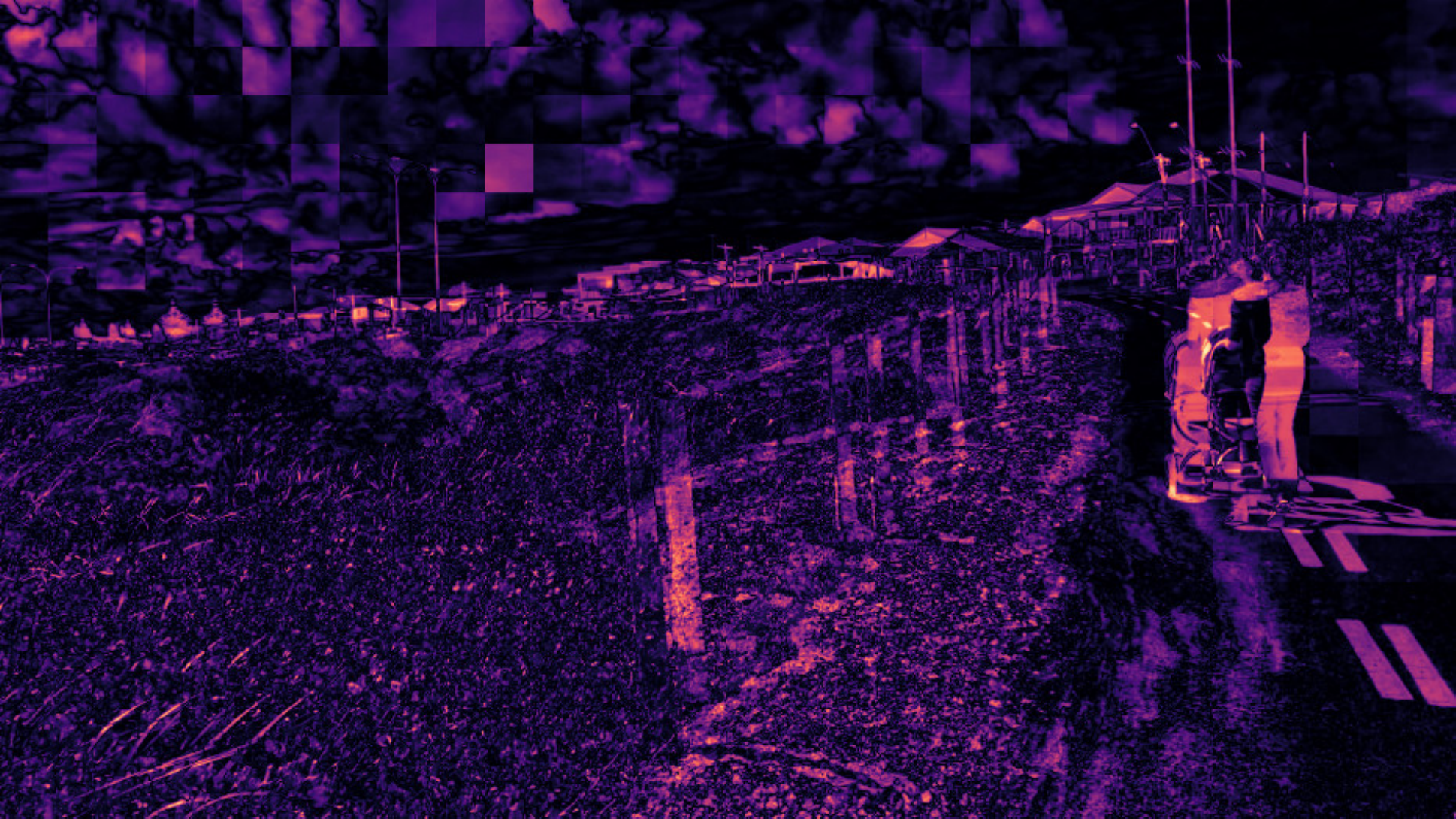}
      \\[0.5pt]
      \rotatebox{90}{\scriptsize\textbf{DehazeFormer}} &
        \includegraphics[width=\zcellW]{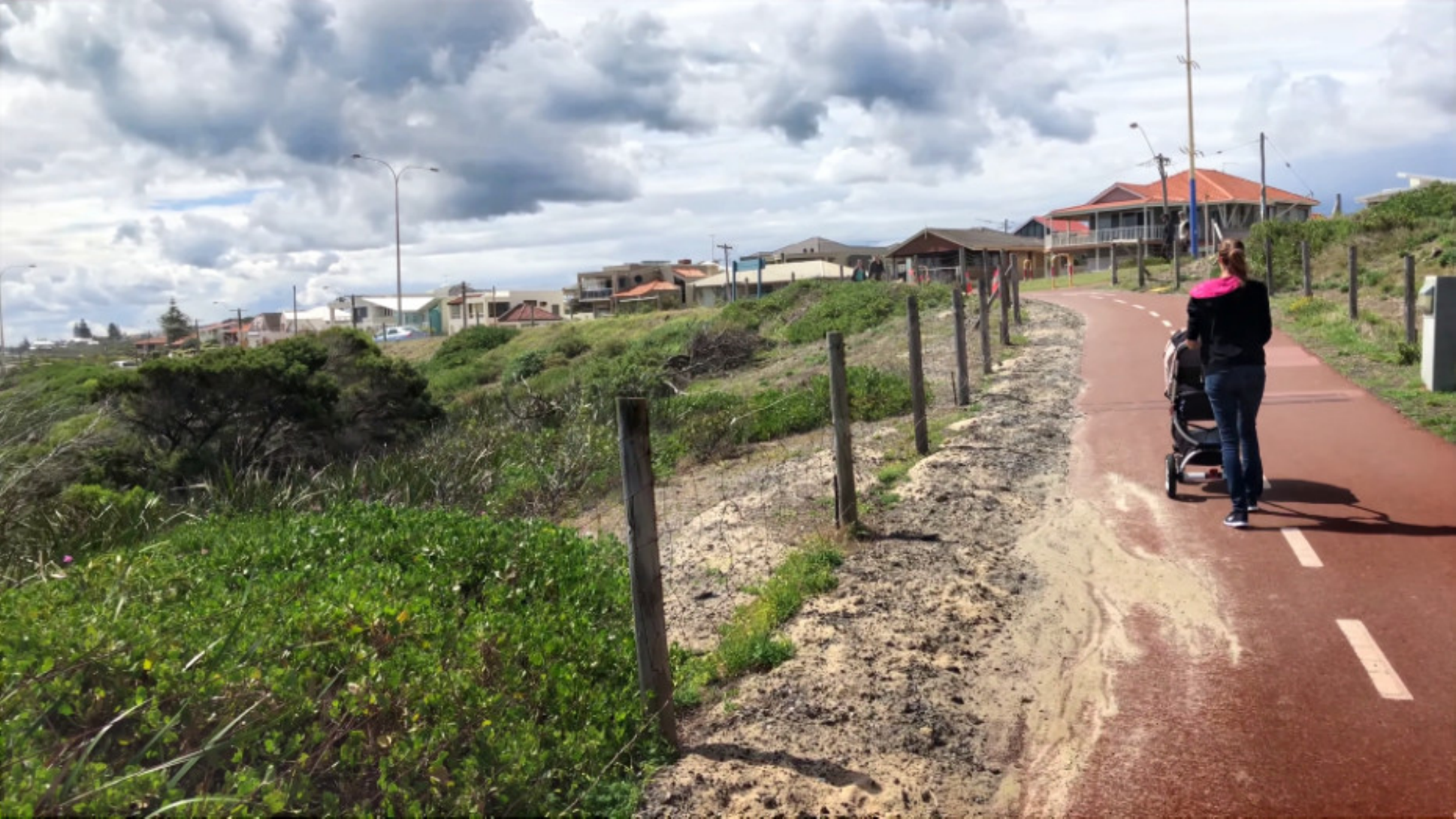} &
        \includegraphics[width=\zcellW]{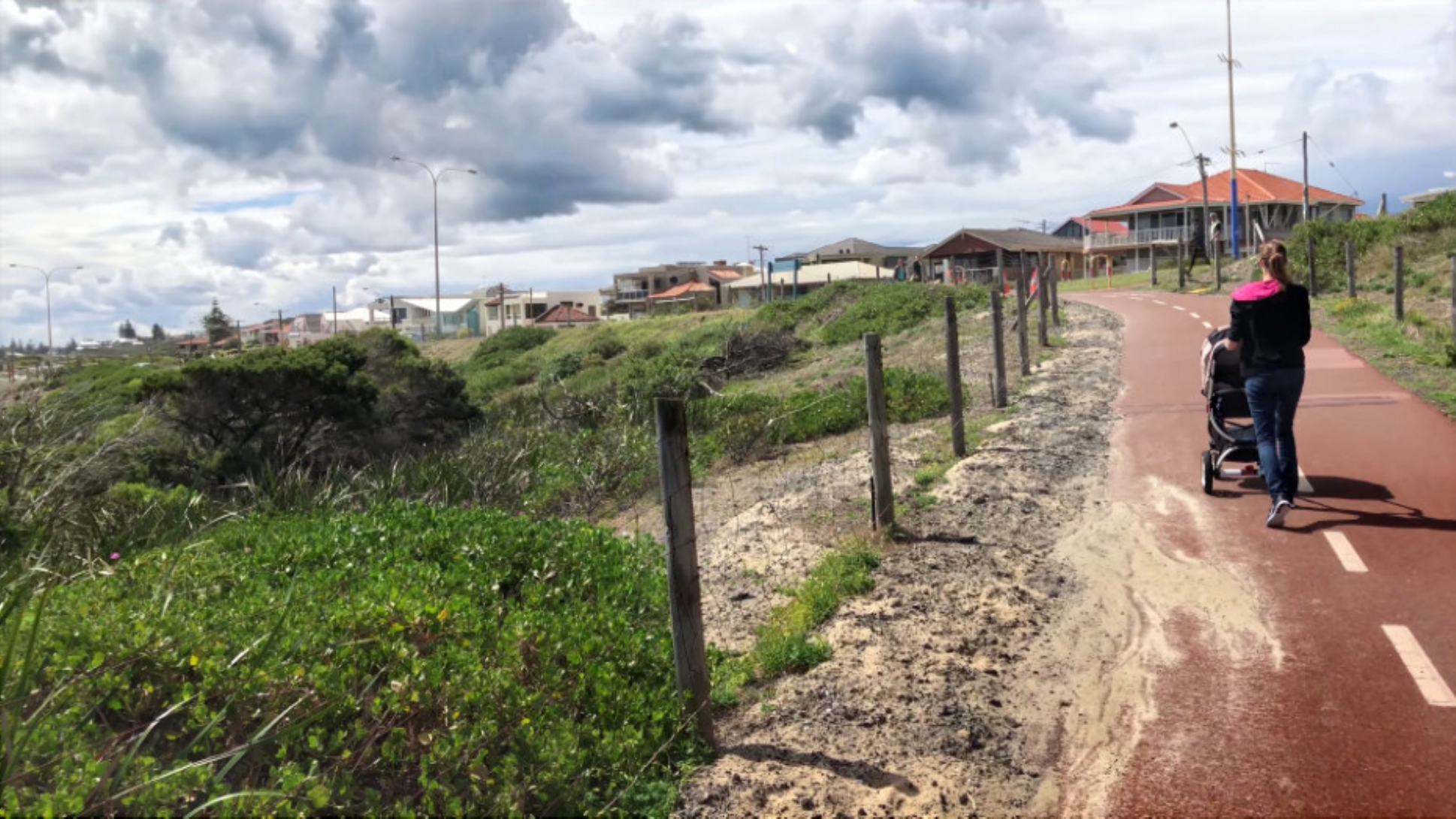} &
        \includegraphics[width=\zcellW]{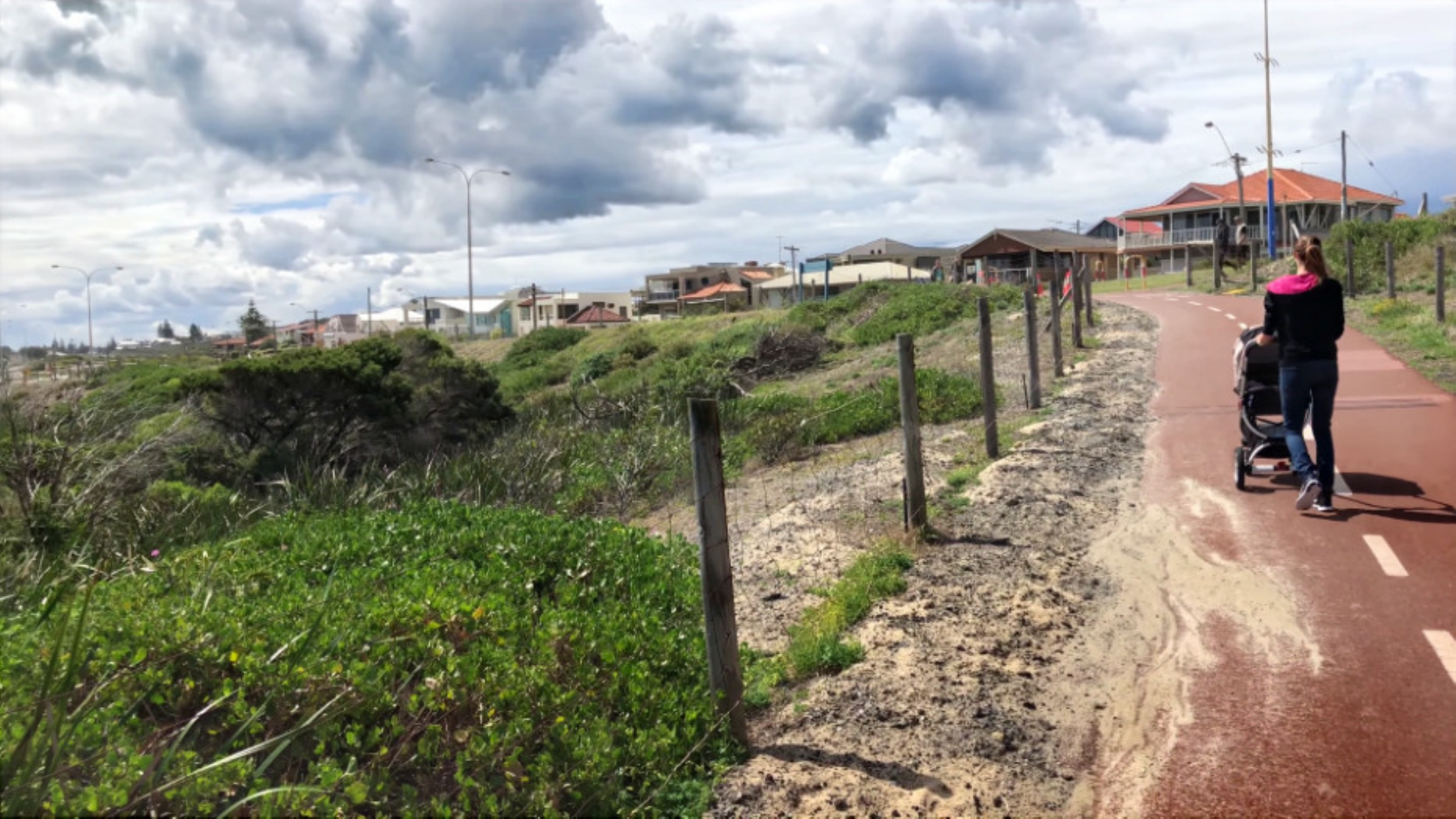} &
        \diffcell{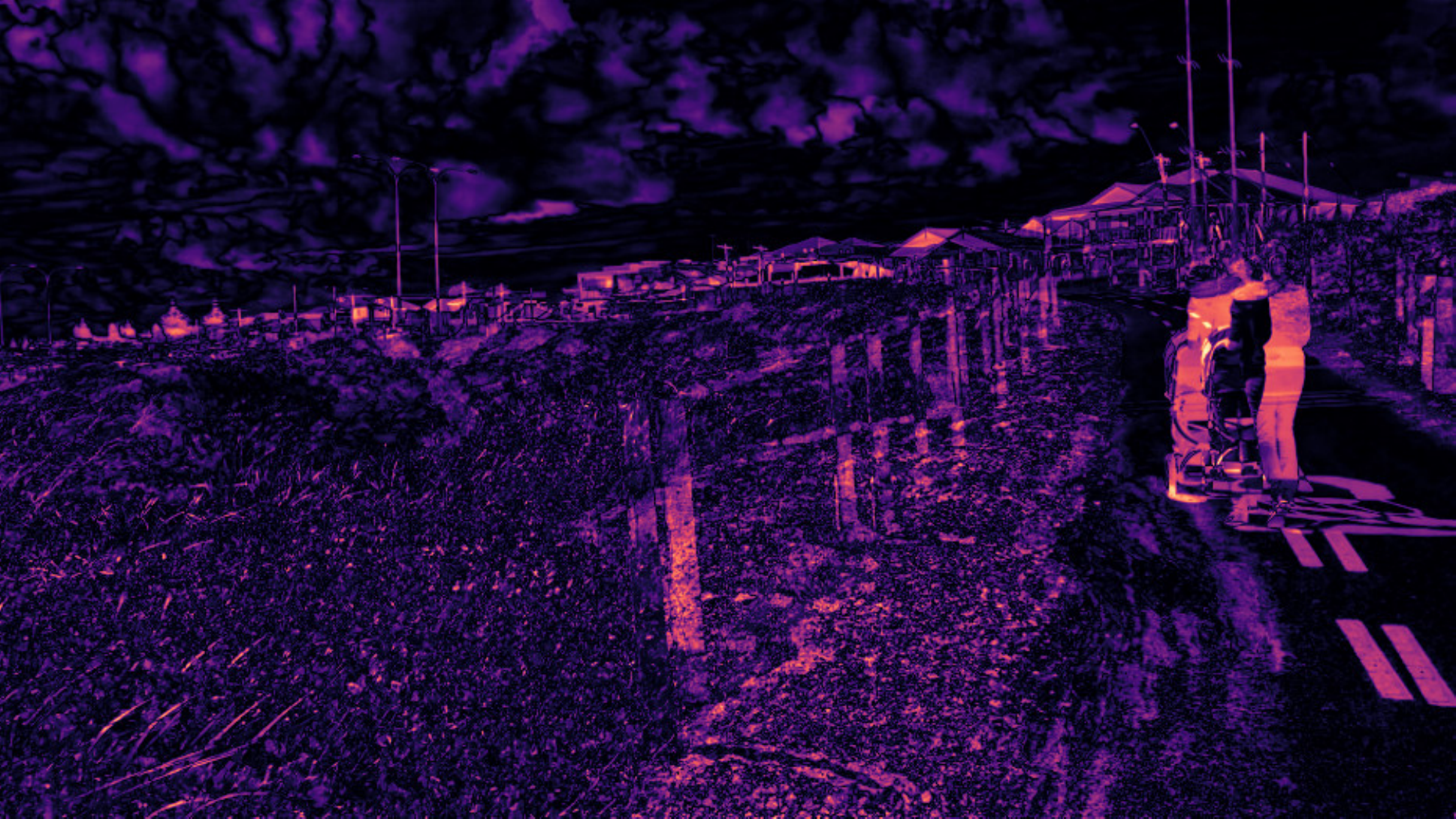}
      \\[0.5pt]
      \rotatebox{90}{\scriptsize\textbf{Ours}} &
        \includegraphics[width=\zcellW]{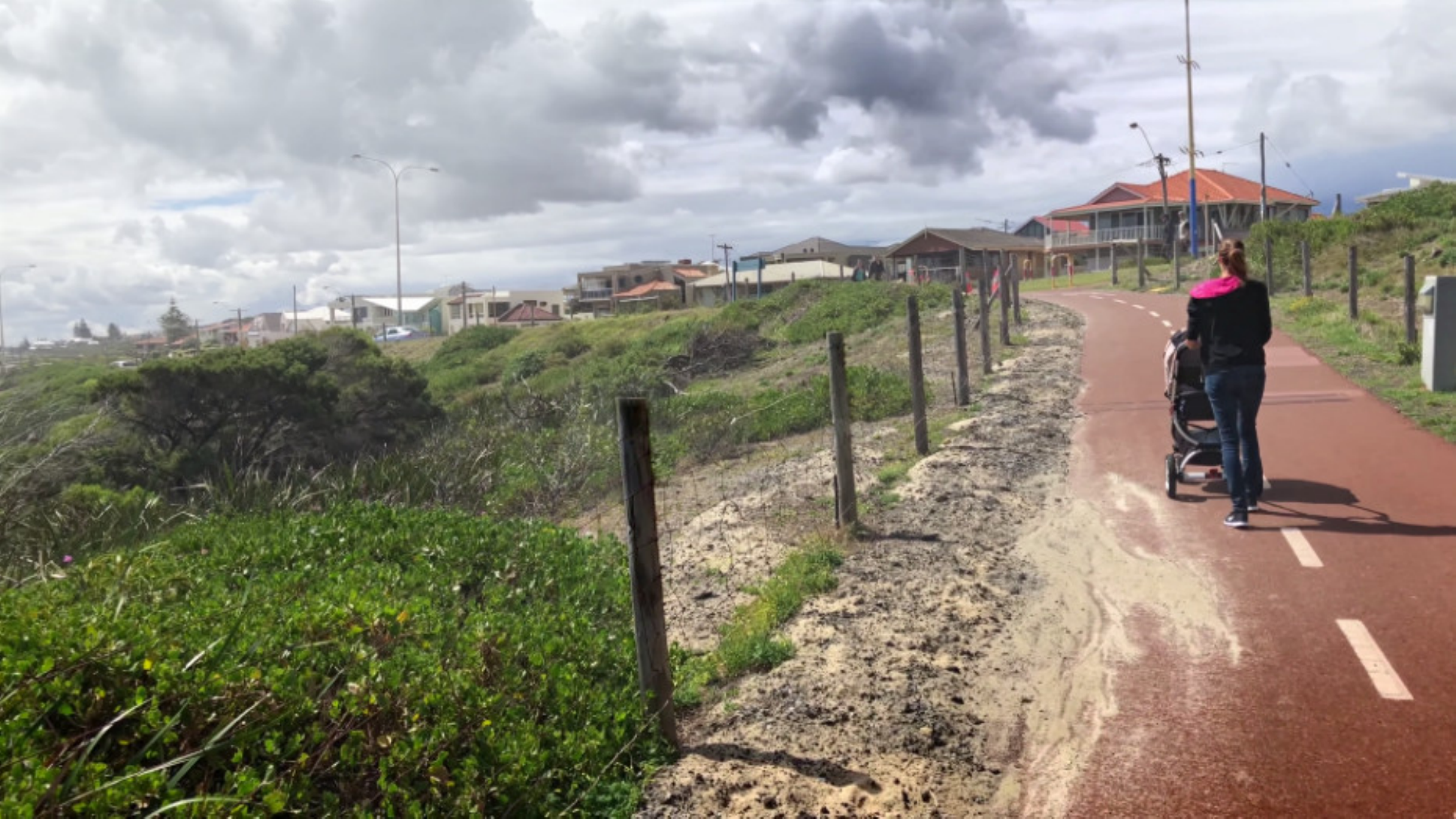} &
        \includegraphics[width=\zcellW]{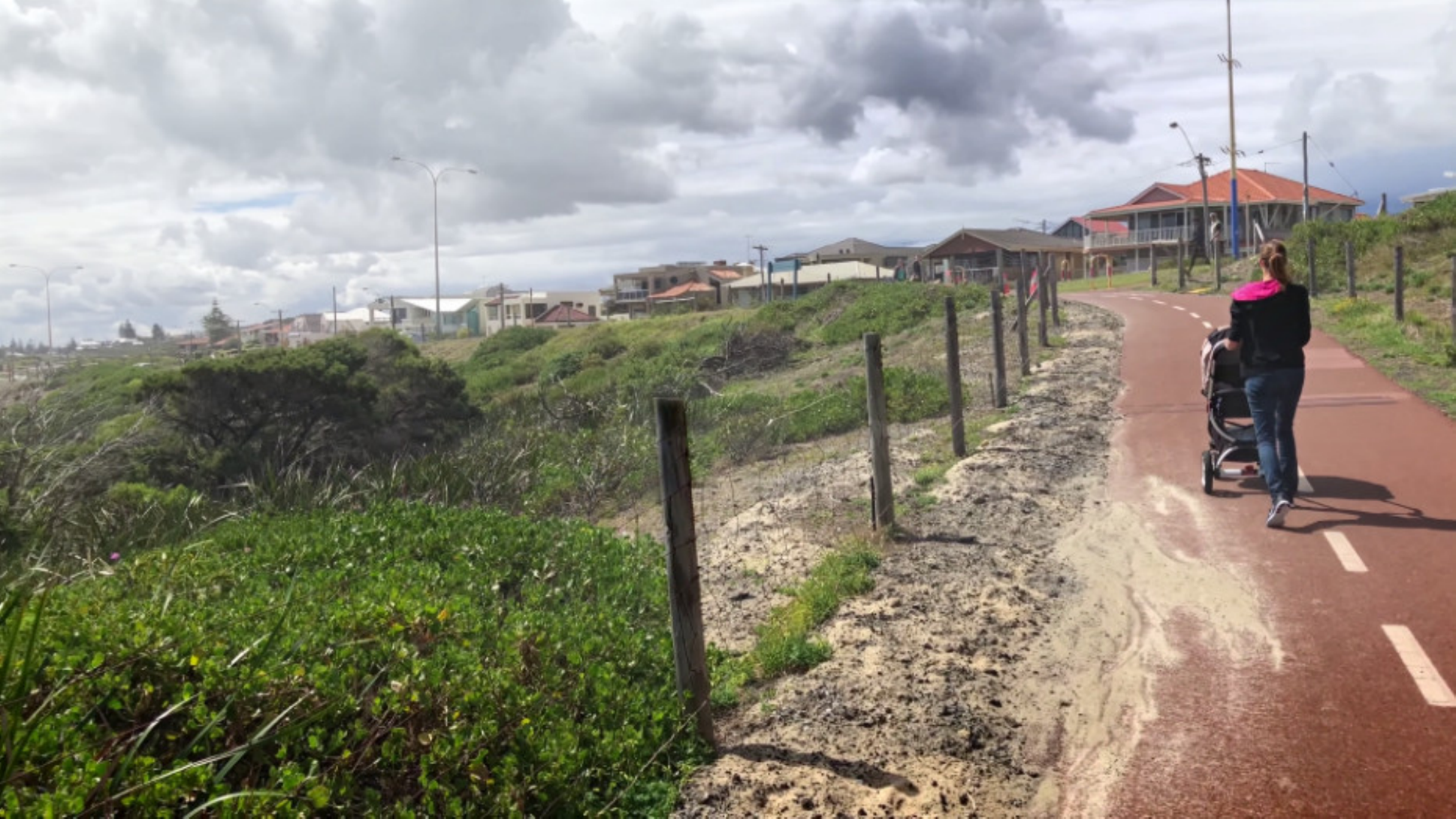} &
        \includegraphics[width=\zcellW]{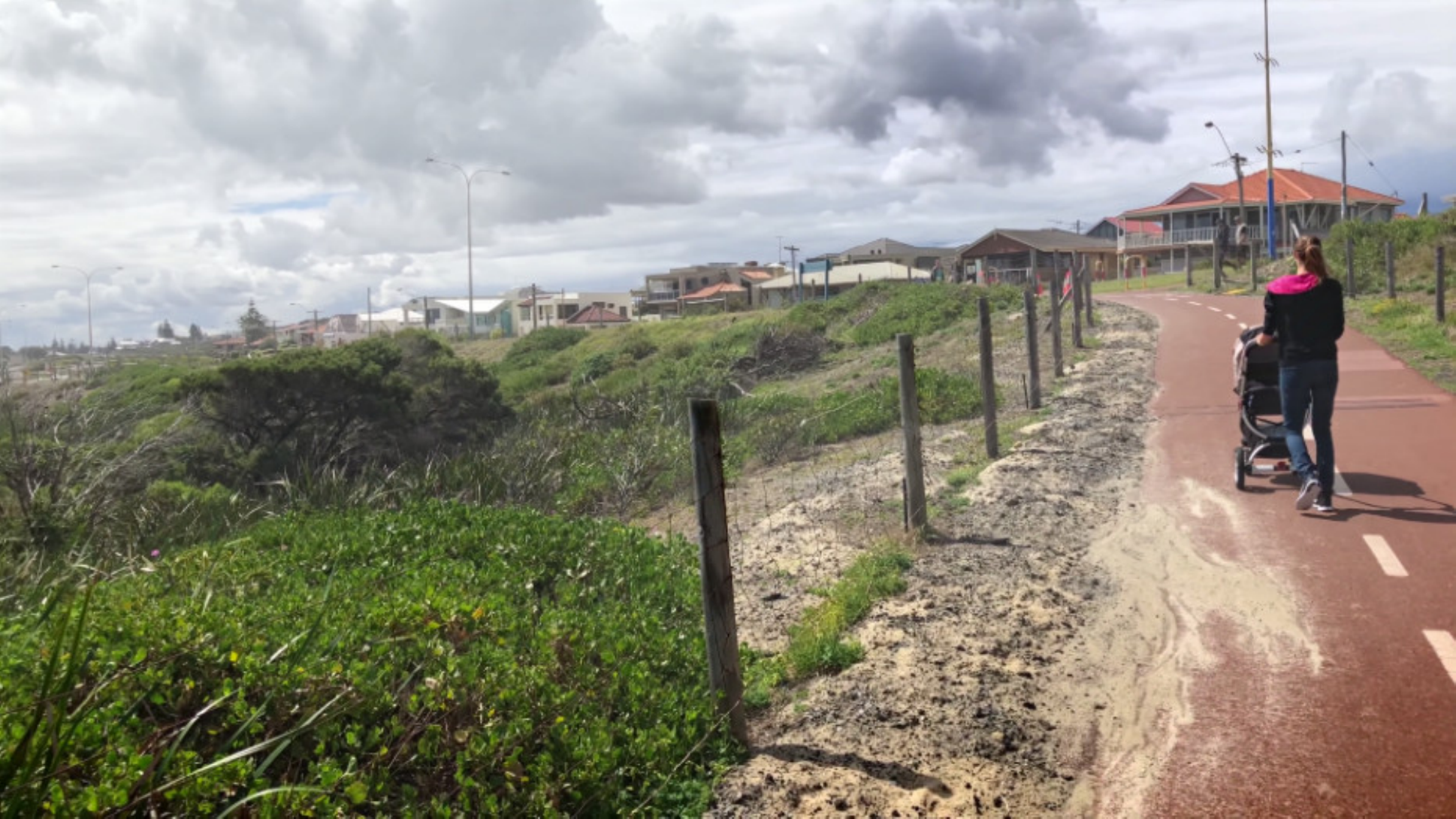} &
        \diffcell{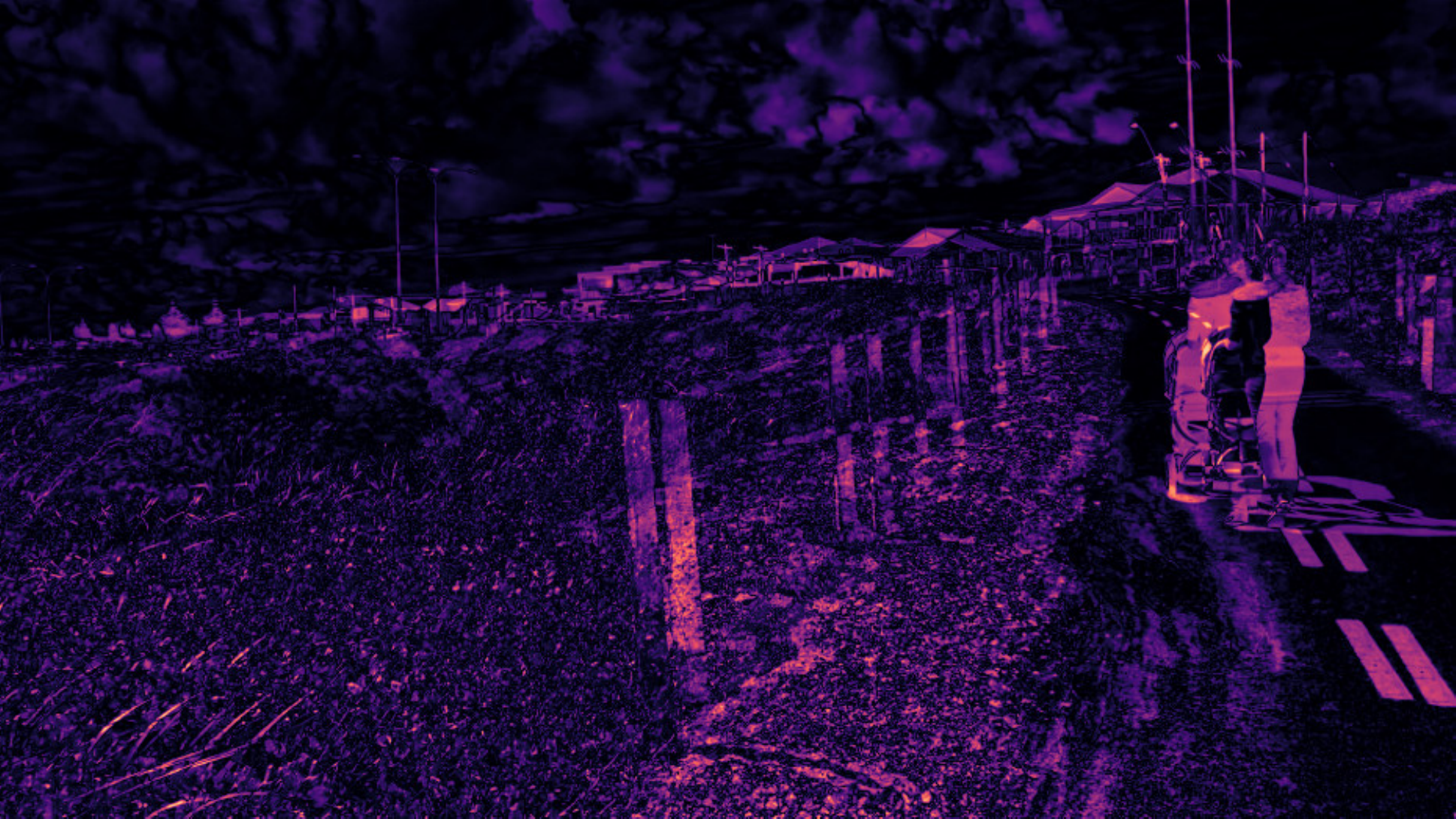}
      \\
    \end{tabular}
  \end{minipage}\hspace{-2mm}%
  \begin{minipage}[c]{0.085\linewidth}
    \centering
    \cbarpanel
  \end{minipage}
  \caption{\textbf{Inter-frame difference heatmaps on a UHV-4K test scene.}
    Columns: three consecutive frames and the per-pixel absolute difference $|\hat{J}_t{-}\hat{J}_{t-1}|$ (shared magma colorbar, $0$--$1$).
    Amber boxes mark two fixed regions of interest. LiBrA-Net's difference map stays uniformly dark, while single-image methods exhibit halos around moving objects.}
  \label{fig:temporal_heatmap}
\end{figure}

\section{External Validity: Extended Results}
\label{app:external}

This section provides the visualization and extended analysis for the downstream perception and real-world generalization experiments summarized in \S\ref{sec:external_validity}.

\subsection{Downstream Perception}
\label{app:downstream_vis}

We feed each method's 4K dehazed output into frozen YOLOv8-X~\citep{jocher2023yolov8} (COCO-pretrained) and SegFormer-B5~\citep{xie2021segformer} (ADE20K-pretrained) at default settings so that any performance gap comes solely from the upstream dehazer. Table~\ref{tab:downstream_nr} reports ten no-reference metrics spanning image quality, detection, and segmentation. LiBrA-Net ranks first on 8 of the 10 indicators; detection confidence matches or slightly exceeds the GT oracle, indicating that frozen detectors treat our output as part of the clean-image distribution. The remaining video baselines either degrade detection counts (CG-IDN, DVD) or lose segmentation detail (MAP-Net).

\begin{table}[h]
  \centering
  \small
  \caption{\textbf{Downstream no-reference evaluation on UHV-4K}. Each method's dehazed output is evaluated by frozen YOLOv8-X and SegFormer-B5. \textit{GT} is a frozen-detector oracle excluded from ranking.}
  \label{tab:downstream_nr}
  \setlength{\tabcolsep}{3pt}
  \renewcommand{\arraystretch}{1}
  \resizebox{\textwidth}{!}{%
  \begin{tabular}{l|ccc|cccc|ccc}
    \toprule
    \multirow{2}{*}{Method} & \multicolumn{3}{c|}{Image quality} & \multicolumn{4}{c|}{YOLOv8-X detection} & \multicolumn{3}{c}{SegFormer-B5 segmentation} \\
    \cmidrule(lr){2-4} \cmidrule(lr){5-8} \cmidrule(lr){9-11}
    & NIQE$\downarrow$ & BRISQUE$\downarrow$ & MUSIQ$\uparrow$ & Det$\uparrow$ & Conf$\uparrow$ & HighConf$\uparrow$ & SmallObj$\uparrow$ & BdryE$\uparrow$ & ClassH$\uparrow$ & NumComp$\uparrow$ \\
    \midrule
    Hazy    & 5.74  & 71.74 & 33.42 & 16.4 & 0.554 & 0.520 & \textbf{1.3} & 0.0074 & 0.261 & 58.7 \\
    CG-IDN & 10.36  & 83.08 & 23.83 & 14.0 & 0.542 & 0.510 & 0.9 & 0.0078 & 0.266 & 64.0 \\
    MAP-Net & 5.58  & 58.50 & 31.92 & 14.3 & 0.545 & 0.502 & 0.6 & 0.0068 & 0.257 & 52.3 \\
    DVD     & 10.43 & 82.69 & 24.97 & 14.7 & 0.541 & 0.493 & 0.9 & 0.0079 & \textbf{0.268} & 64.5 \\
    \rowcolor{ourgreen} \textbf{Ours} \tabgain{8/10 best} & \textbf{5.04} & \textbf{50.03} & \textbf{37.45} & \textbf{17.0} & \textbf{0.558} & \textbf{0.523} & \underline{1.2} & \textbf{0.0079} & \underline{0.266} & \textbf{67.6} \\
    \midrule
    \muted{\textit{GT (oracle)}} & \muted{\textit{5.14}} & \muted{\textit{55.56}} & \muted{\textit{38.23}} & \muted{\textit{17.2}} & \muted{\textit{0.557}} & \muted{\textit{0.518}} & \muted{\textit{1.2}} & \muted{\textit{0.0080}} & \muted{\textit{0.265}} & \muted{\textit{69.0}} \\
    \bottomrule
  \end{tabular}%
  }
\end{table}

Figure~\ref{fig:downstream} visualizes a representative scene. LiBrA-Net recovers distant vehicles and low-contrast pedestrians that YOLOv8 misses on the hazy input. Segmentation boundaries track ground-truth masks more tightly than MAP-Net or DVD, particularly for thin structures whose edges are lost to haze-induced contrast collapse.

\begin{figure*}[t]
  \centering
  \setlength{\zcellW}{0.188\linewidth}%
  \setlength{\tabcolsep}{1pt}%
  \renewcommand{\arraystretch}{0}%
  \begin{tabular}{@{}c@{\hspace{3pt}}ccccc@{}}
    &
    \scriptsize Hazy &
    \scriptsize MAP-Net &
    \scriptsize DVD &
    \scriptsize\textbf{Ours} &
    \scriptsize\textit{GT}
    \\[1.5pt]
    \rotatebox{90}{\scriptsize\textbf{(a) Det.}} &
      \zcell{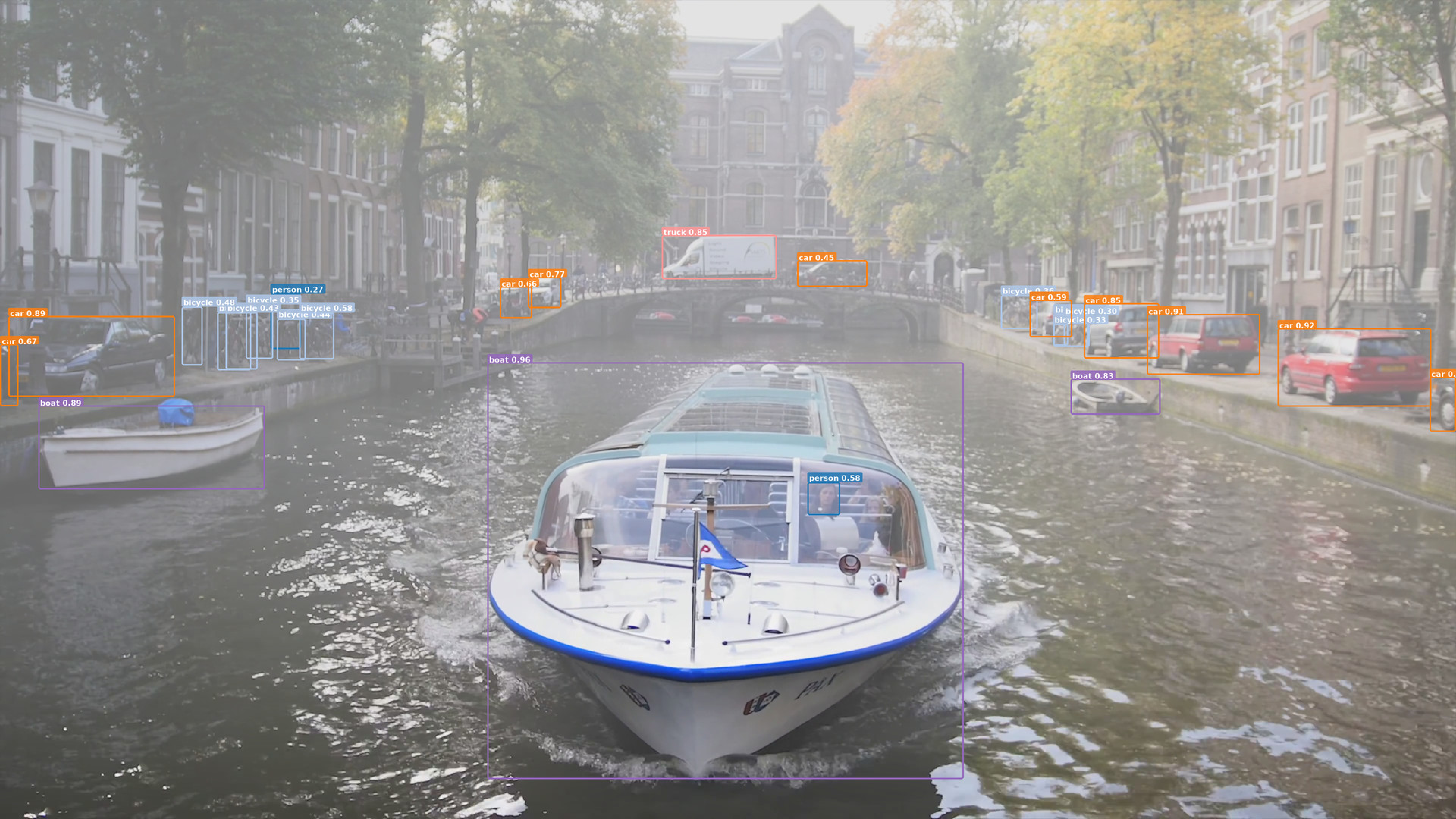}%
            {1152 1188 1152 432}{0.3}{0.55}{0.7}{0.8} &
      \zcell{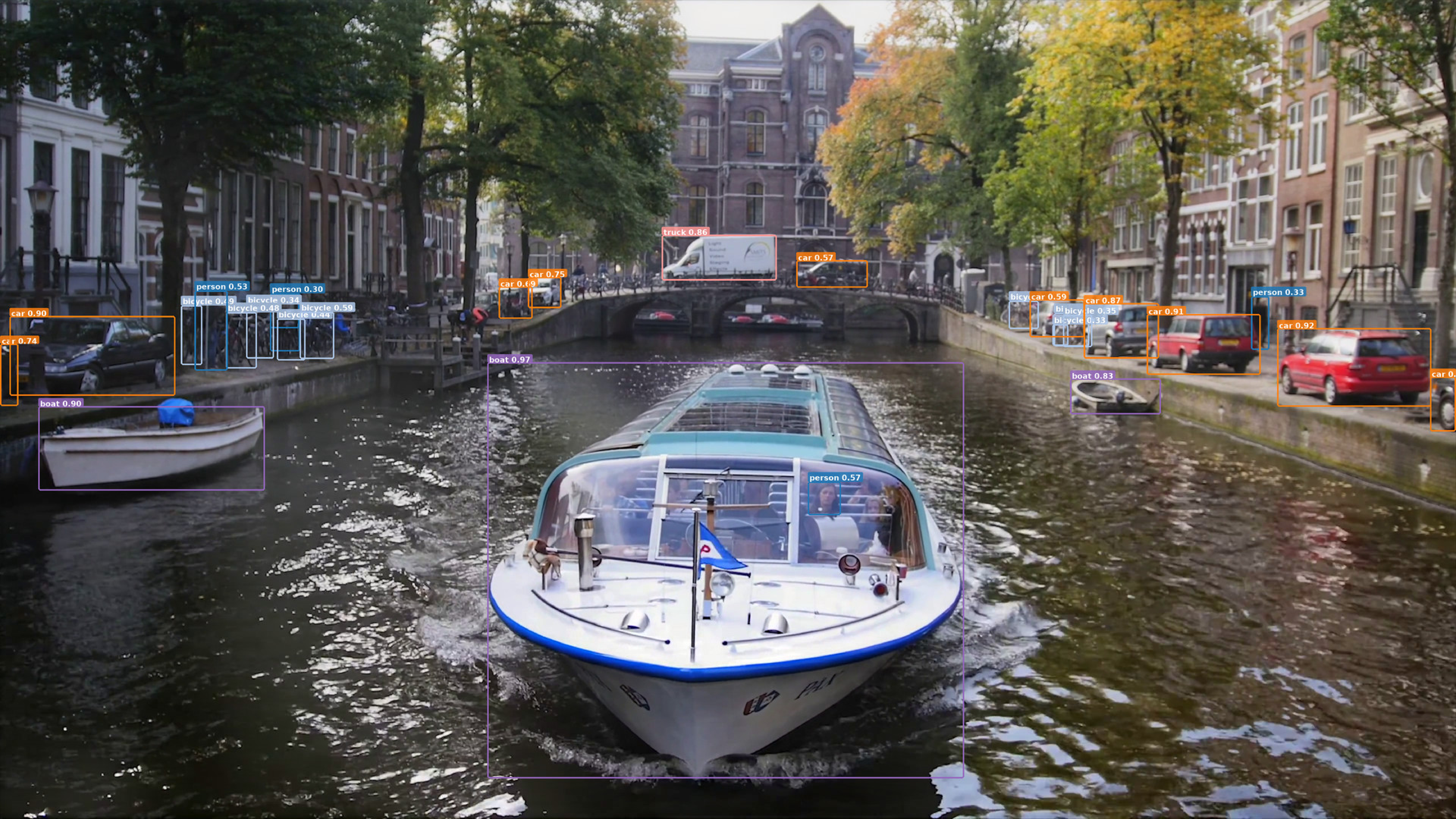}%
            {1152 1188 1152 432}{0.3}{0.55}{0.7}{0.8} &
      \zcell{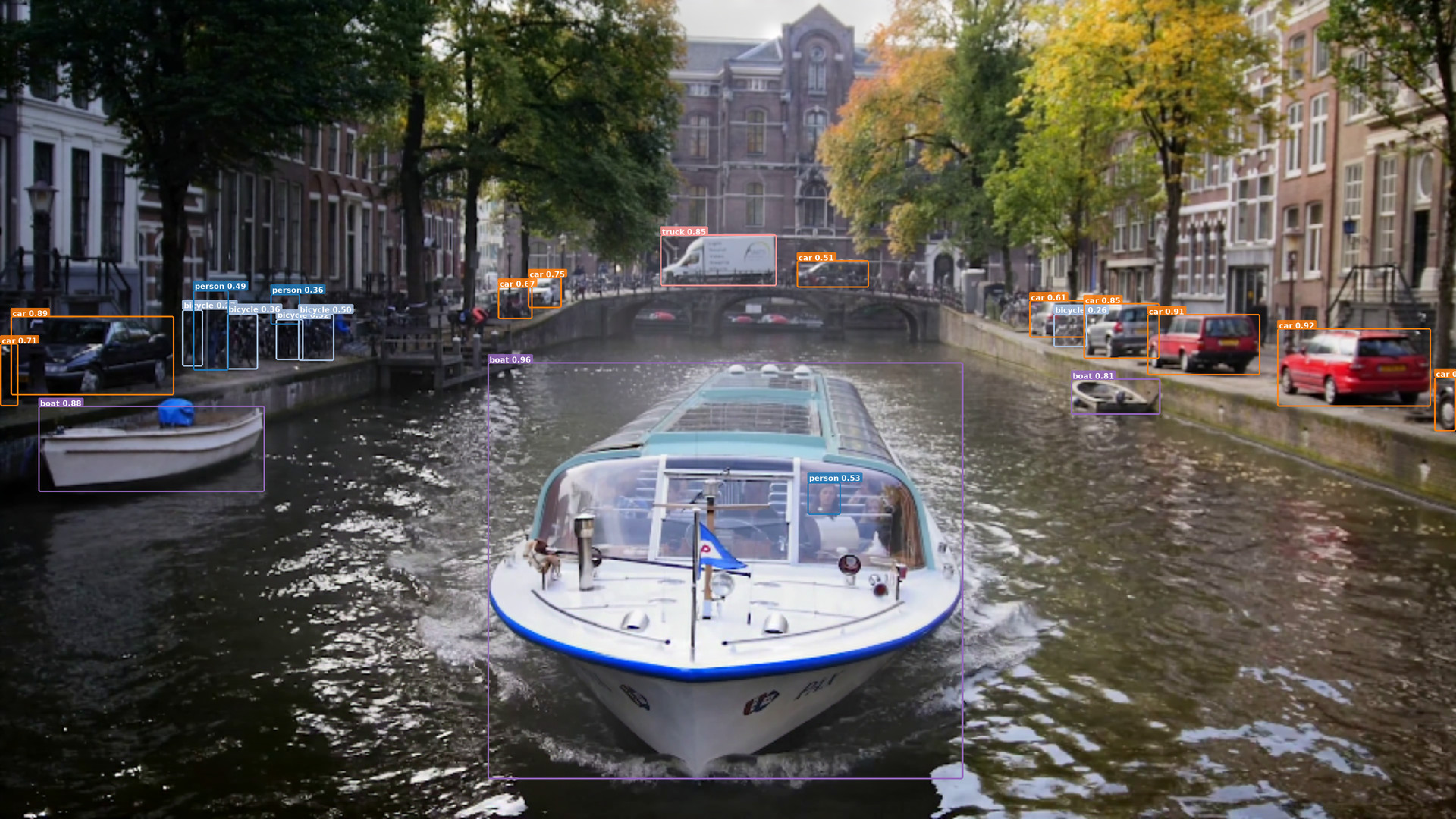}%
            {1152 1188 1152 432}{0.3}{0.55}{0.7}{0.8} &
      \zcell{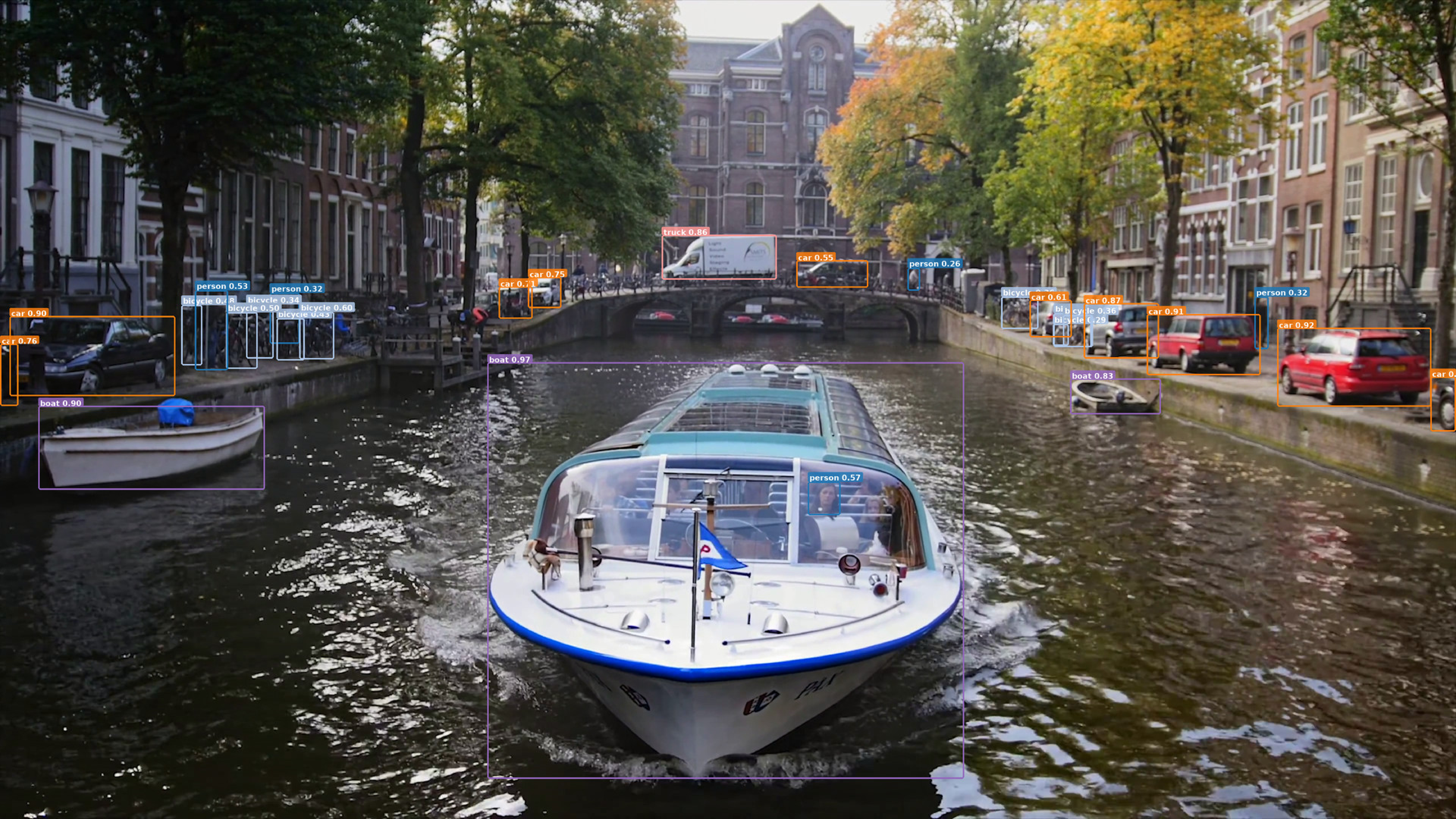}%
            {1152 1188 1152 432}{0.3}{0.55}{0.7}{0.8} &
      \zcell{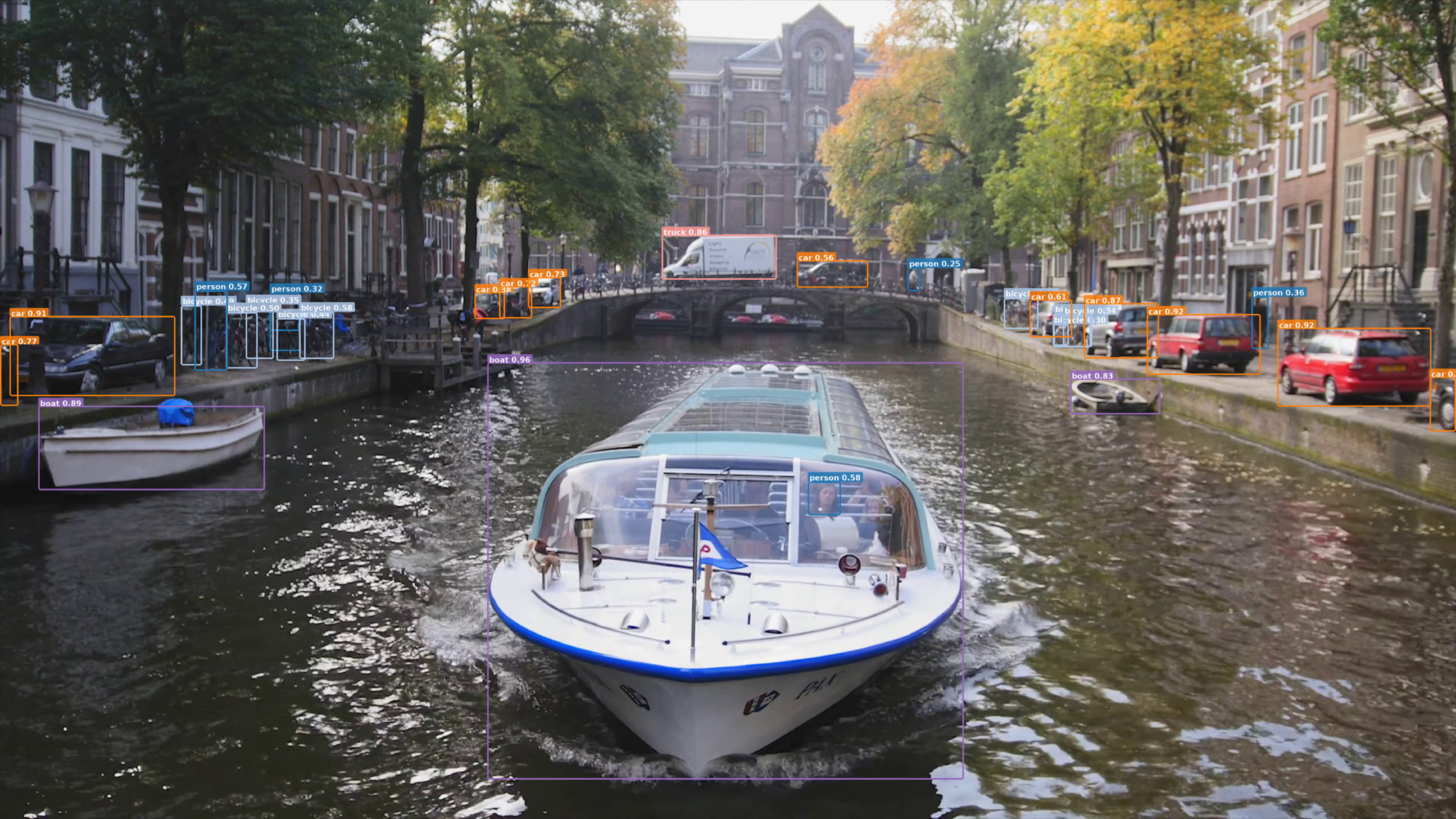}%
            {1152 1188 1152 432}{0.3}{0.55}{0.7}{0.8}
    \\[0.5pt]
    \rotatebox{90}{\scriptsize\textbf{(b) Seg.}} &
      \zcell{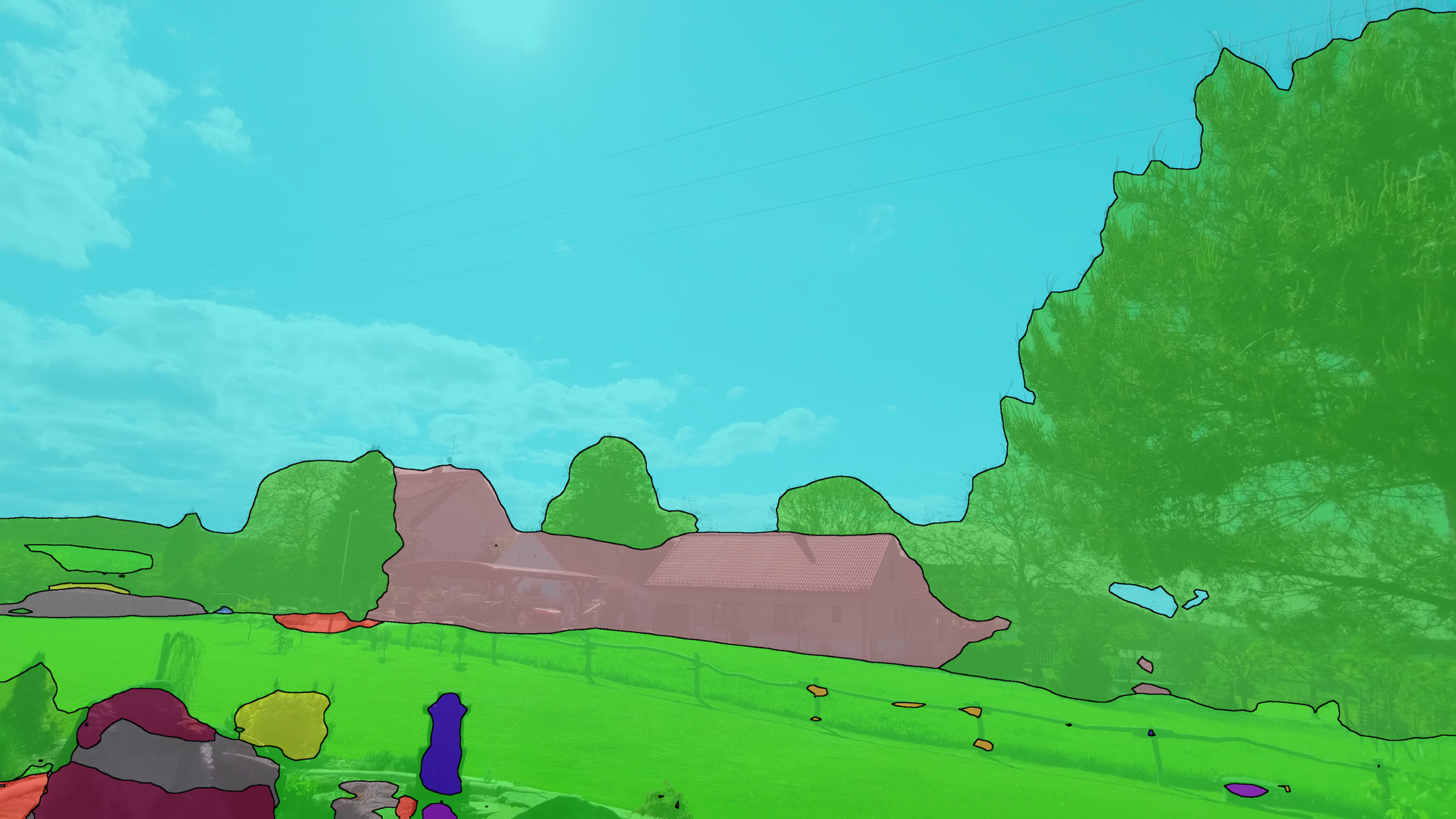}%
            {0 0 2304 1188}{0.0}{0.0}{0.4}{0.5} &
      \zcell{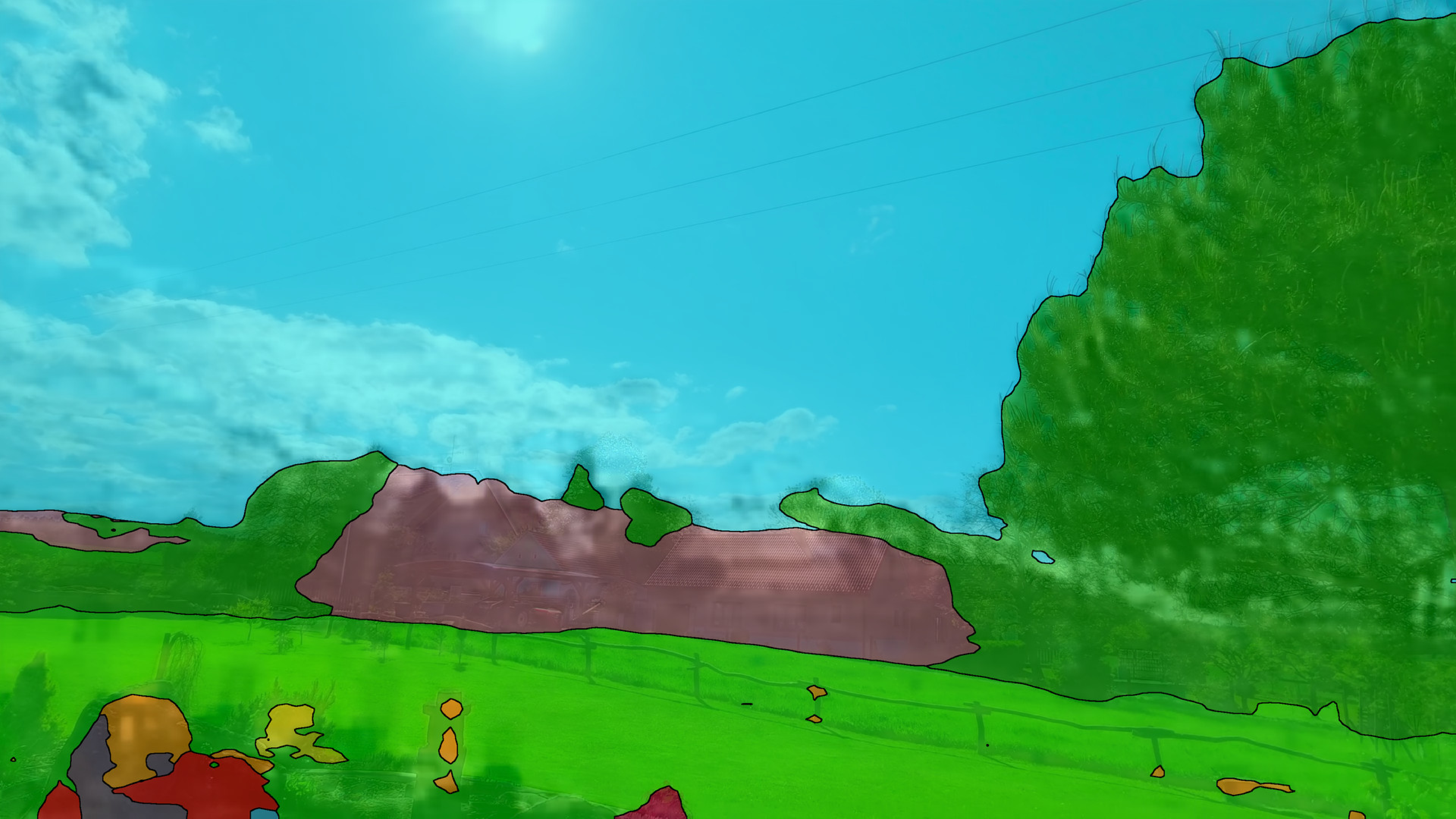}%
            {0 0 2304 1188}{0.0}{0.0}{0.4}{0.5} &
      \zcell{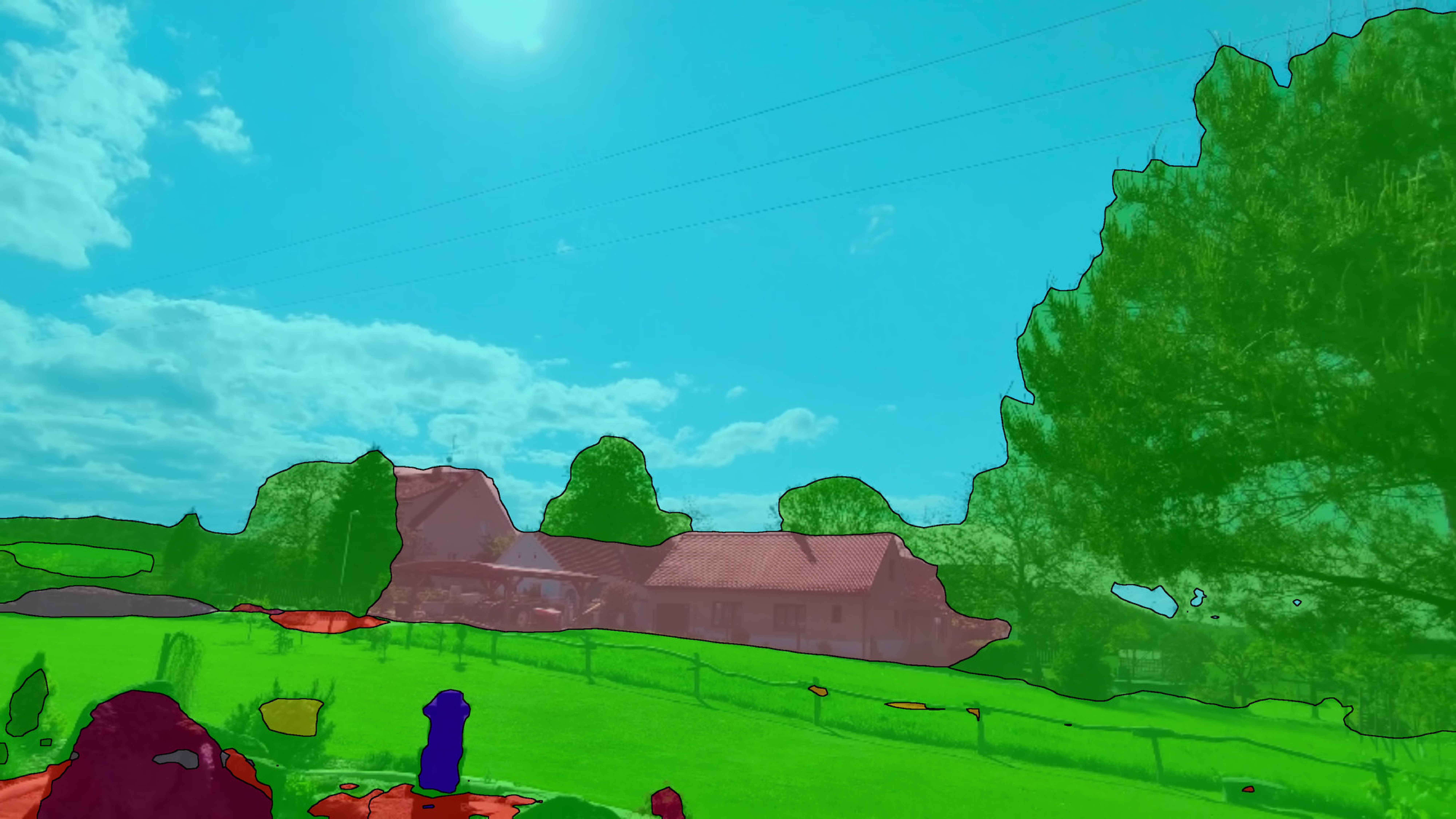}%
            {0 0 2304 1188}{0.0}{0.0}{0.4}{0.5} &
      \zcell{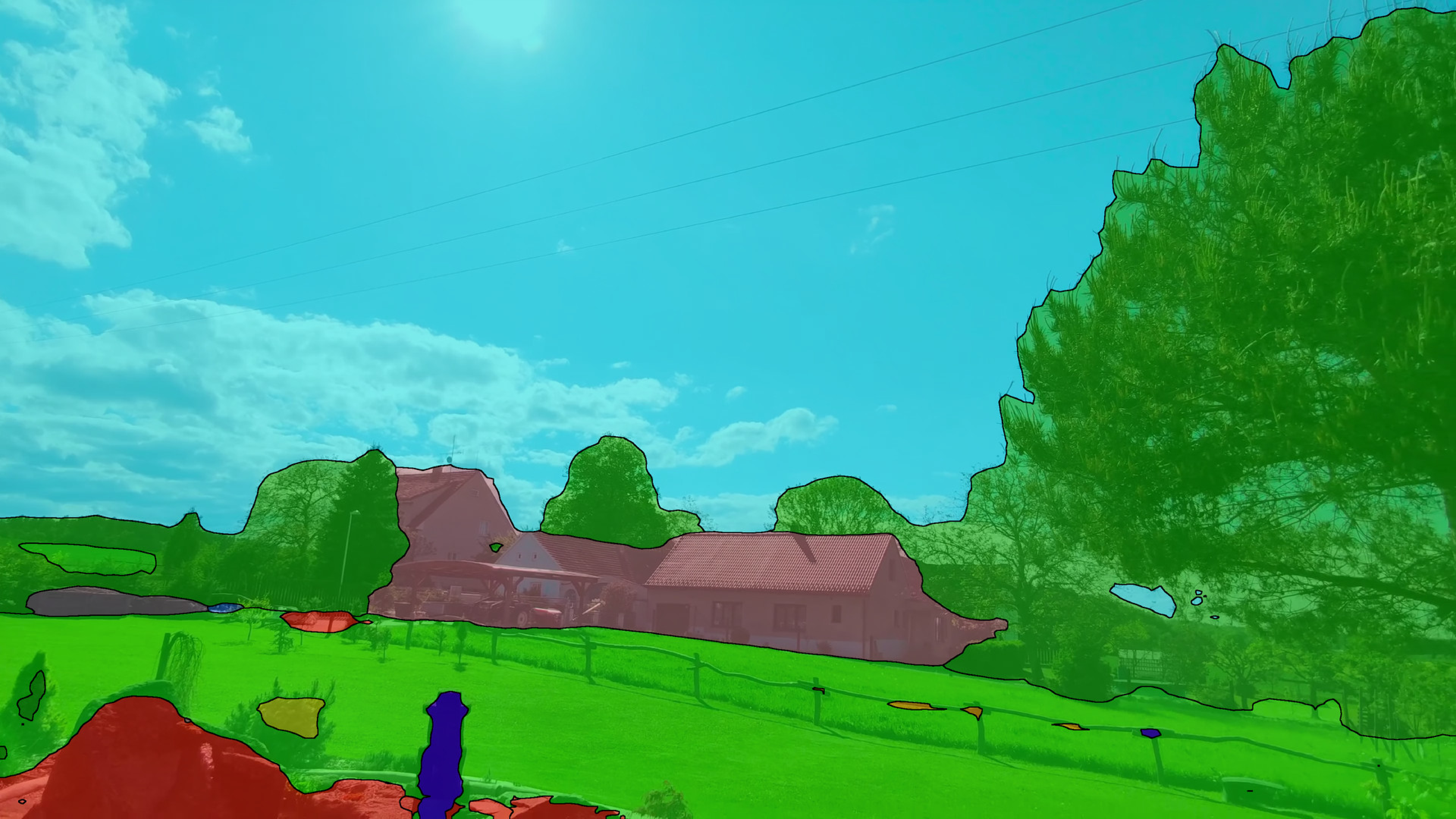}%
            {0 0 2304 1188}{0.0}{0.0}{0.4}{0.5} &
      \zcell{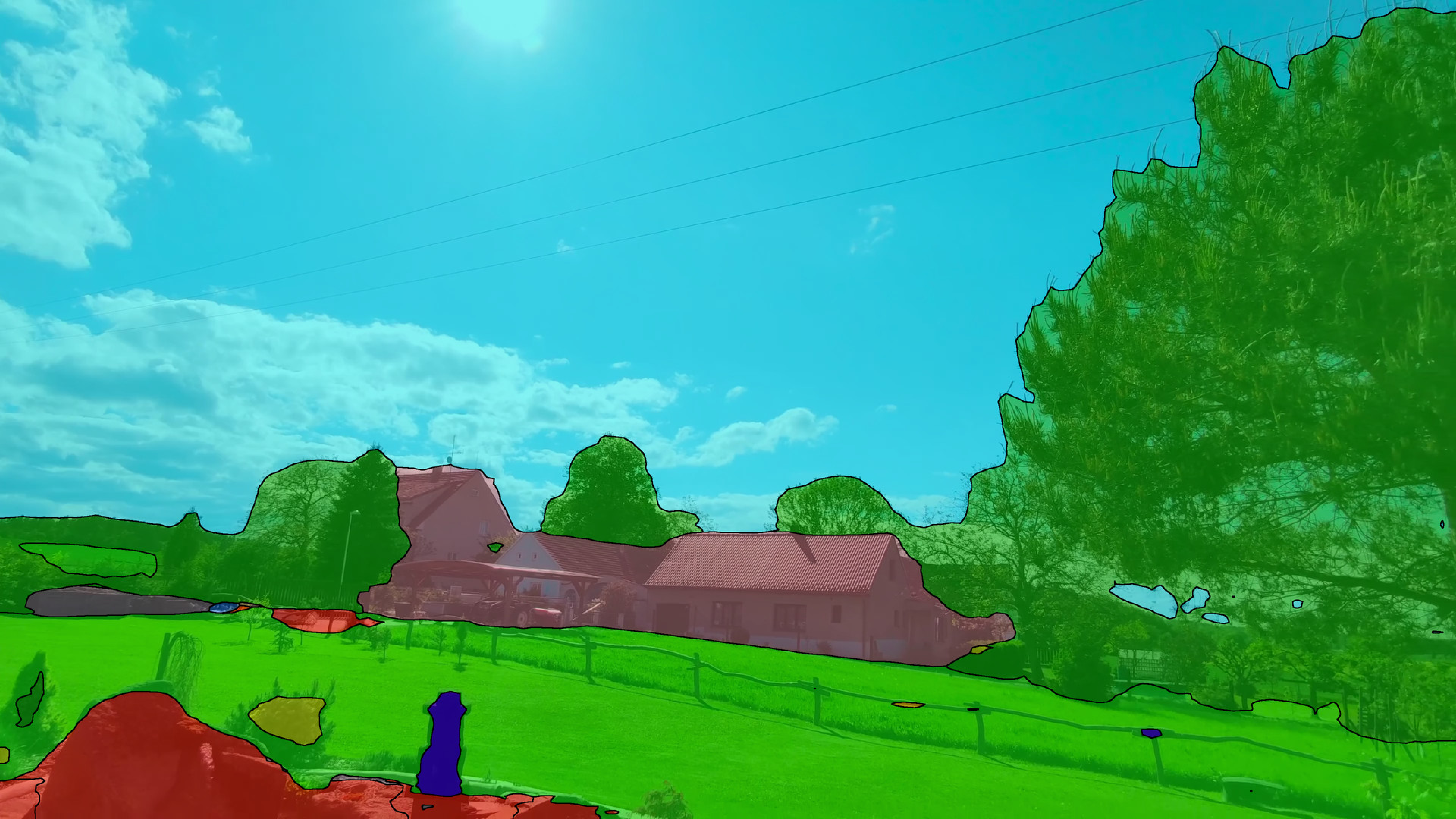}%
            {0 0 2304 1188}{0.0}{0.0}{0.4}{0.45}
    \\
  \end{tabular}
  \caption{%
    \textbf{Downstream perception on a UHV-4K test scene.}
    \textbf{(a)}~Object detection (YOLOv8-X, COCO-pretrained).
    \textbf{(b)}~Semantic segmentation (SegFormer-B5, ADE20K-pretrained).
    Both models run at default settings on each method's dehazed output; any performance gap comes solely from the upstream dehazer.%
  }
  \label{fig:downstream}
\end{figure*}

\subsection{Real-World No-Reference Quality and Qualitative Results}
\label{app:realworld_qual}

Figure~\ref{fig:realworld_nriqa} summarizes the no-reference image quality assessment across all five video methods on eight real-world 4K hazy videos. LiBrA-Net is the only method that improves MUSIQ, BRISQUE, and NIQE simultaneously over the hazy input. MAP-Net moves marginally; CG-IDN, DVD, and ViWS-Net all degrade NIQE and BRISQUE, penalized for halos, ringing, and chromatic overshoot that their full-resolution temporal decoders inject when the atmospheric statistics drift from the training distribution.

\begin{figure}[t]
  \centering
  \begin{subfigure}[t]{0.3\linewidth}
    \includegraphics[width=\linewidth]{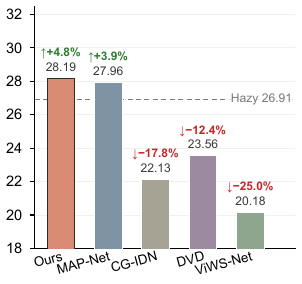}
    \caption{MUSIQ}
    \label{fig:realworld_musiq}
  \end{subfigure}\hfill
  \begin{subfigure}[t]{0.3\linewidth}
    \includegraphics[width=\linewidth]{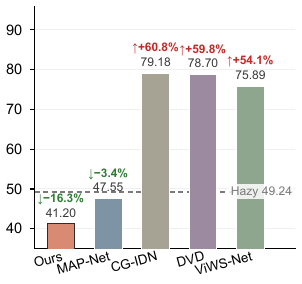}
    \caption{BRISQUE}
    \label{fig:realworld_brisque}
  \end{subfigure}\hfill
  \begin{subfigure}[t]{0.3\linewidth}
    \includegraphics[width=\linewidth]{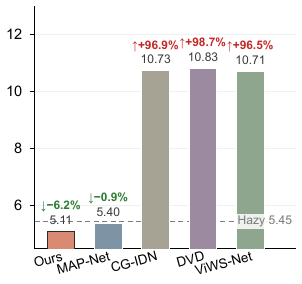}
    \caption{NIQE}
    \label{fig:realworld_niqe}
  \end{subfigure}
  \caption{\textbf{No-reference image quality on real-world 4K hazy videos.}
    Bars show each video method's score relative to the \emph{Hazy} input (dashed line). LiBrA-Net is the only method that improves all three metrics.}
  \label{fig:realworld_nriqa}
\end{figure}

Figure~\ref{fig:realworld_qualitative} presents frame-by-frame comparisons on two real-world 4K hazy videos (no ground truth available). The first scene (lake and skyline) tests fine-structure preservation: only LiBrA-Net keeps the thin members of the lattice transmission tower intact, while CG-IDN, DVD, and ViWS-Net wash them into the foliage or surround them with halo bands. The second scene (Venice street) tests color fidelity: CG-IDN, DVD, and ViWS-Net introduce saturated color casts on the awning and signage, MAP-Net leaves residual haze on the right-hand fa\c{c}ade, and LiBrA-Net recovers the strongest contrast without color shift. Across both scenes, the temporal stability of LiBrA-Net's output is visually apparent from the consistent appearance across eight consecutive frames.

\begin{figure}[t]
  \centering
  \setlength{\zcellW}{0.108\linewidth}%
  \setlength{\tabcolsep}{0.6pt}%
  \renewcommand{\arraystretch}{0.35}%
  \fboxsep=0pt%
  \begin{tabular}{@{}c@{\hspace{1pt}}cccccccc@{}}
     & \scriptsize Frame $t$ & \scriptsize Frame $t{+}1$ & \scriptsize Frame $t{+}2$ & \scriptsize Frame $t{+}3$ & \scriptsize Frame $t{+}4$ & \scriptsize Frame $t{+}5$ & \scriptsize Frame $t{+}6$ & \scriptsize Frame $t{+}7$ \\[0.5pt]
    \rotatebox{90}{\scriptsize\textbf{Video\,1}} &
      \fullroi{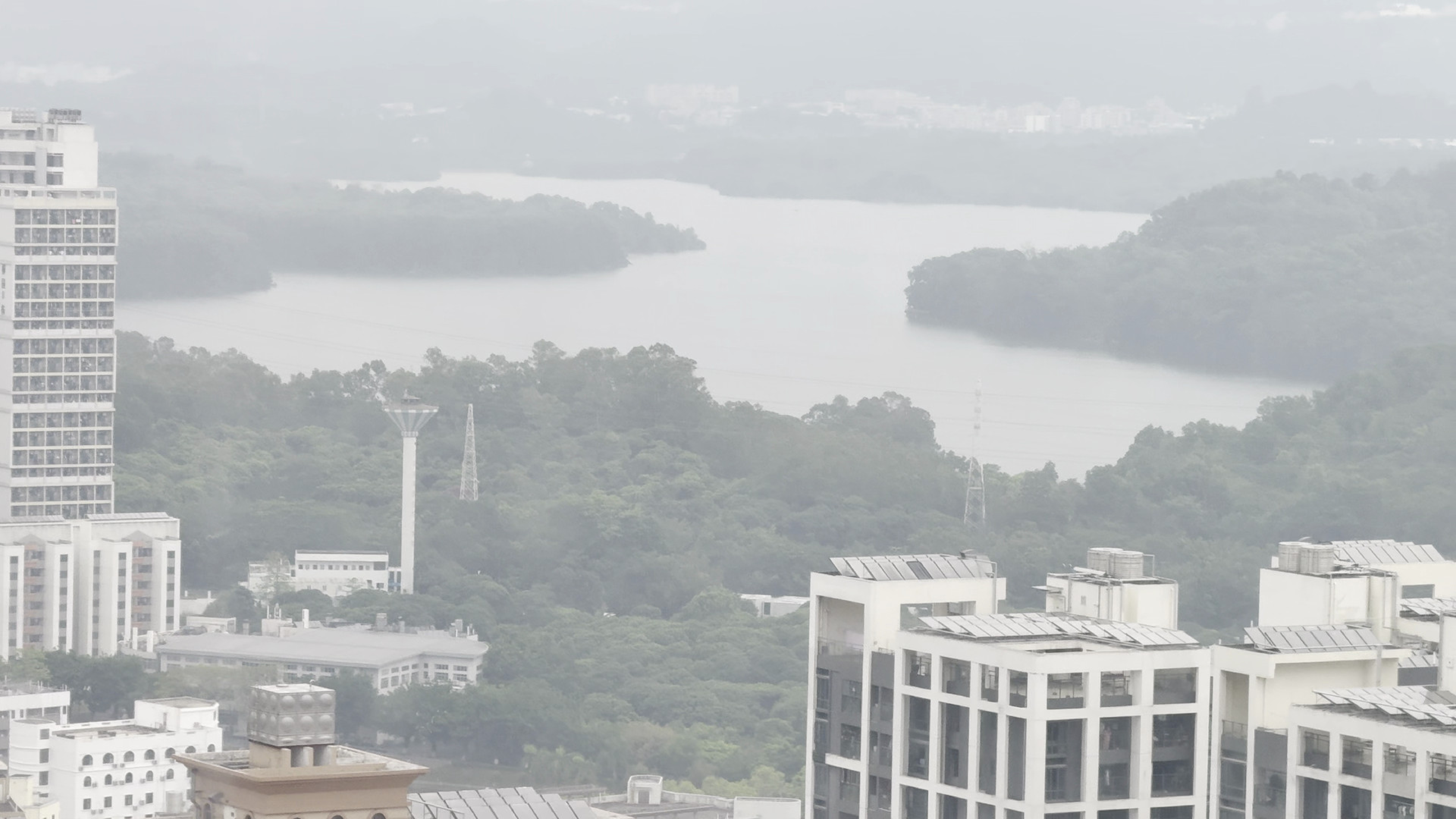}{0.02}{0.05}{0.36}{0.43} &
      \fullroi{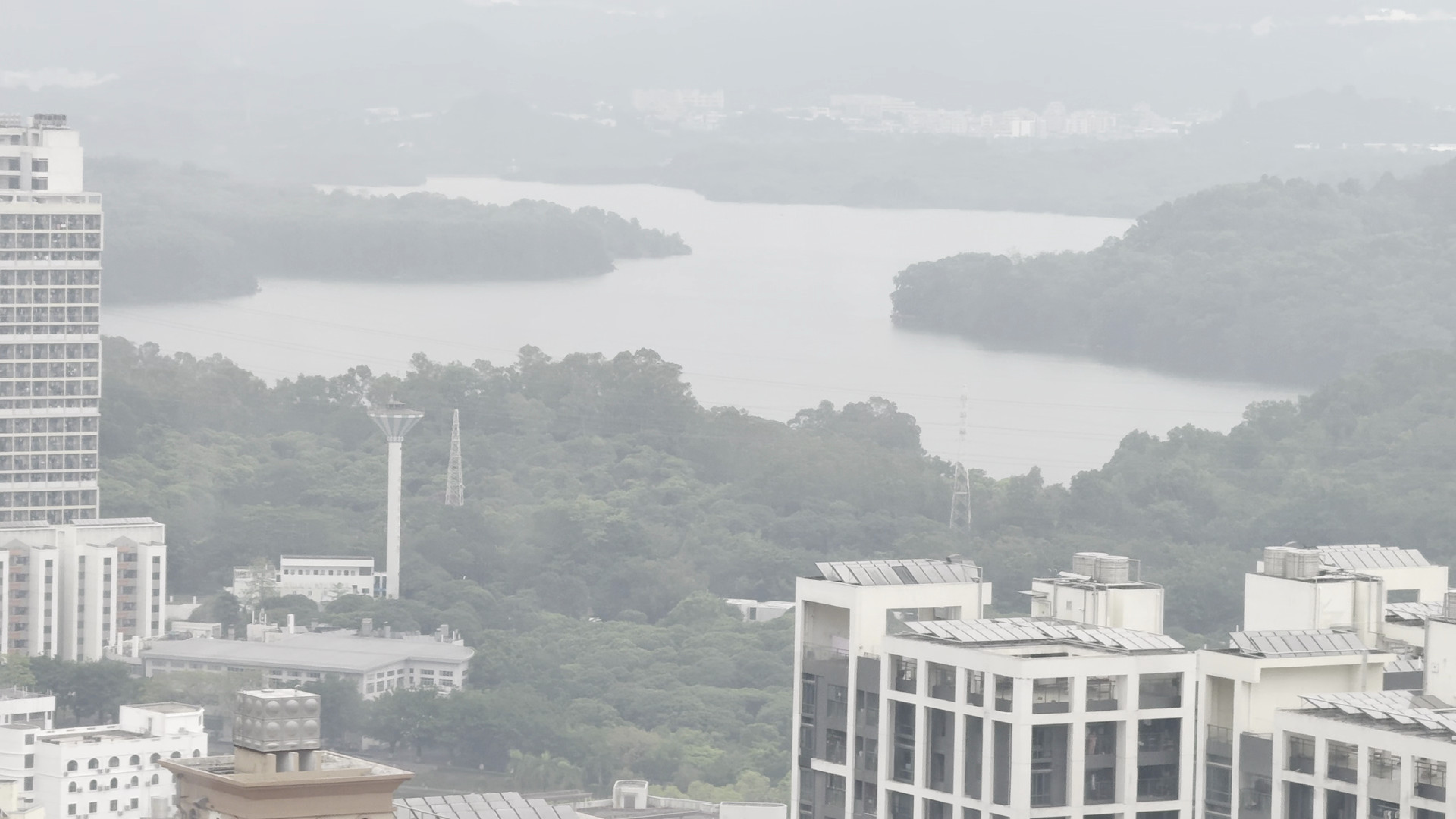}{0.02}{0.05}{0.36}{0.43} &
      \fullroi{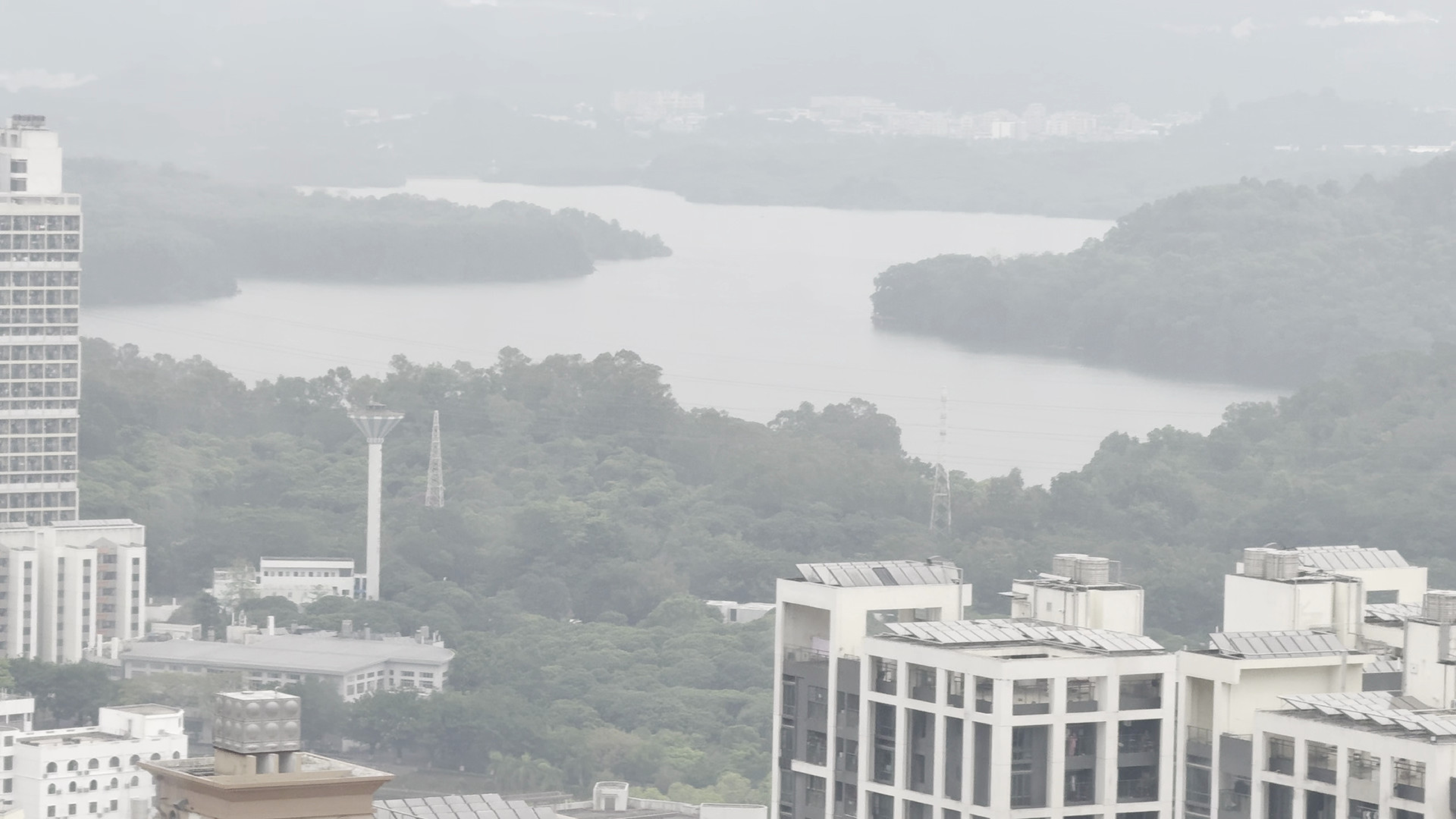}{0.02}{0.05}{0.36}{0.43} &
      \fullroi{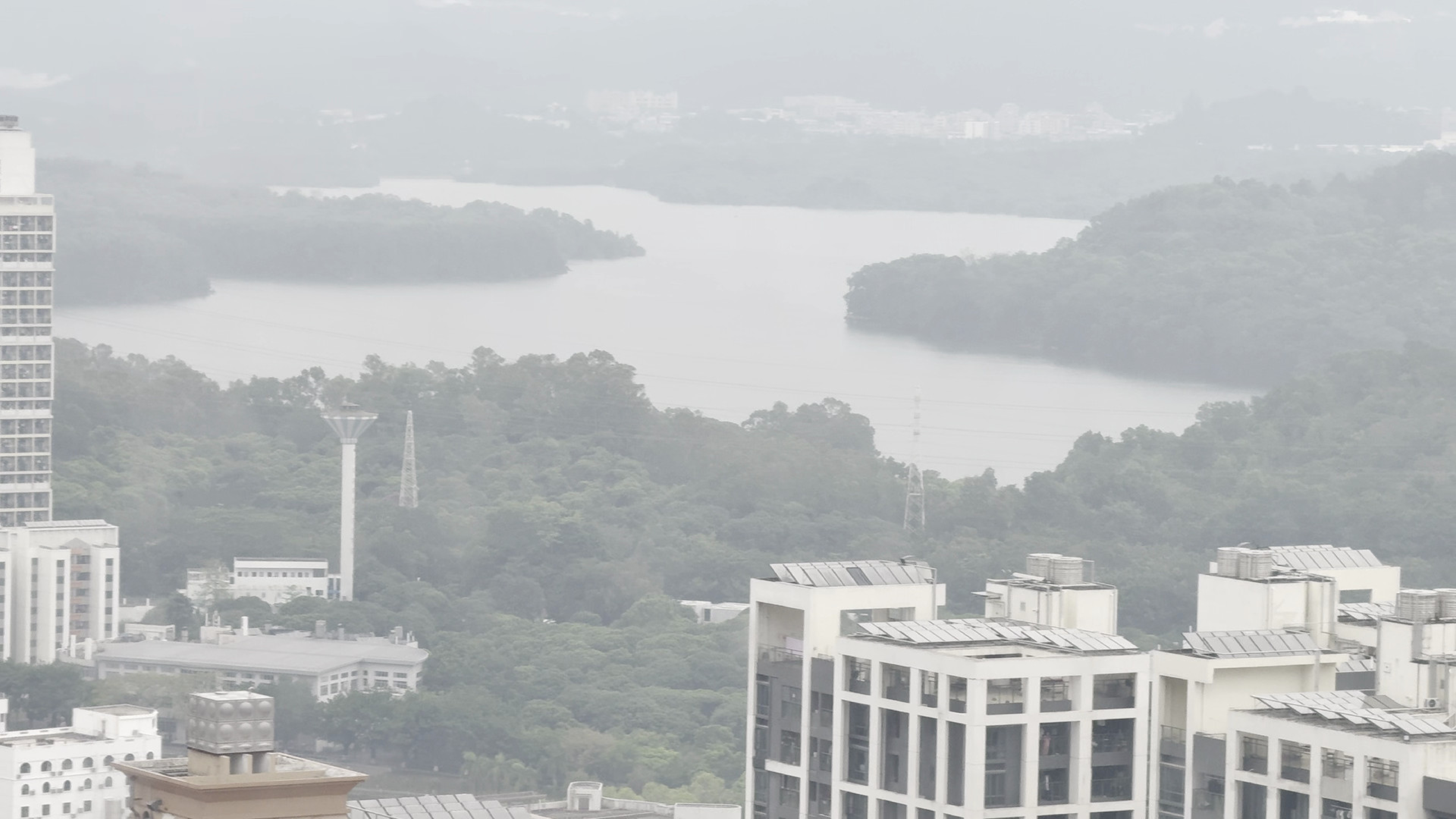}{0.02}{0.05}{0.36}{0.43} &
      \fullroi{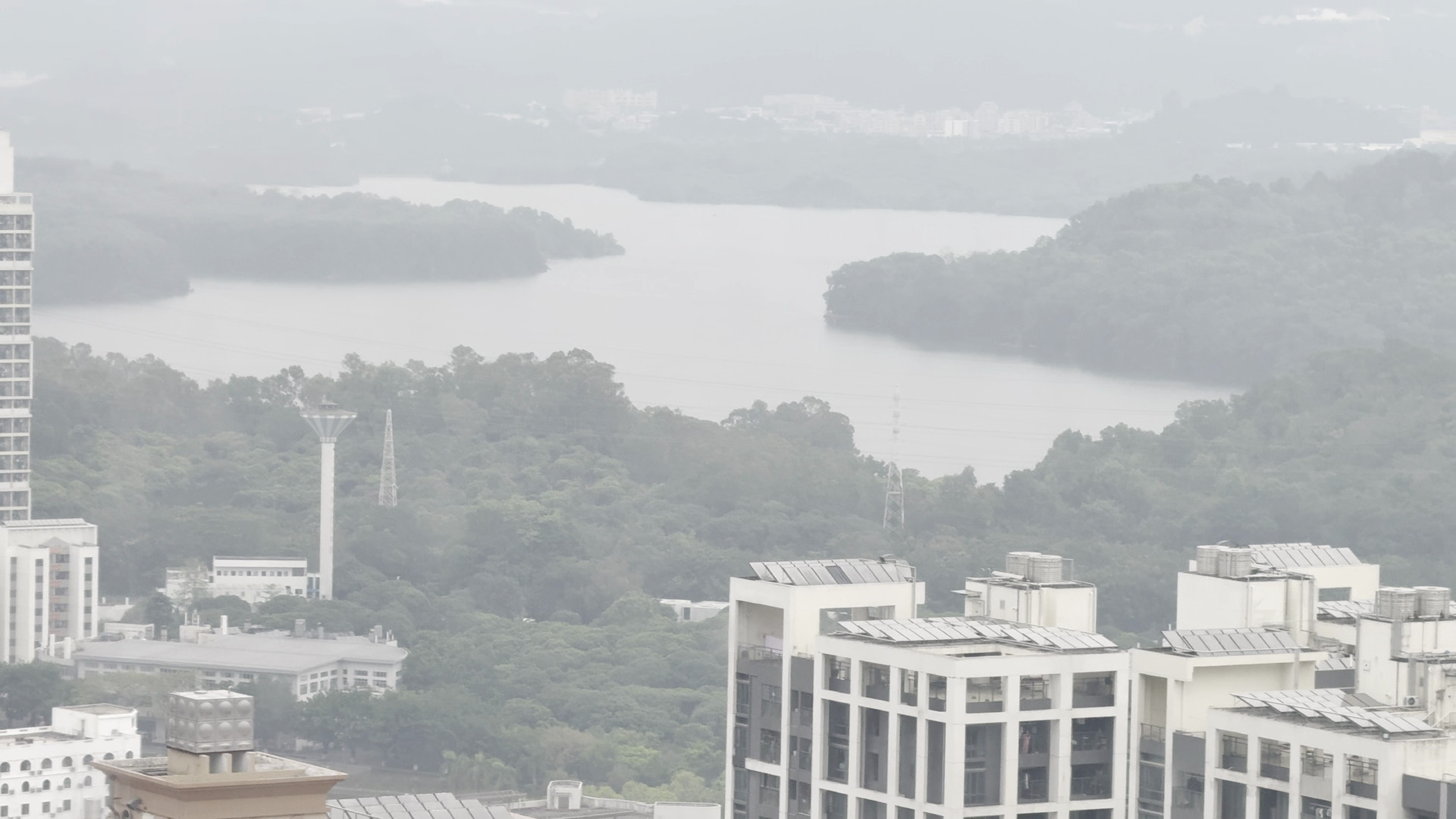}{0.02}{0.05}{0.36}{0.43} &
      \fullroi{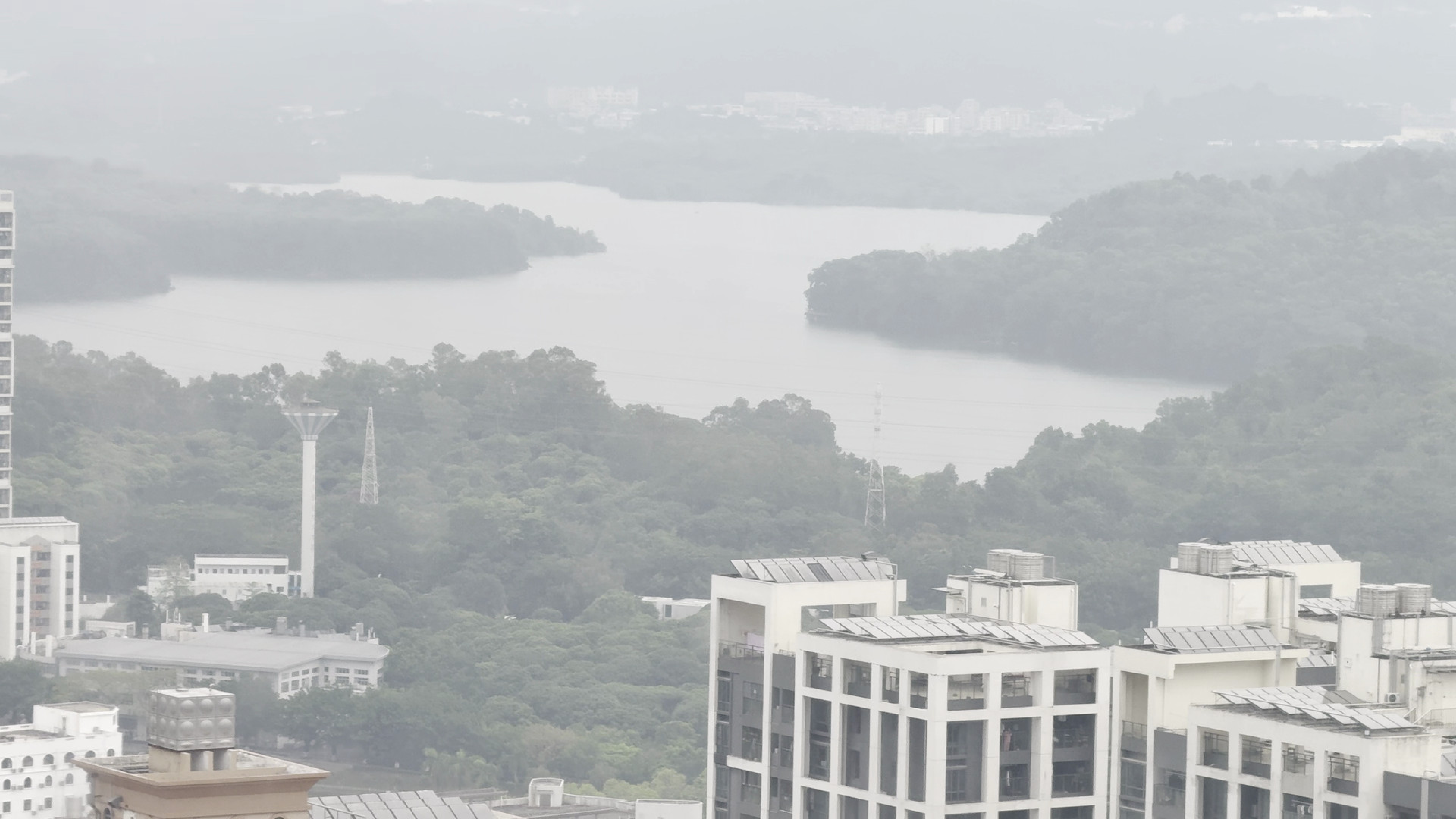}{0.02}{0.05}{0.36}{0.43} &
      \fullroi{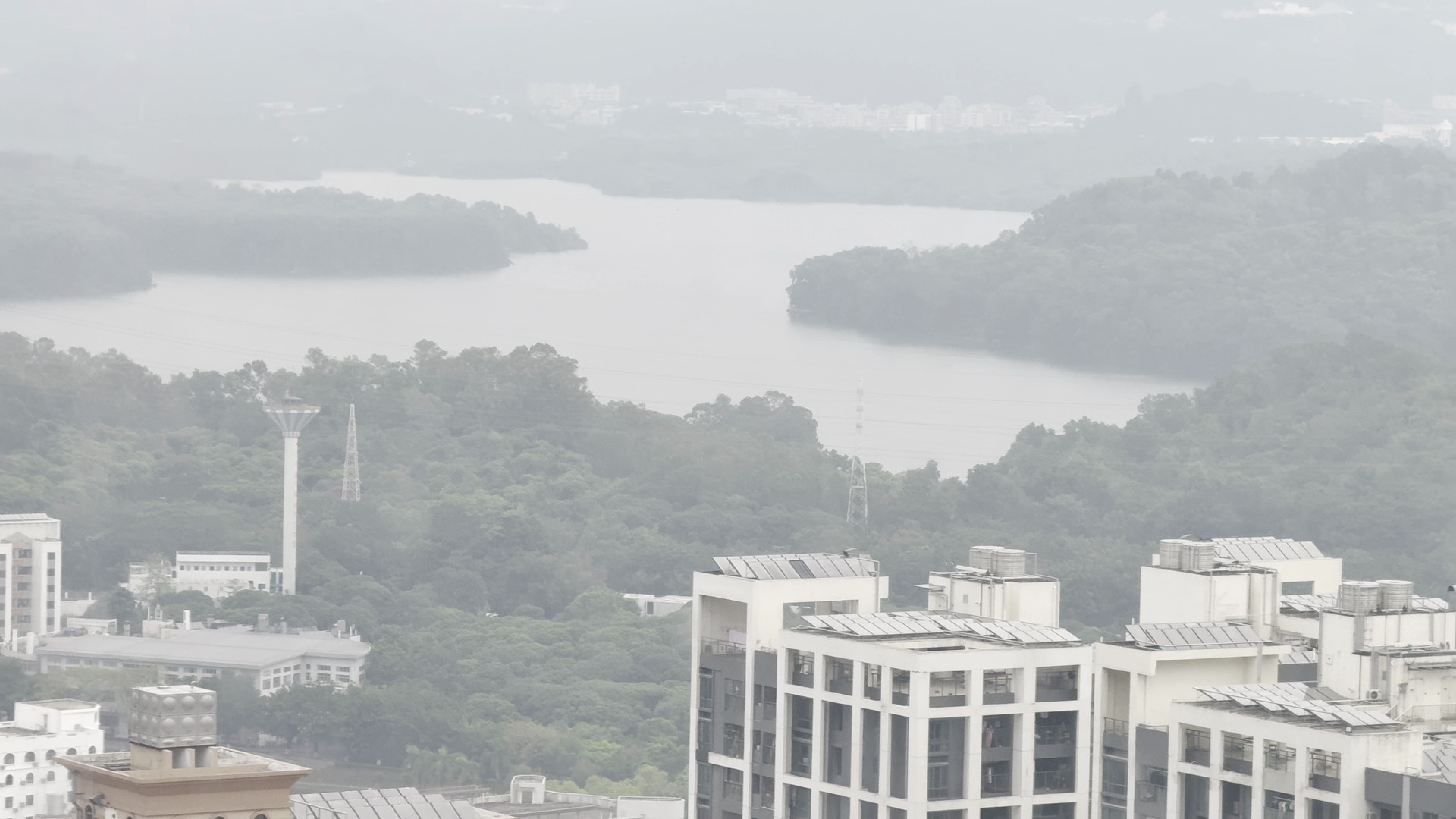}{0.02}{0.05}{0.36}{0.43} &
      \fullroi{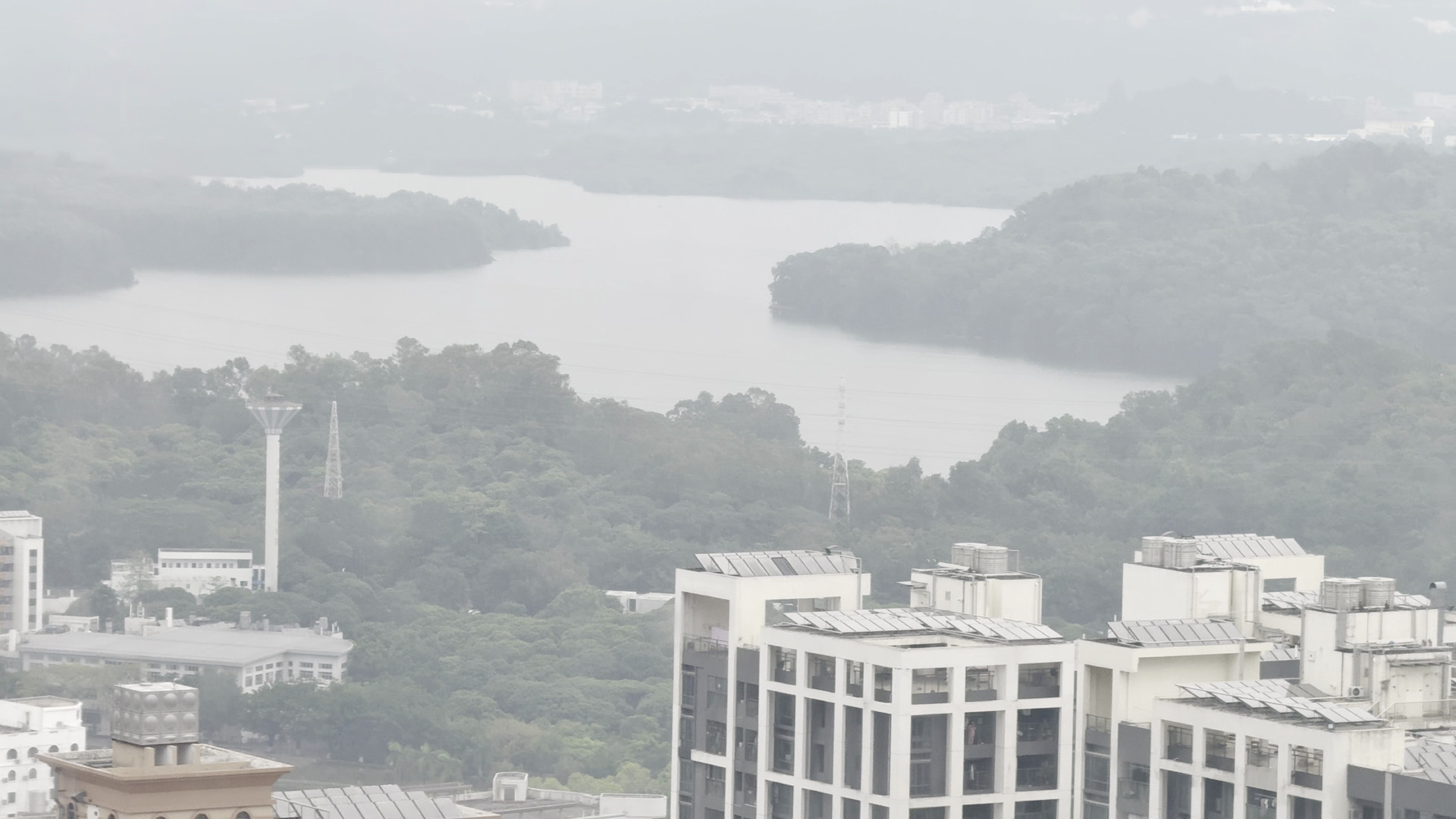}{0.02}{0.05}{0.36}{0.43} \\[0pt]
    \rotatebox{0}{\scriptsize\textbf{(a)}} &
      \zoomcrop{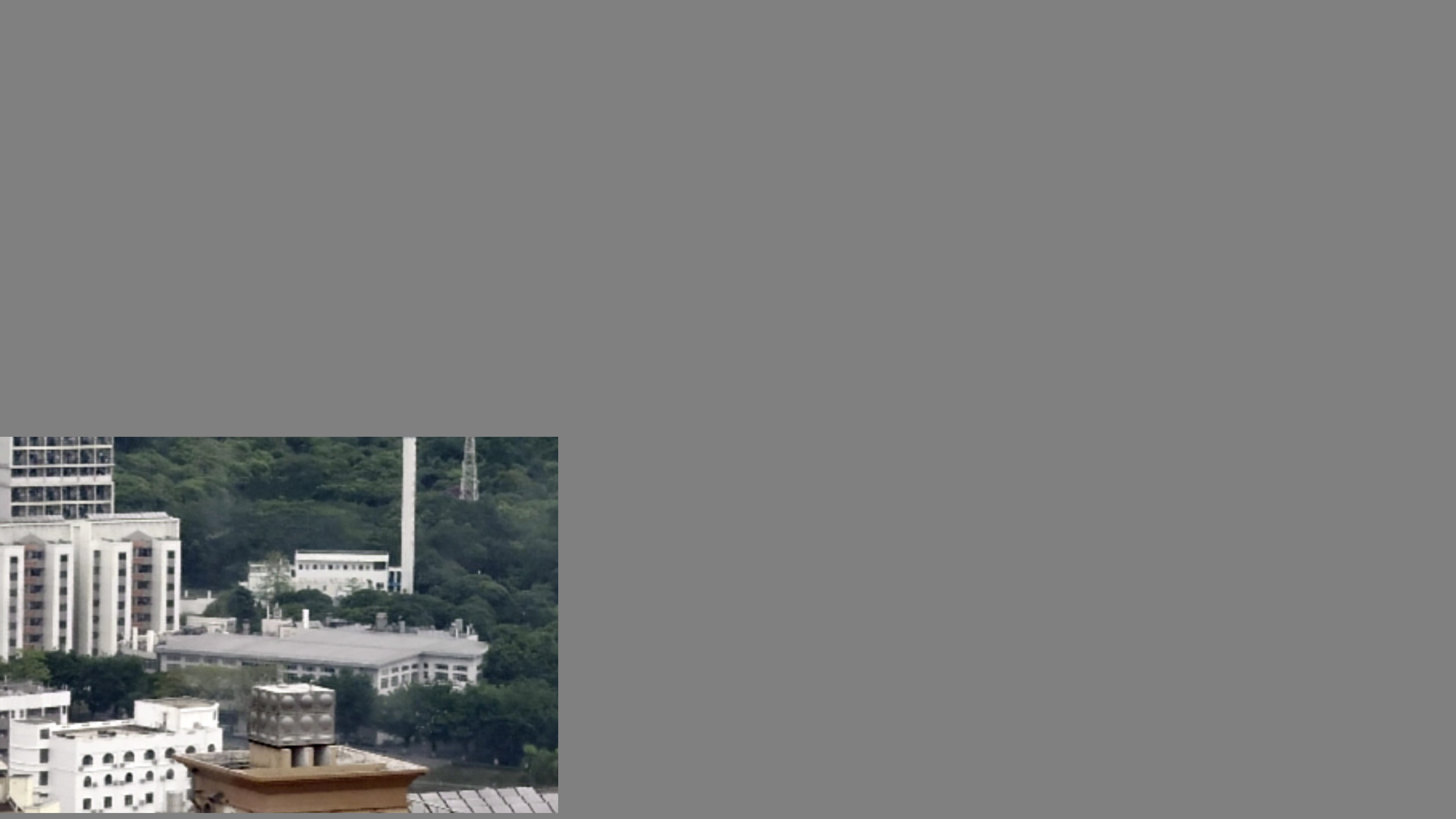}{77 108 2458 1231} &
      \zoomcrop{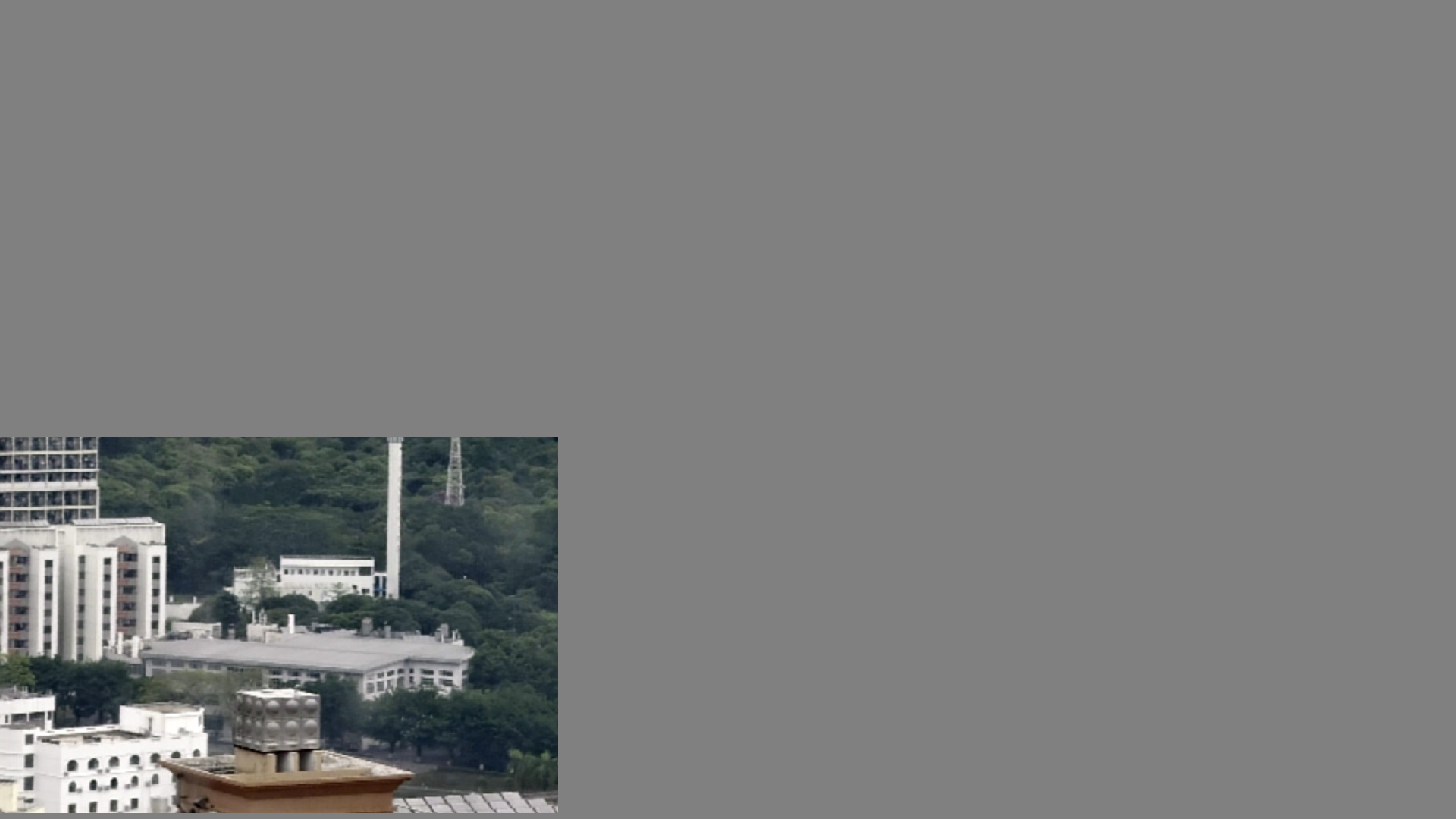}{77 108 2458 1231} &
      \zoomcrop{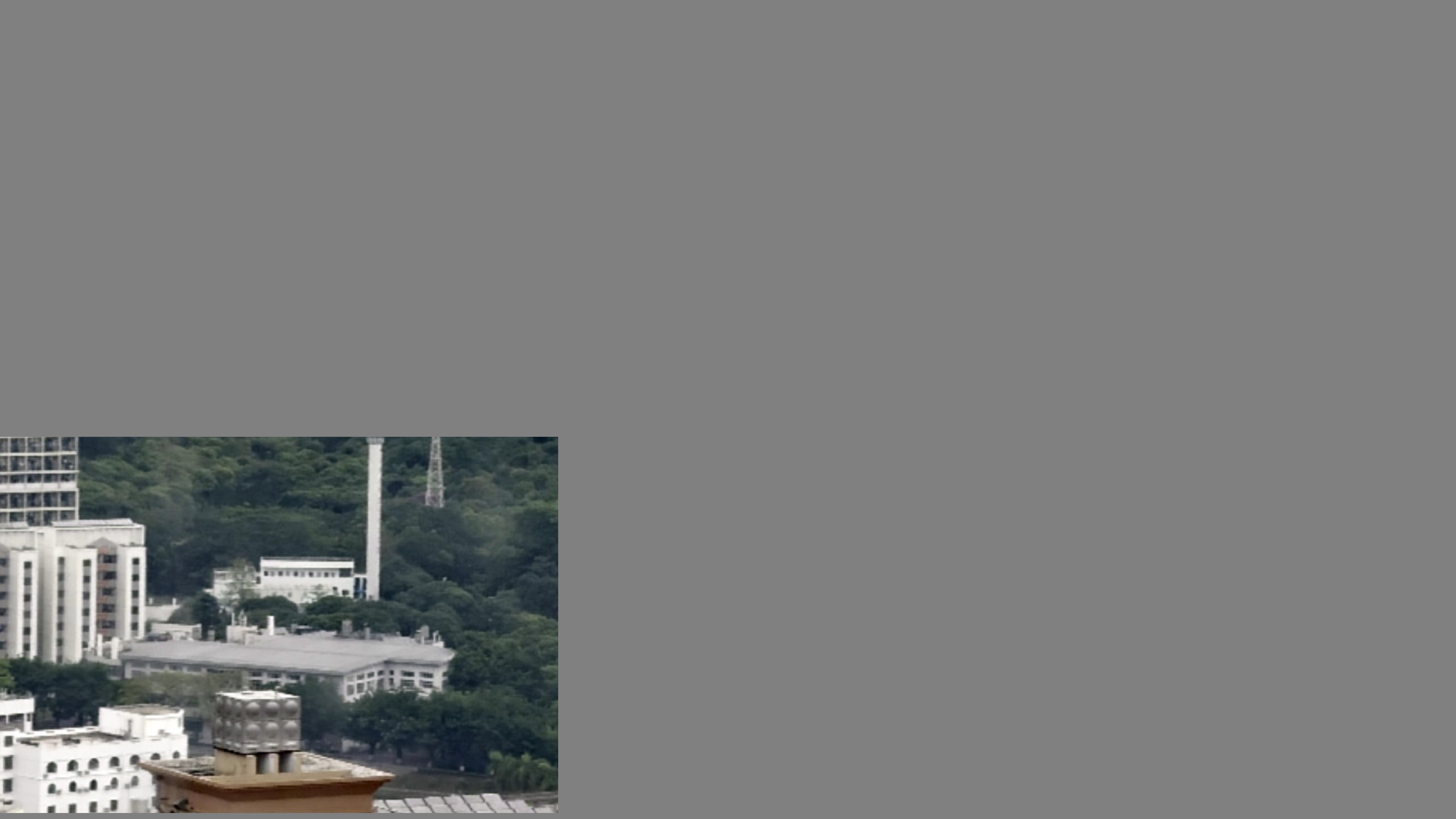}{77 108 2458 1231} &
      \zoomcrop{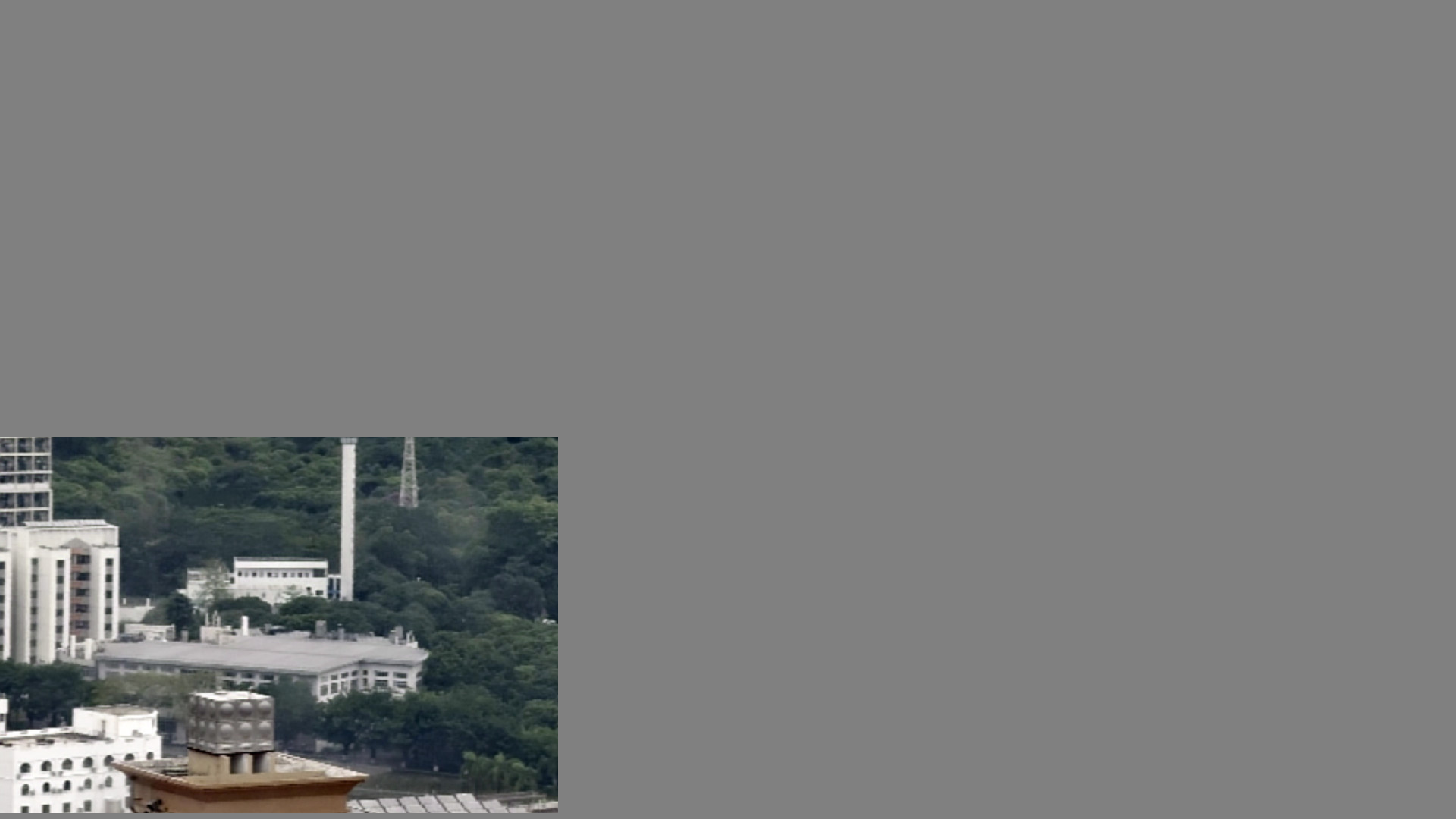}{77 108 2458 1231} &
      \zoomcrop{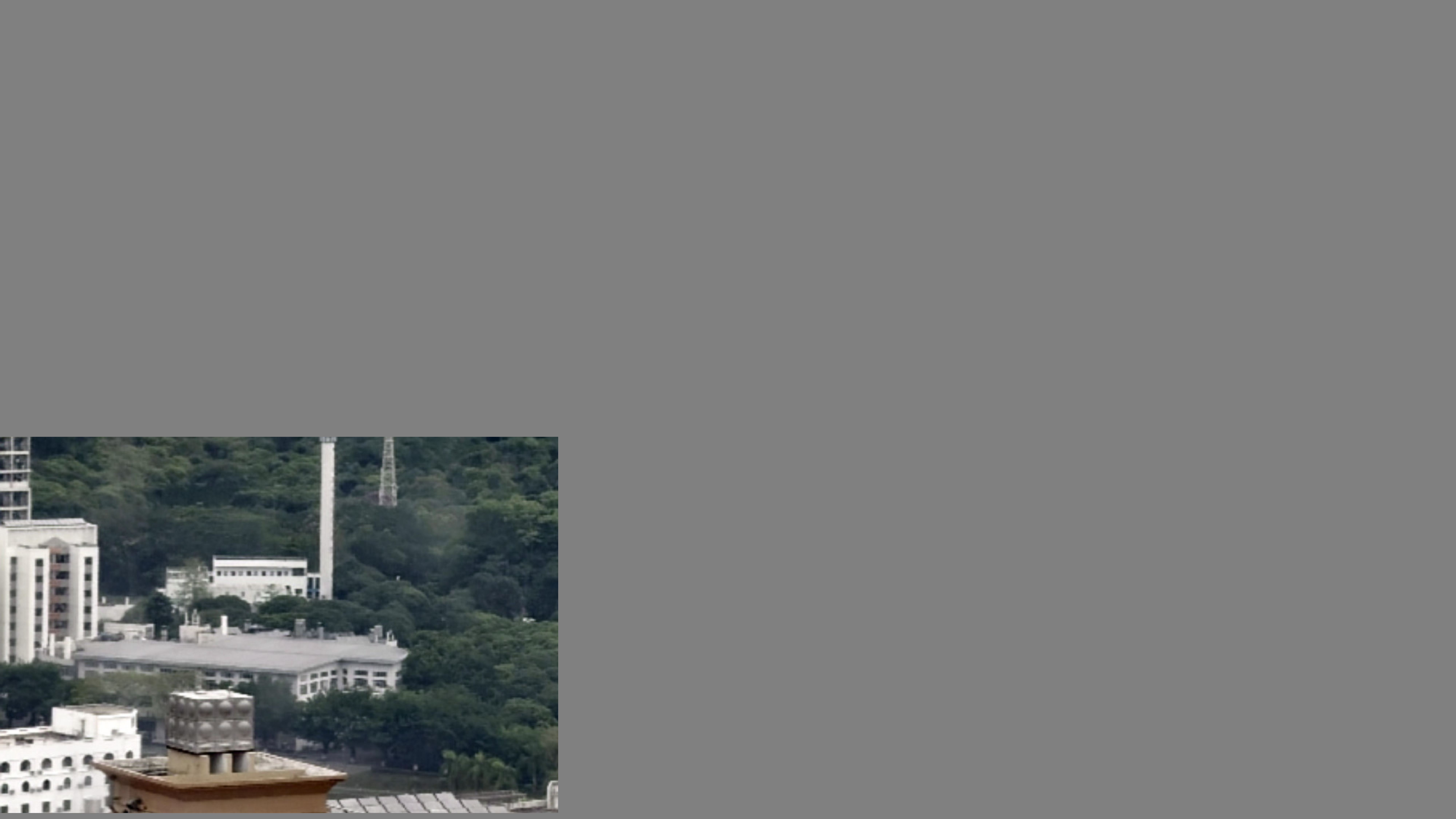}{77 108 2458 1231} &
      \zoomcrop{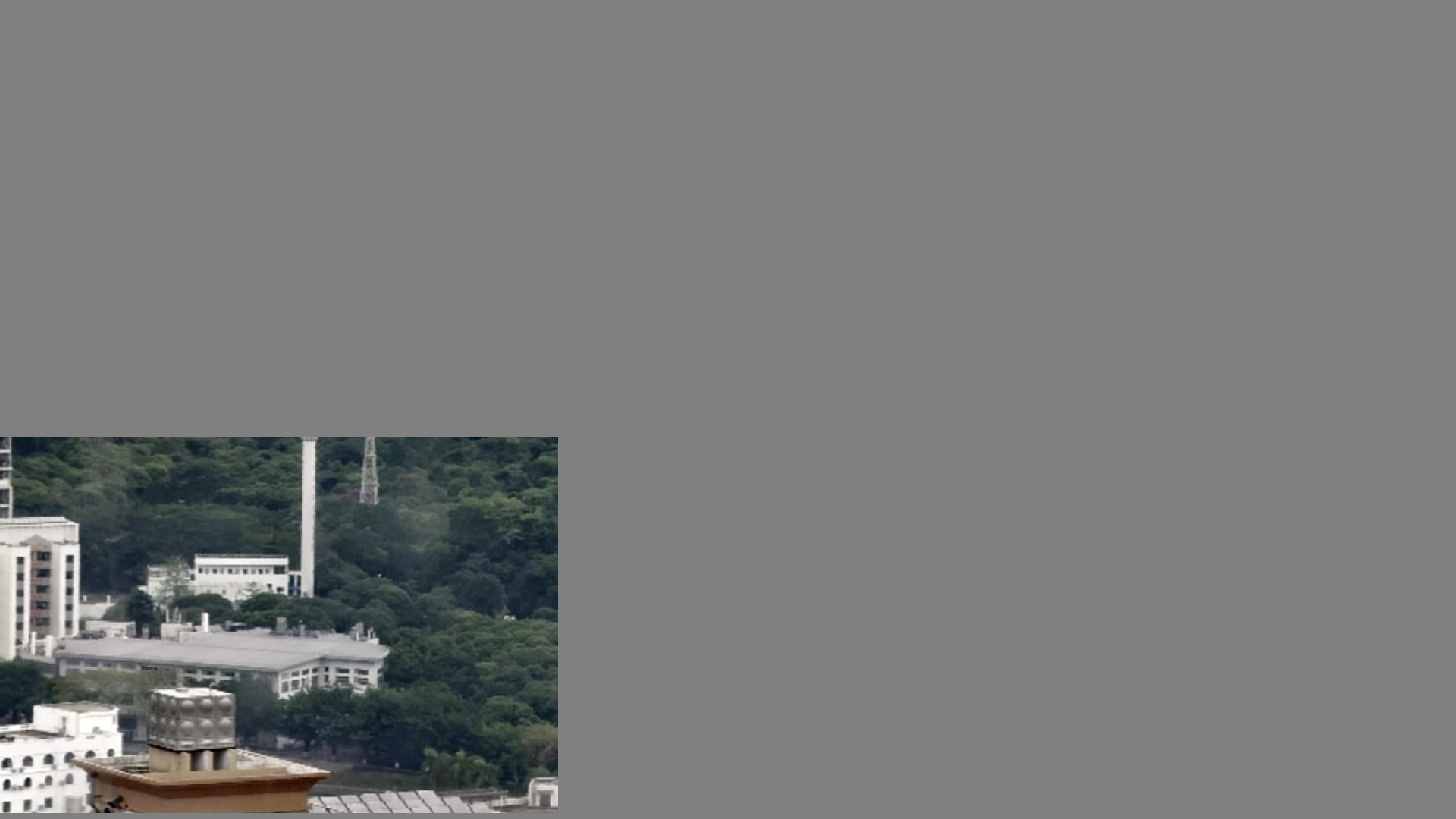}{77 108 2458 1231} &
      \zoomcrop{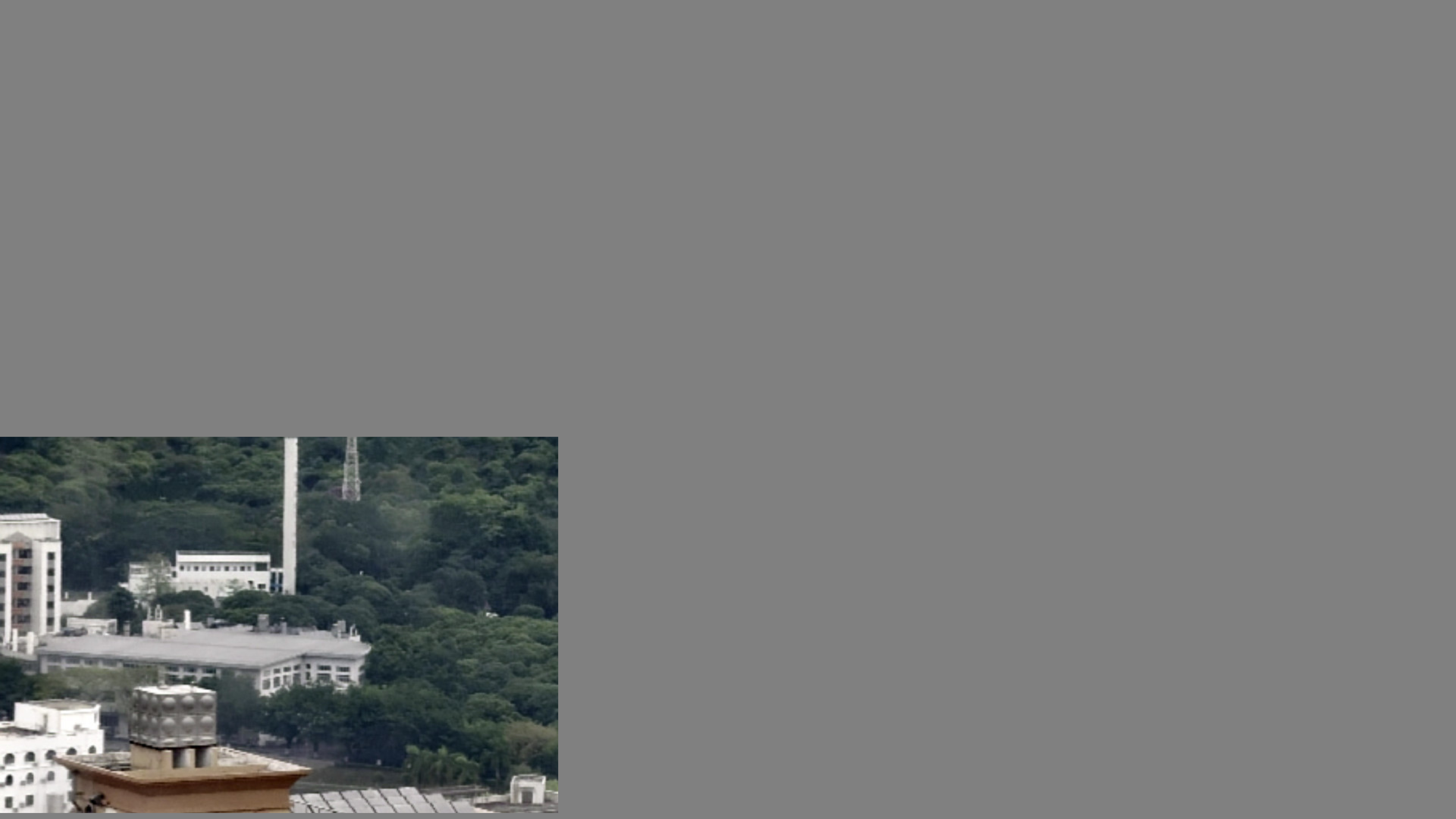}{77 108 2458 1231} &
      \zoomcrop{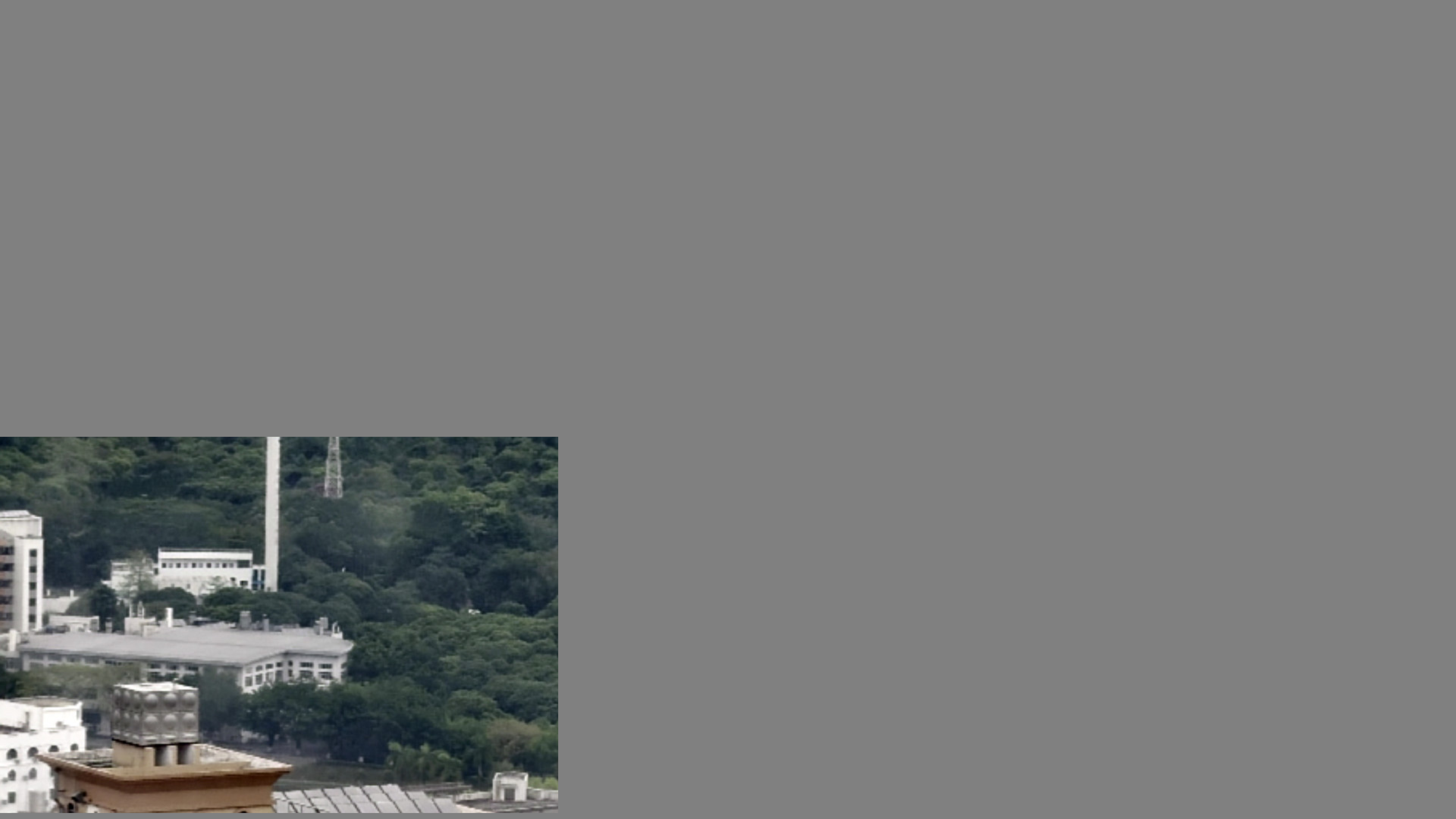}{77 108 2458 1231} \\[0pt]
    \rotatebox{0}{\scriptsize\textbf{(b)}} &
      \zoomcrop{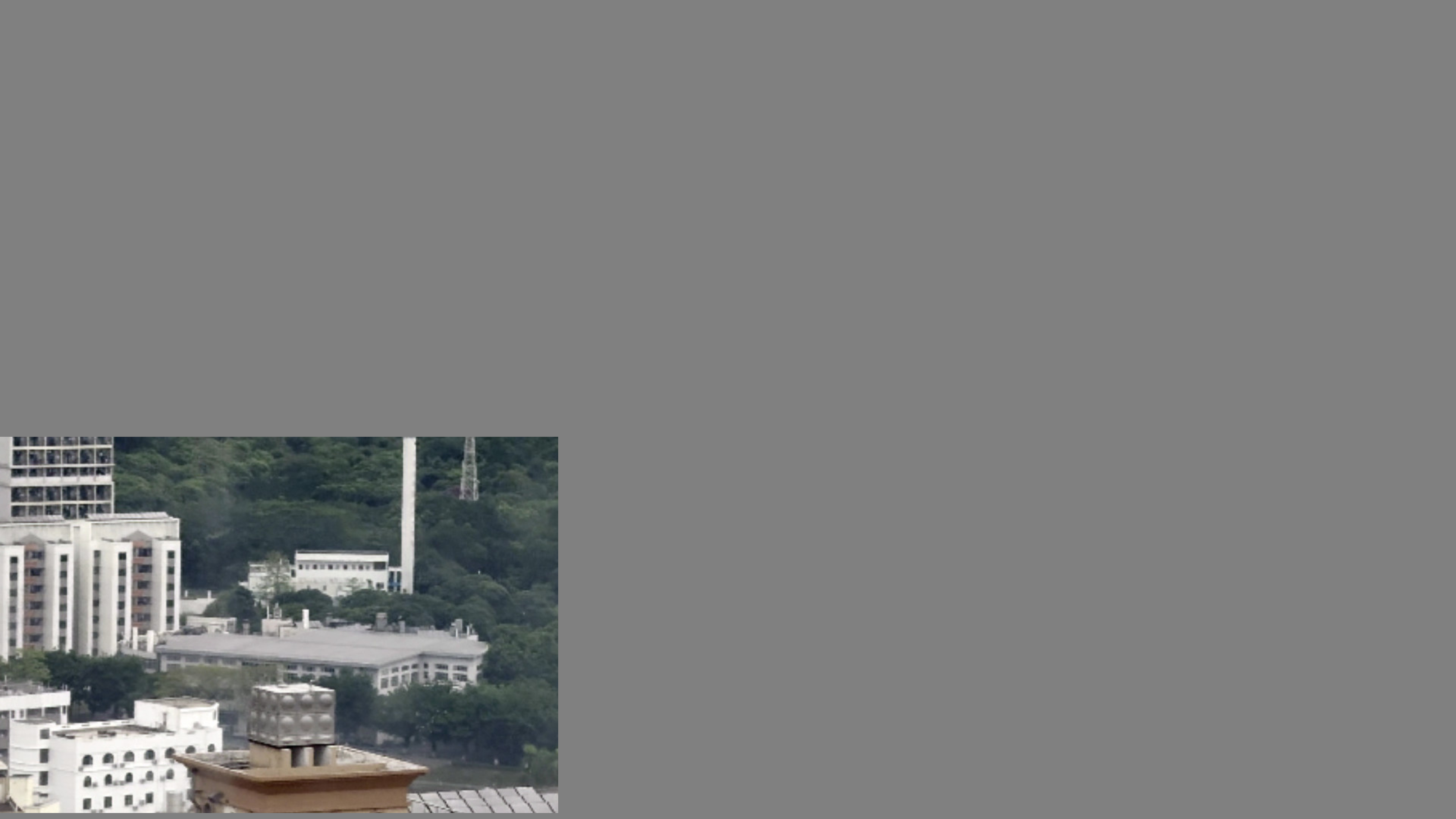}{77 108 2458 1231} &
      \zoomcrop{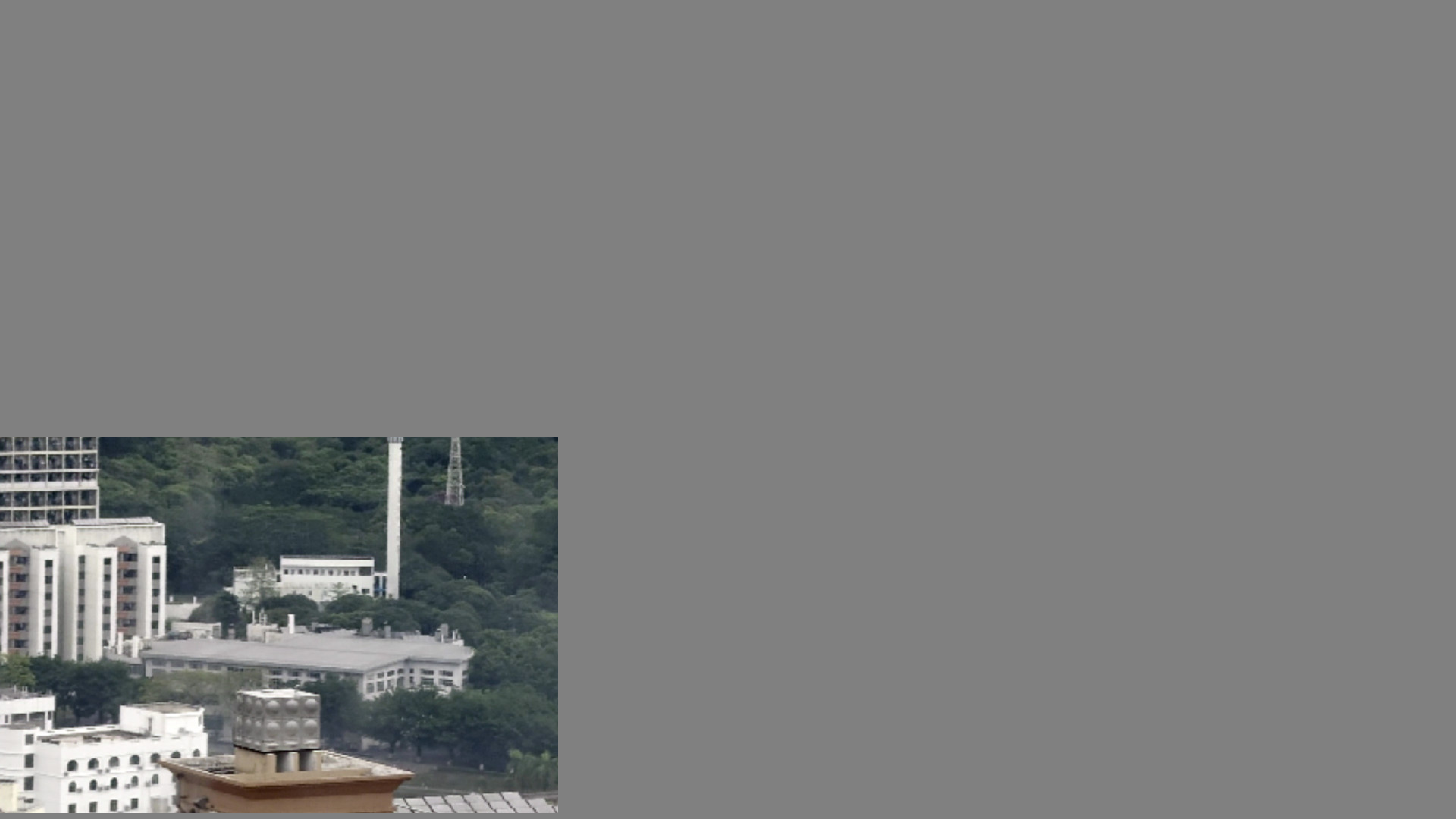}{77 108 2458 1231} &
      \zoomcrop{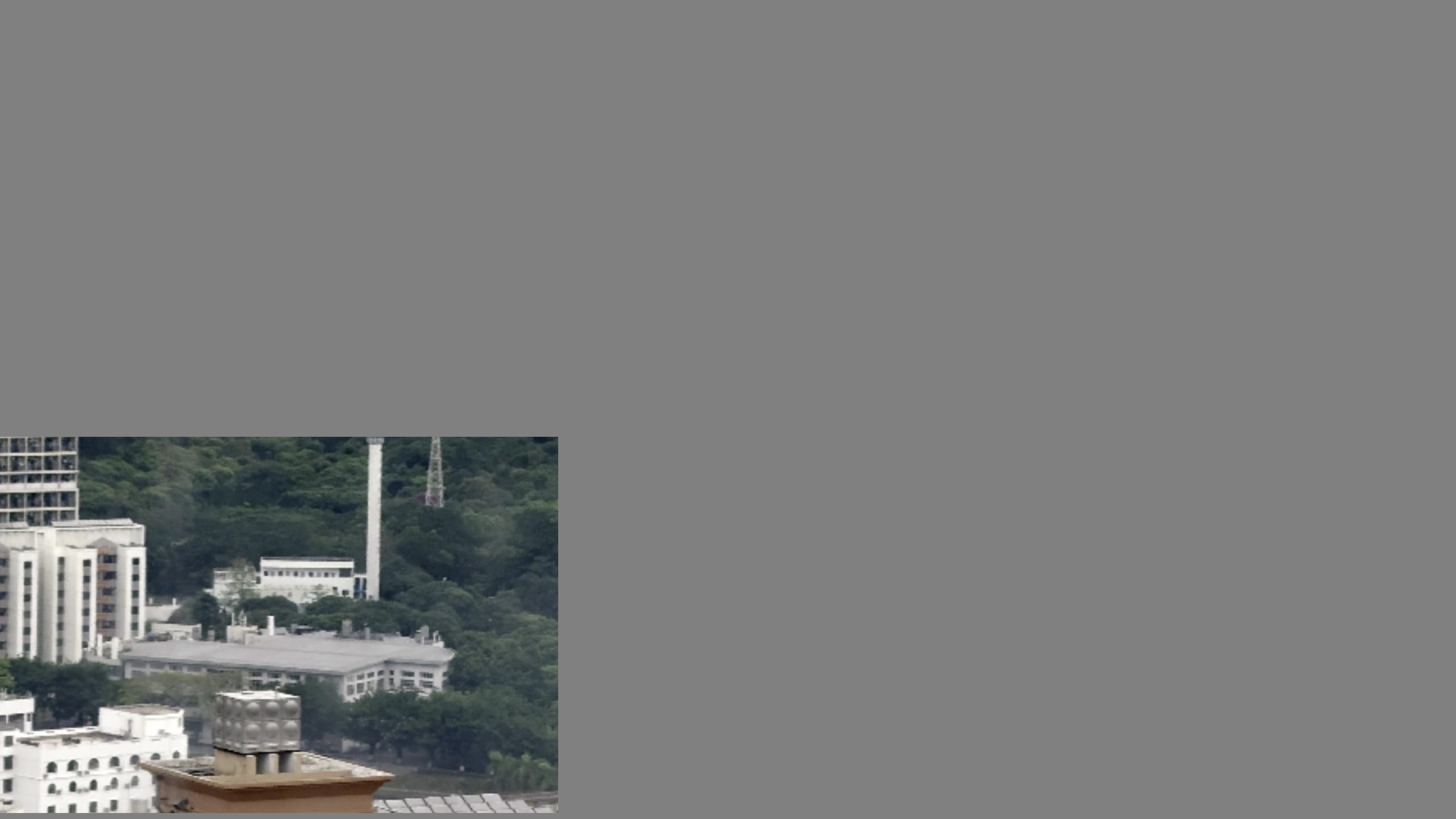}{77 108 2458 1231} &
      \zoomcrop{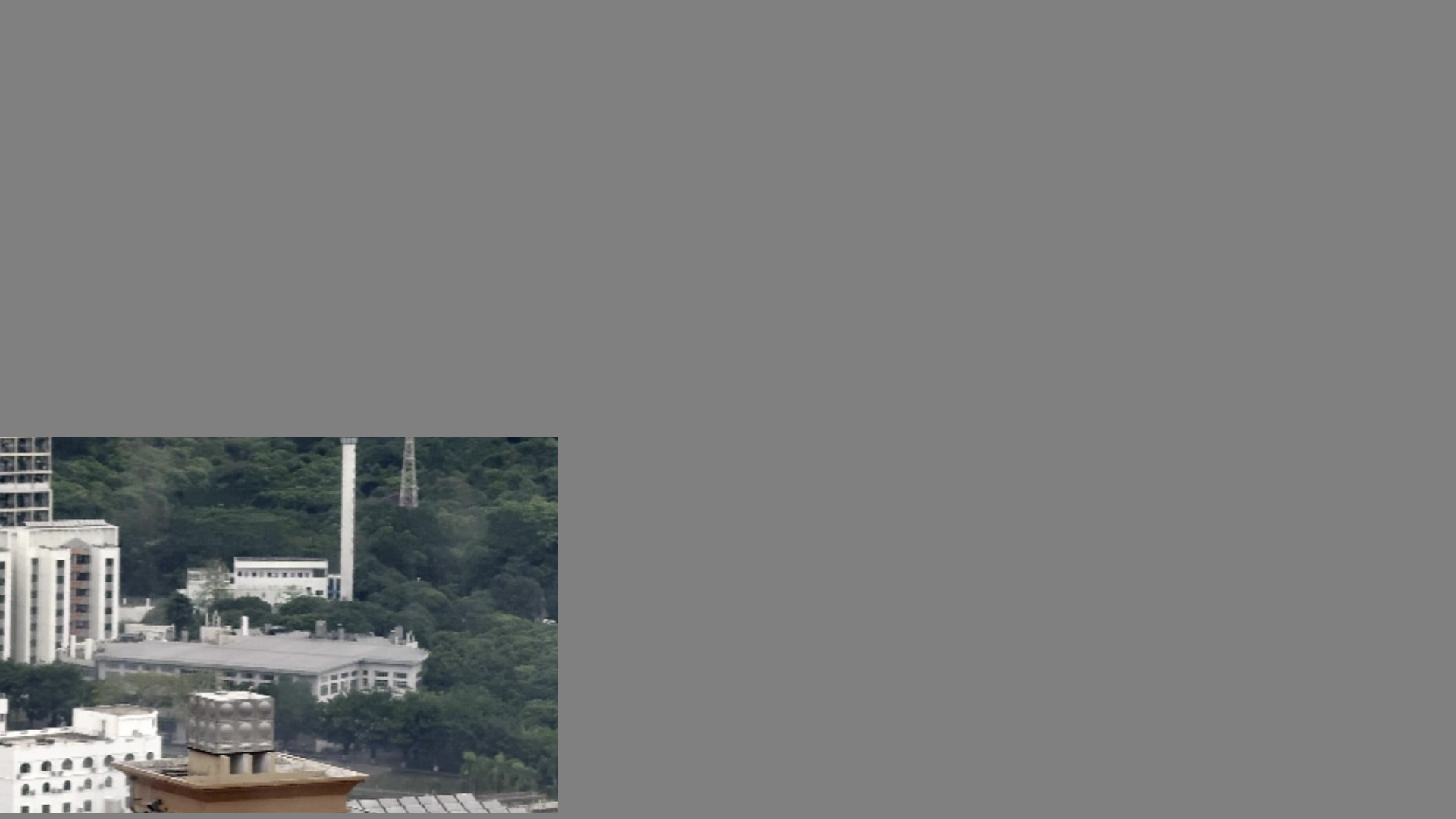}{77 108 2458 1231} &
      \zoomcrop{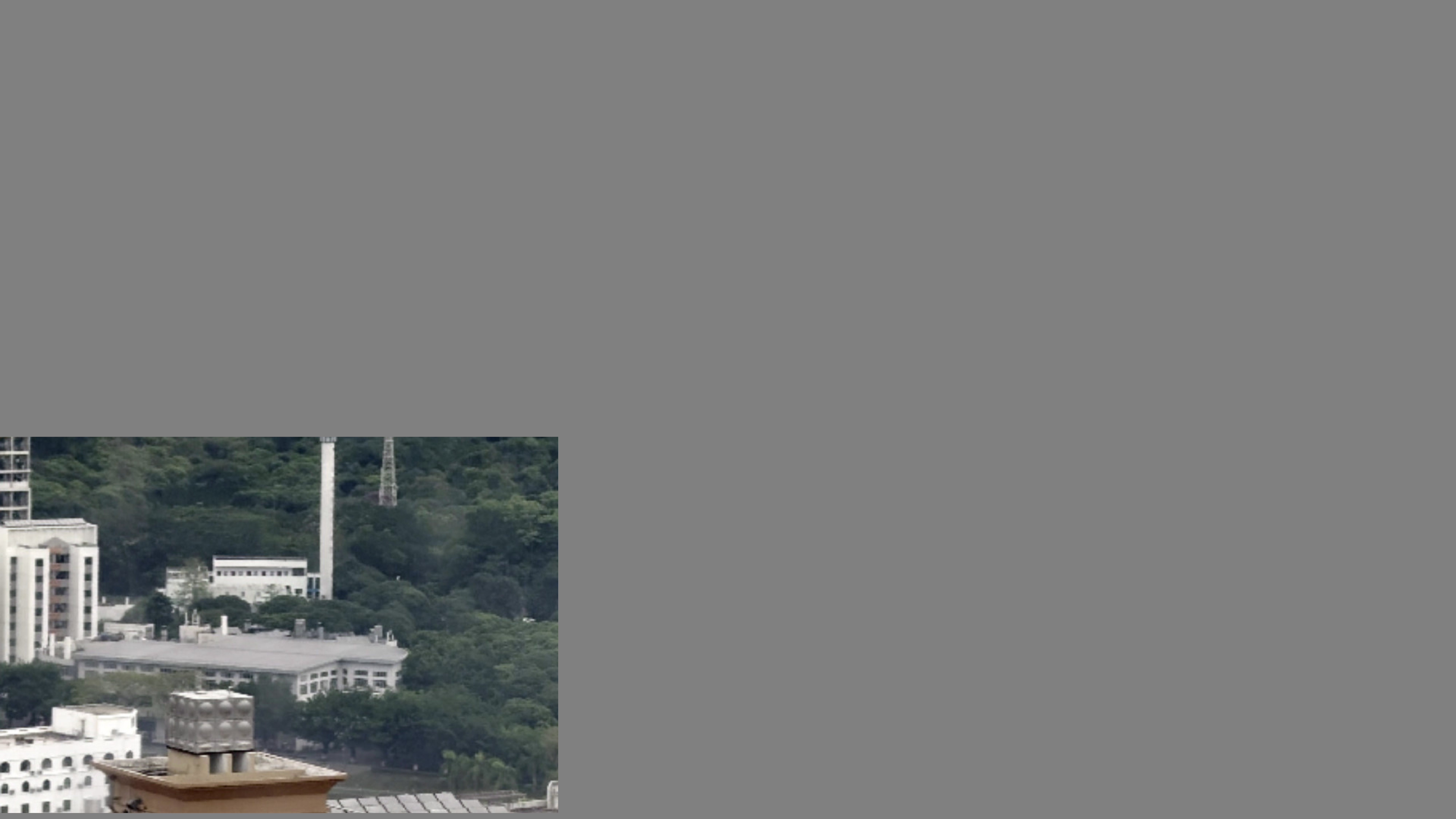}{77 108 2458 1231} &
      \zoomcrop{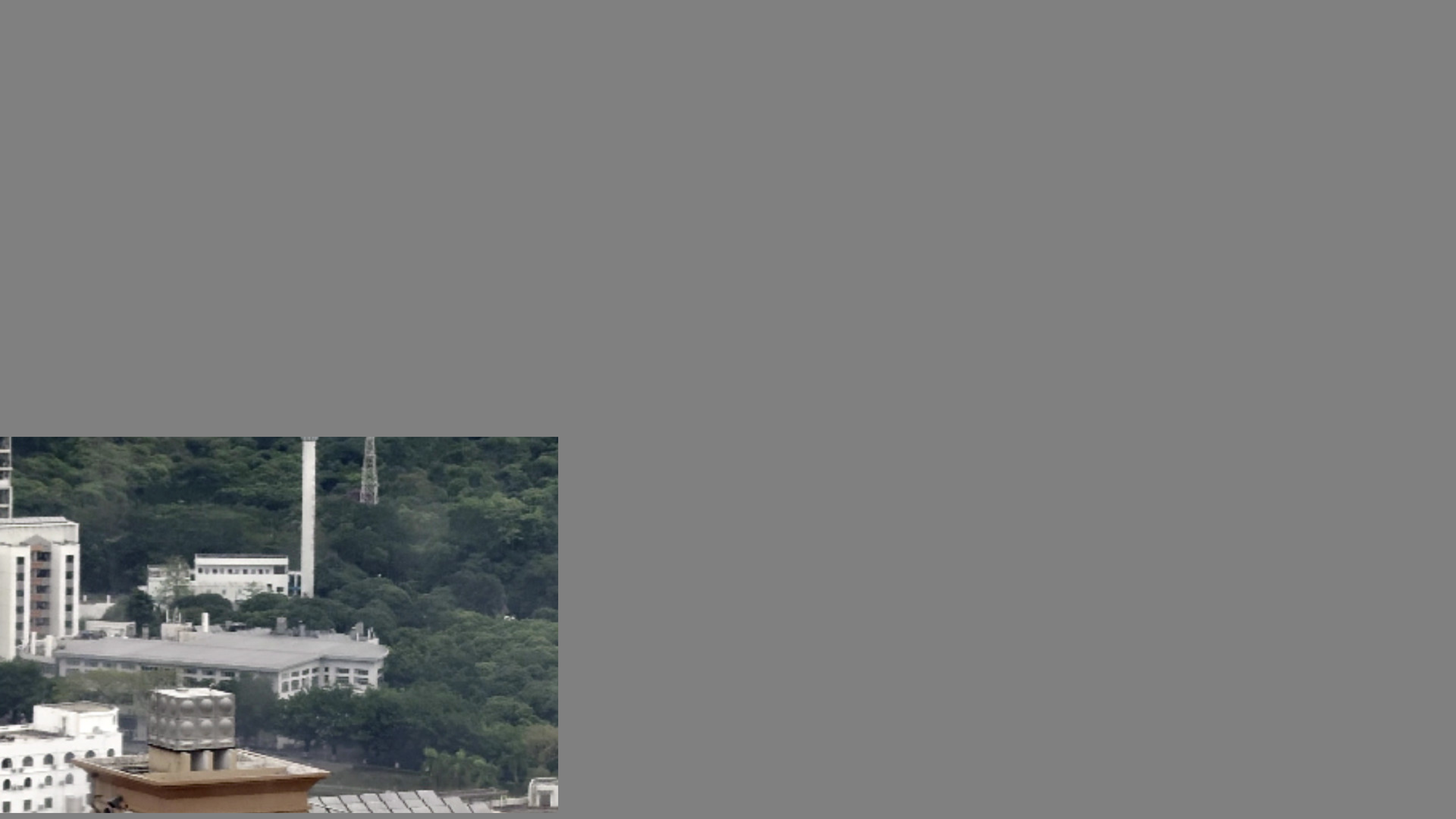}{77 108 2458 1231} &
      \zoomcrop{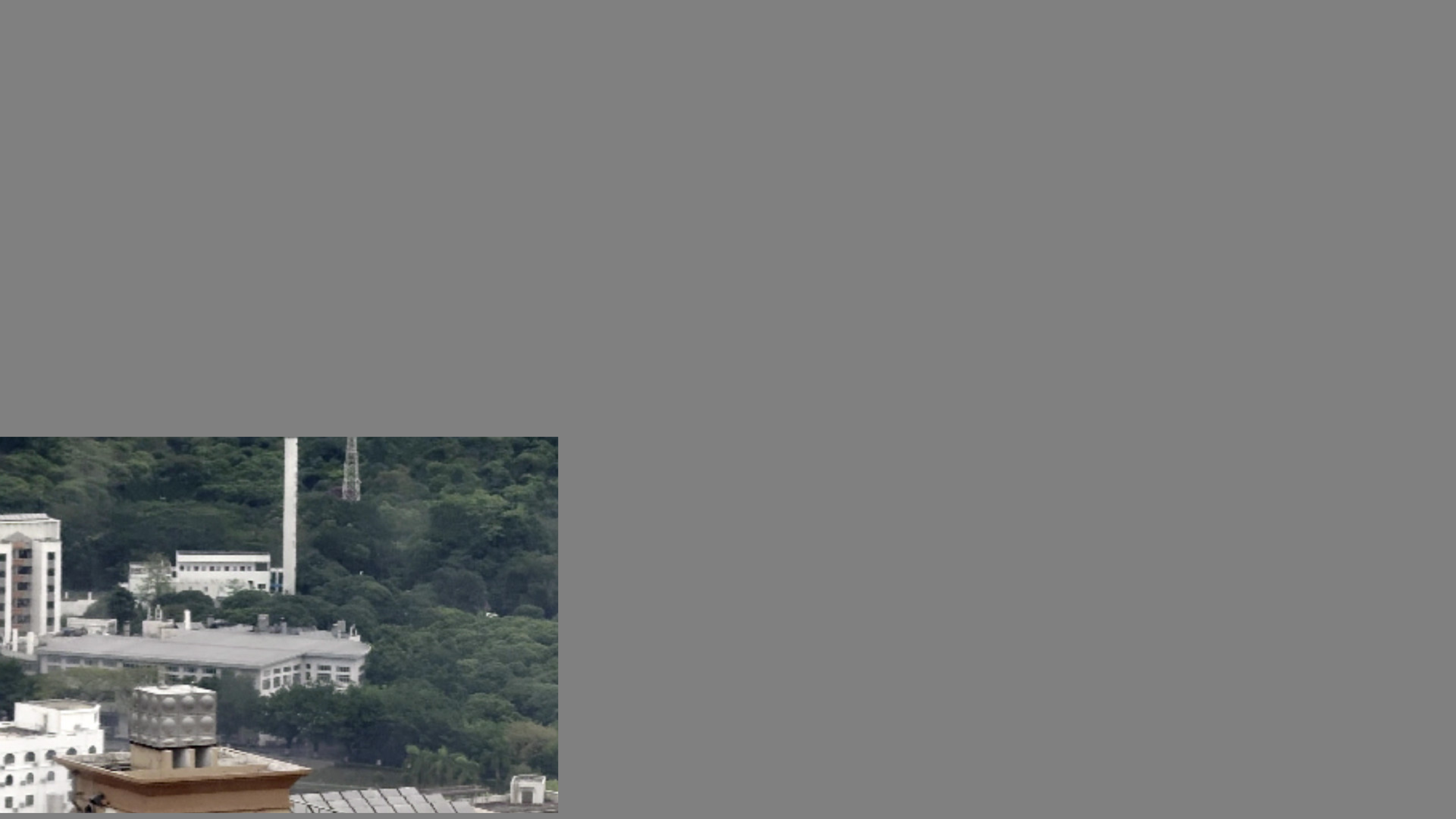}{77 108 2458 1231} &
      \zoomcrop{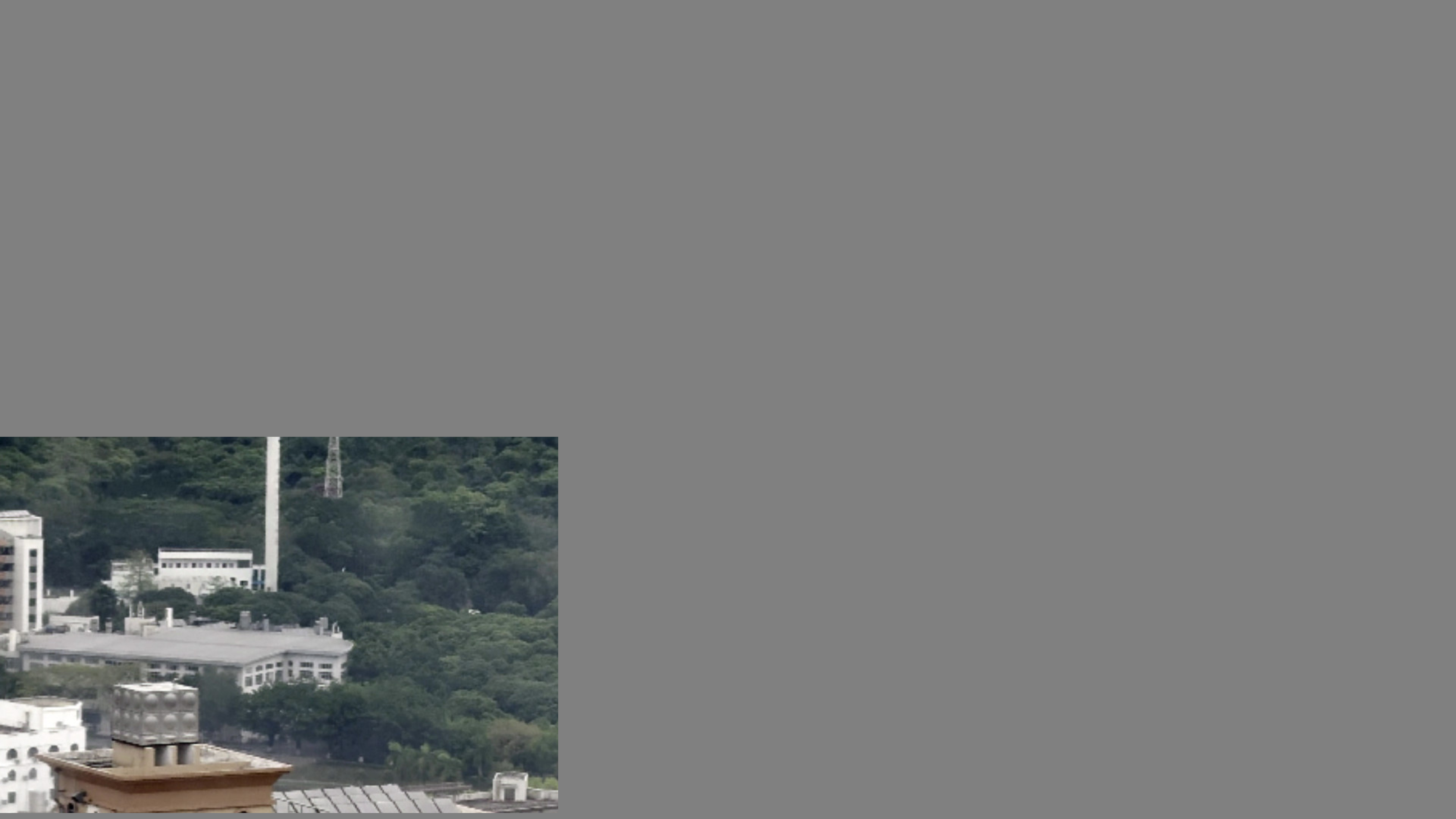}{77 108 2458 1231} \\[0pt]
    \rotatebox{0}{\scriptsize\textbf{(c)}} &
      \zoomcrop{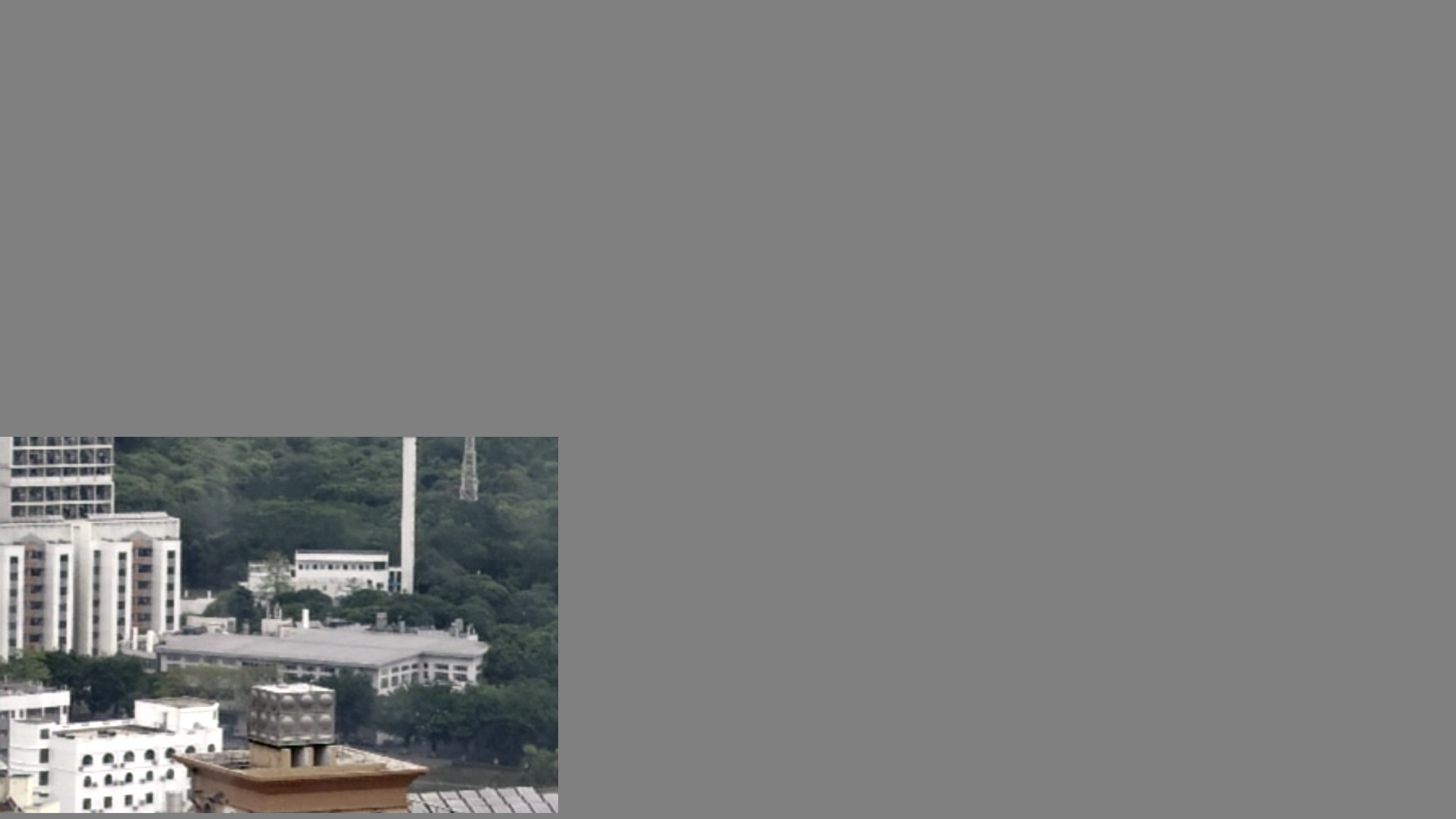}{77 108 2458 1231} &
      \zoomcrop{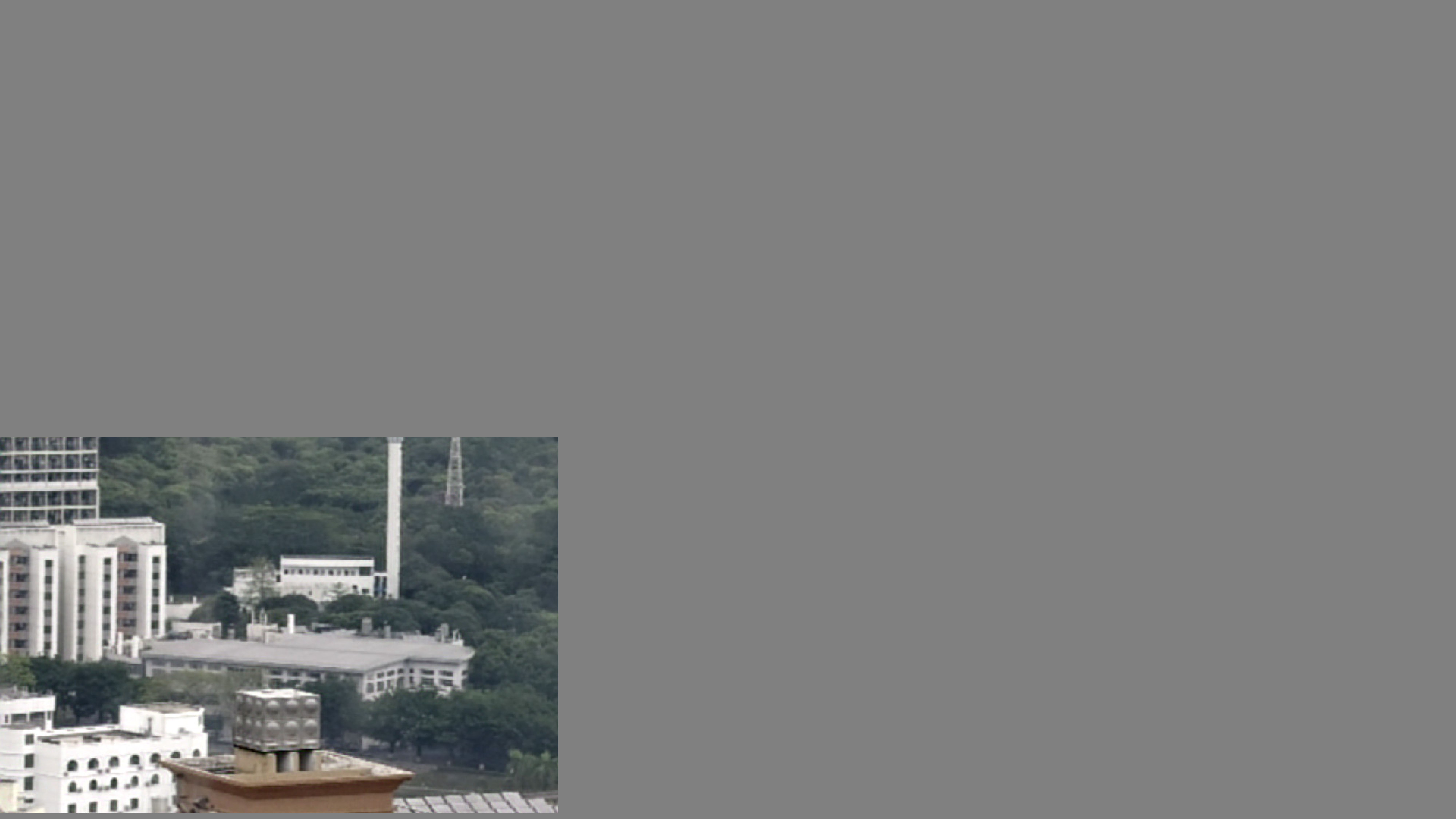}{77 108 2458 1231} &
      \zoomcrop{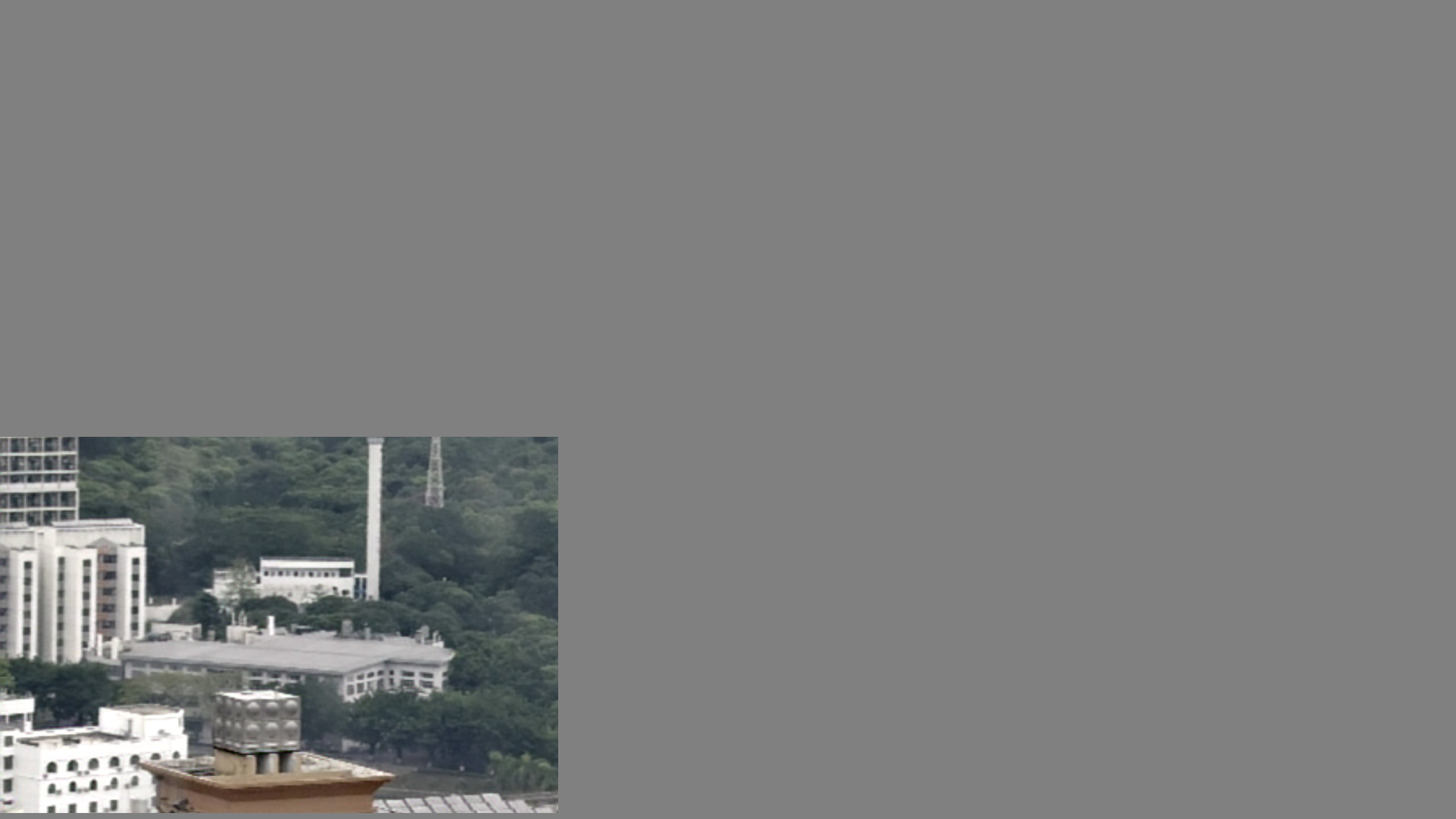}{77 108 2458 1231} &
      \zoomcrop{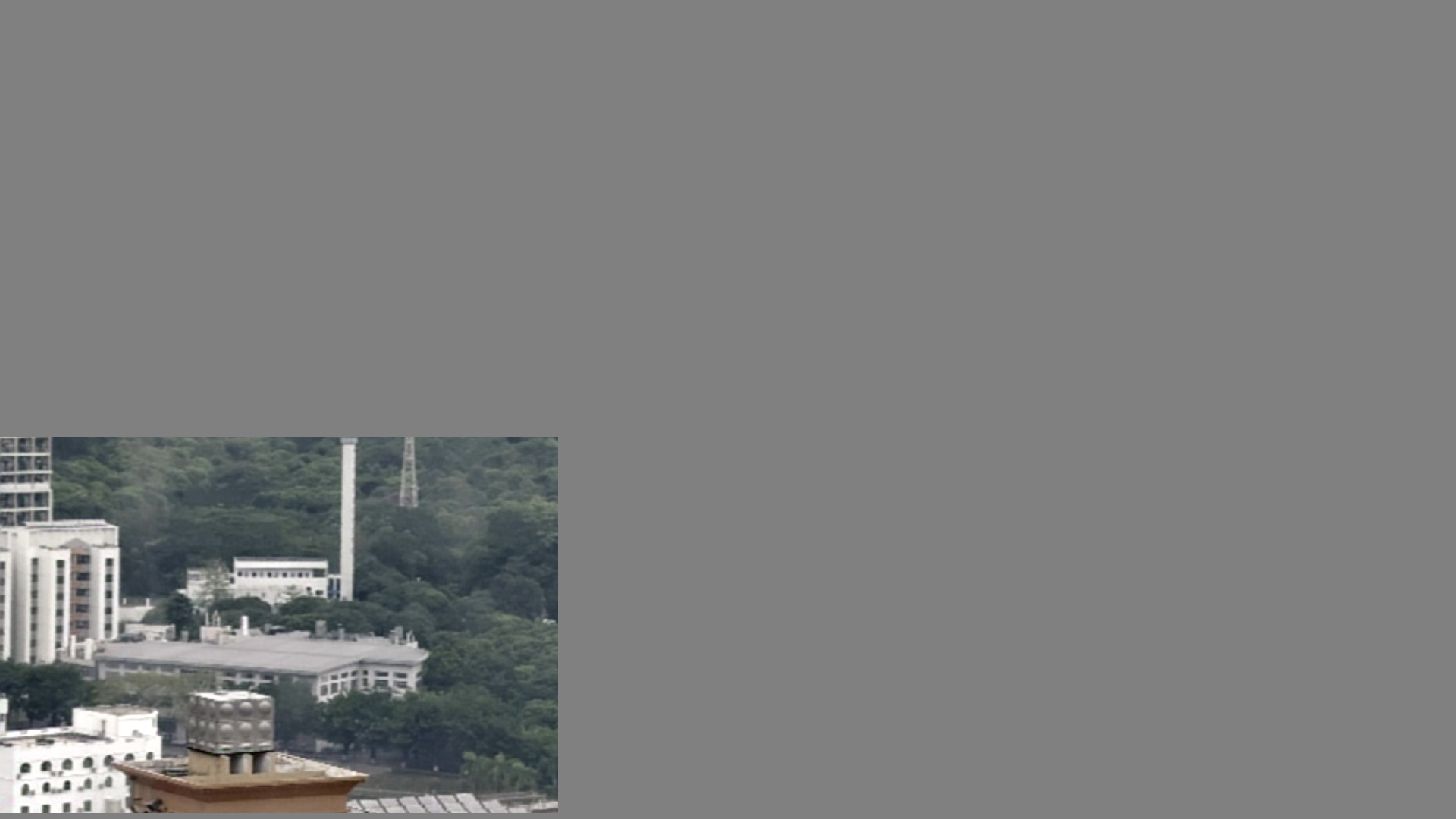}{77 108 2458 1231} &
      \zoomcrop{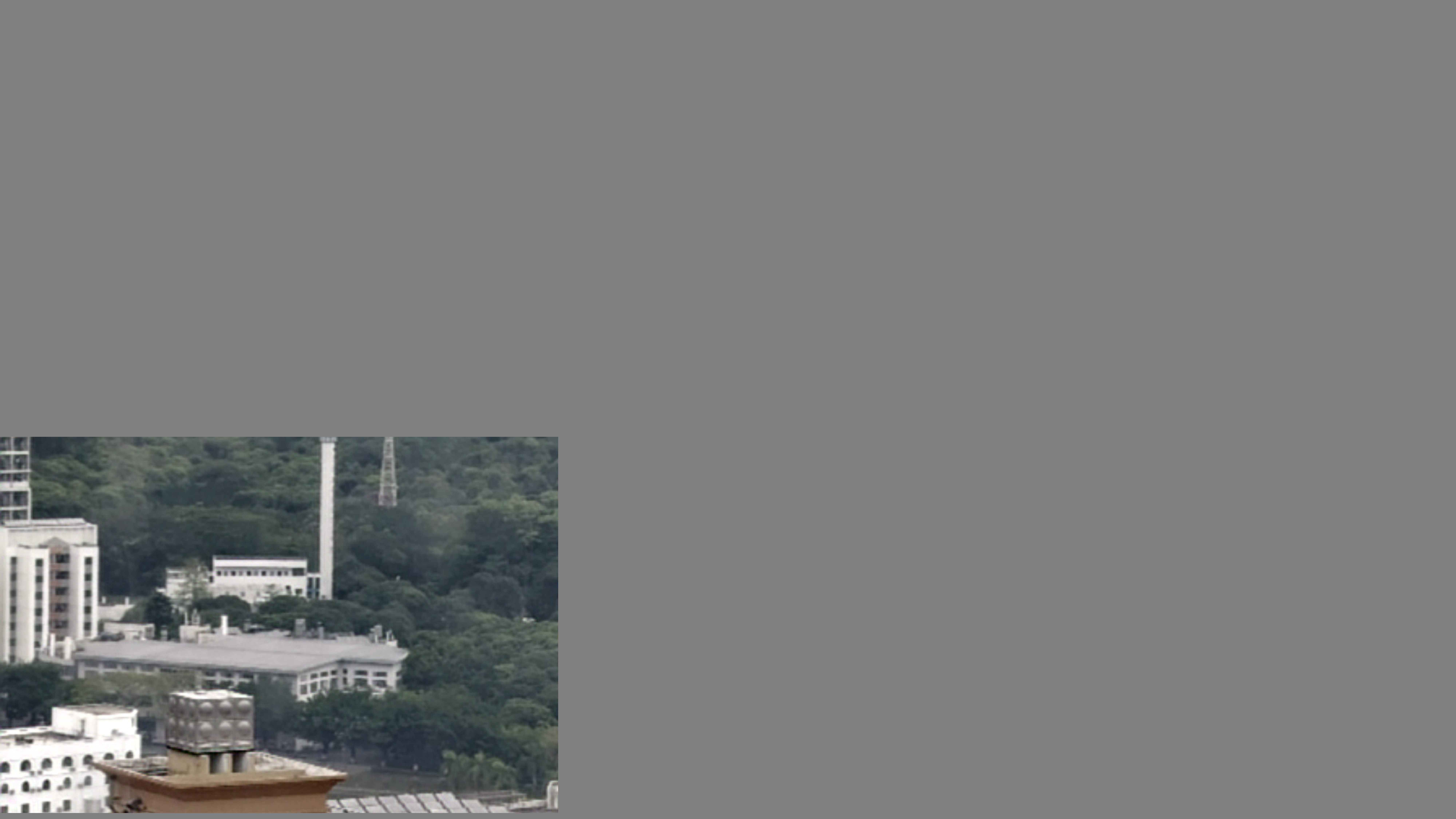}{77 108 2458 1231} &
      \zoomcrop{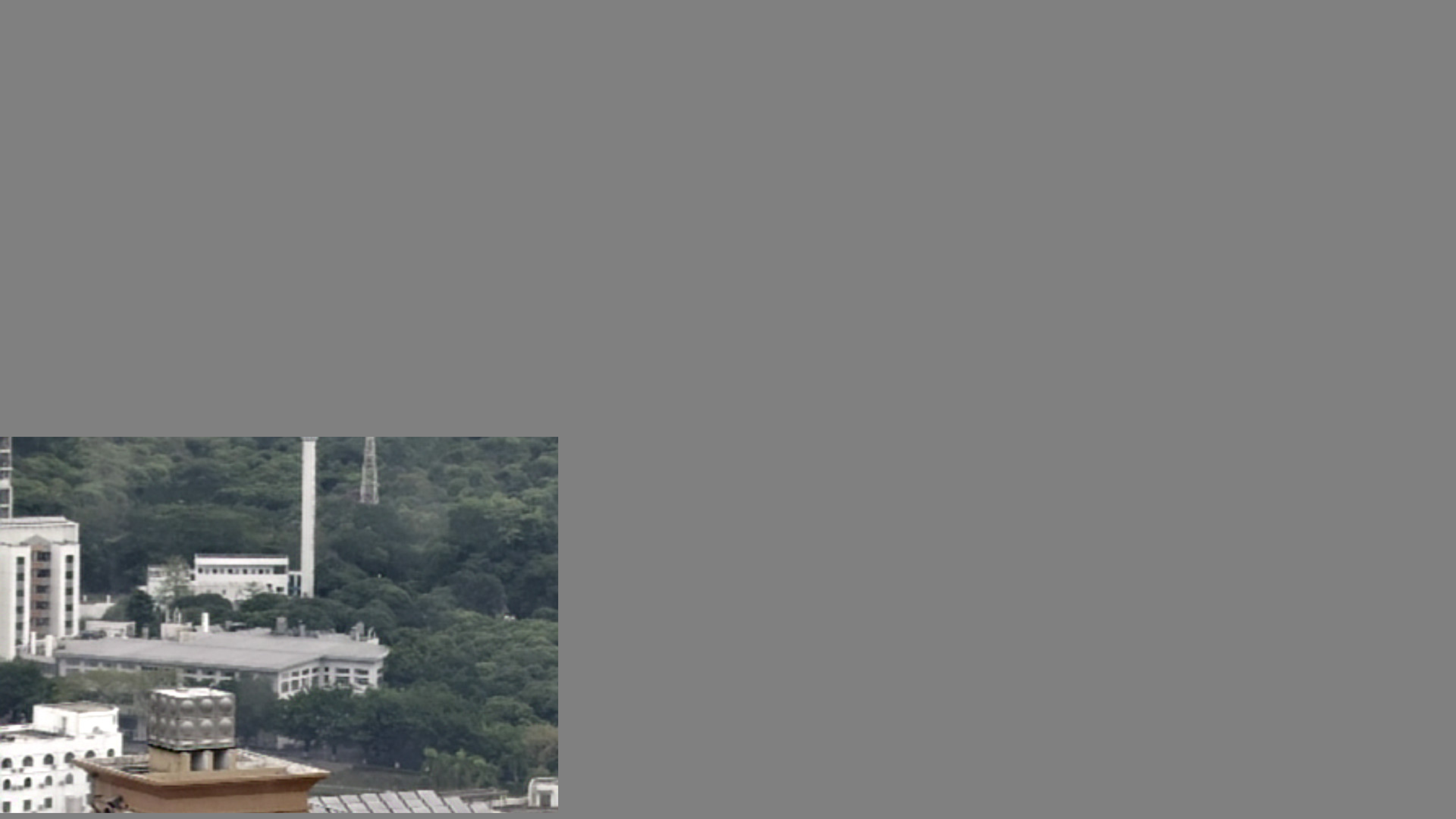}{77 108 2458 1231} &
      \zoomcrop{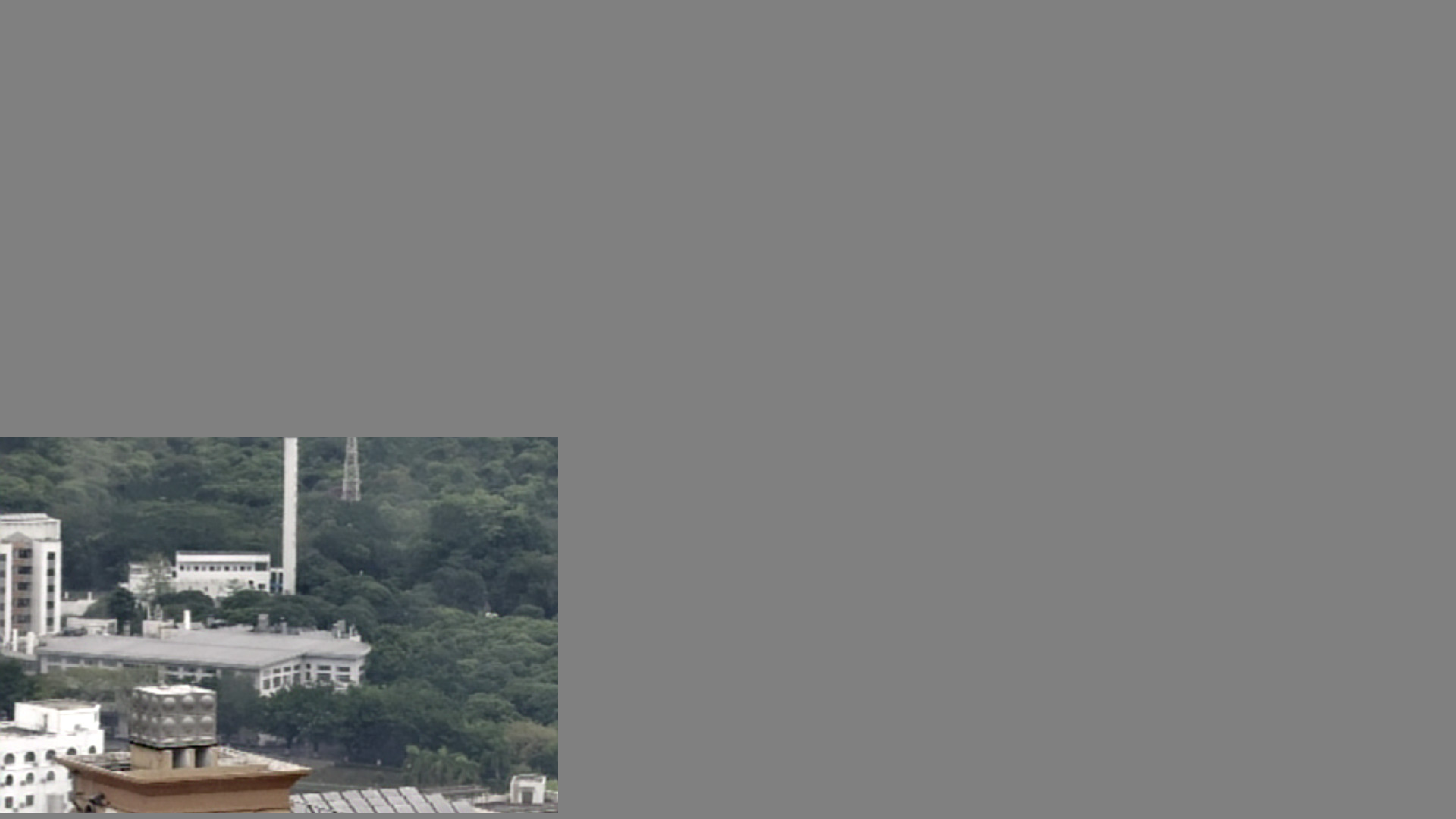}{77 108 2458 1231} &
      \zoomcrop{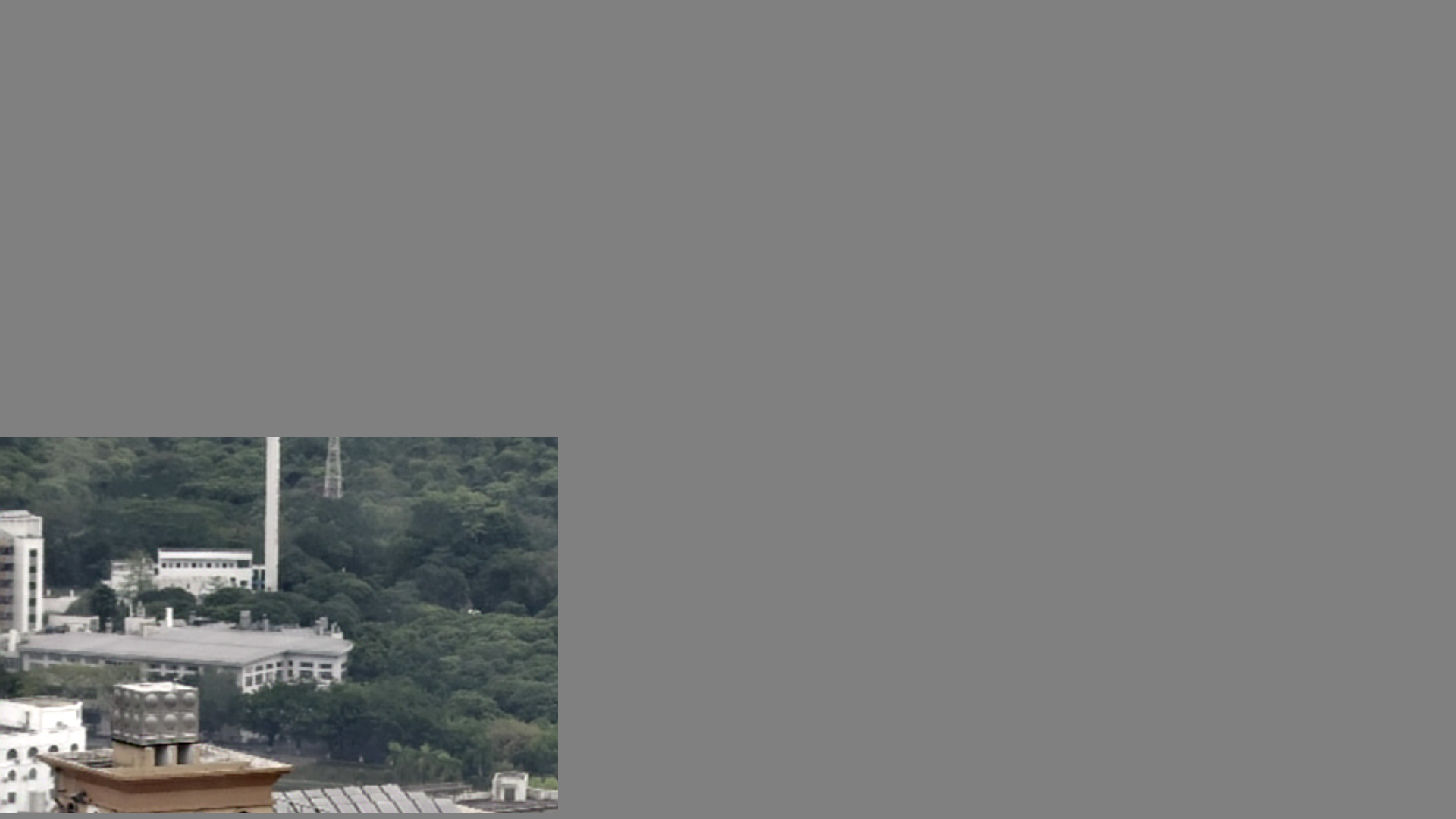}{77 108 2458 1231} \\[0pt]
    \rotatebox{0}{\scriptsize\textbf{(d)}} &
      \zoomcrop{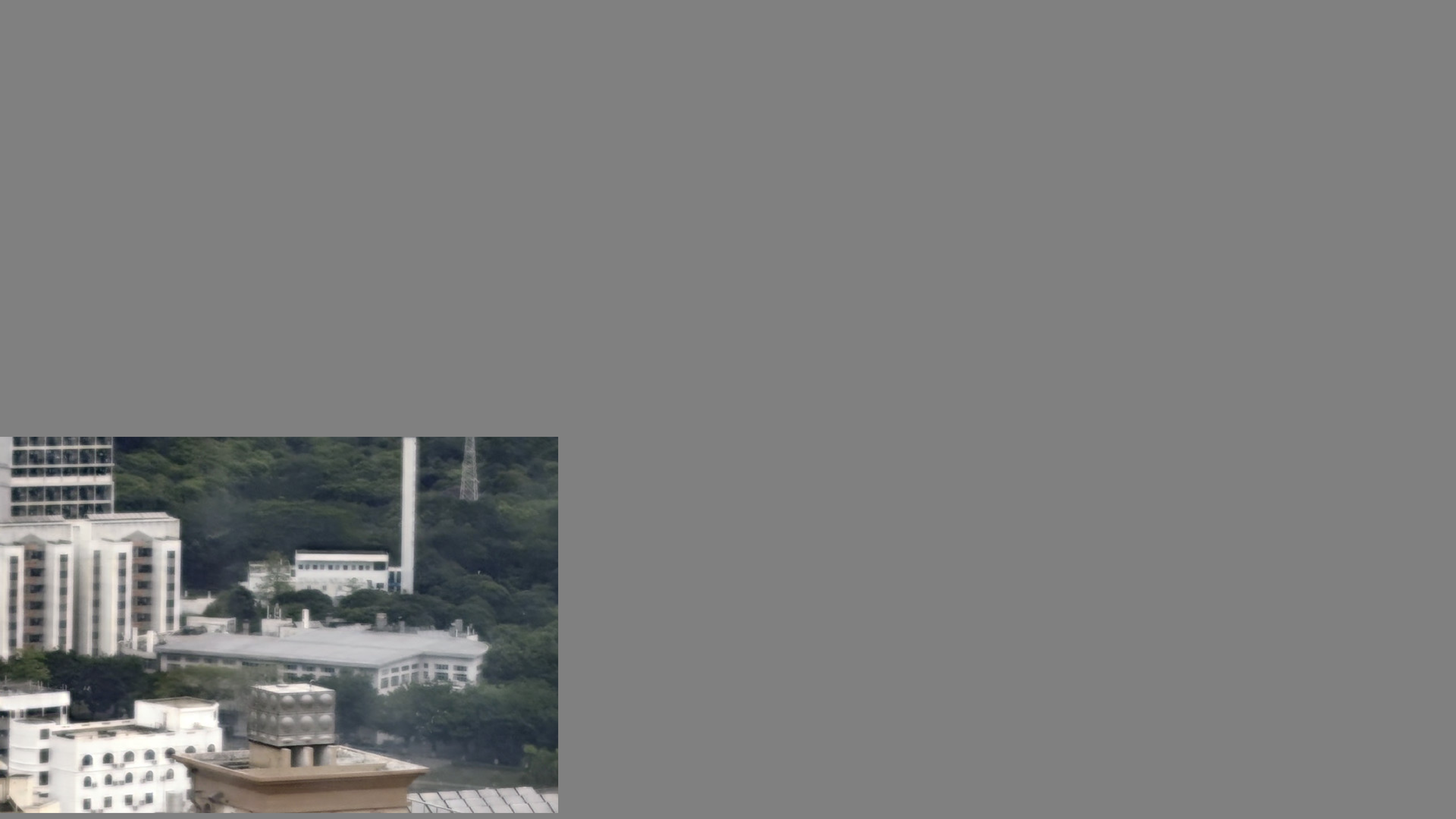}{77 108 2458 1231} &
      \zoomcrop{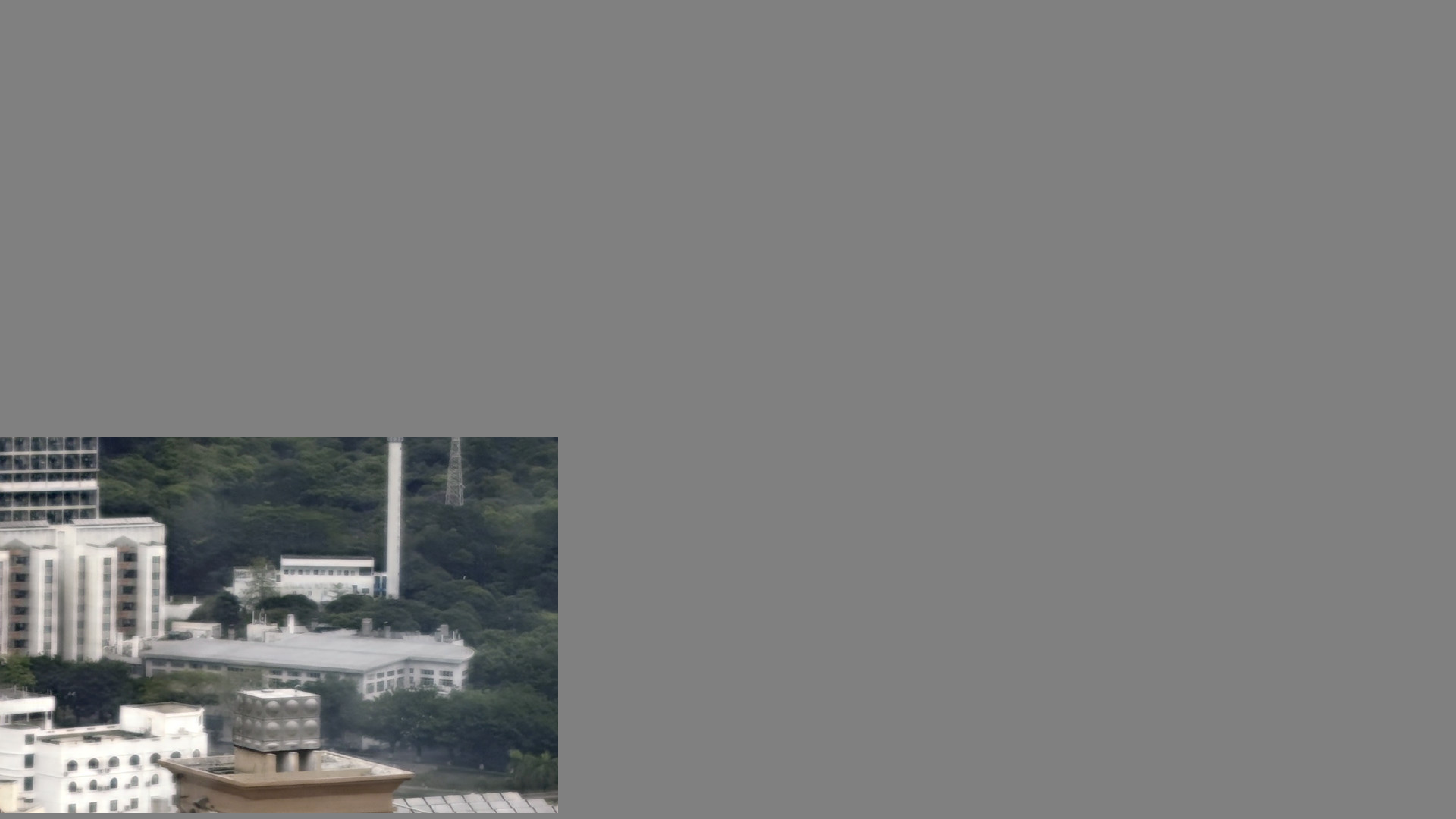}{77 108 2458 1231} &
      \zoomcrop{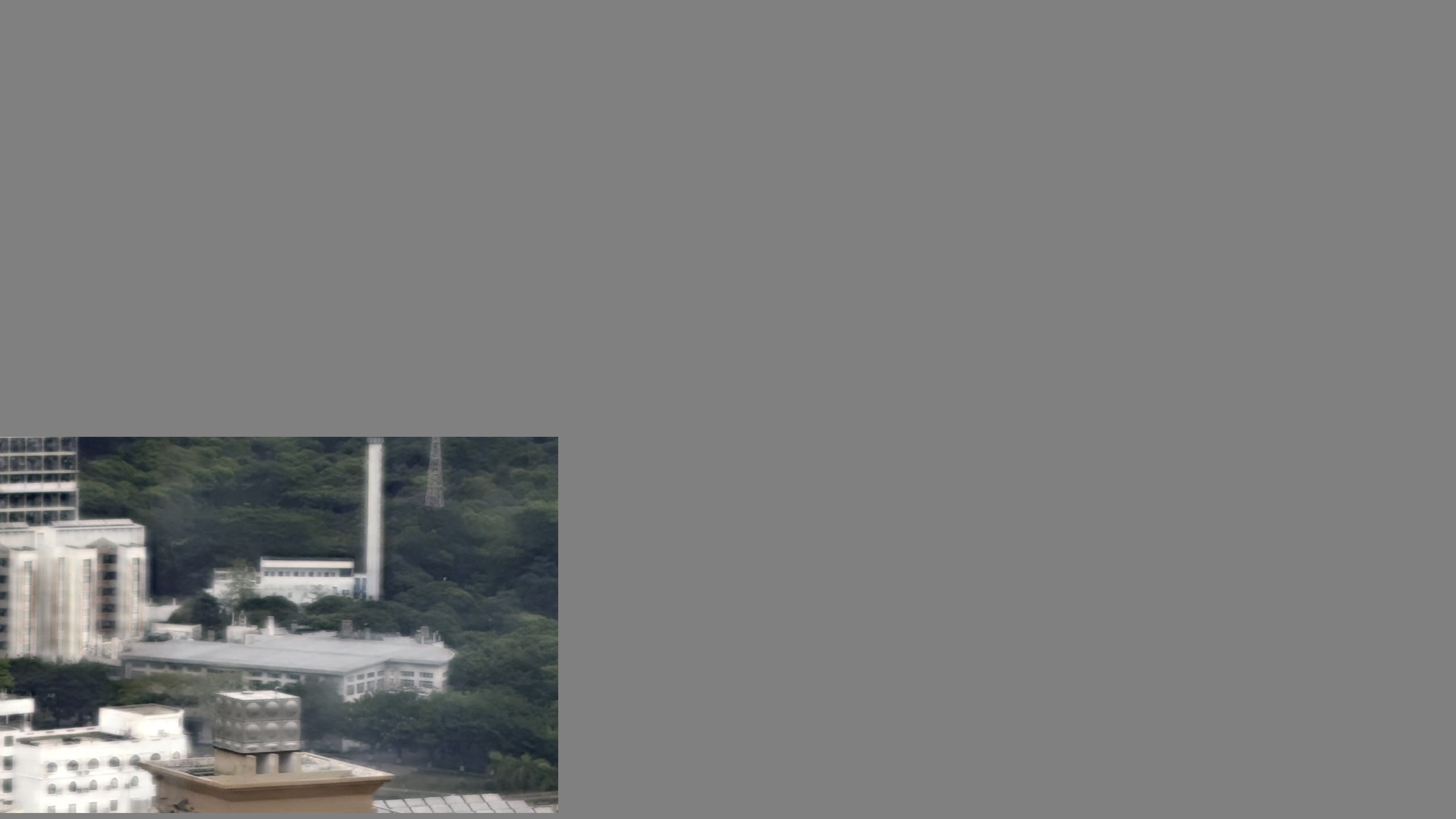}{77 108 2458 1231} &
      \zoomcrop{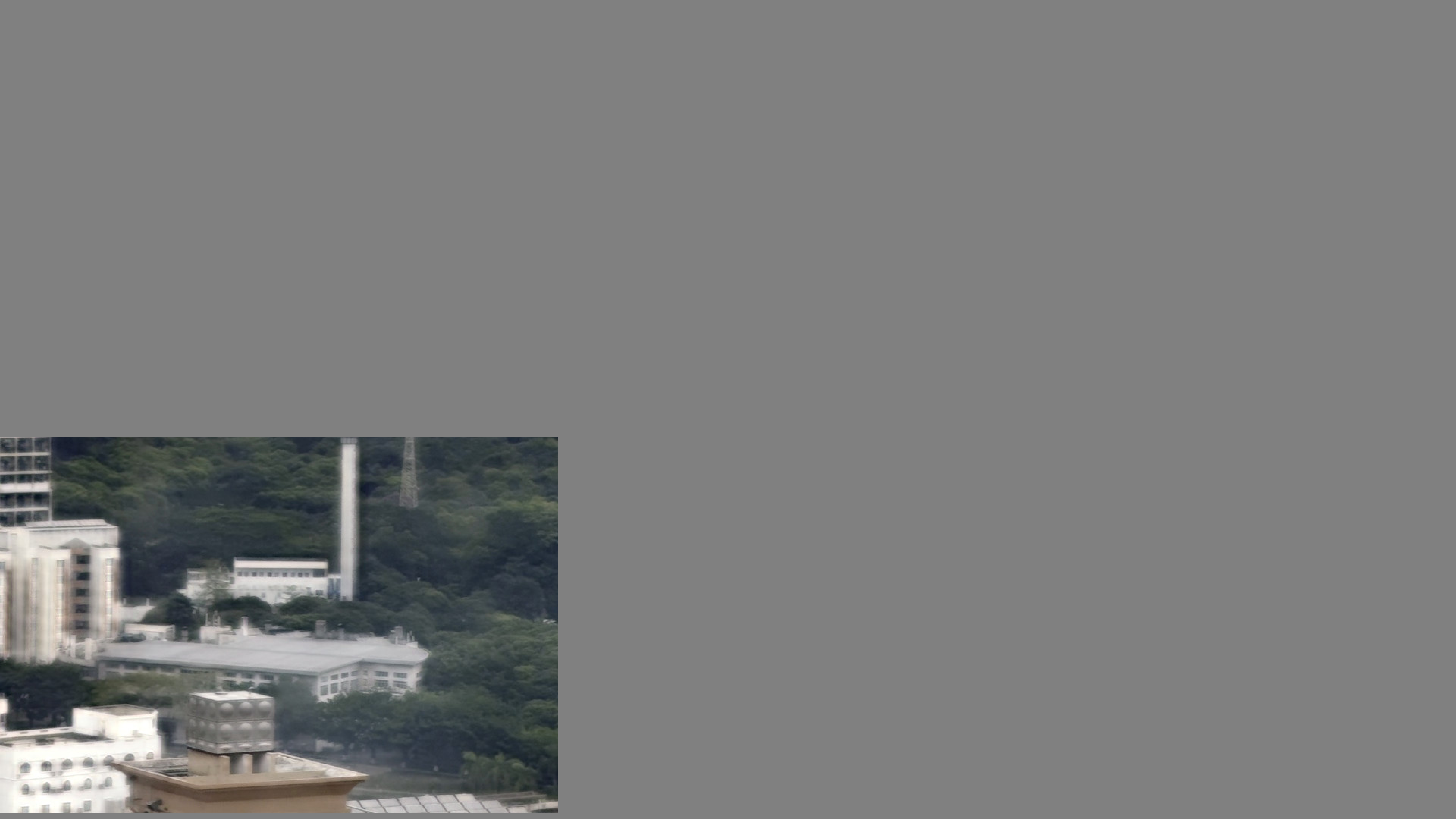}{77 108 2458 1231} &
      \zoomcrop{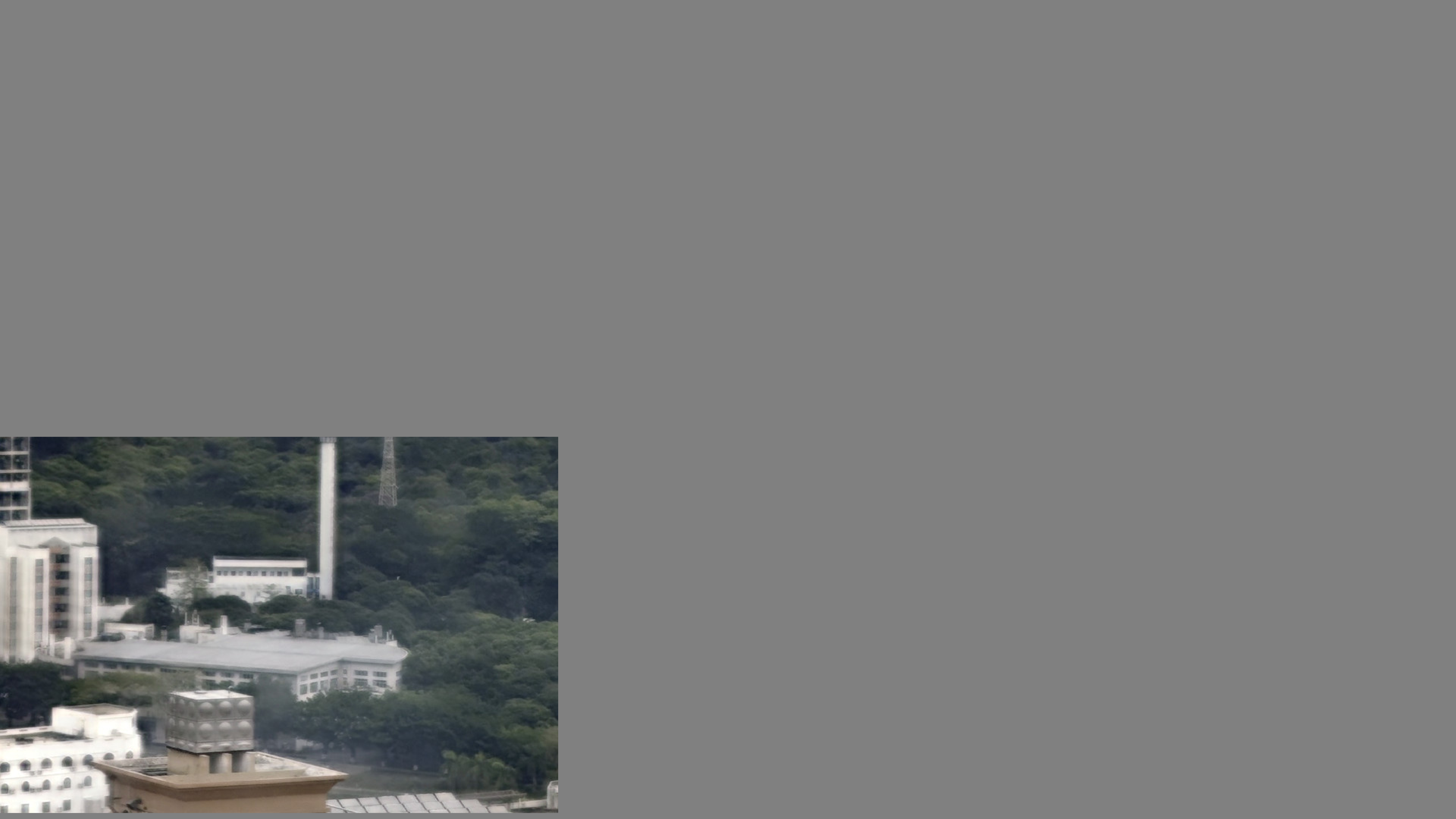}{77 108 2458 1231} &
      \zoomcrop{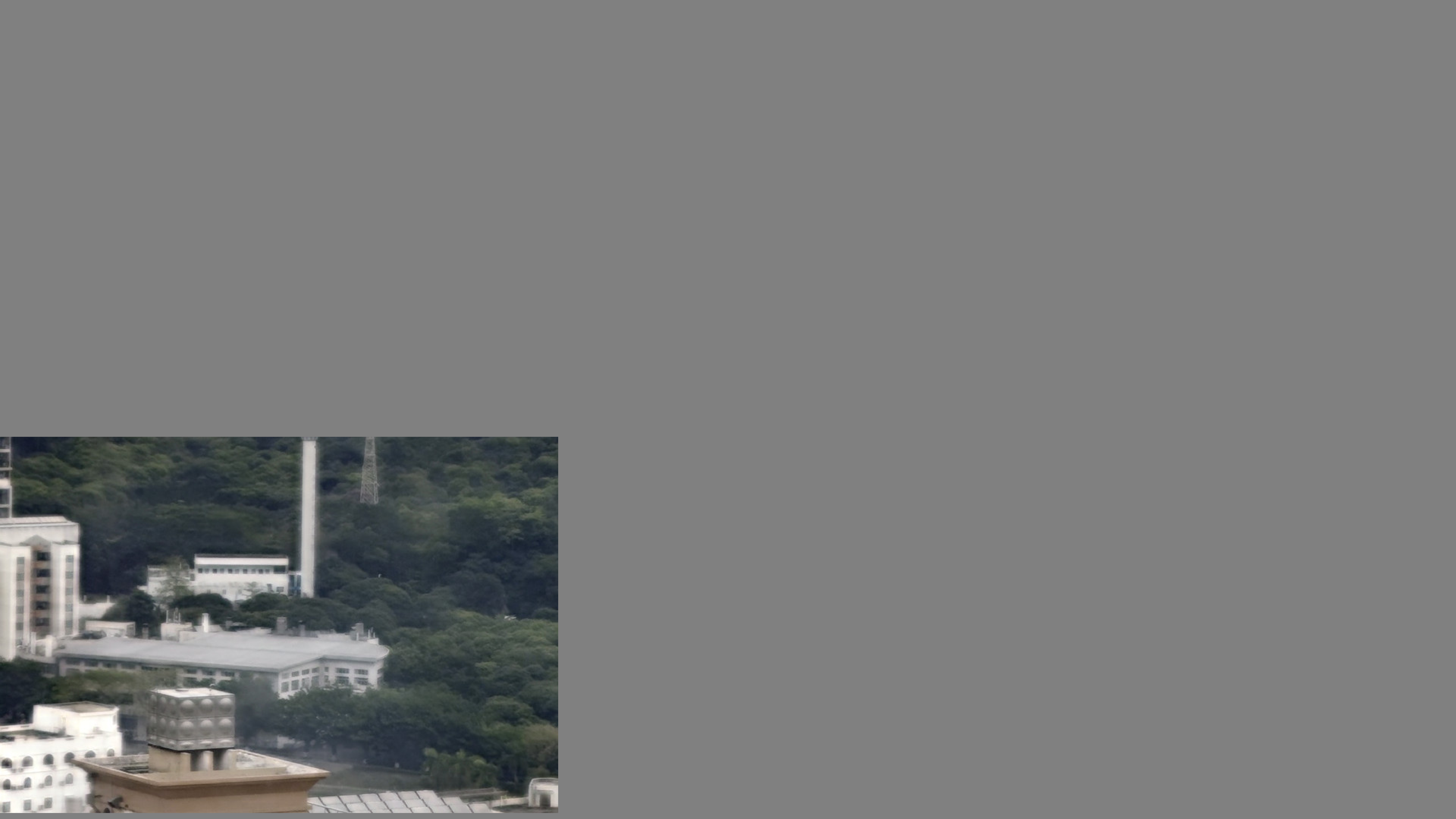}{77 108 2458 1231} &
      \zoomcrop{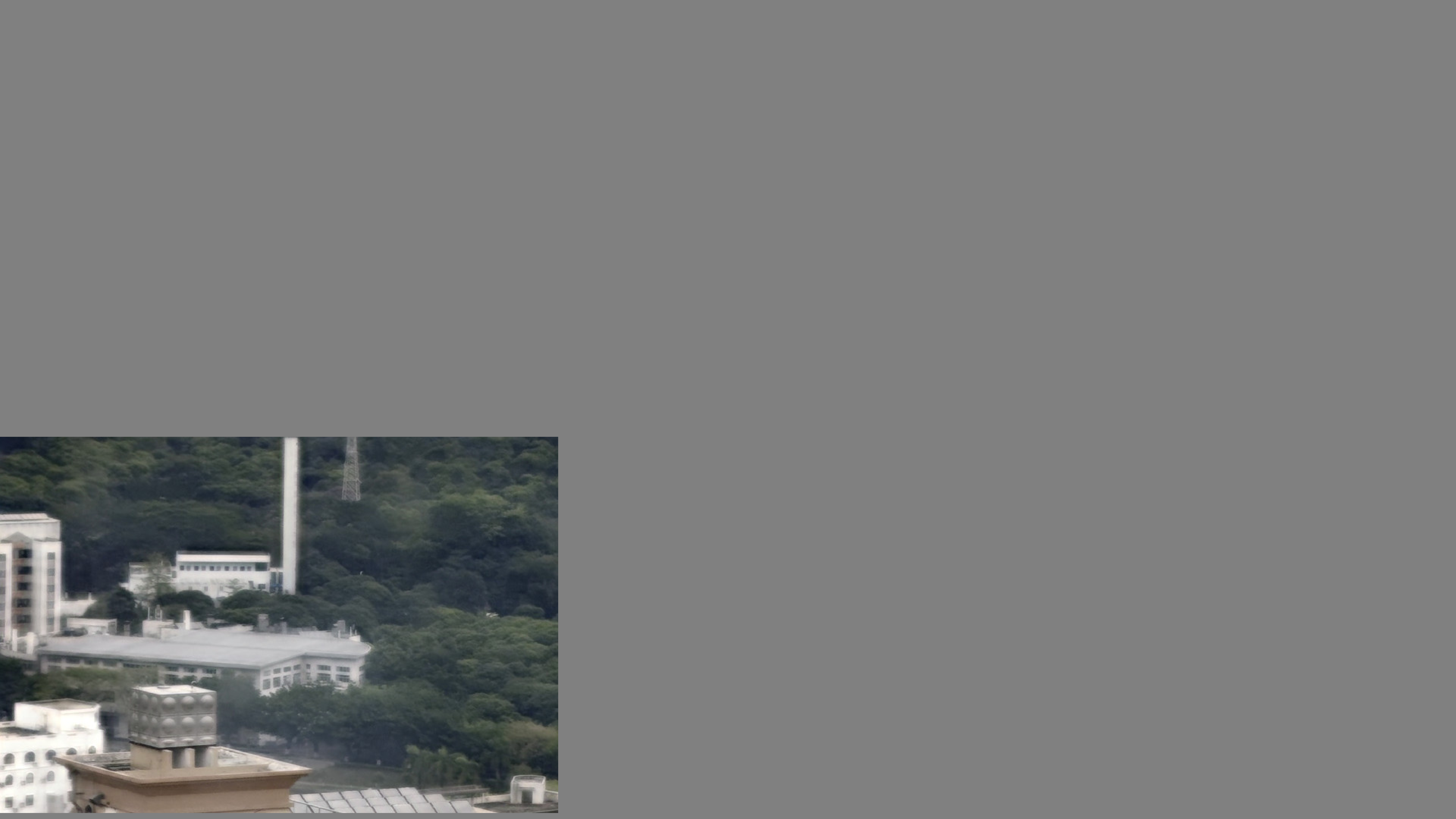}{77 108 2458 1231} &
      \zoomcrop{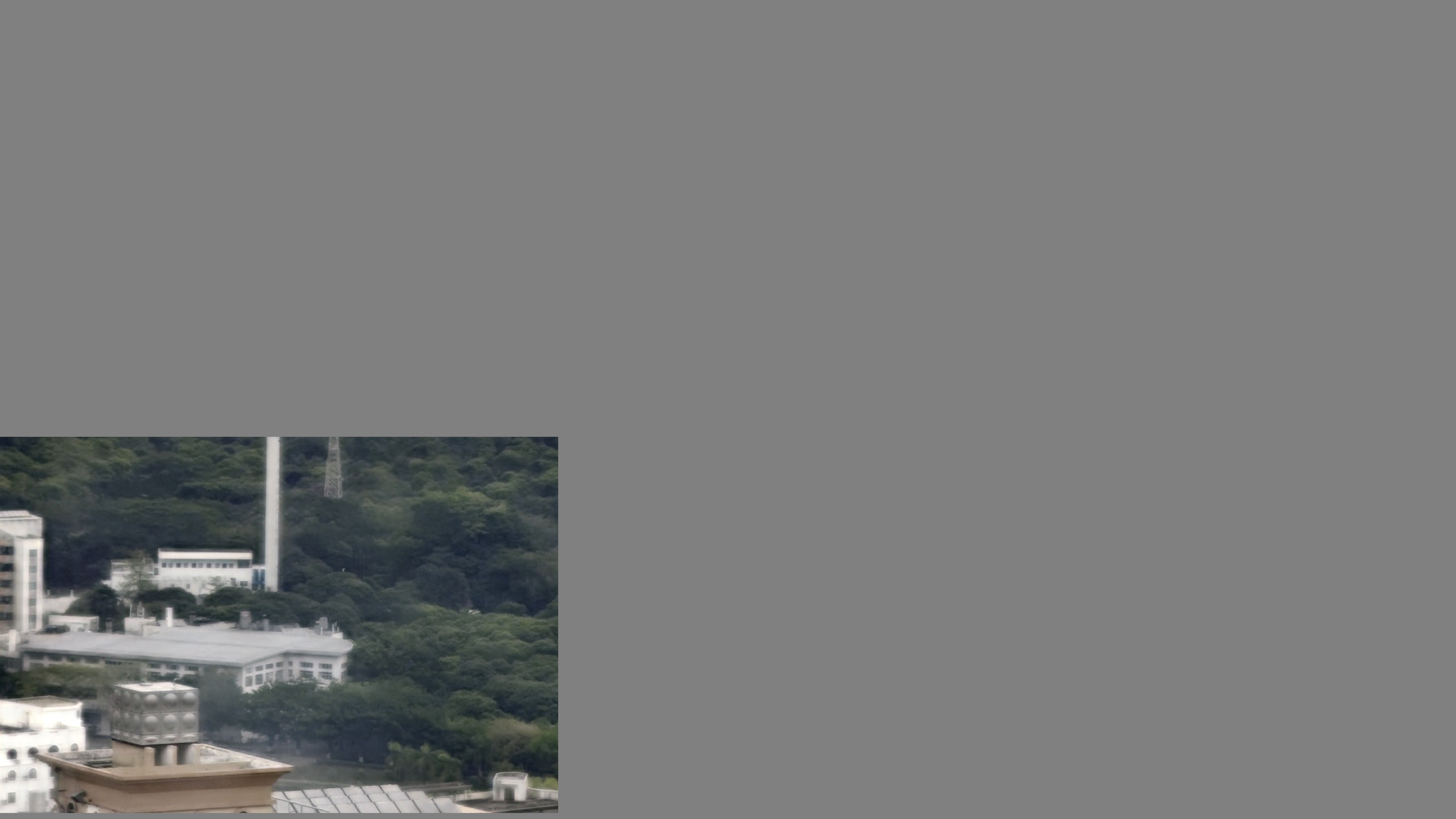}{77 108 2458 1231} \\[0pt]
    \rotatebox{90}{\scriptsize\textbf{(e)}} &
      \colorbox{ourgreen}{\zoomcrop{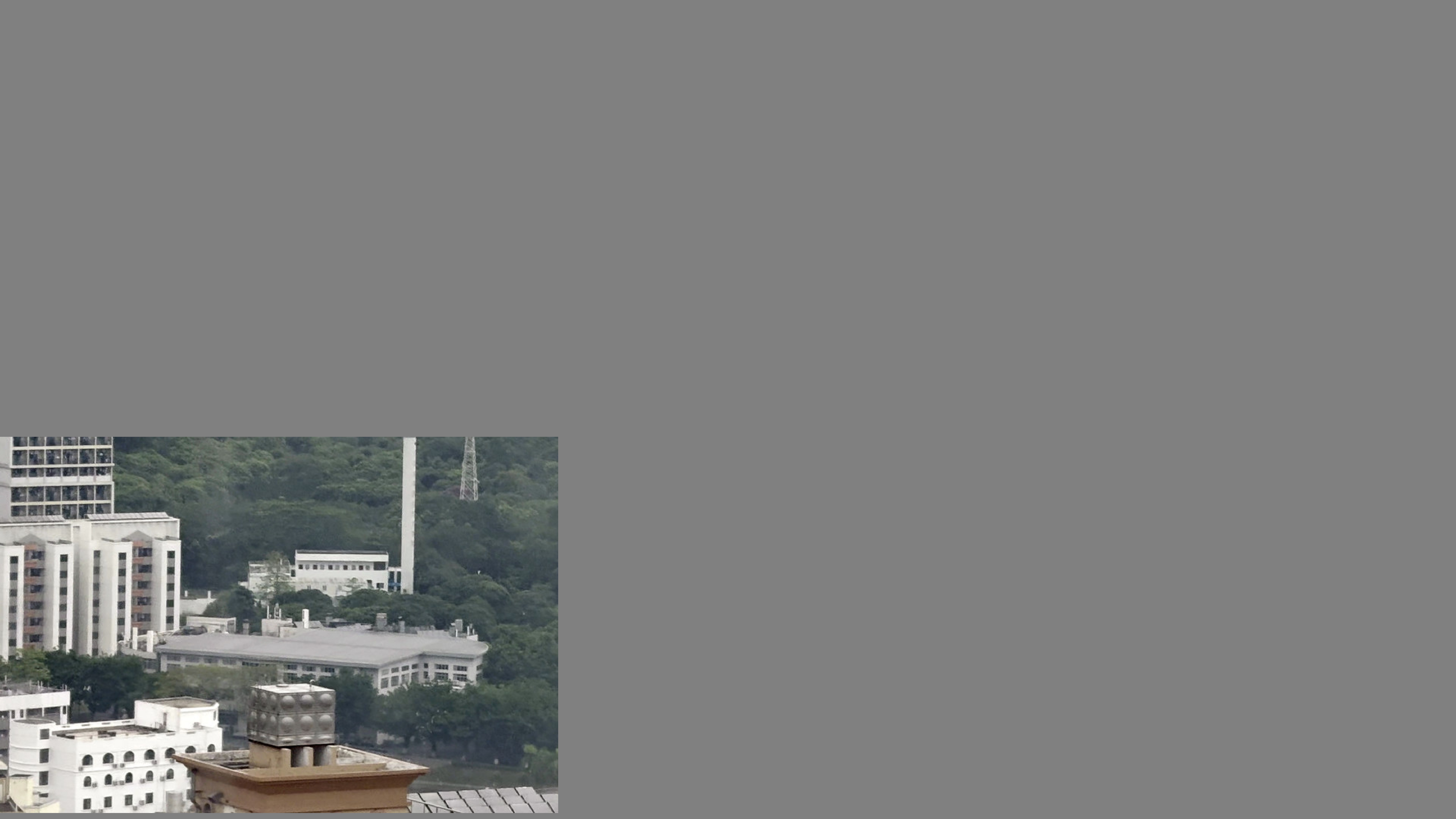}{77 108 2458 1231}} &
      \colorbox{ourgreen}{\zoomcrop{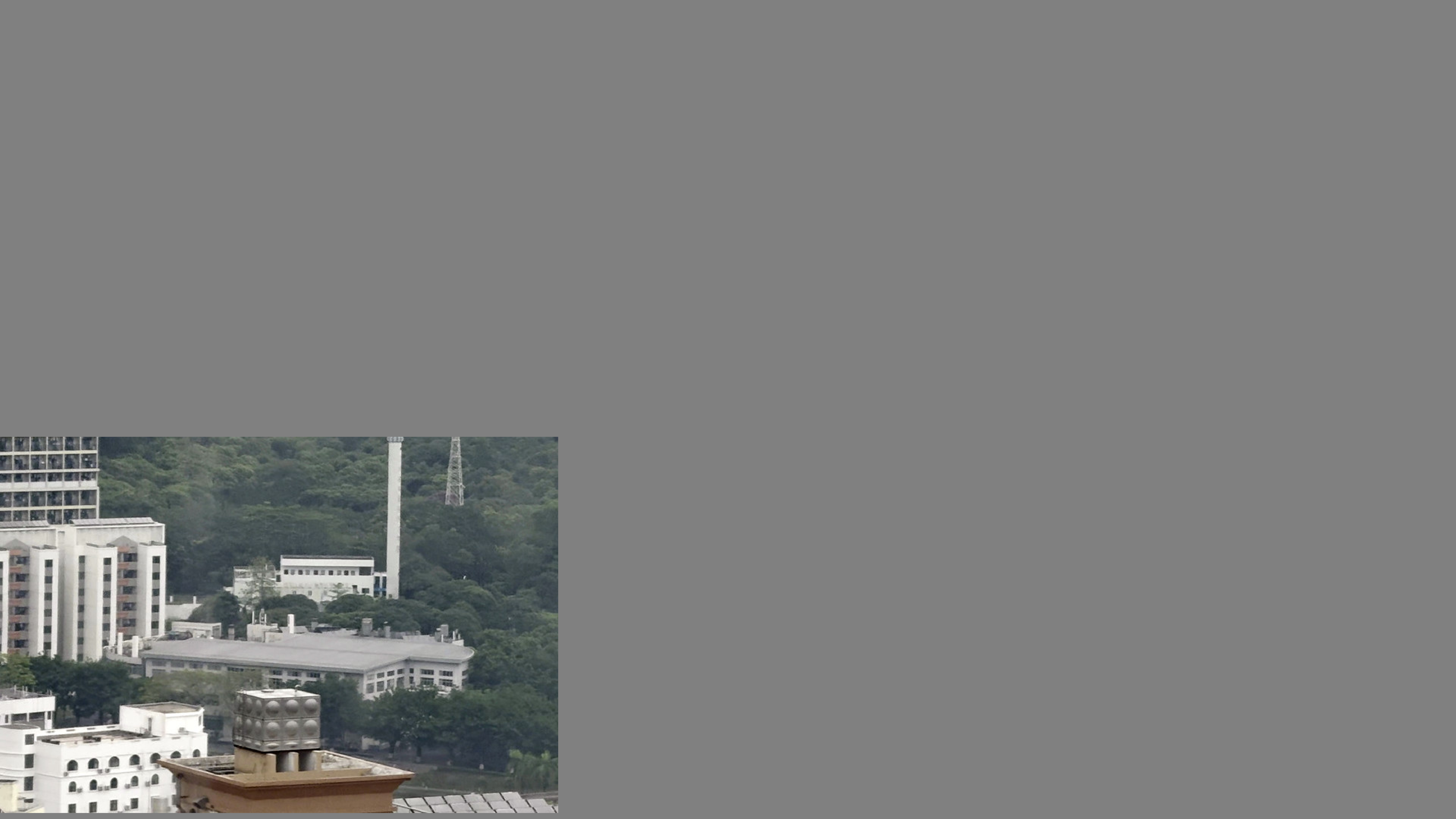}{77 108 2458 1231}} &
      \colorbox{ourgreen}{\zoomcrop{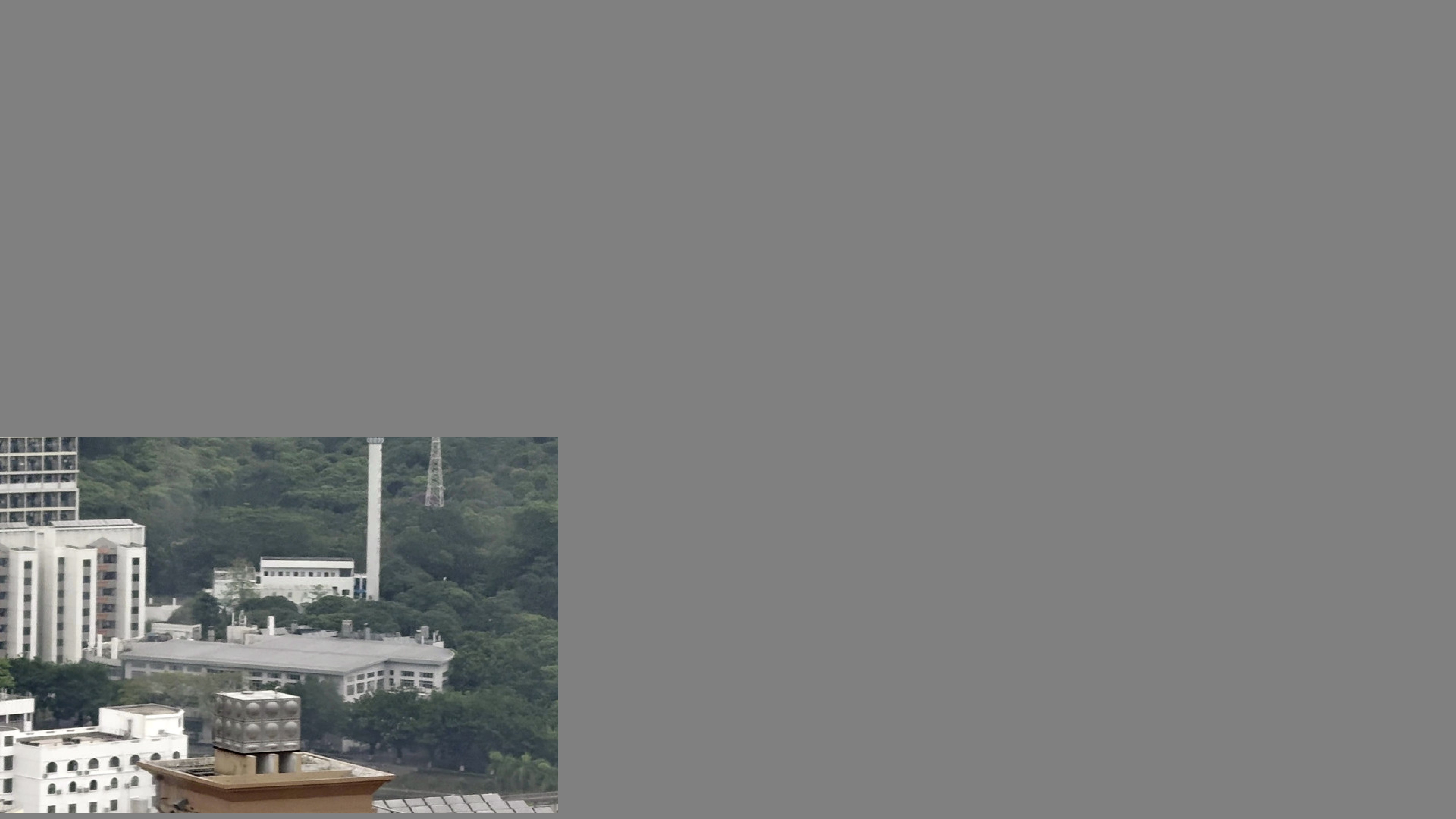}{77 108 2458 1231}} &
      \colorbox{ourgreen}{\zoomcrop{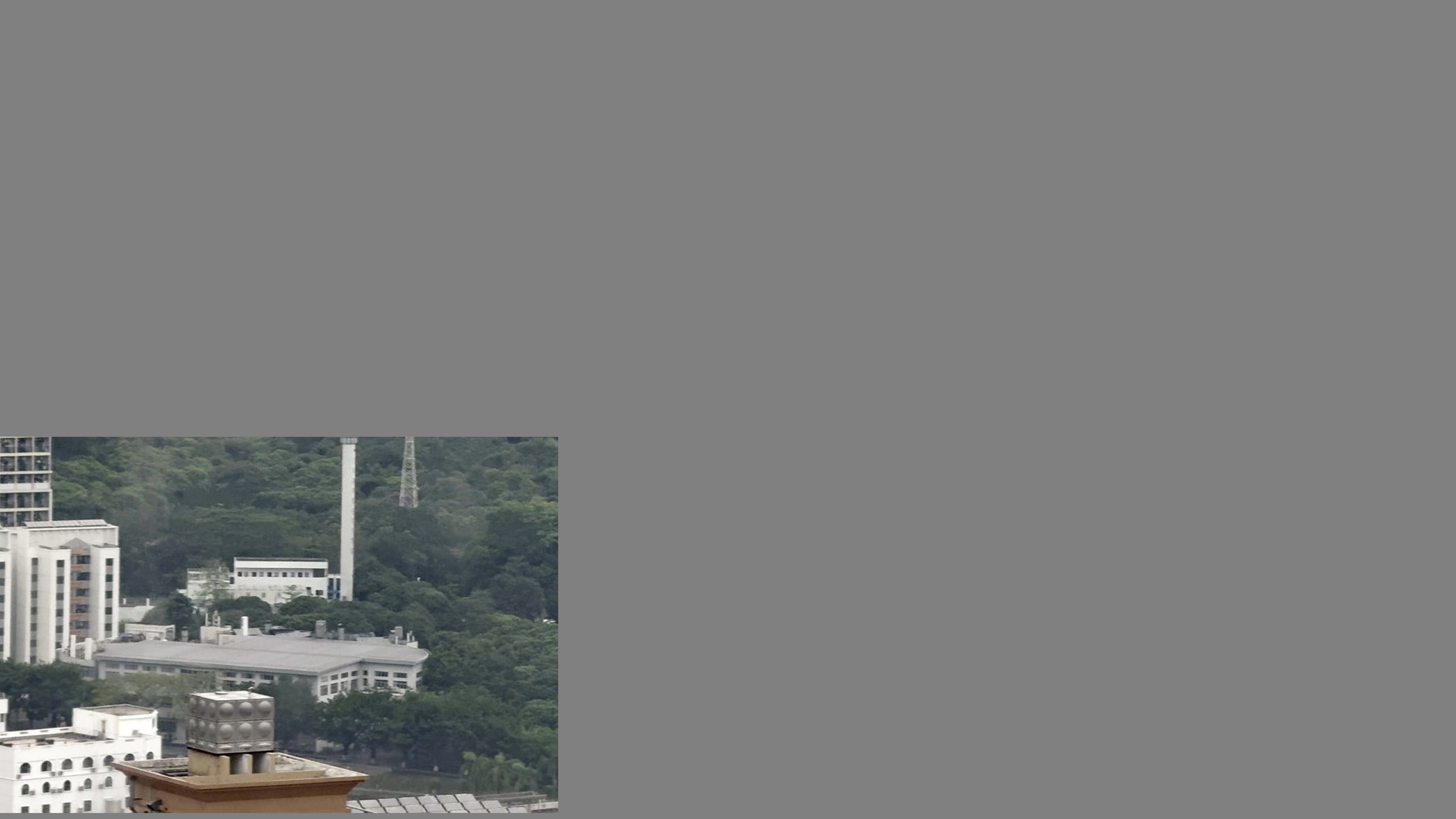}{77 108 2458 1231}} &
      \colorbox{ourgreen}{\zoomcrop{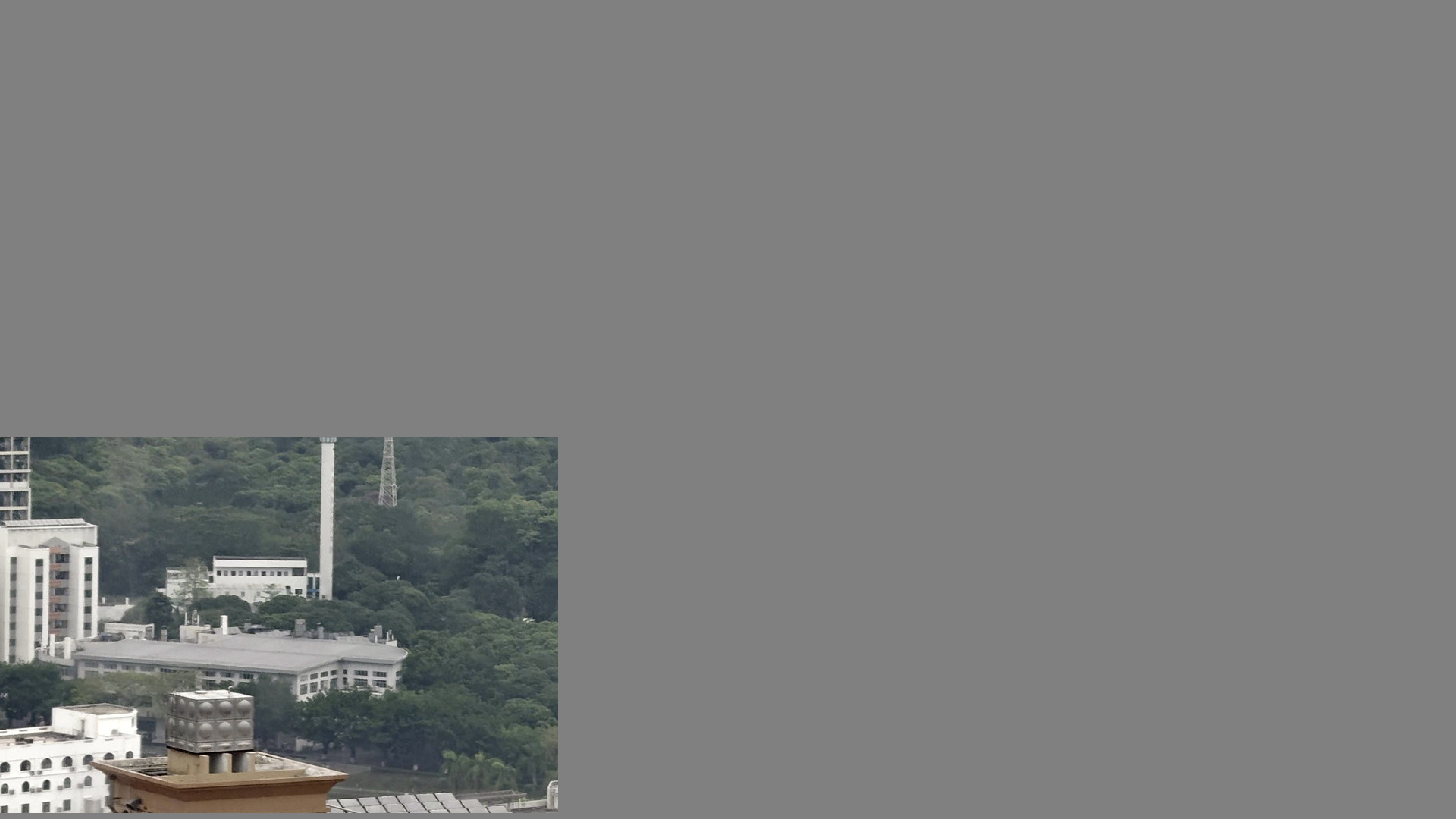}{77 108 2458 1231}} &
      \colorbox{ourgreen}{\zoomcrop{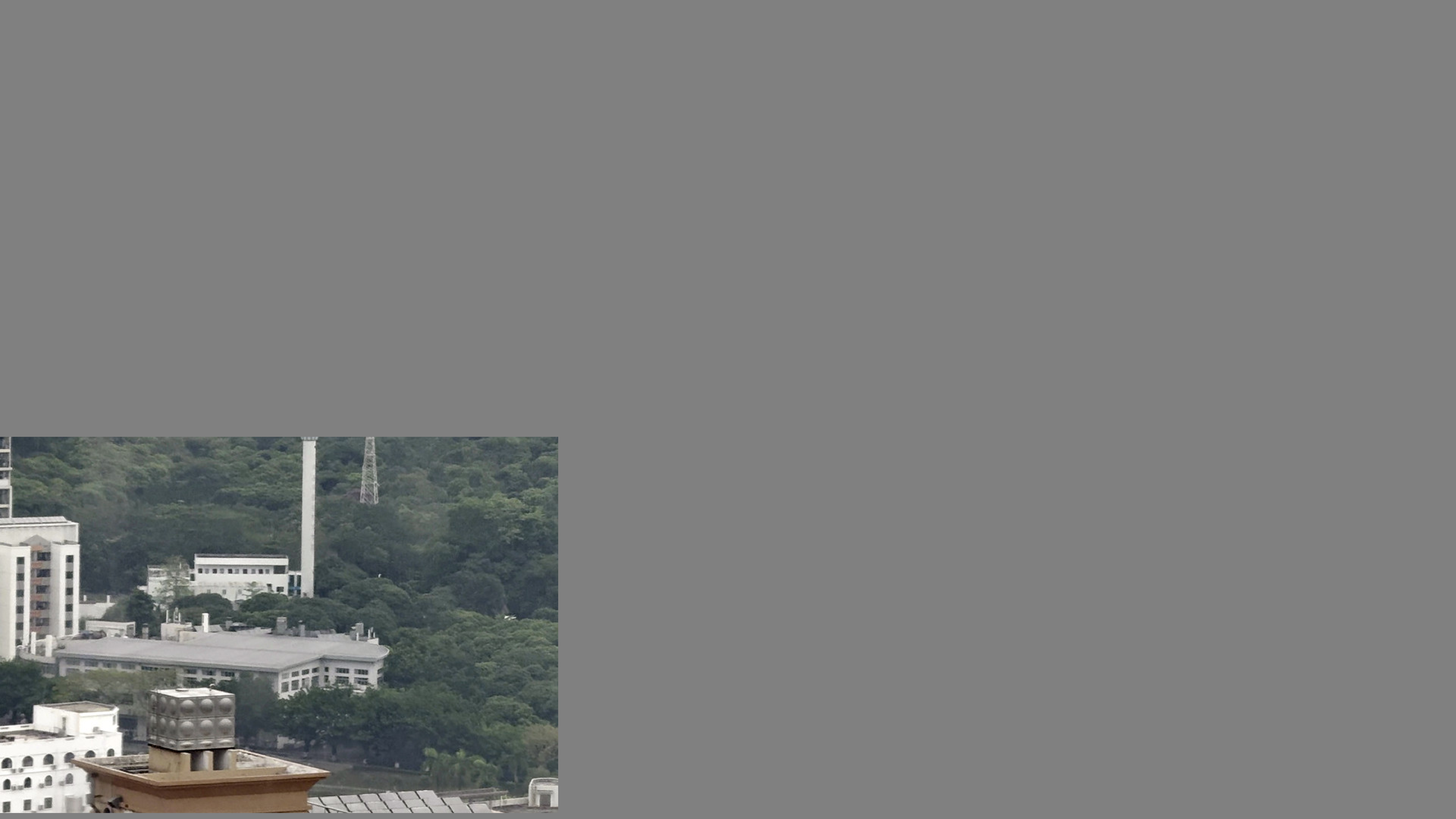}{77 108 2458 1231}} &
      \colorbox{ourgreen}{\zoomcrop{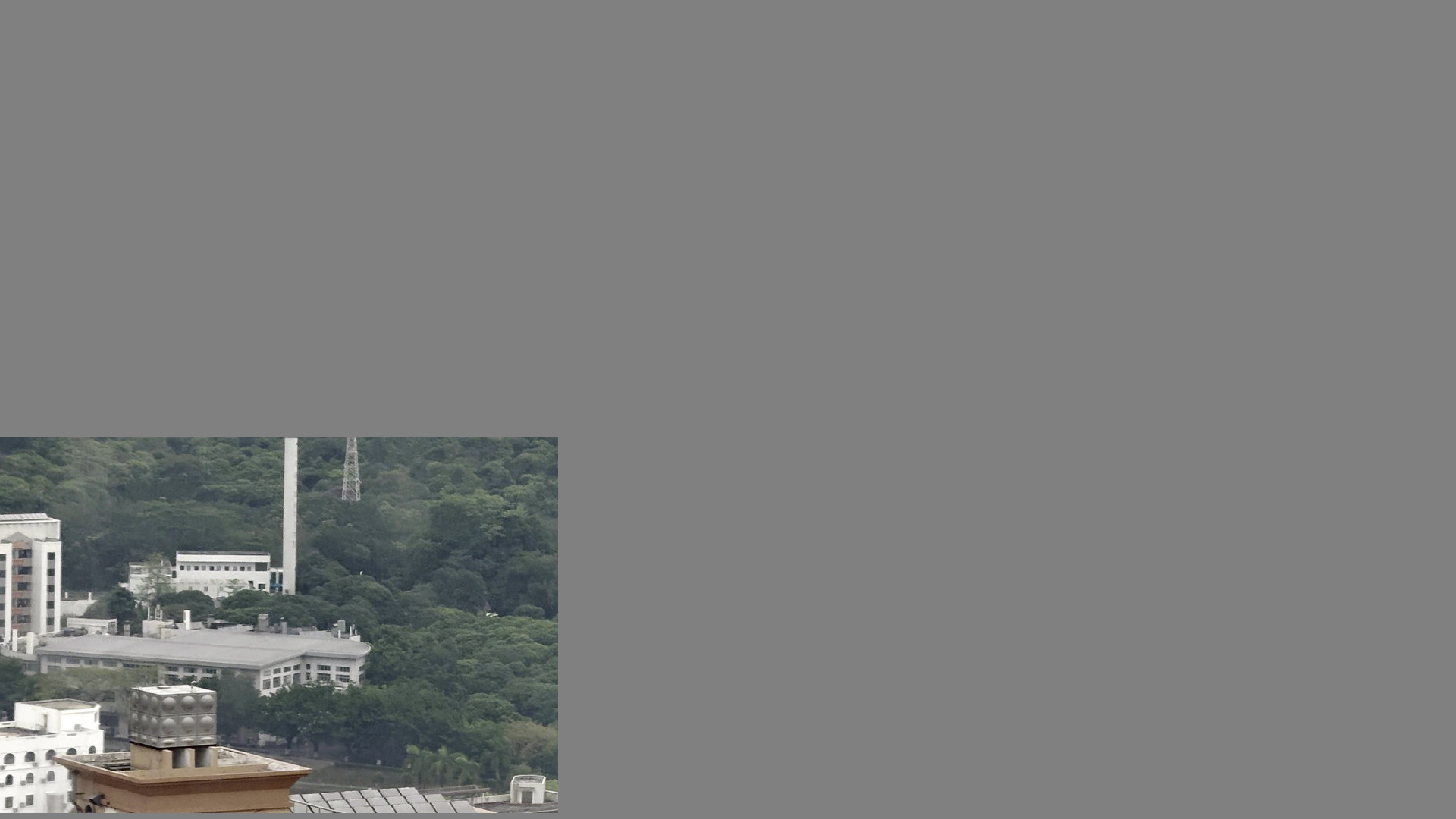}{77 108 2458 1231}} &
      \colorbox{ourgreen}{\zoomcrop{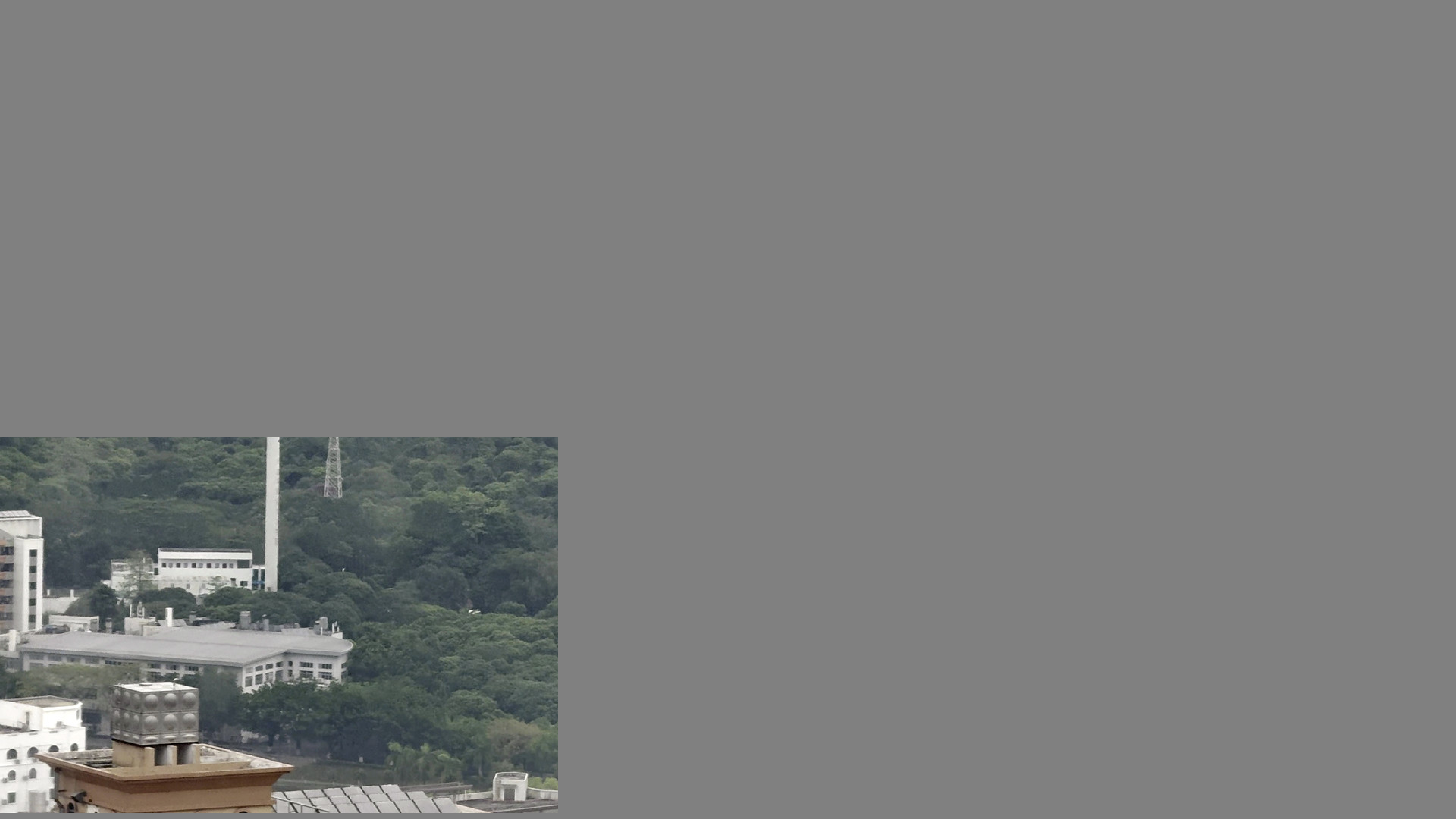}{77 108 2458 1231}} \\[0pt]
    \rotatebox{90}{\scriptsize\textbf{Video\,2}} &
      \fullroi{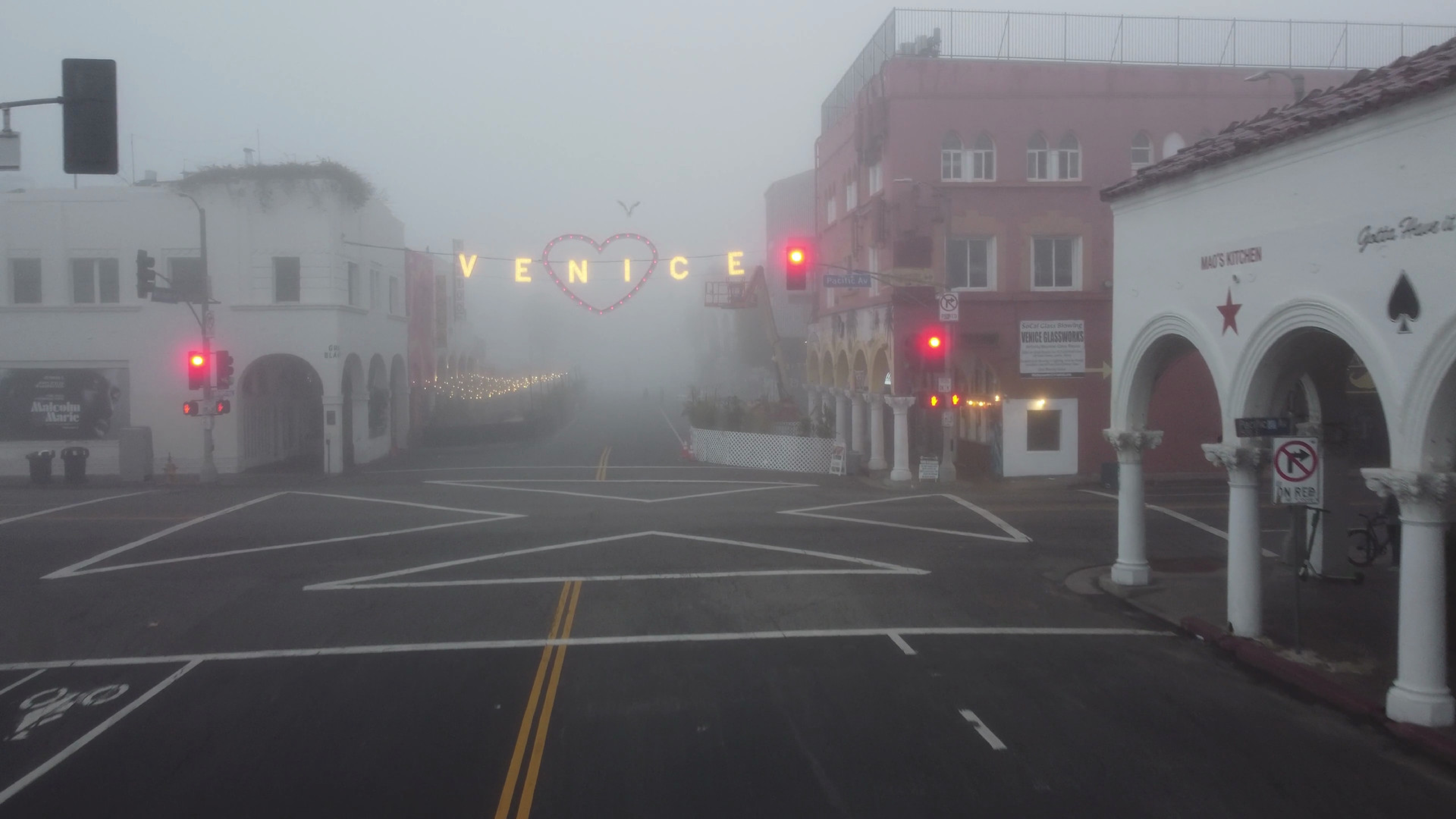}{0.54}{0.40}{0.83}{0.704} &
      \fullroi{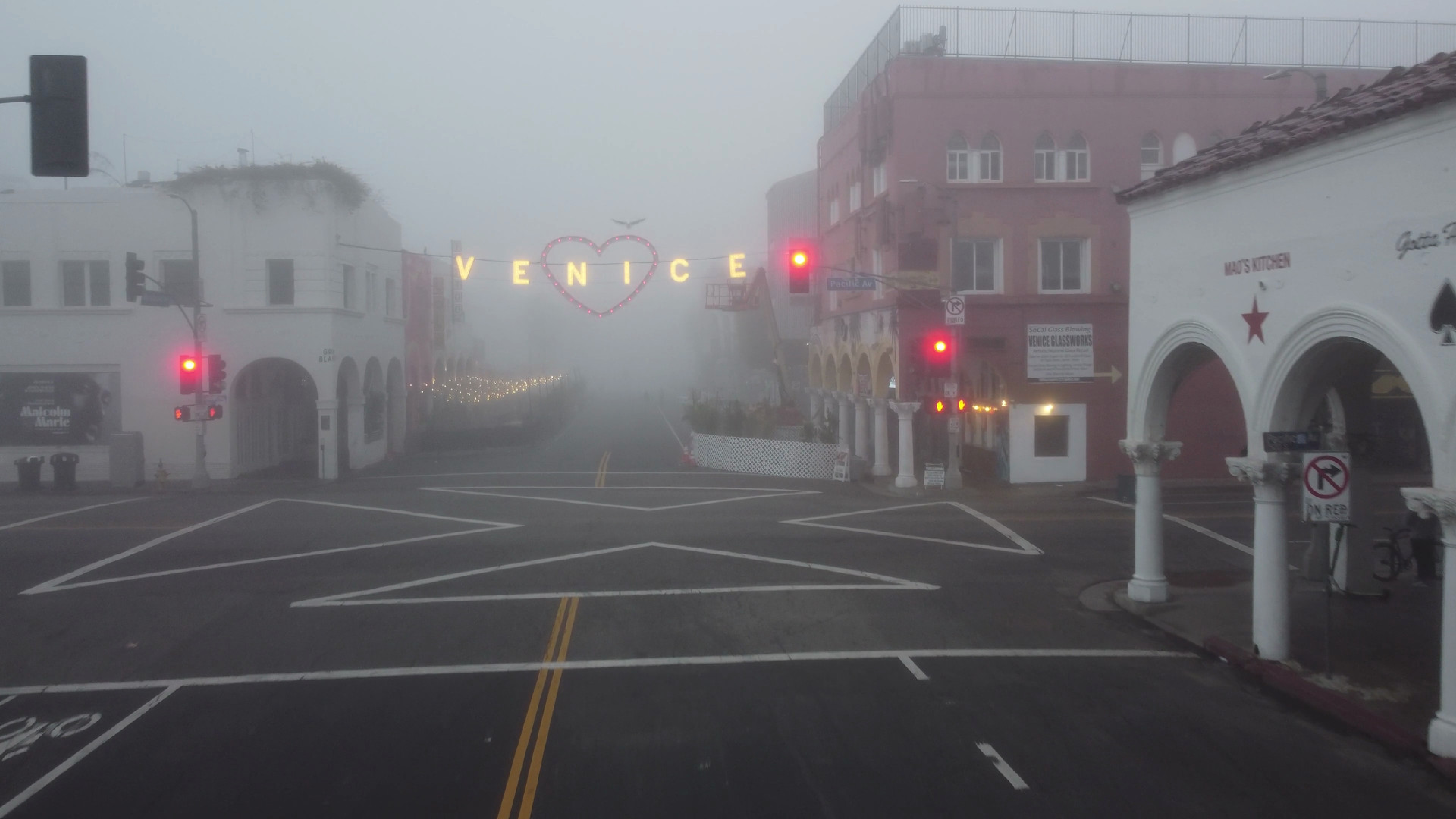}{0.54}{0.40}{0.83}{0.704} &
      \fullroi{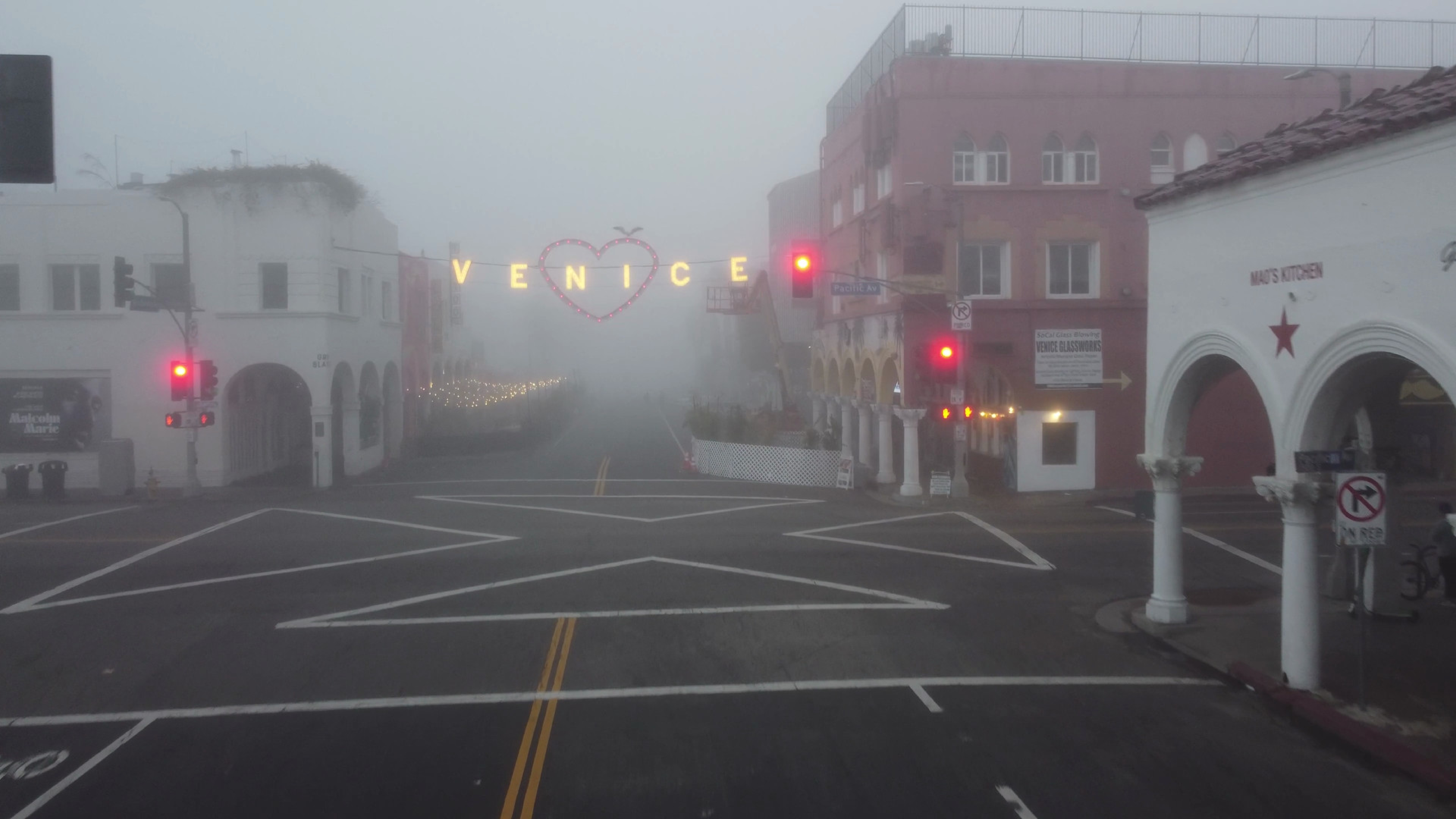}{0.54}{0.40}{0.83}{0.704} &
      \fullroi{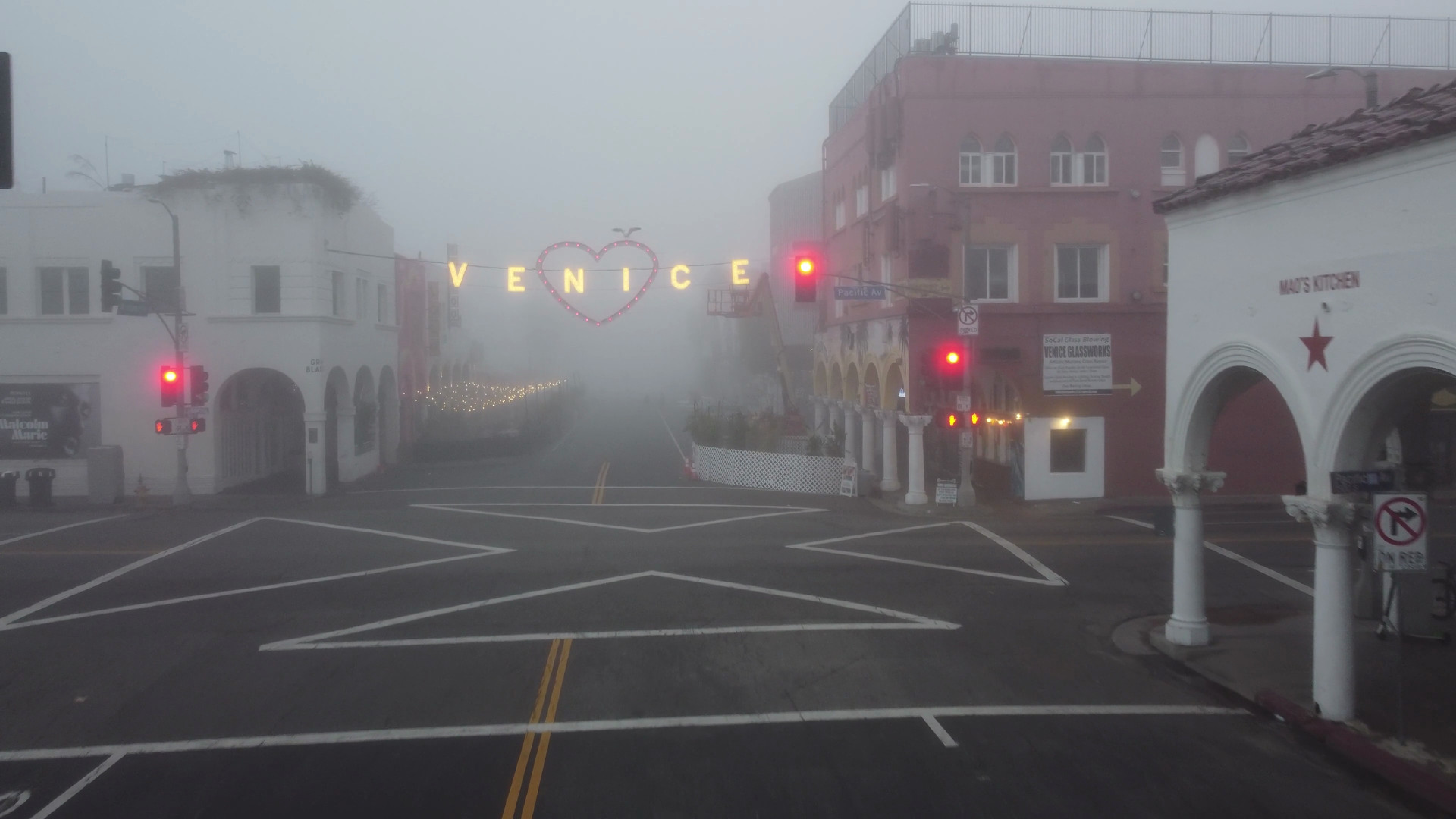}{0.54}{0.40}{0.83}{0.704} &
      \fullroi{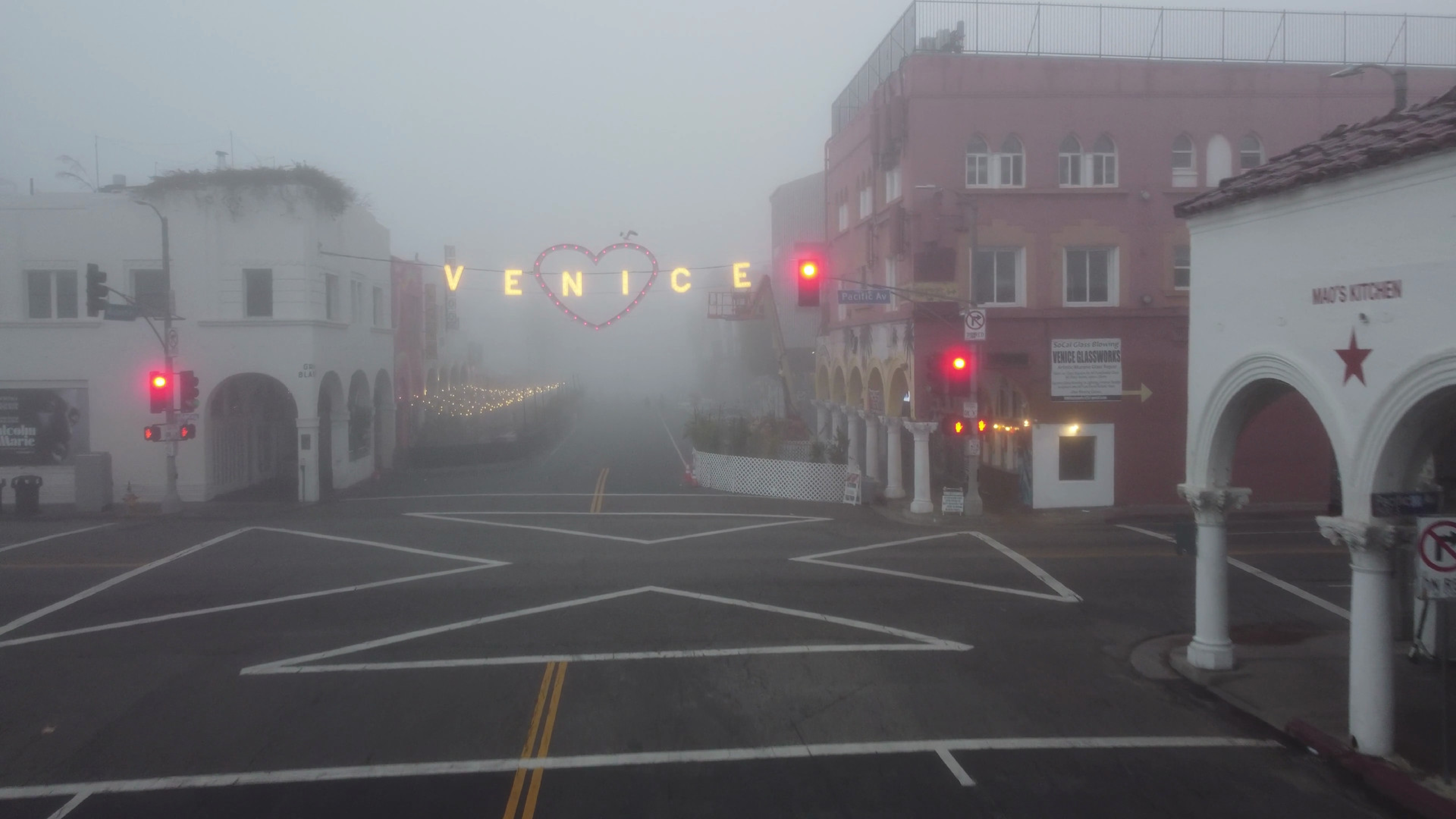}{0.54}{0.40}{0.83}{0.704} &
      \fullroi{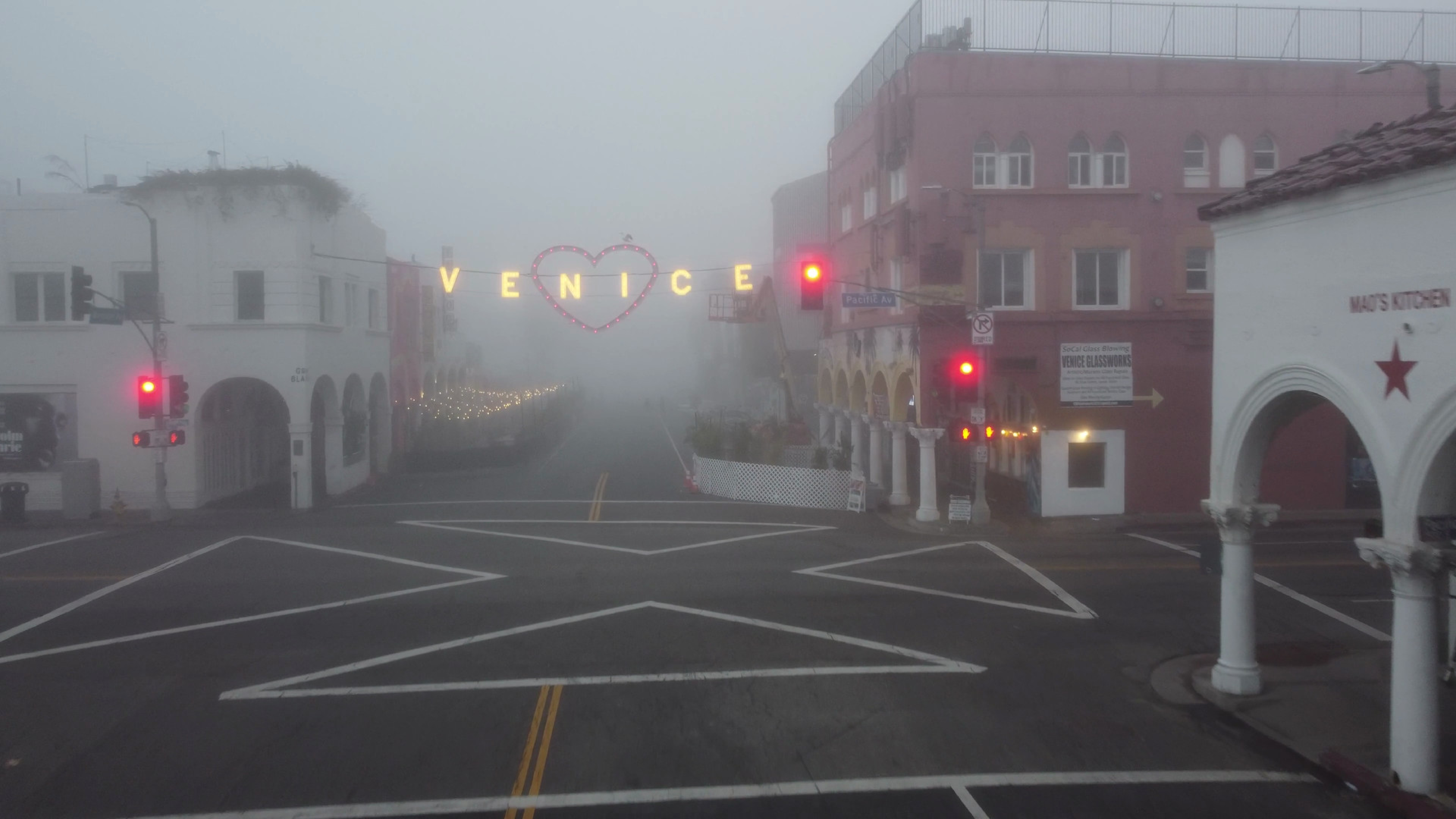}{0.54}{0.40}{0.83}{0.704} &
      \fullroi{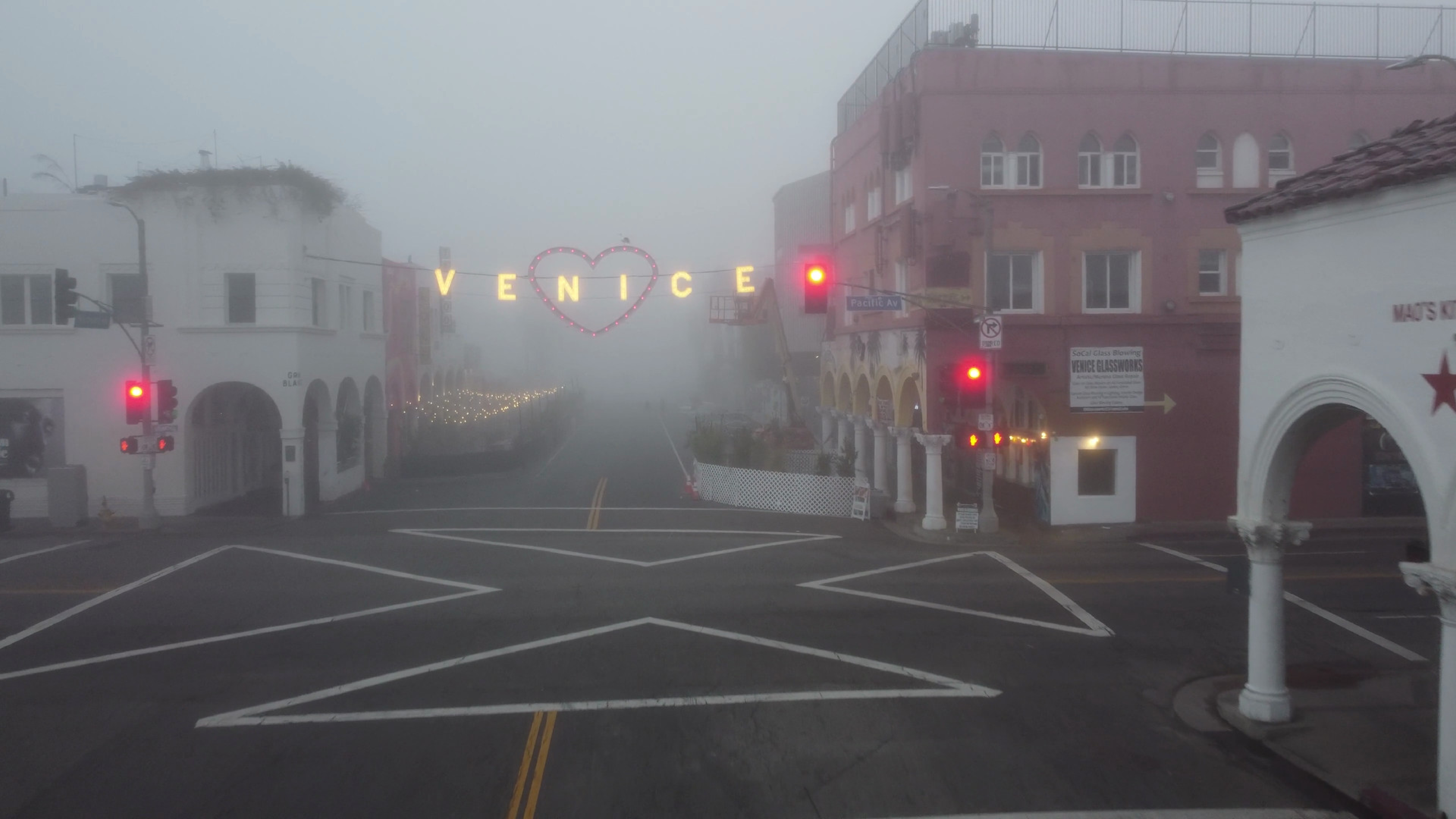}{0.54}{0.40}{0.83}{0.704} &
      \fullroi{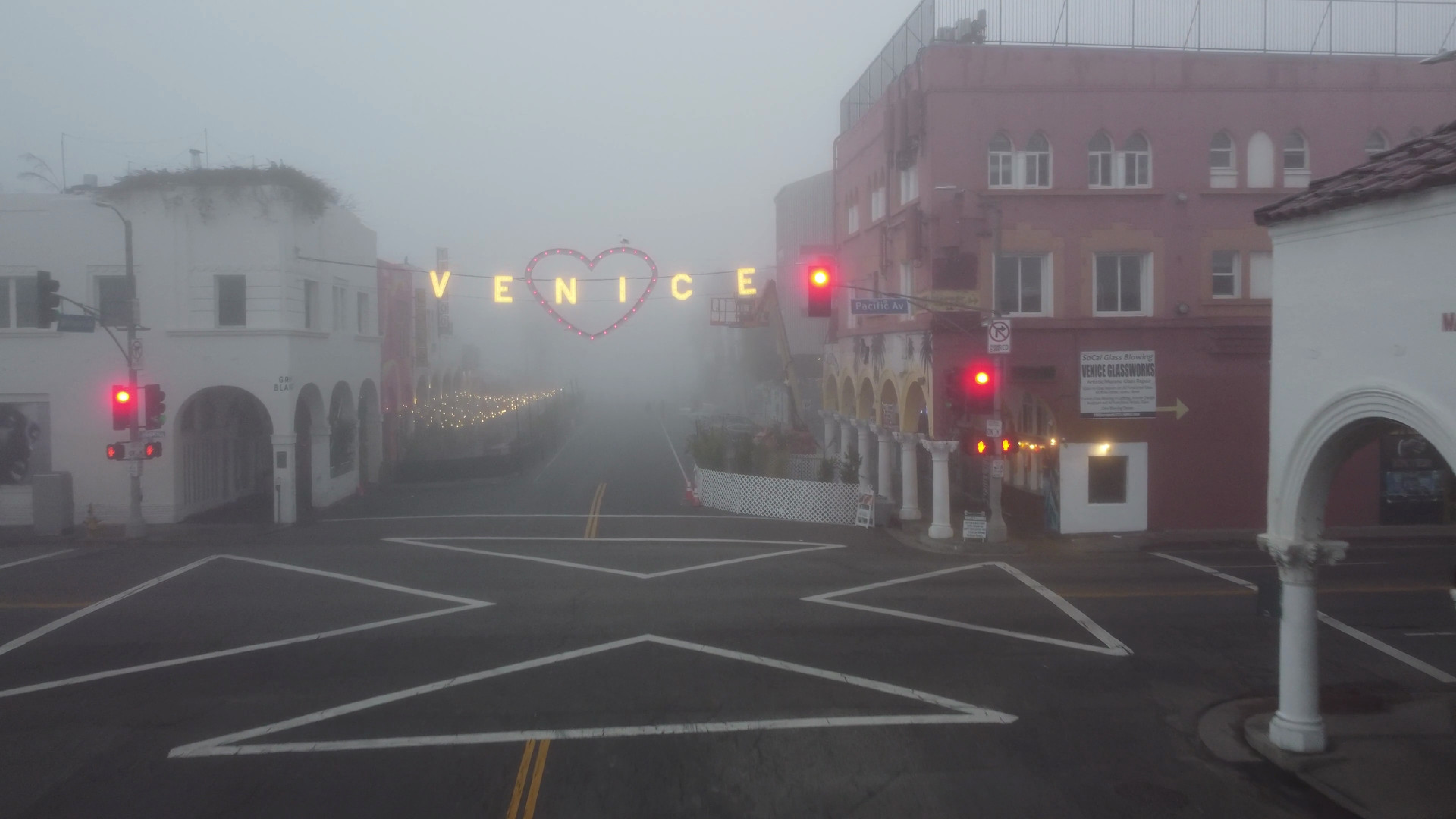}{0.54}{0.40}{0.83}{0.704} \\[0pt]
    \rotatebox{0}{\scriptsize\textbf{(a)}} &
      \zoomcrop{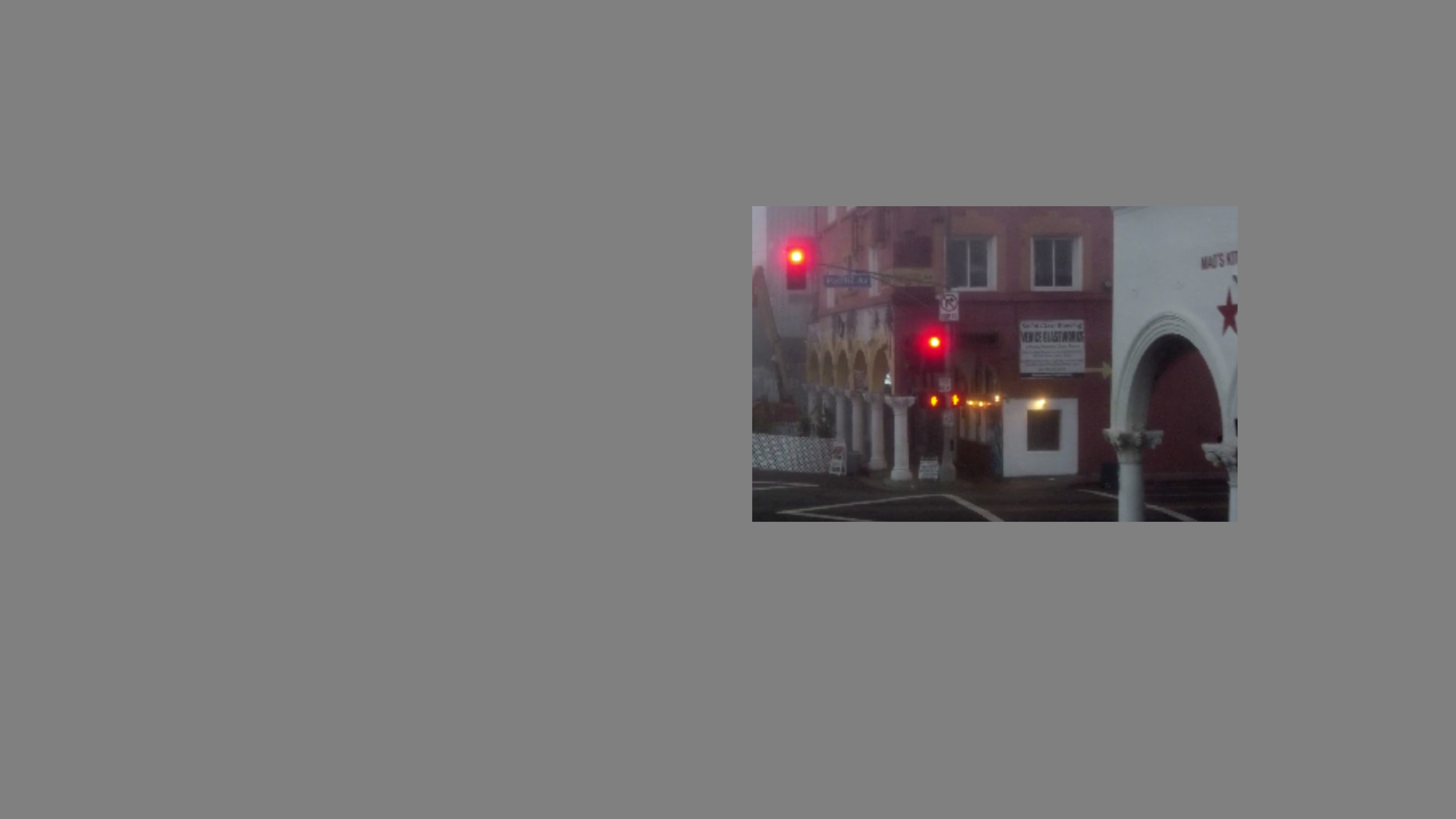}{2074 864 653 639} &
      \zoomcrop{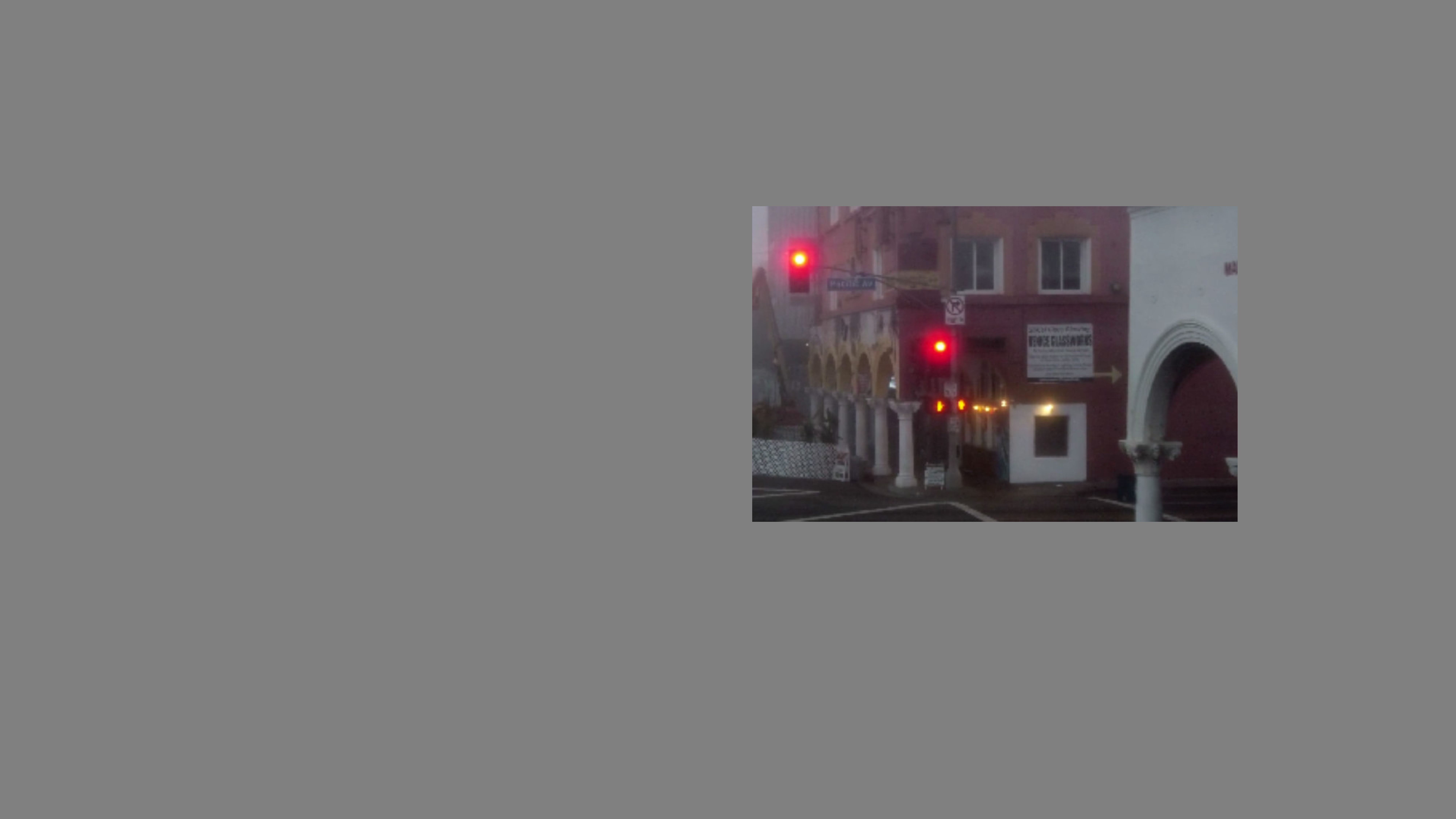}{2074 864 653 639} &
      \zoomcrop{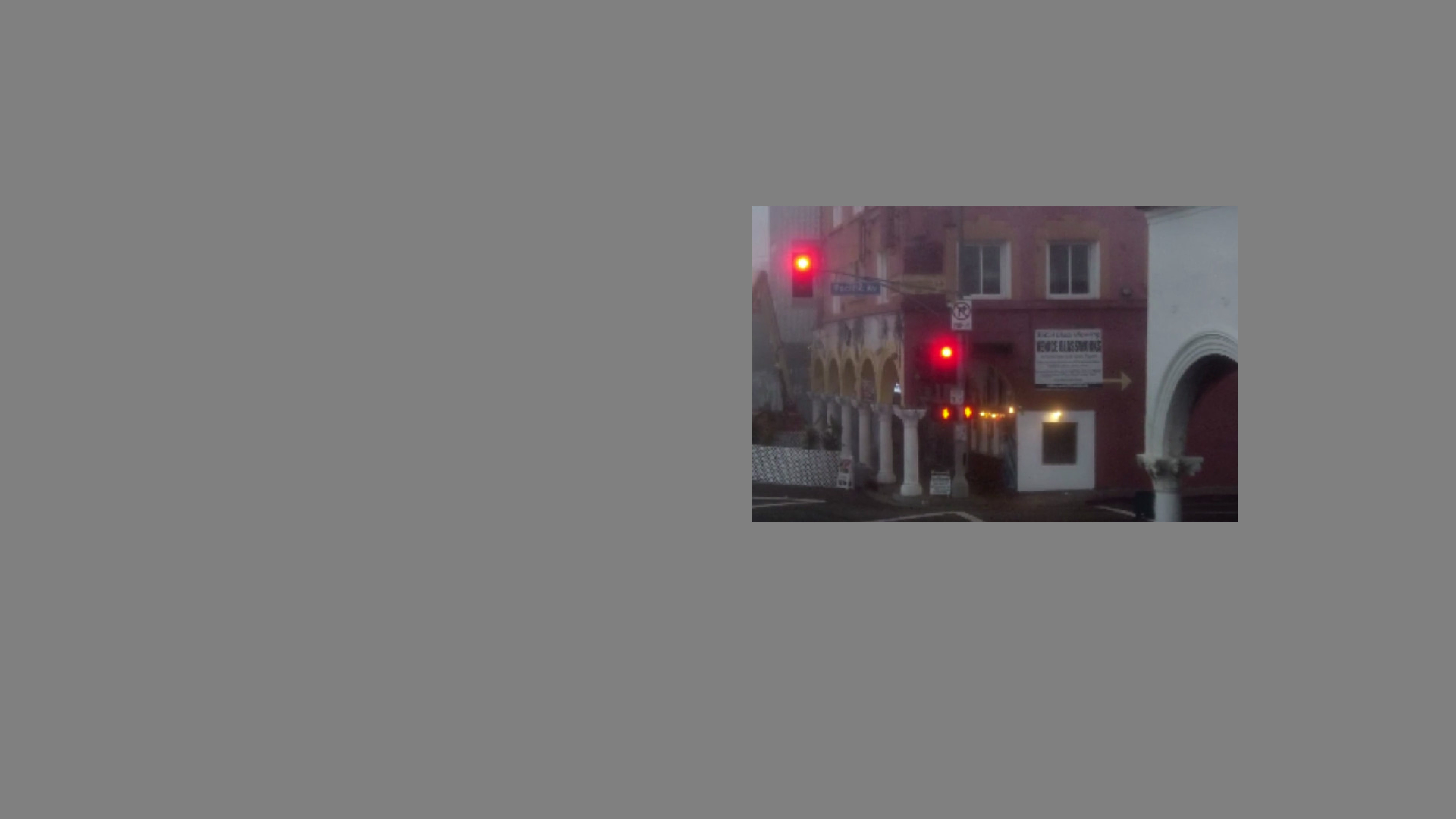}{2074 864 653 639} &
      \zoomcrop{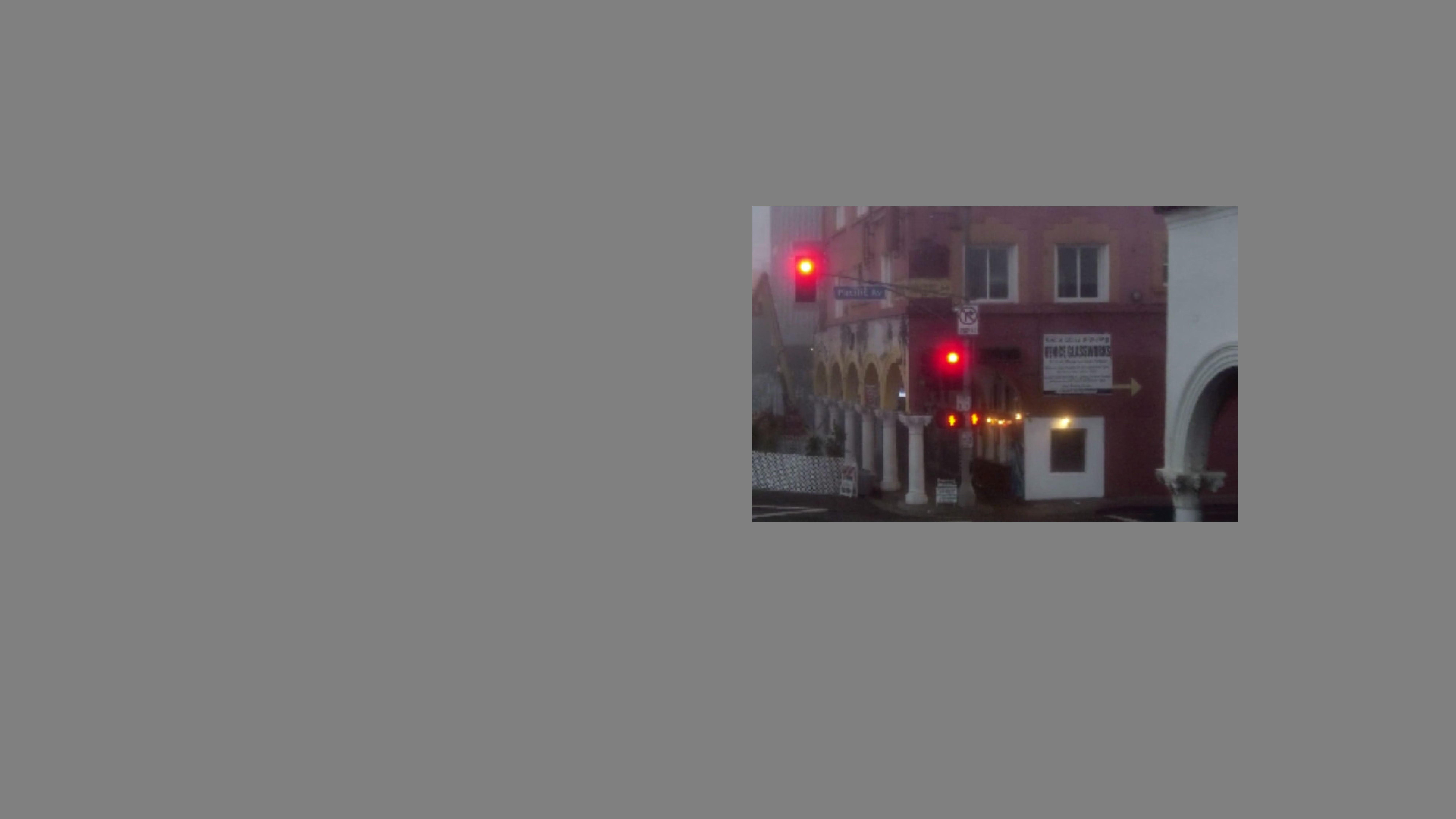}{2074 864 653 639} &
      \zoomcrop{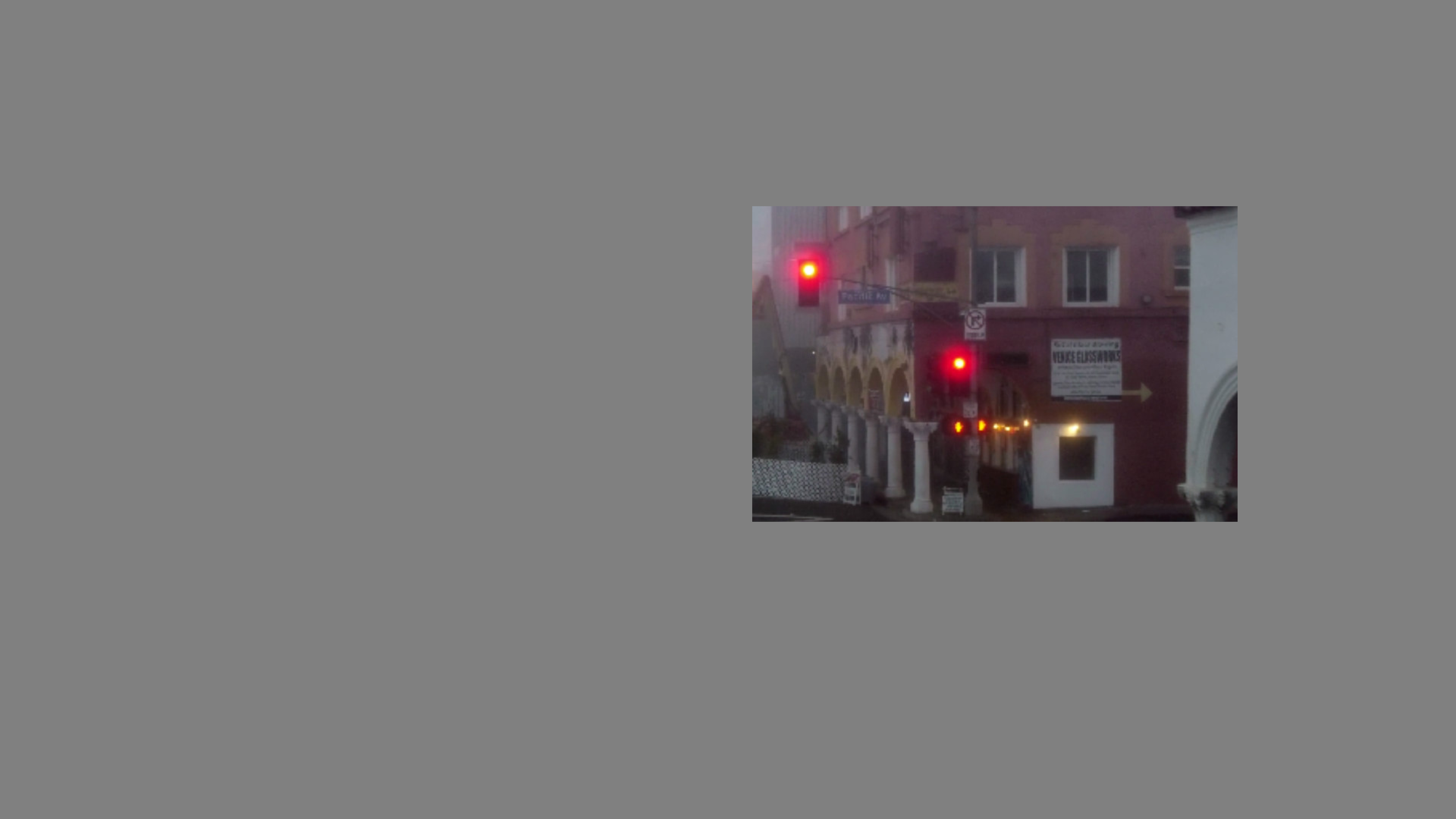}{2074 864 653 639} &
      \zoomcrop{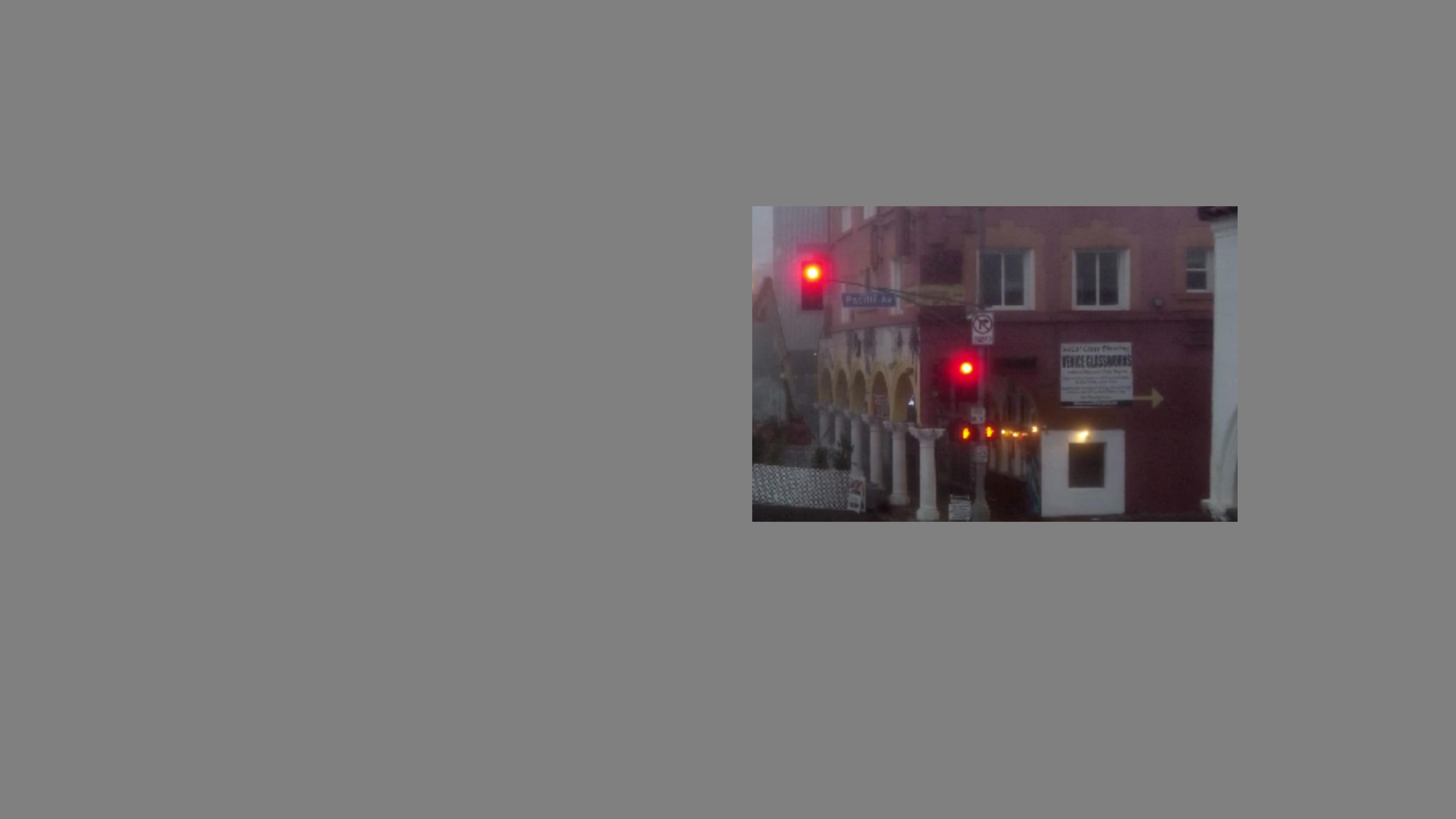}{2074 864 653 639} &
      \zoomcrop{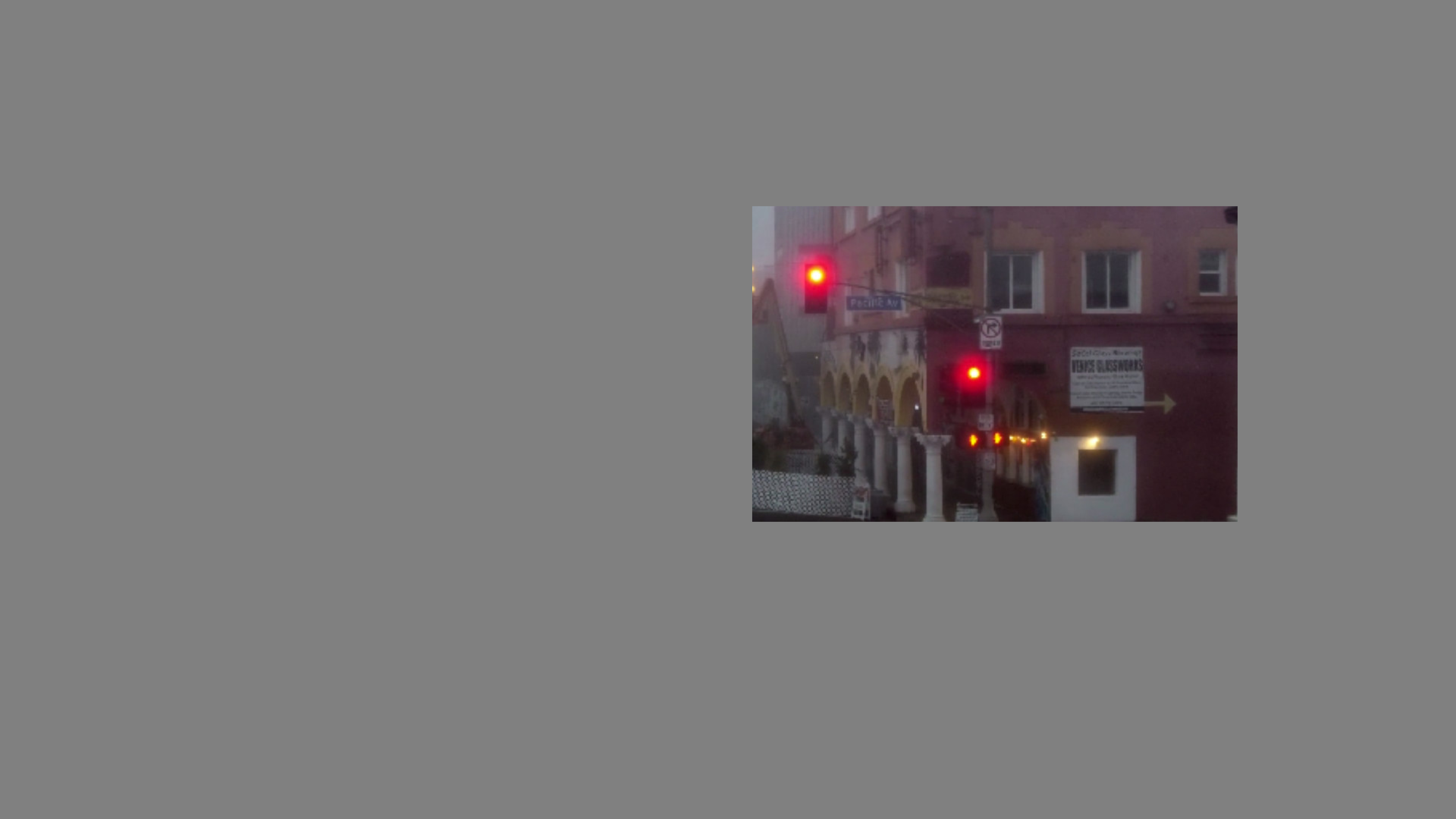}{2074 864 653 639} &
      \zoomcrop{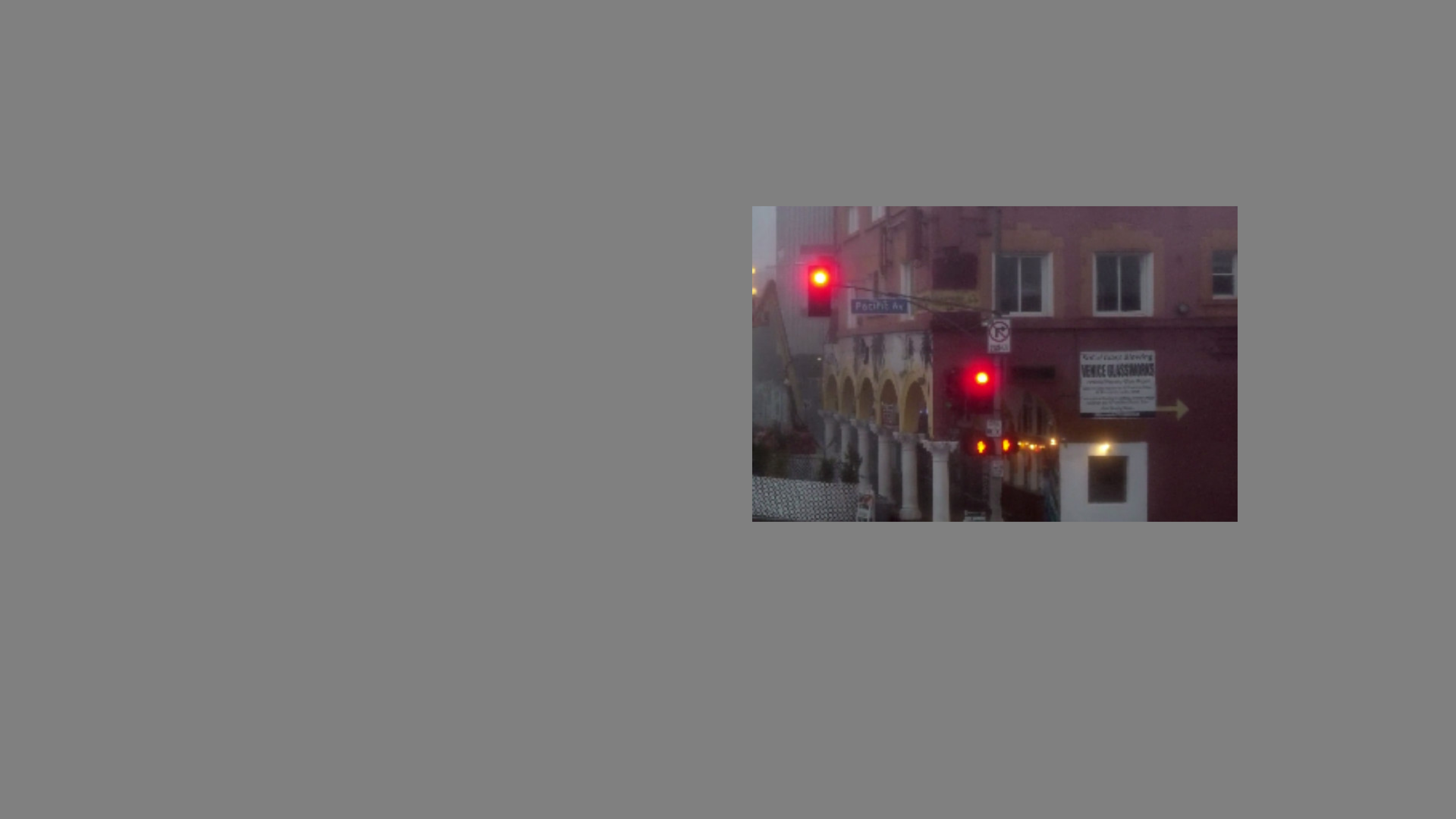}{2074 864 653 639} \\[0pt]
    \rotatebox{0}{\scriptsize\textbf{(b)}} &
      \zoomcrop{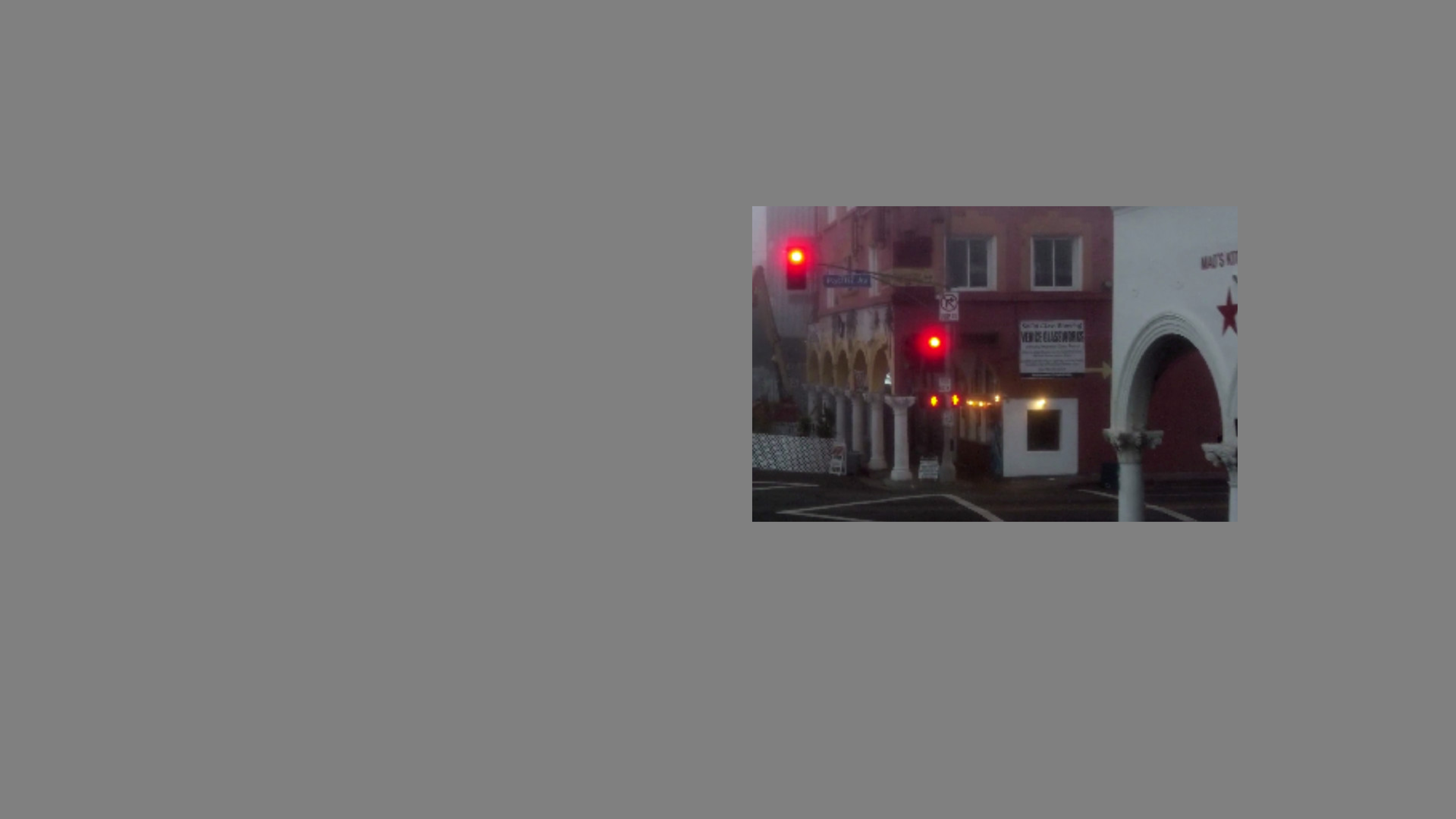}{2074 864 653 639} &
      \zoomcrop{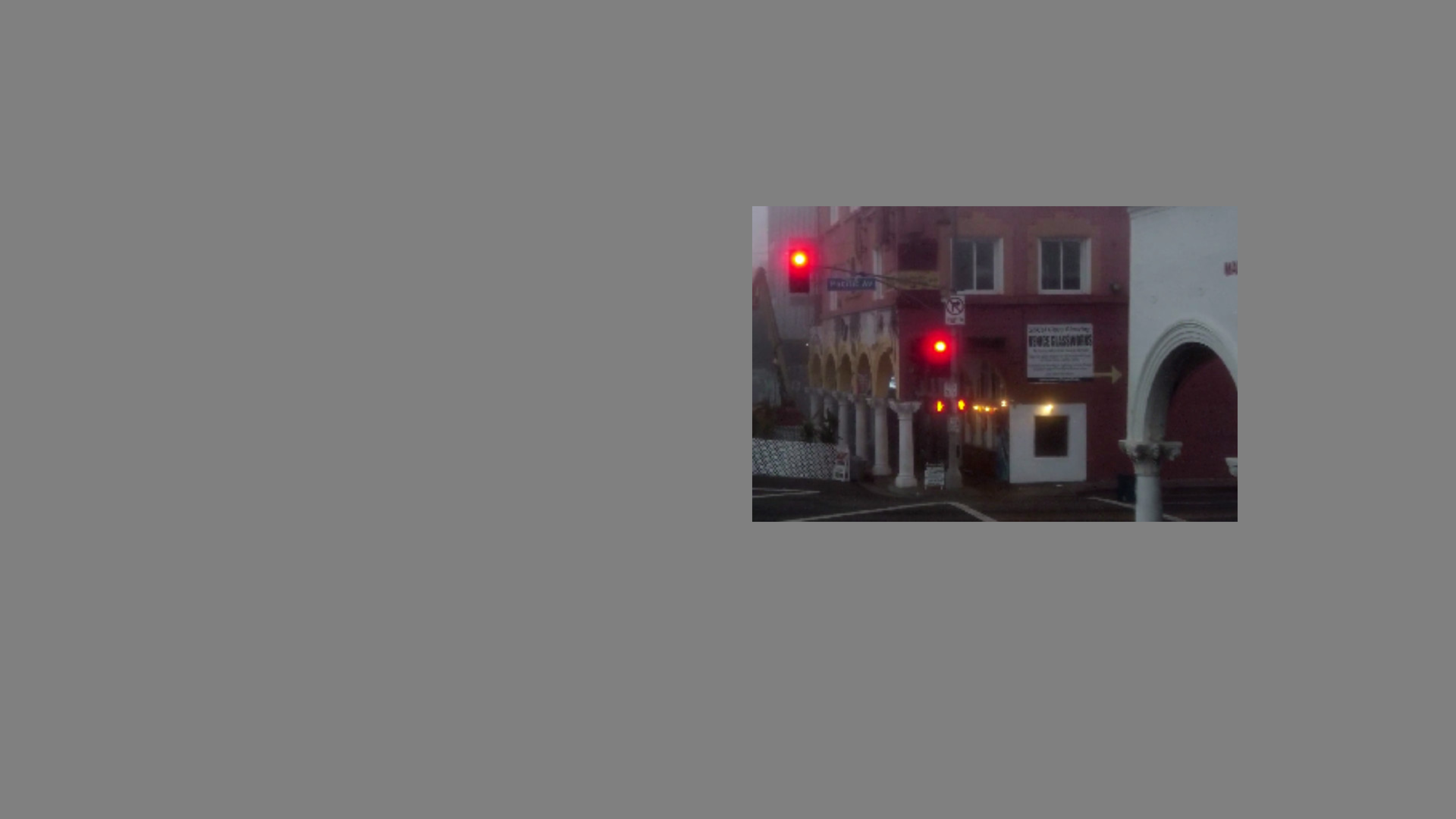}{2074 864 653 639} &
      \zoomcrop{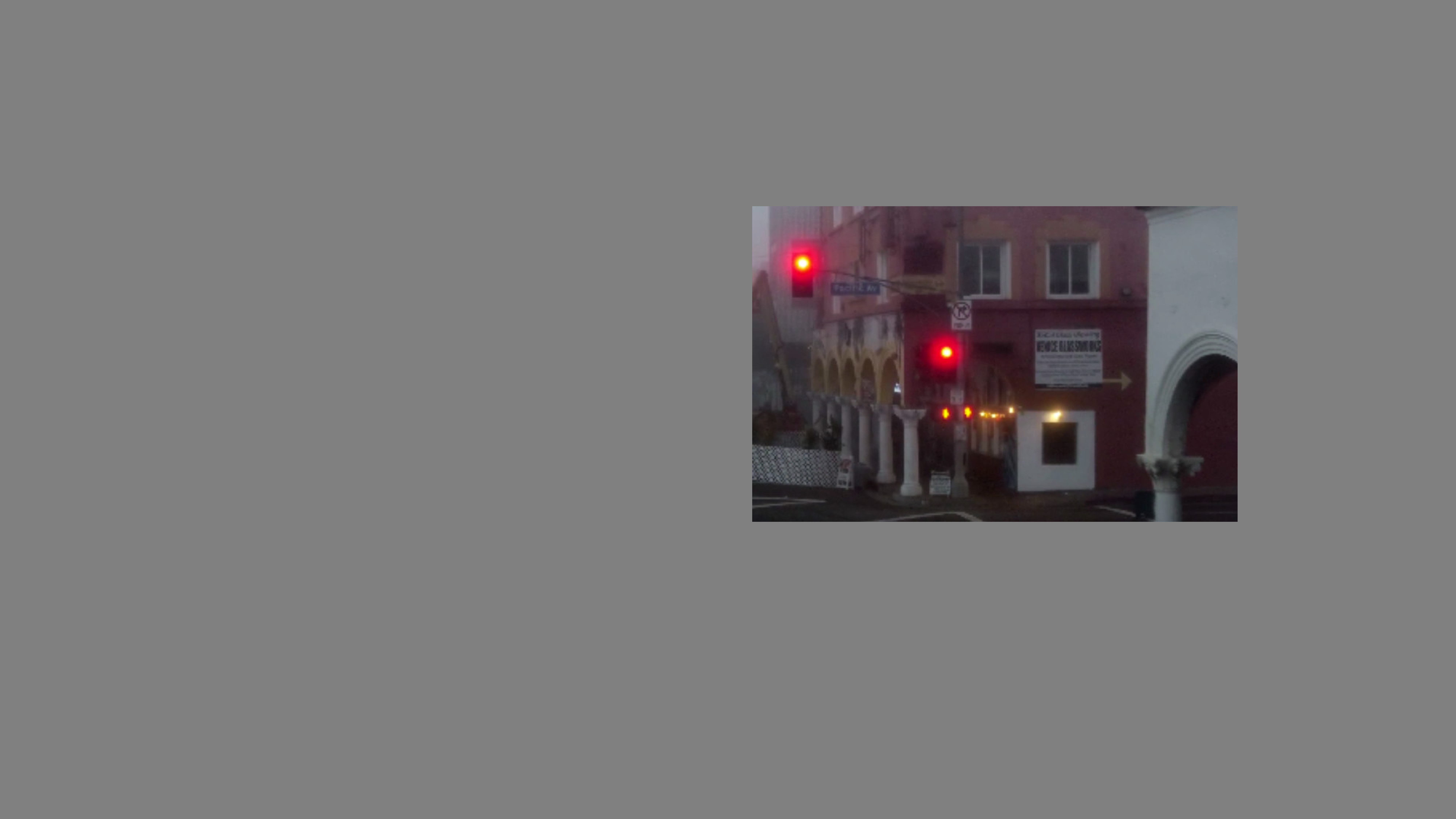}{2074 864 653 639} &
      \zoomcrop{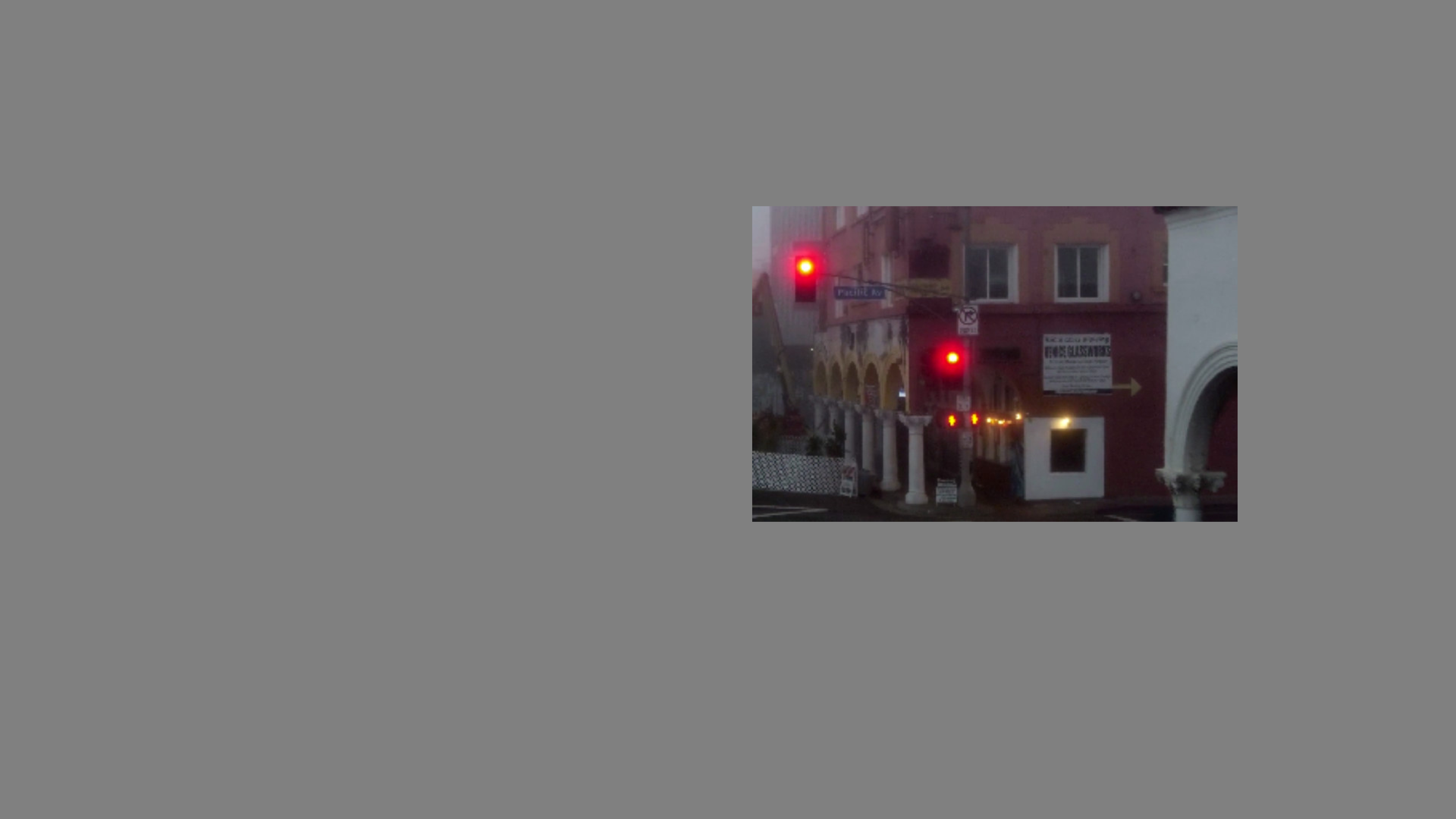}{2074 864 653 639} &
      \zoomcrop{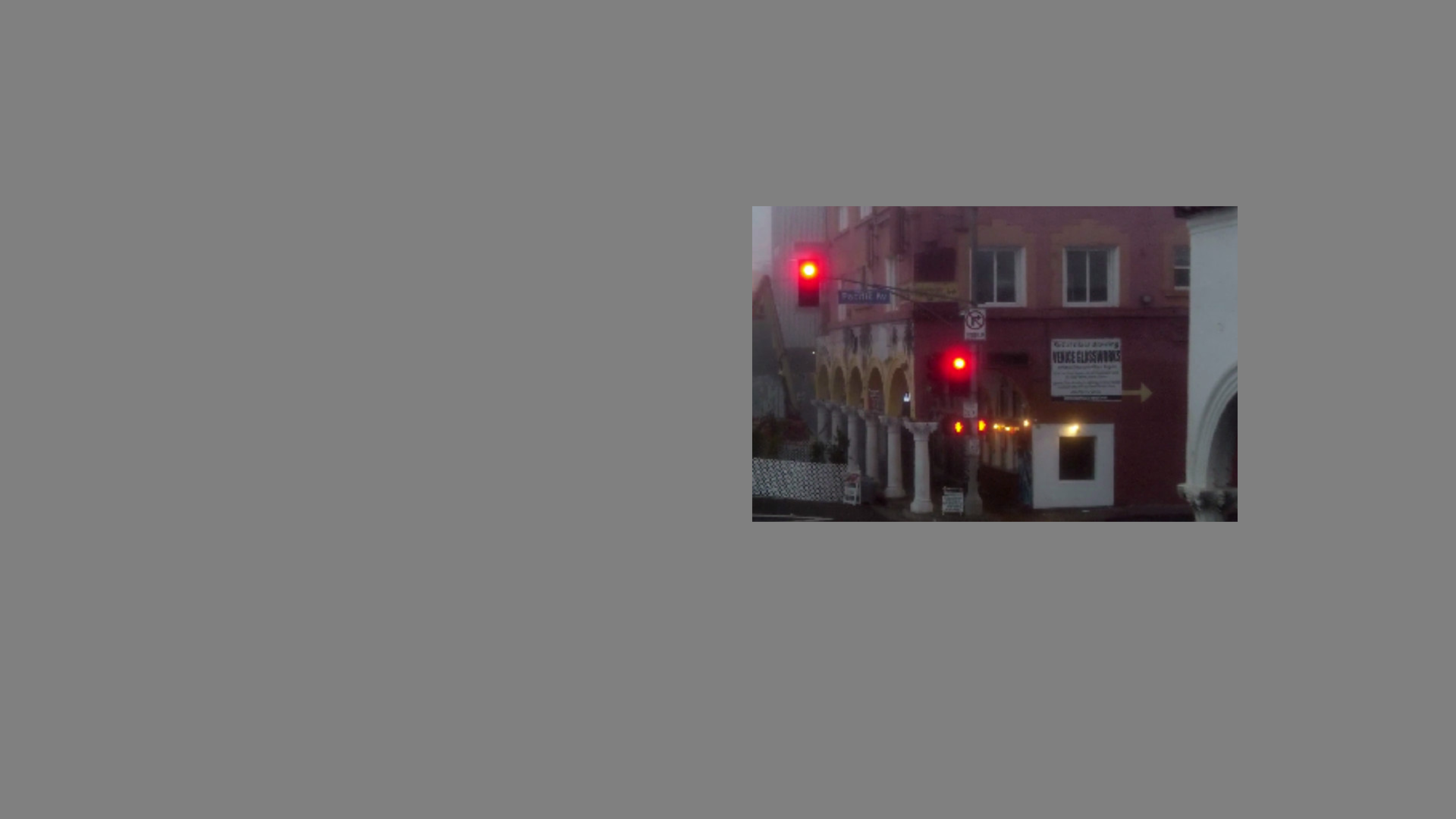}{2074 864 653 639} &
      \zoomcrop{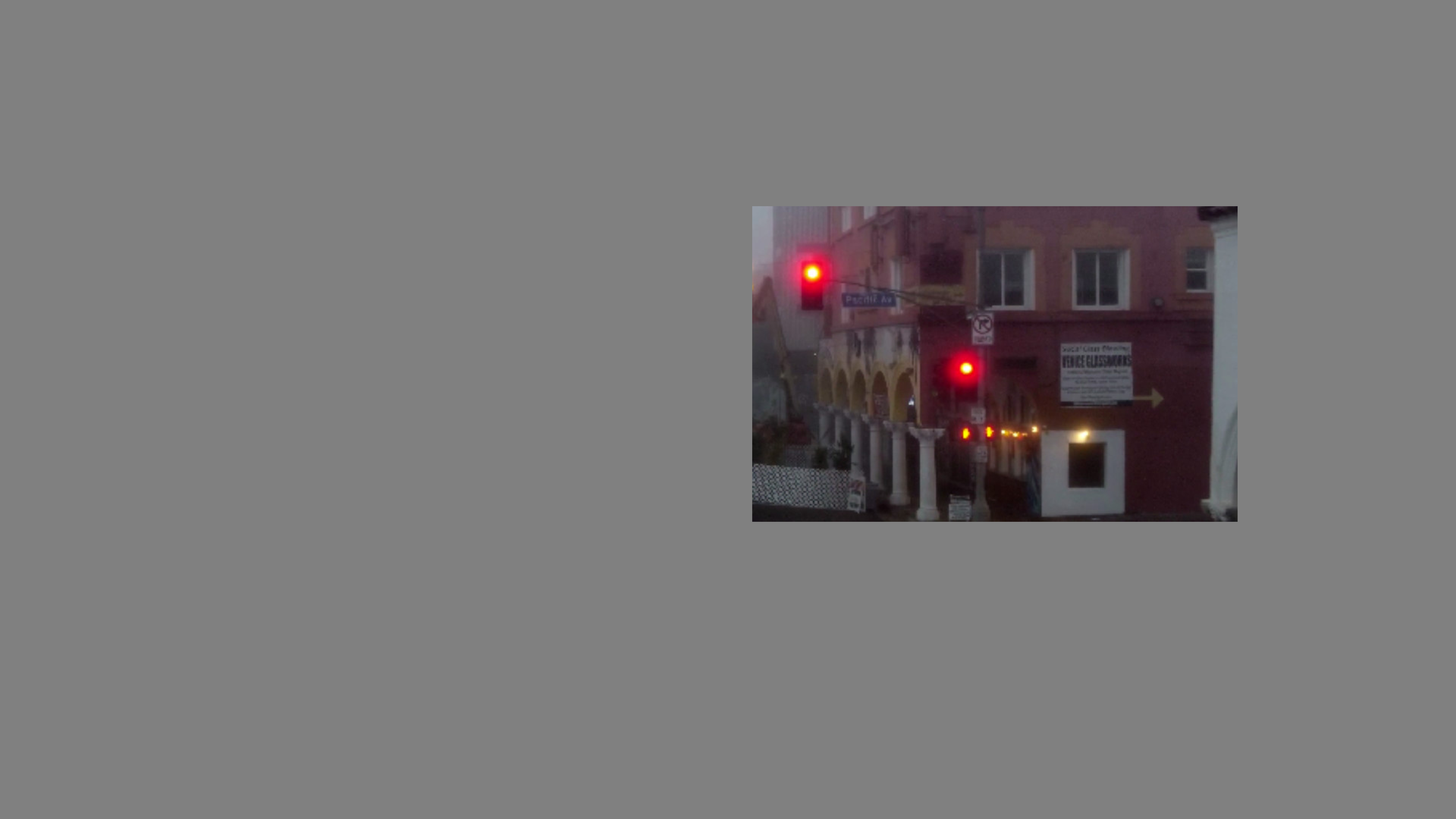}{2074 864 653 639} &
      \zoomcrop{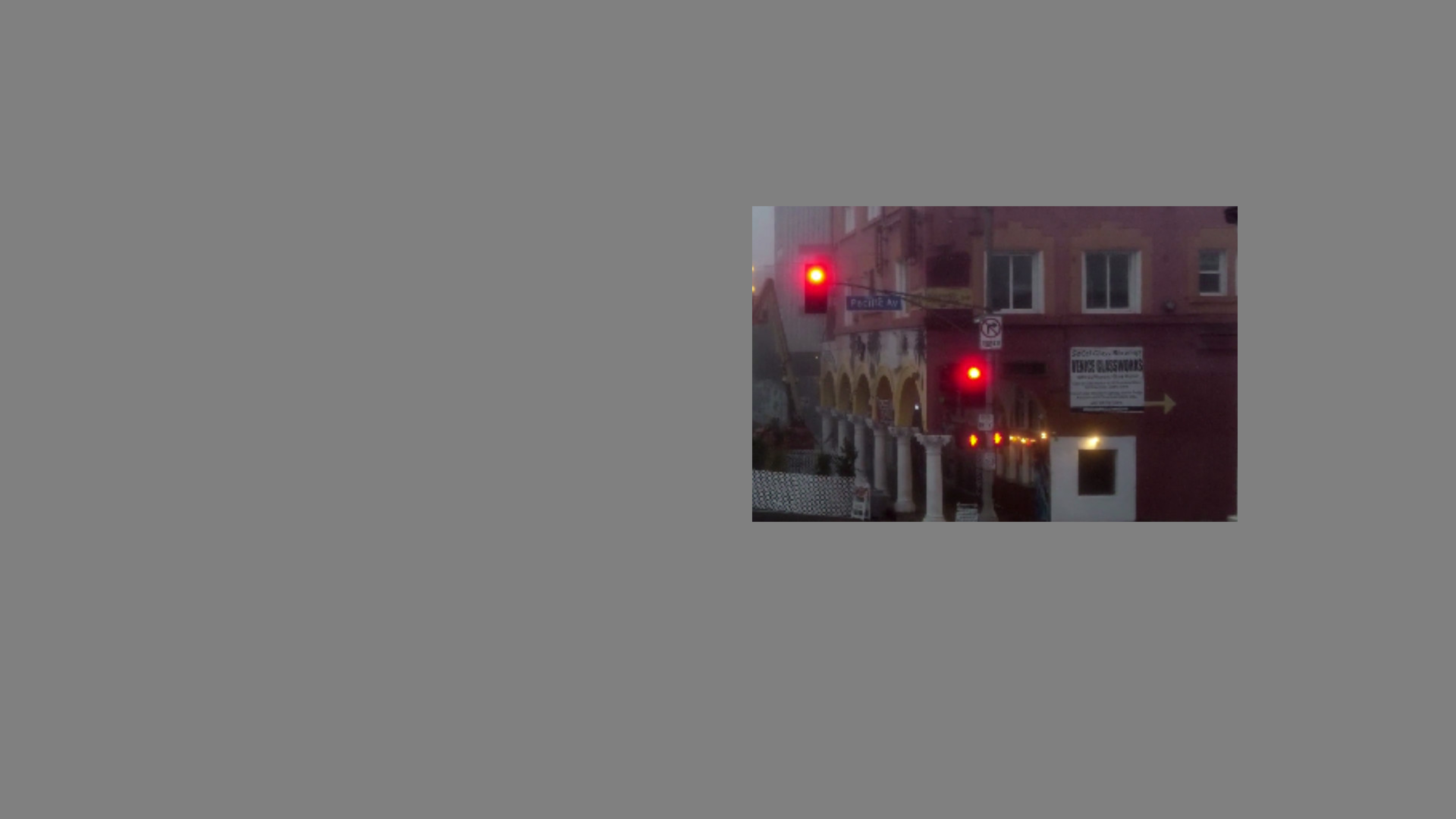}{2074 864 653 639} &
      \zoomcrop{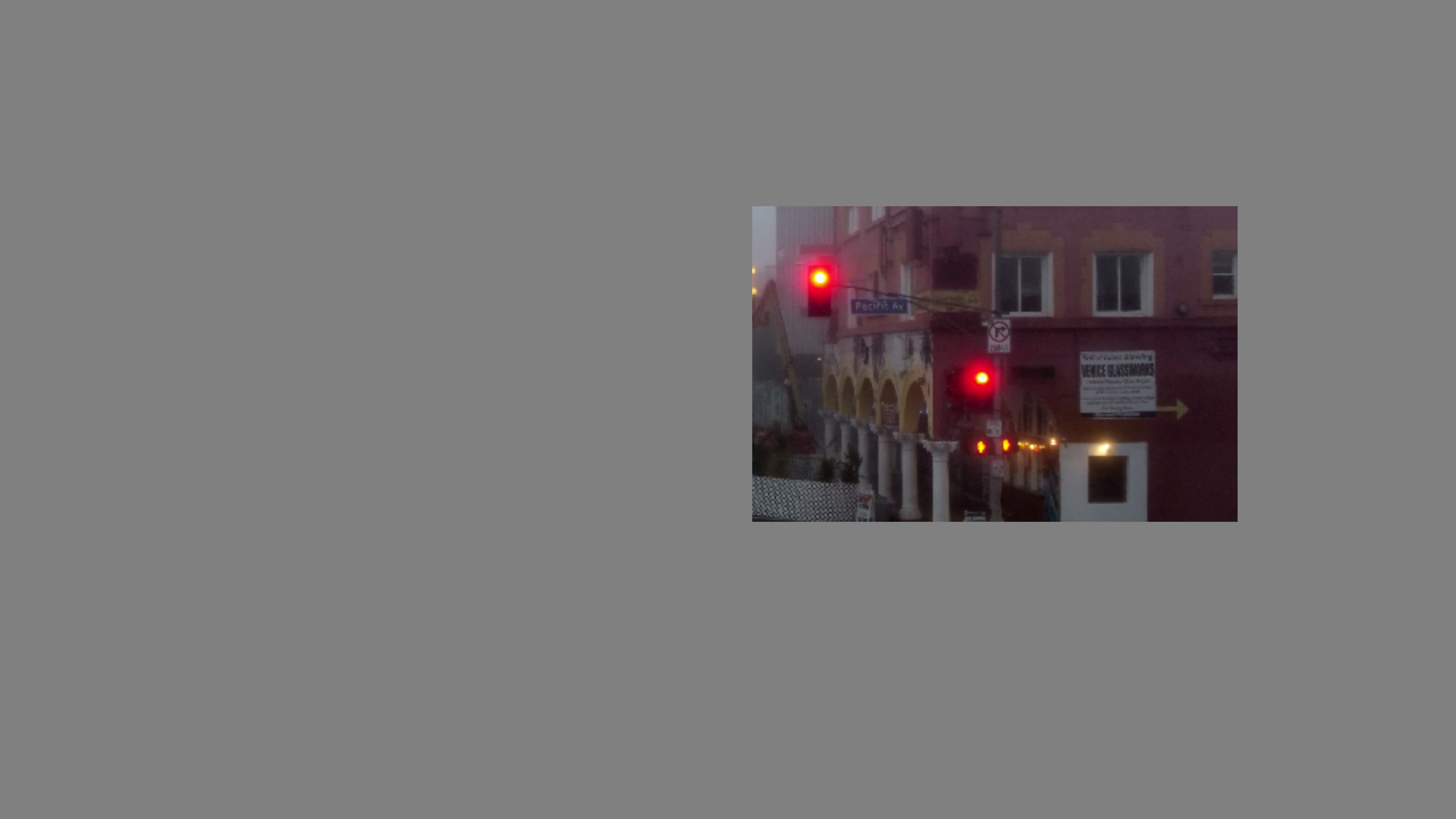}{2074 864 653 639} \\[0pt]
    \rotatebox{0}{\scriptsize\textbf{(c)}} &
      \zoomcrop{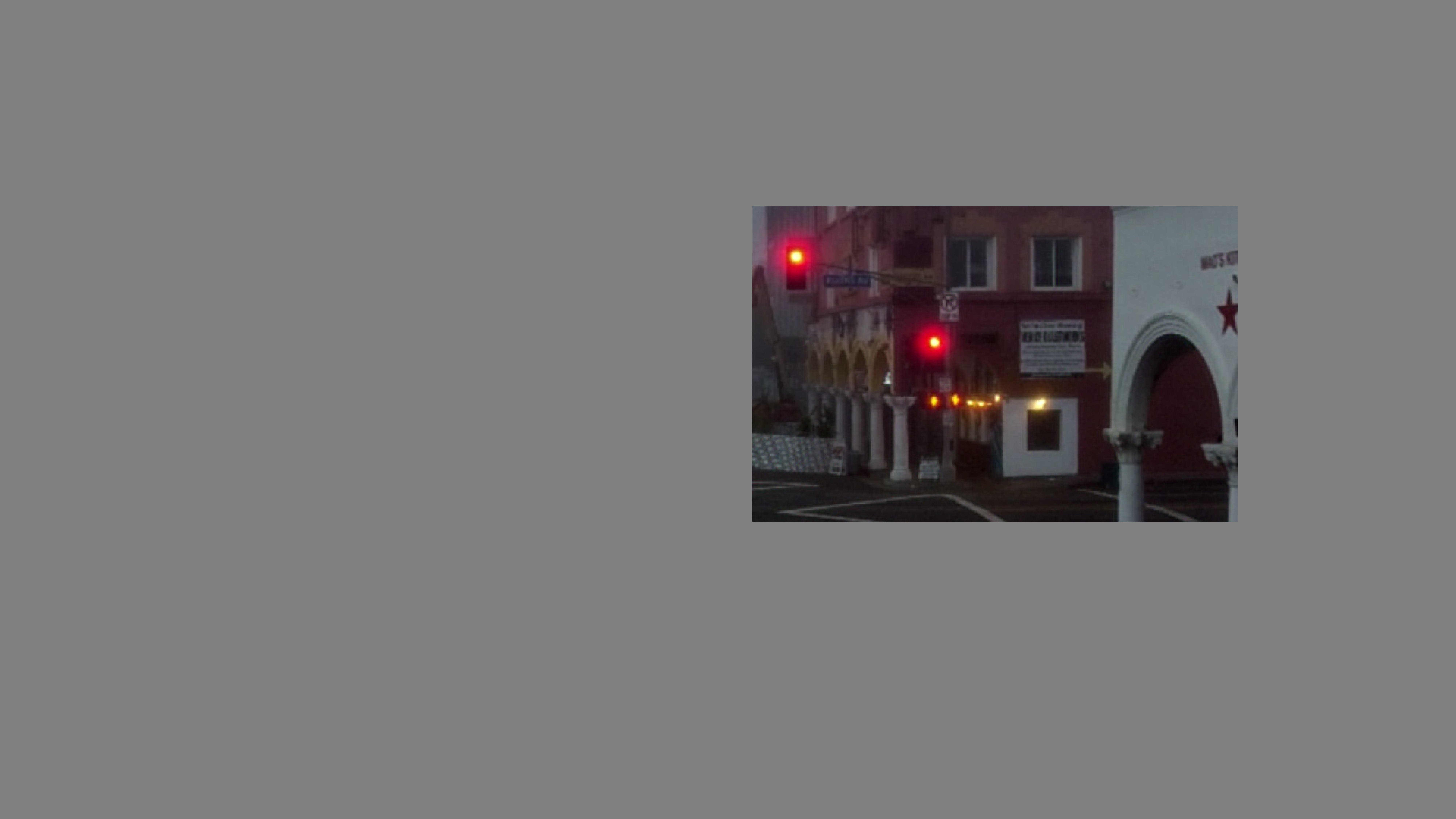}{2074 864 653 639} &
      \zoomcrop{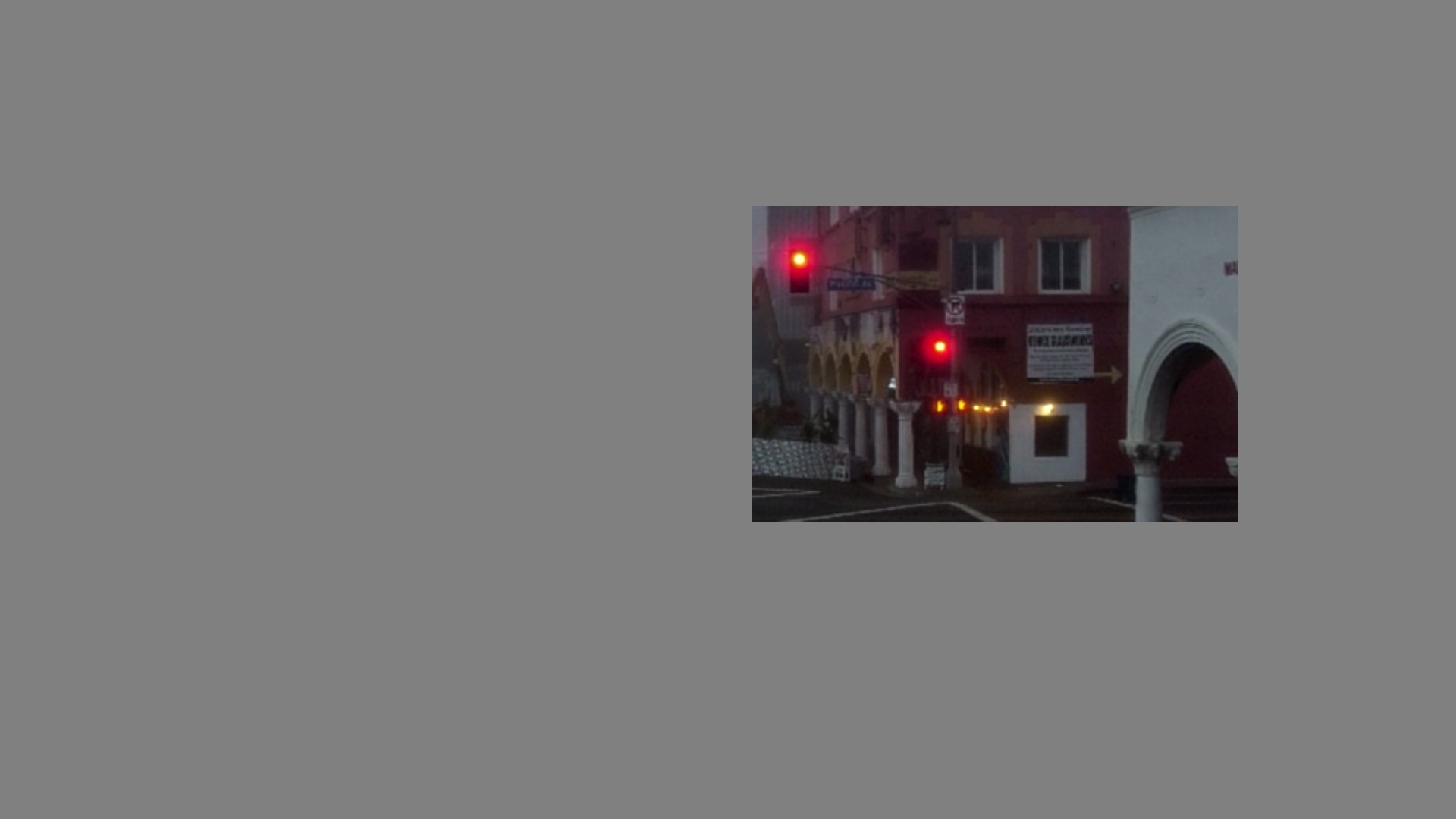}{2074 864 653 639} &
      \zoomcrop{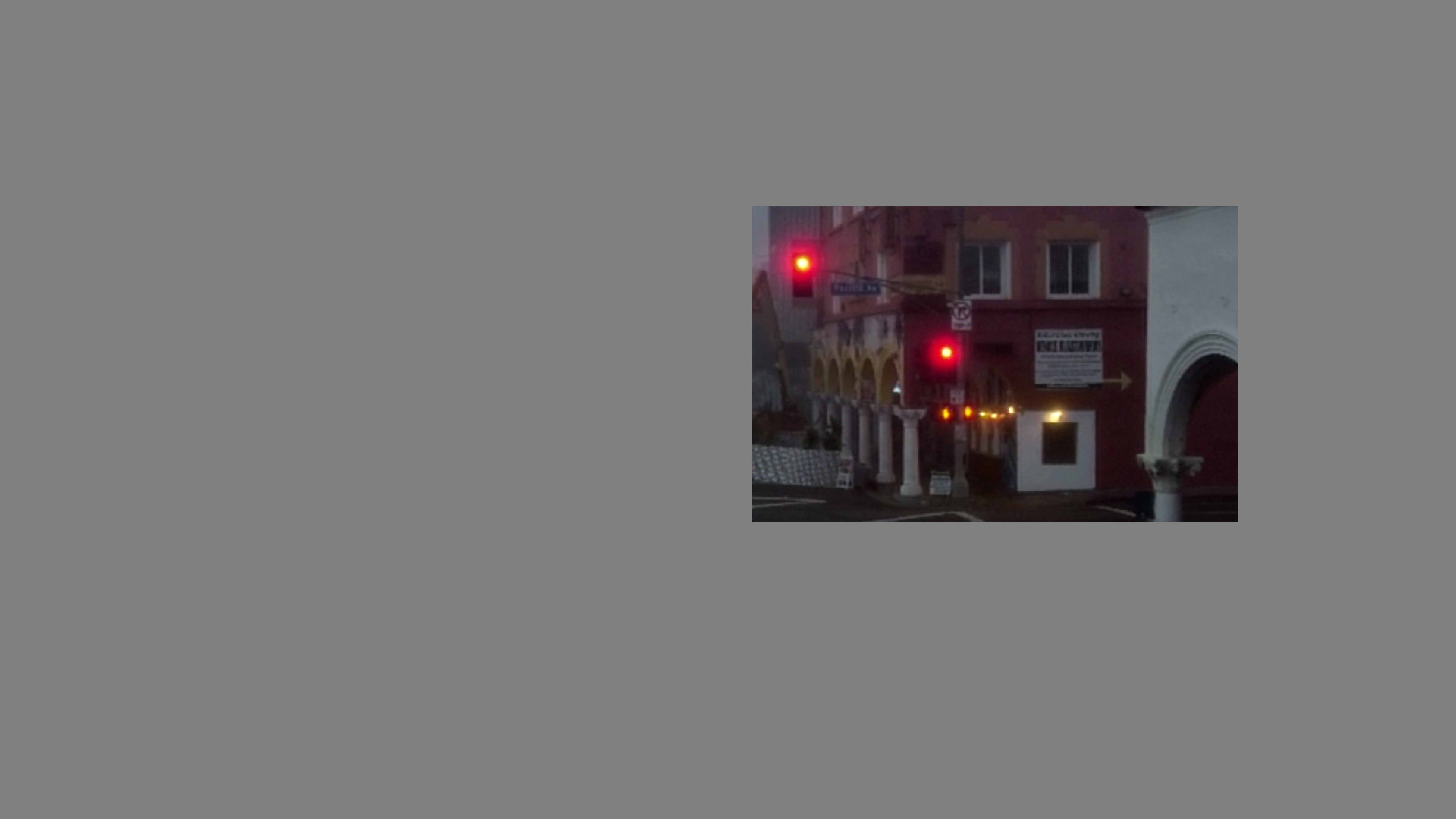}{2074 864 653 639} &
      \zoomcrop{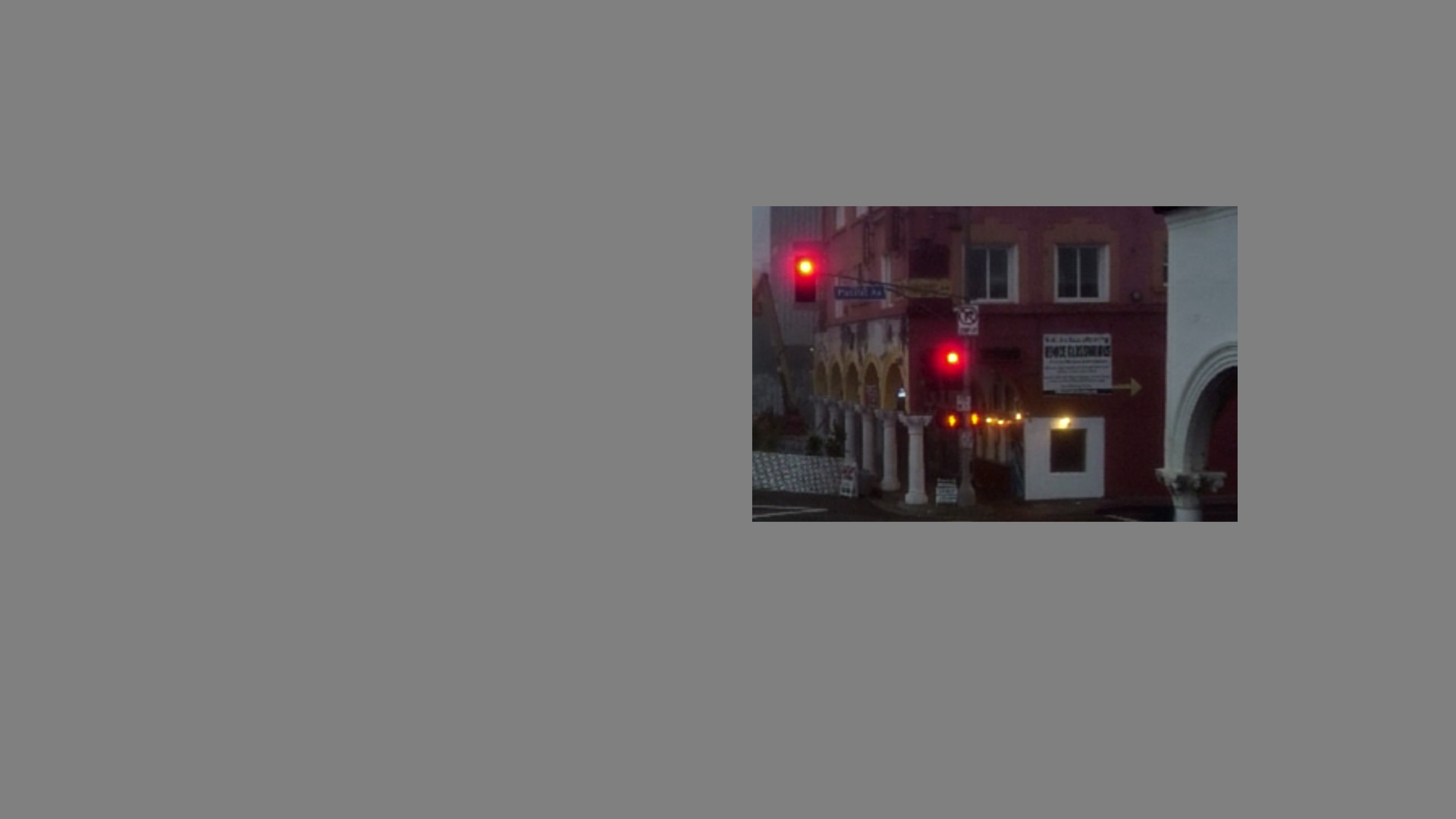}{2074 864 653 639} &
      \zoomcrop{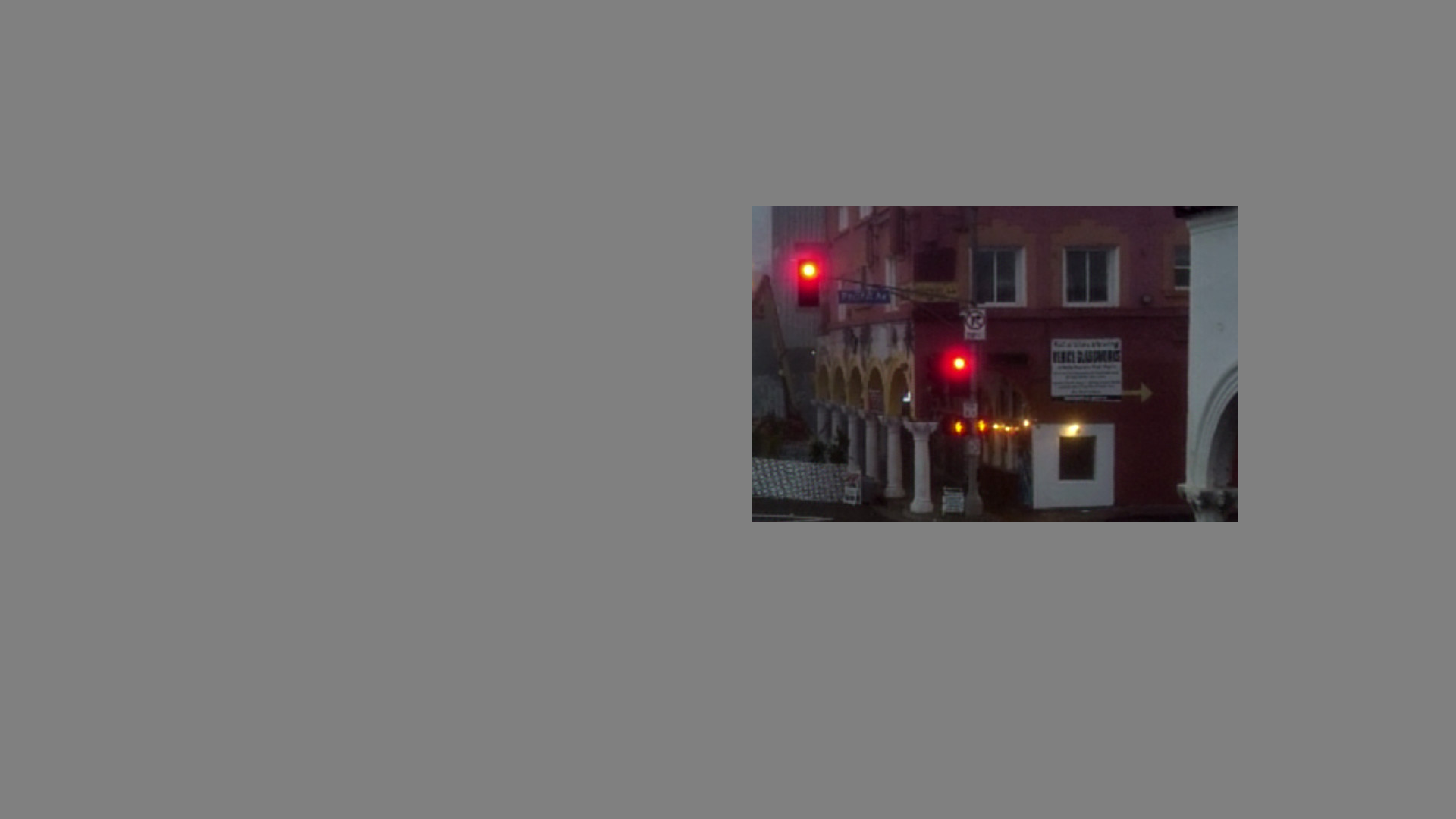}{2074 864 653 639} &
      \zoomcrop{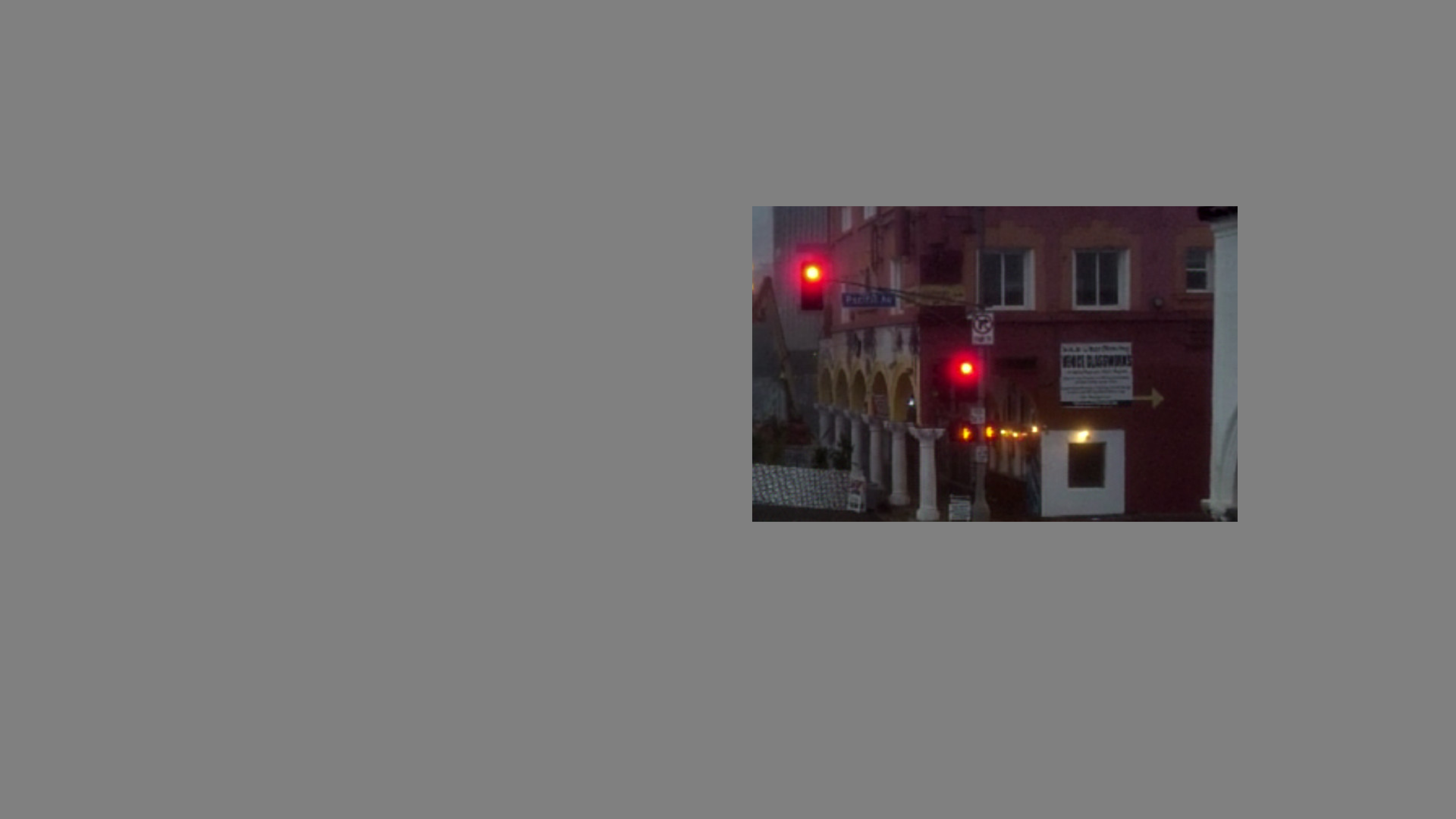}{2074 864 653 639} &
      \zoomcrop{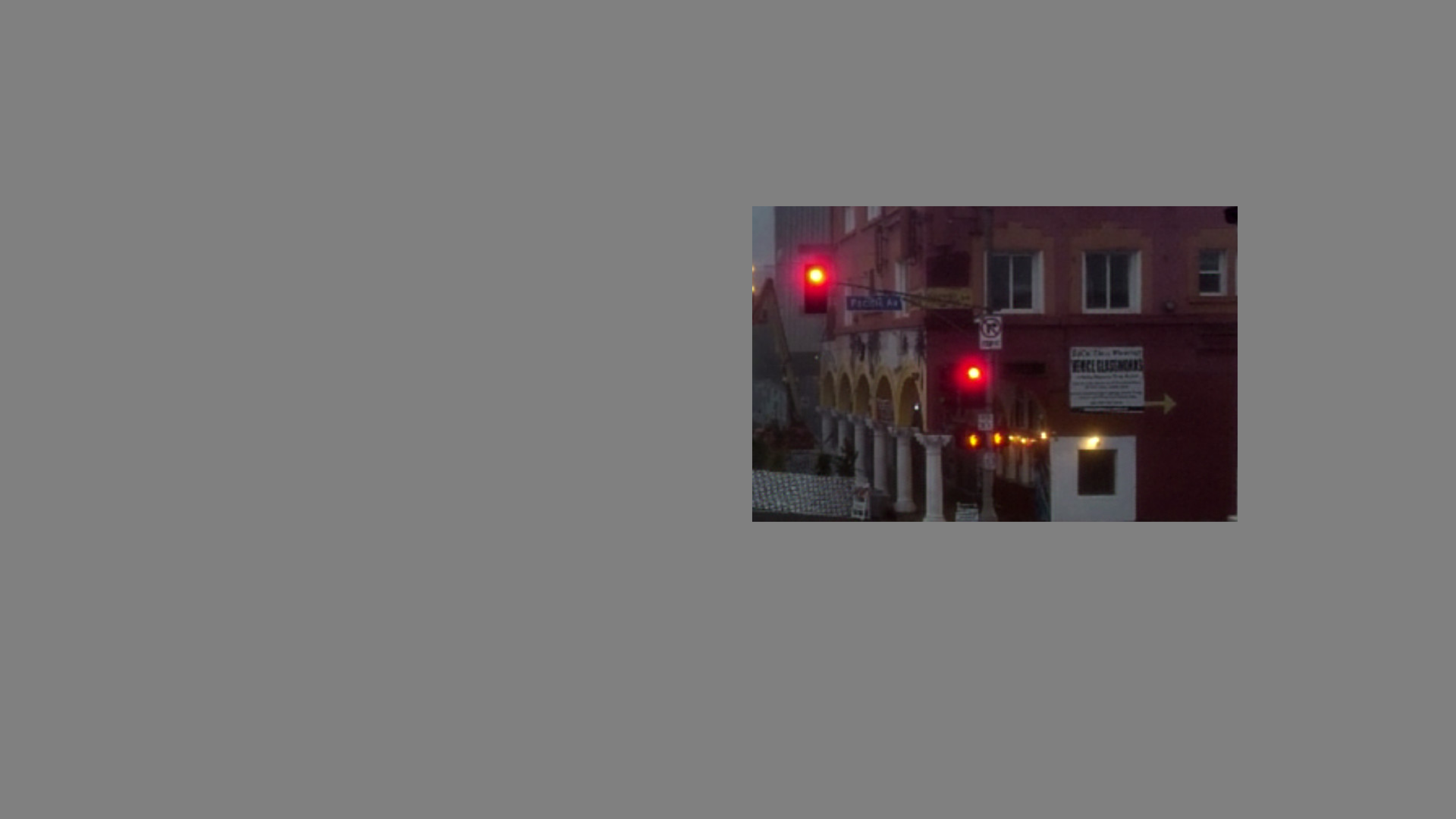}{2074 864 653 639} &
      \zoomcrop{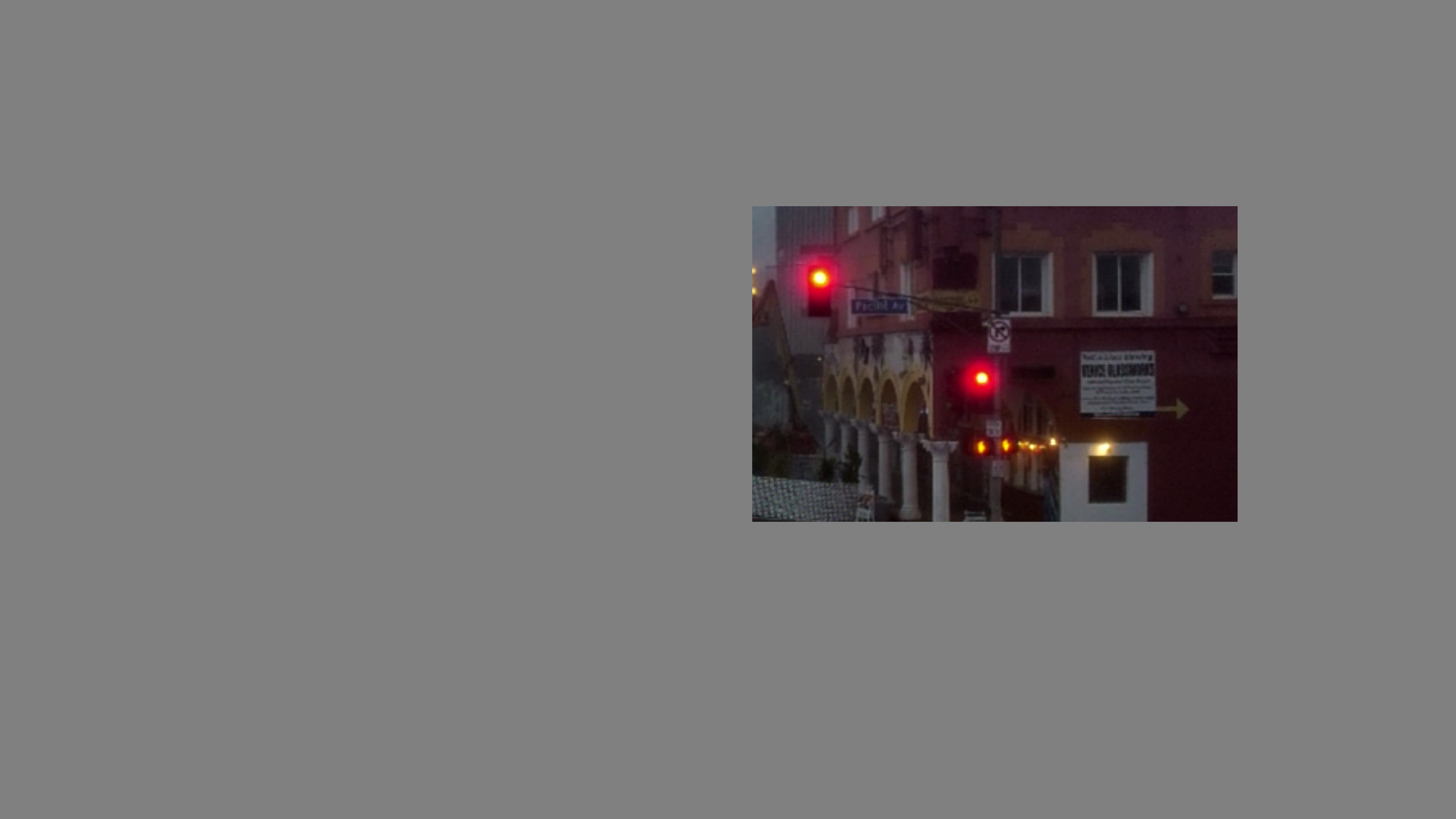}{2074 864 653 639} \\[0pt]
    \rotatebox{0}{\scriptsize\textbf{(d)}} &
      \zoomcrop{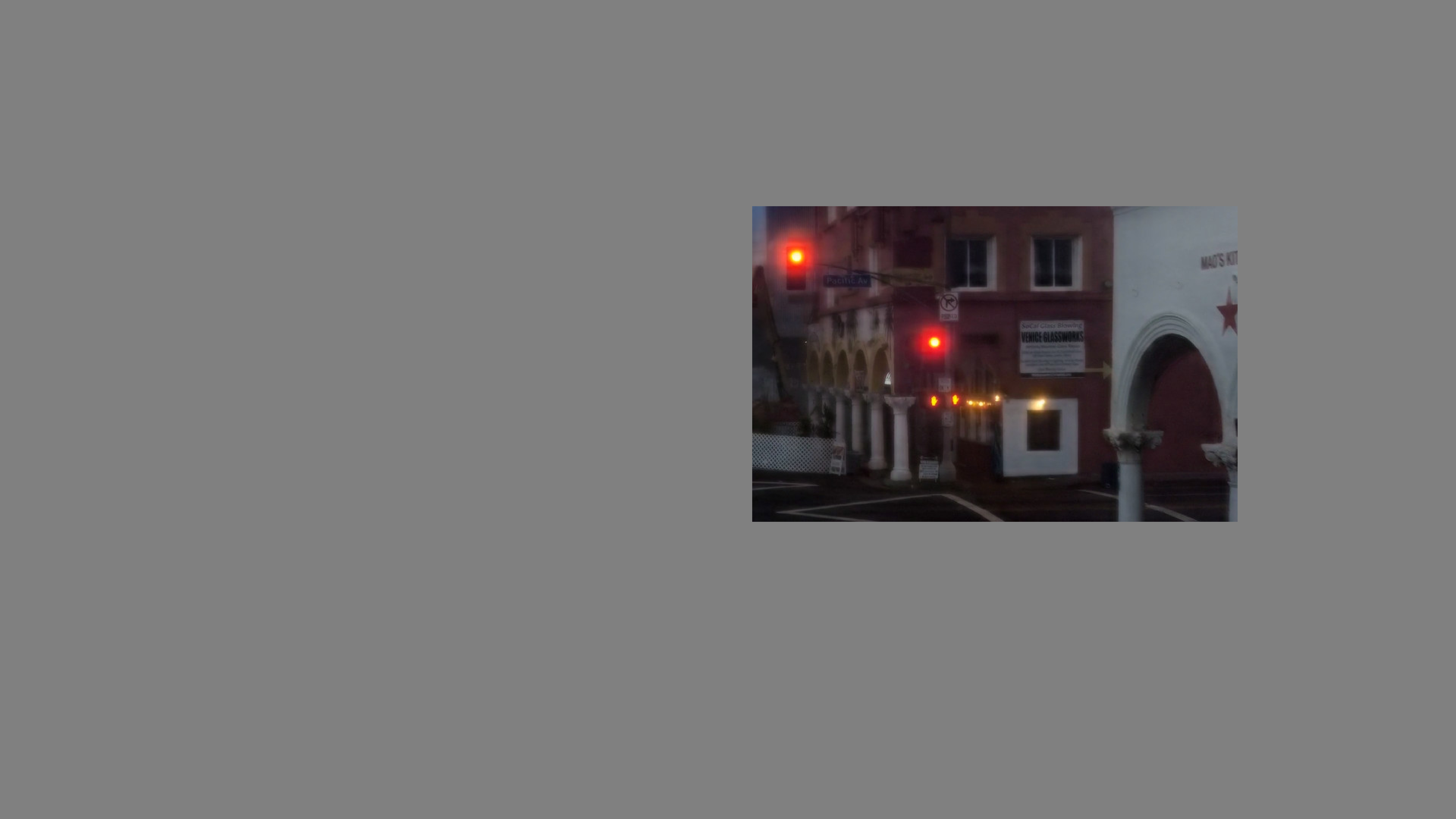}{2074 864 653 639} &
      \zoomcrop{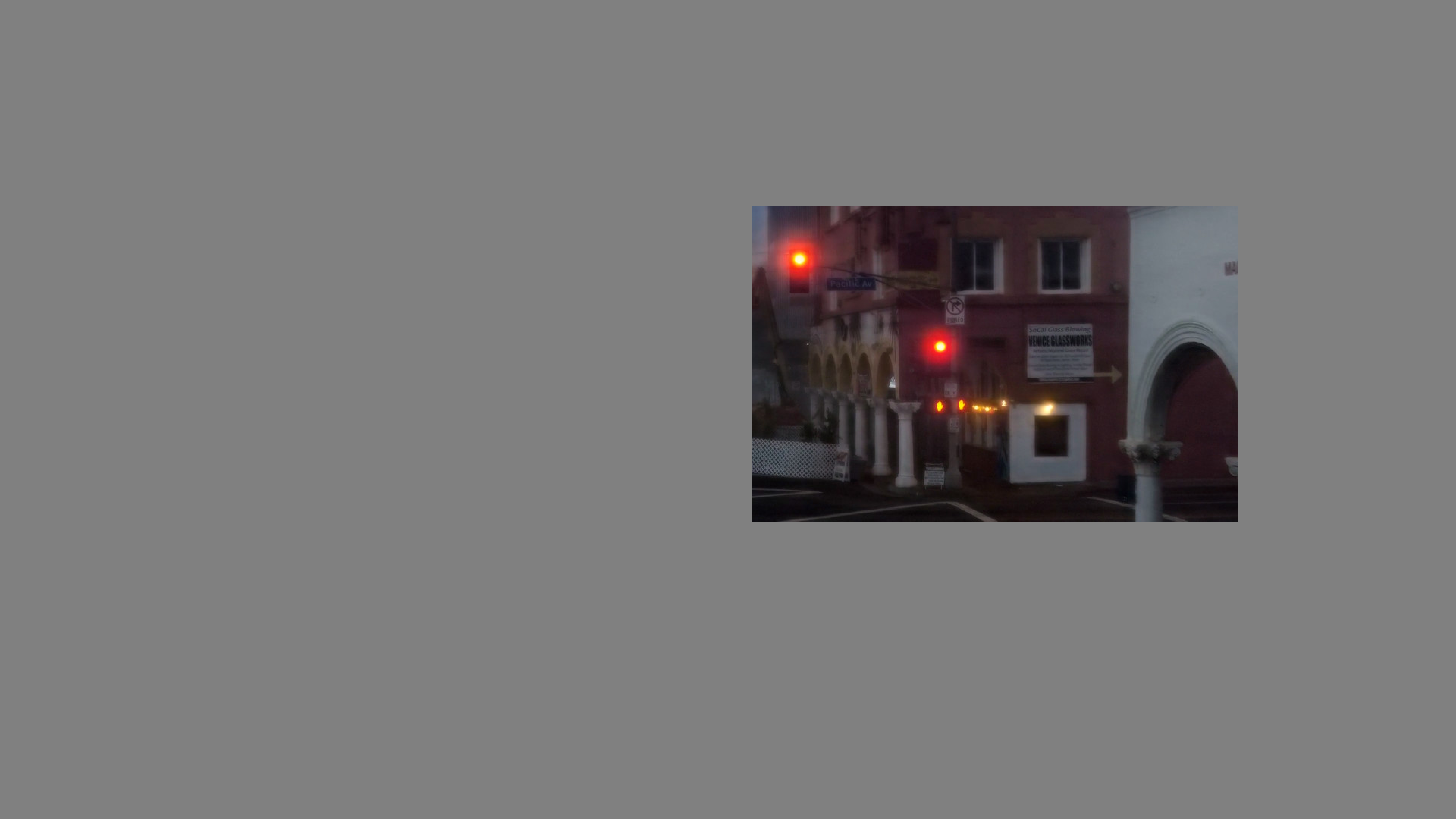}{2074 864 653 639} &
      \zoomcrop{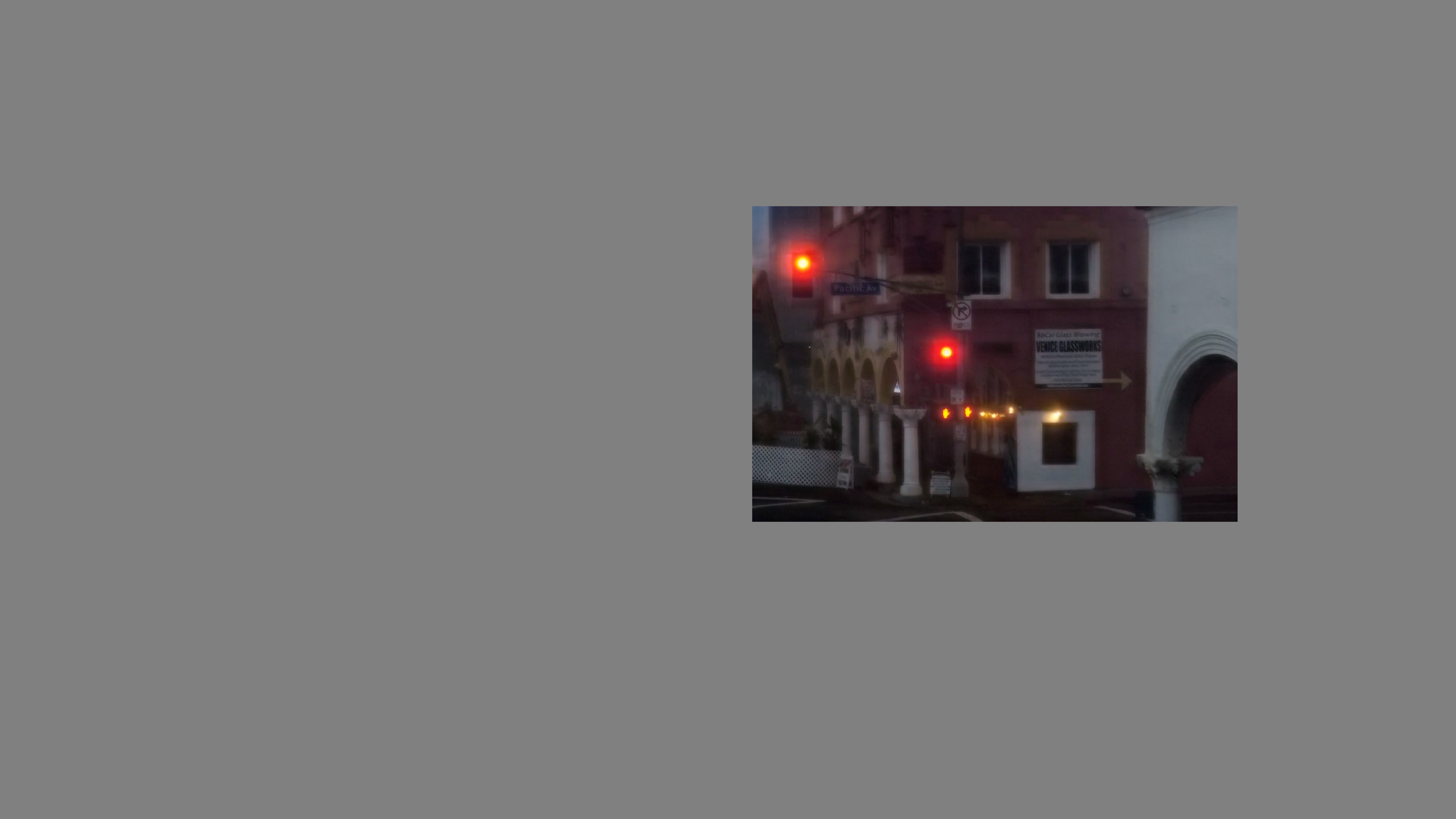}{2074 864 653 639} &
      \zoomcrop{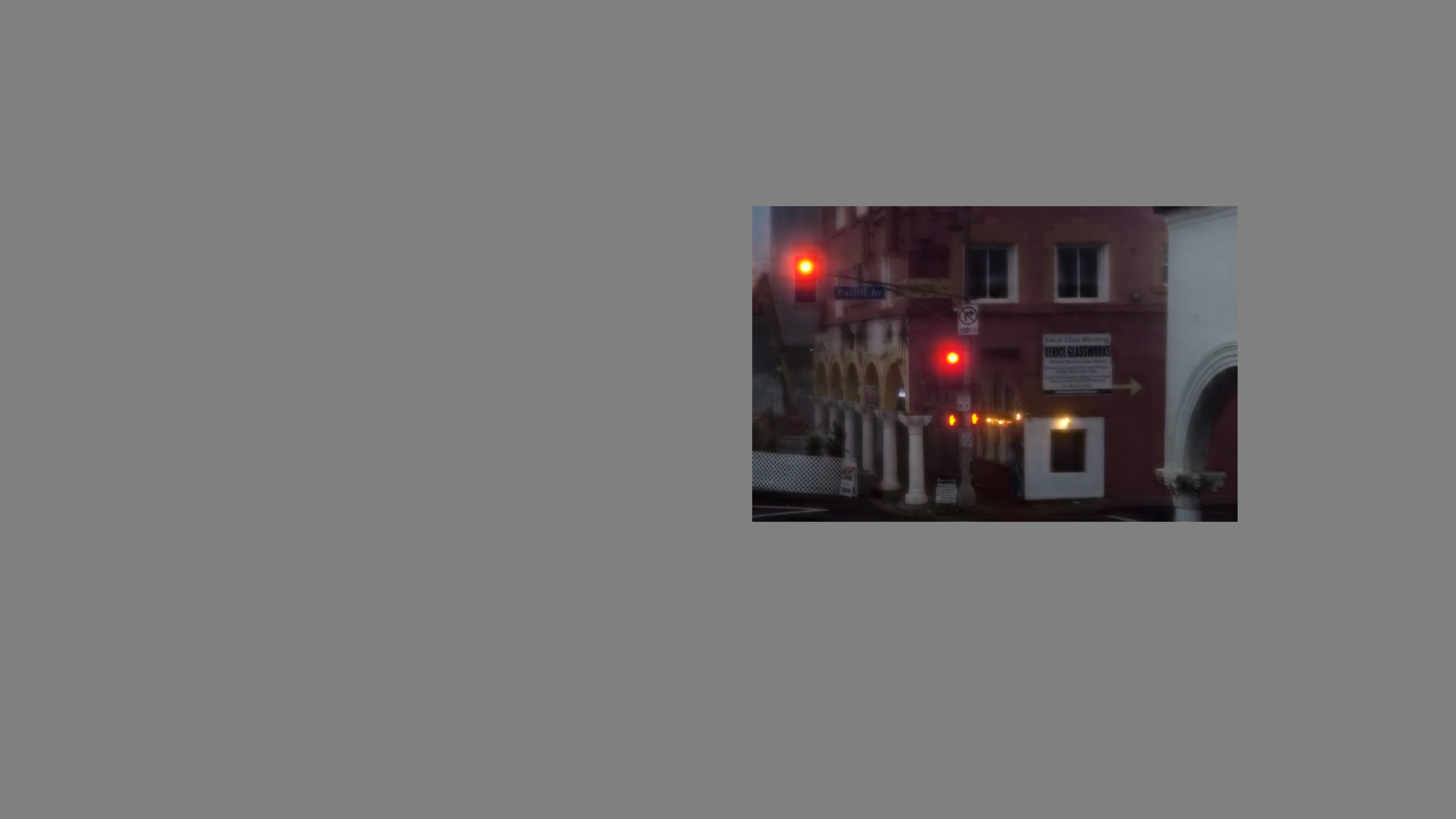}{2074 864 653 639} &
      \zoomcrop{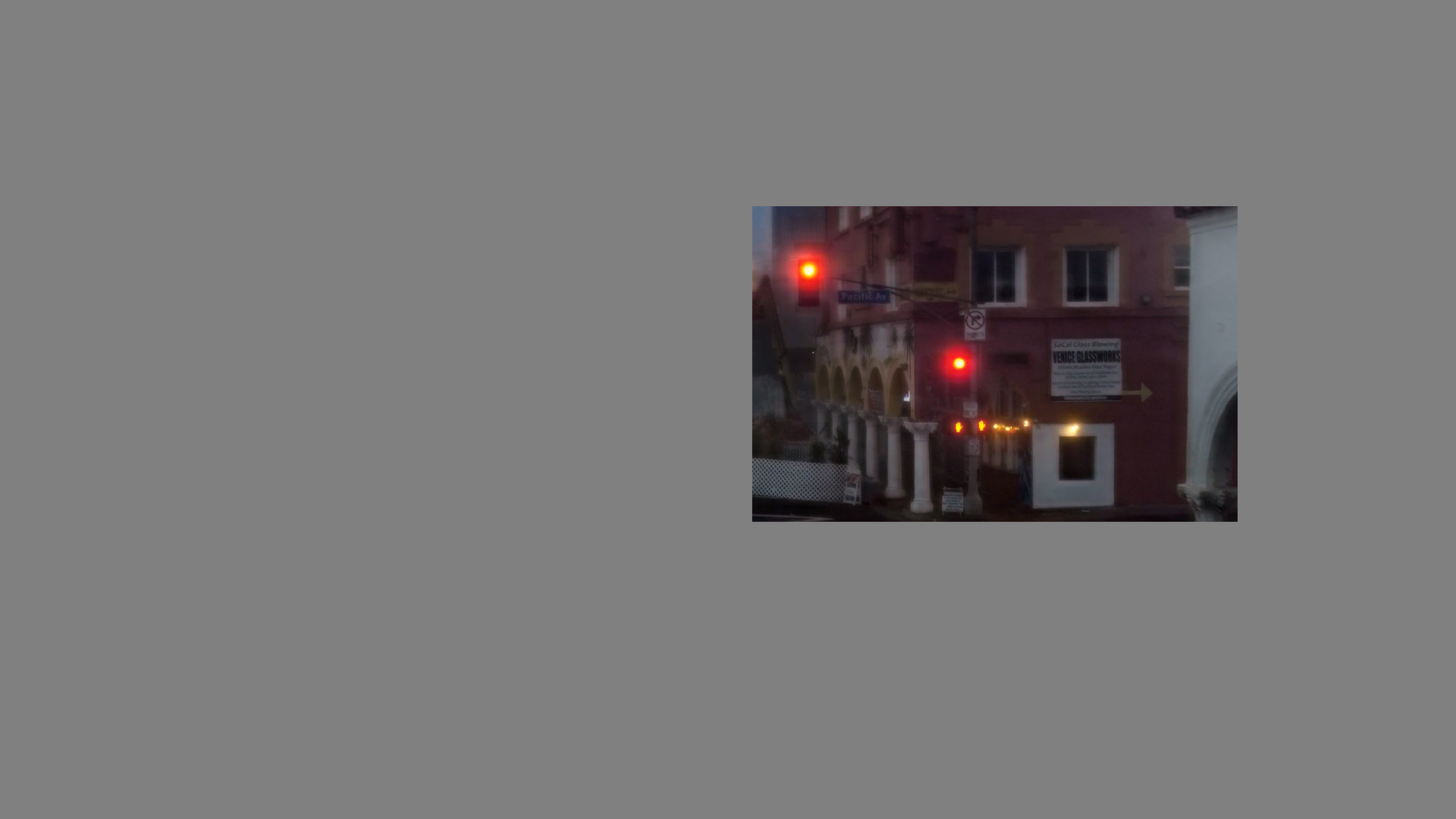}{2074 864 653 639} &
      \zoomcrop{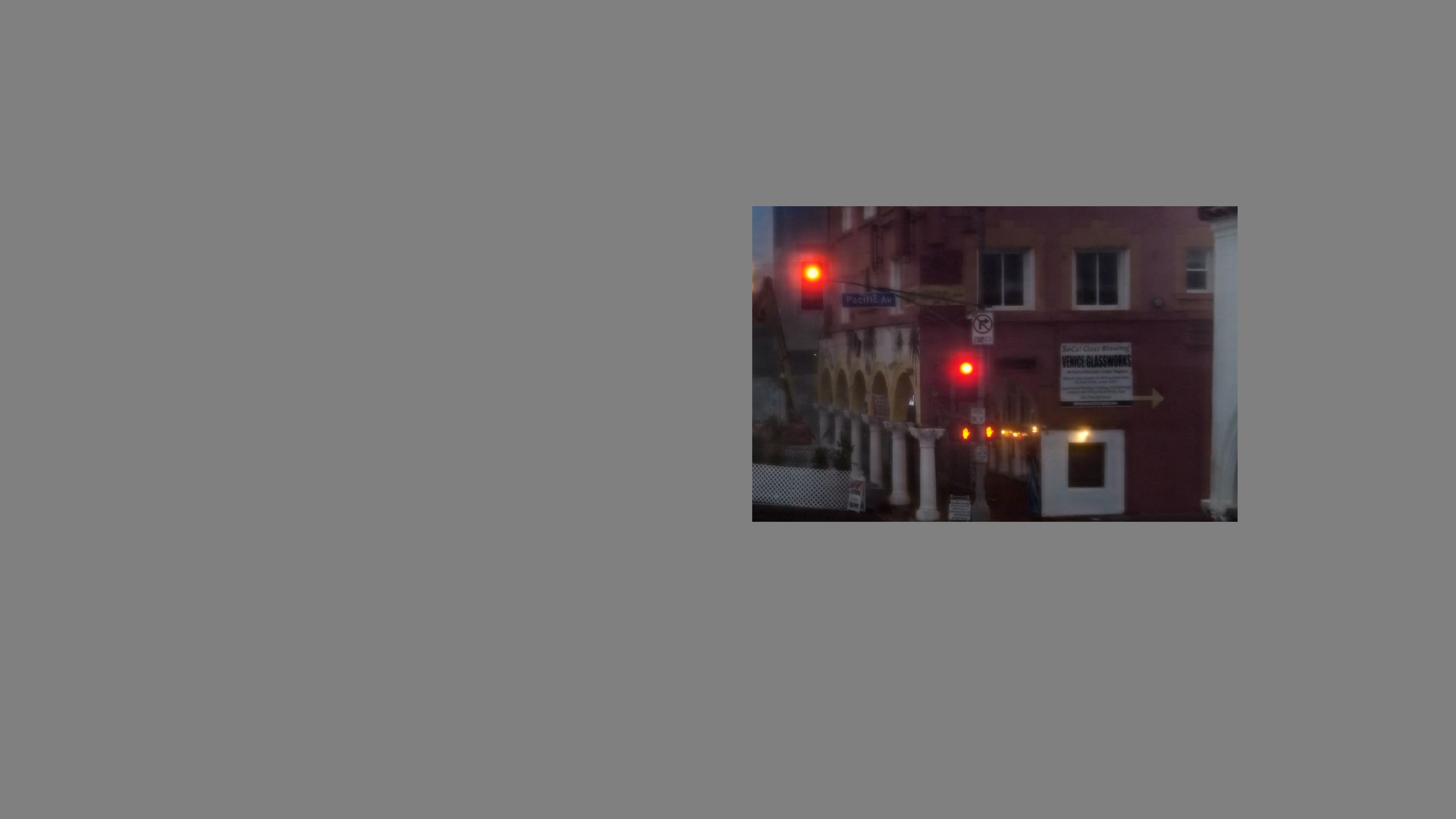}{2074 864 653 639} &
      \zoomcrop{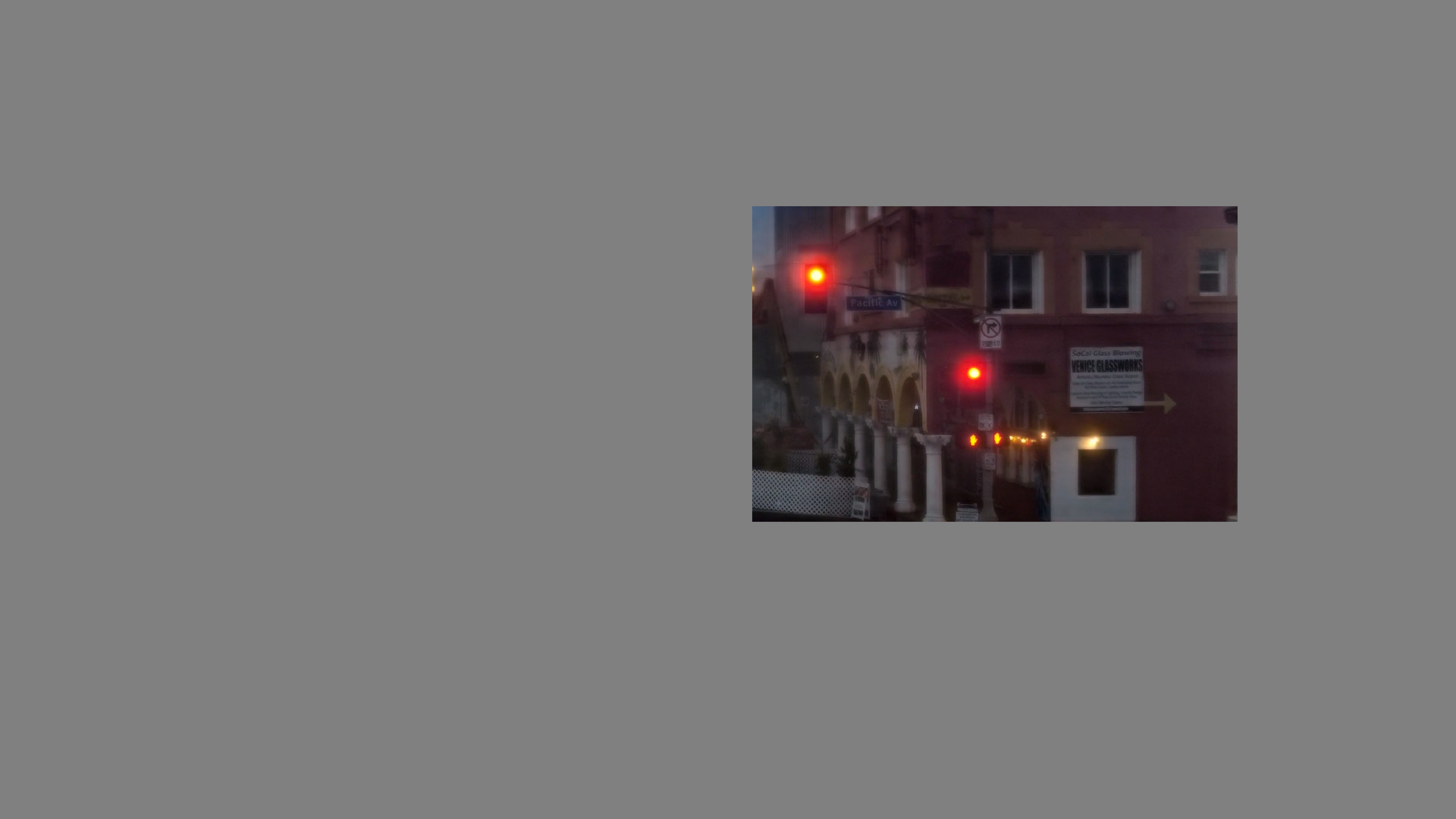}{2074 864 653 639} &
      \zoomcrop{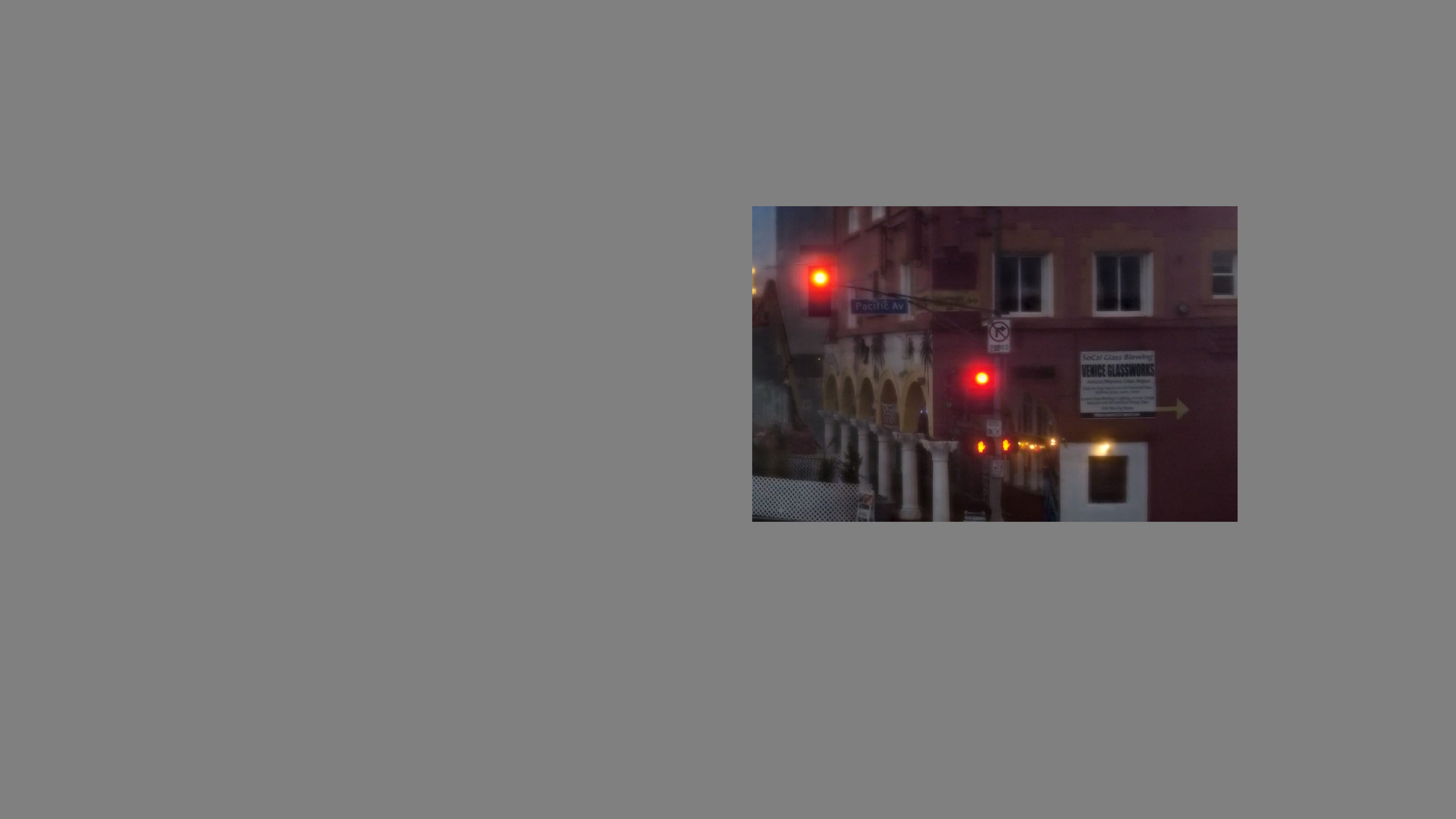}{2074 864 653 639} \\[0pt]
    \rotatebox{0}{\scriptsize\textbf{(e)}} &
      \colorbox{ourgreen}{\zoomcrop{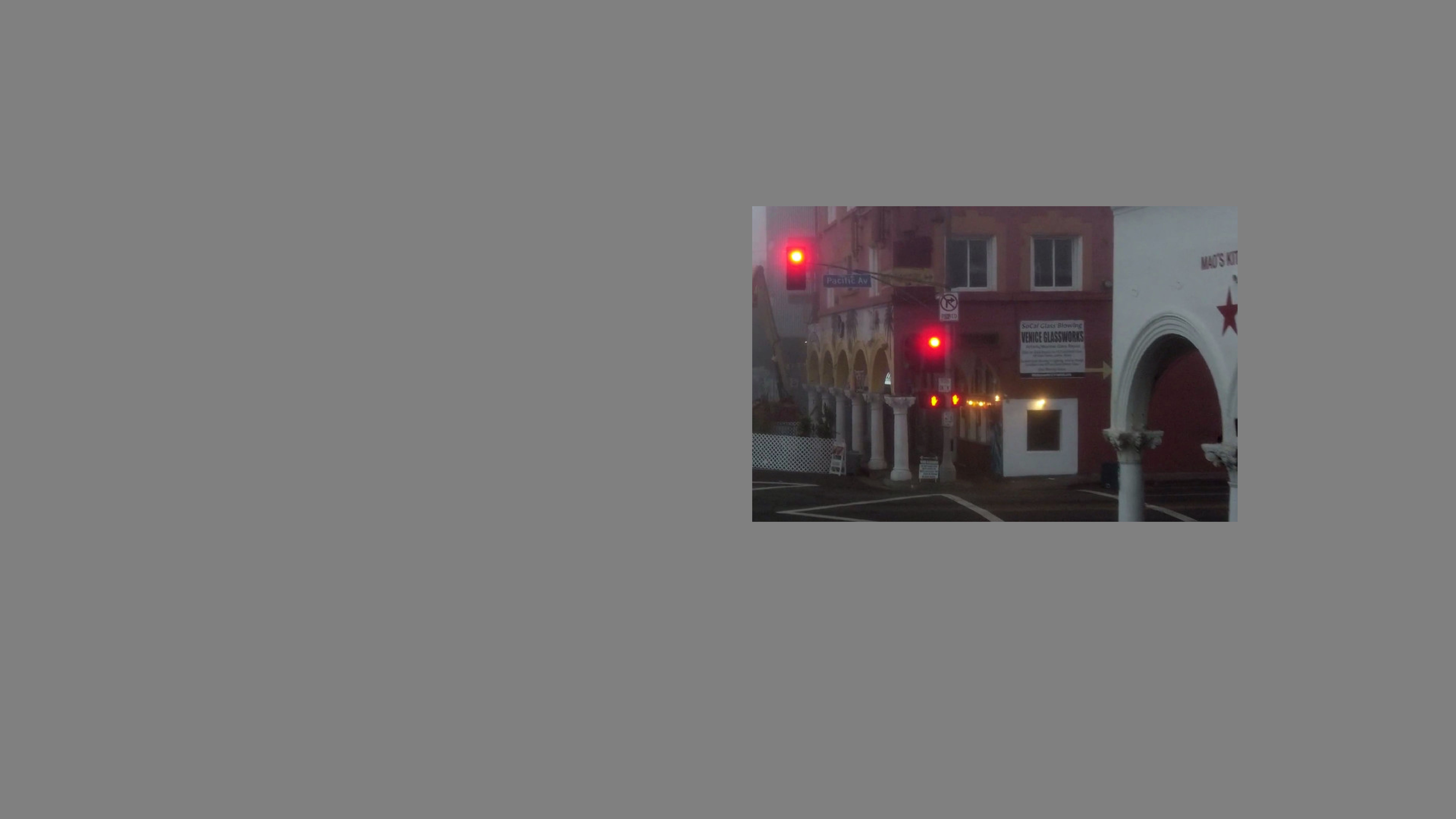}{2074 864 653 639}} &
      \colorbox{ourgreen}{\zoomcrop{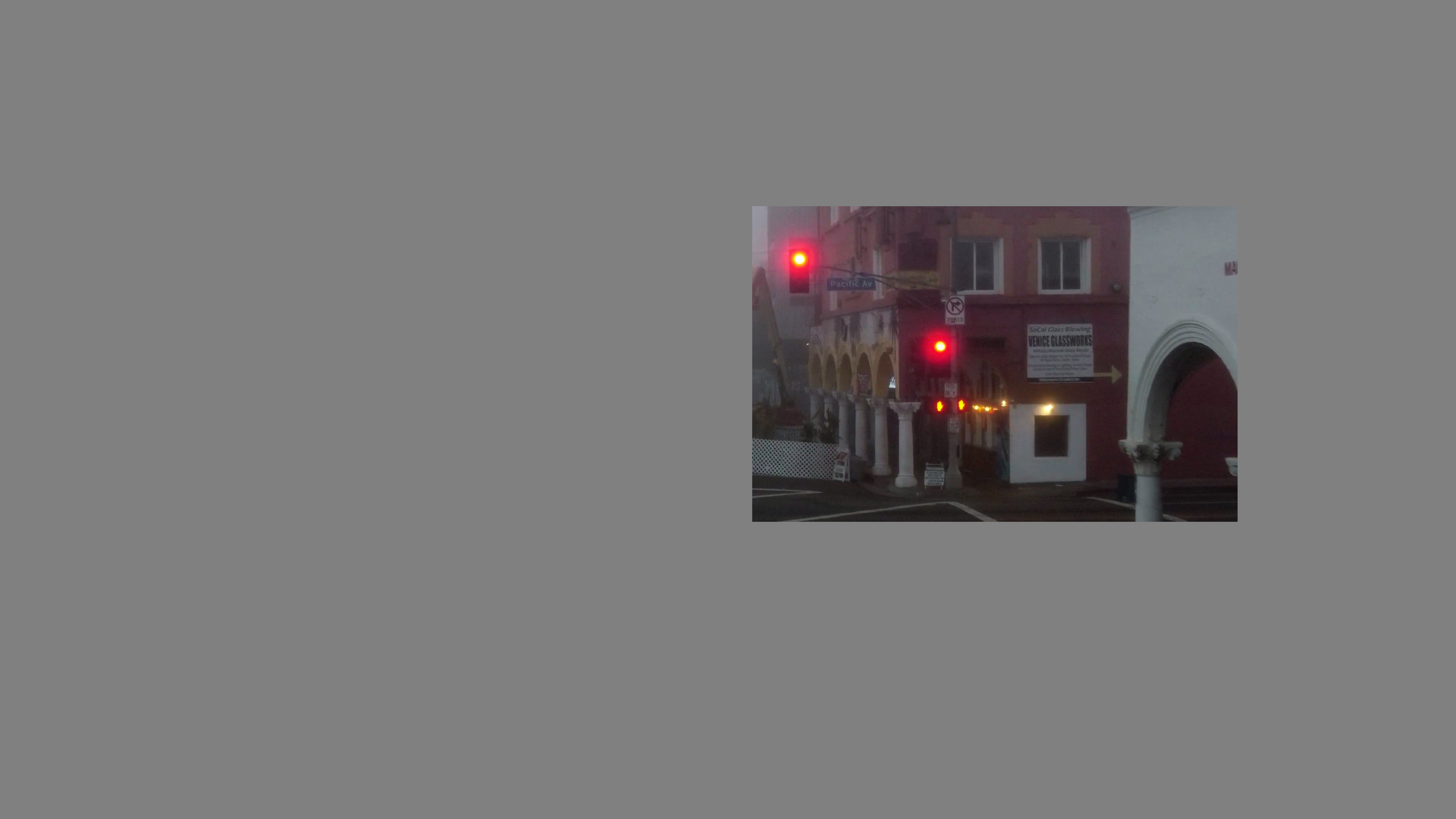}{2074 864 653 639}} &
      \colorbox{ourgreen}{\zoomcrop{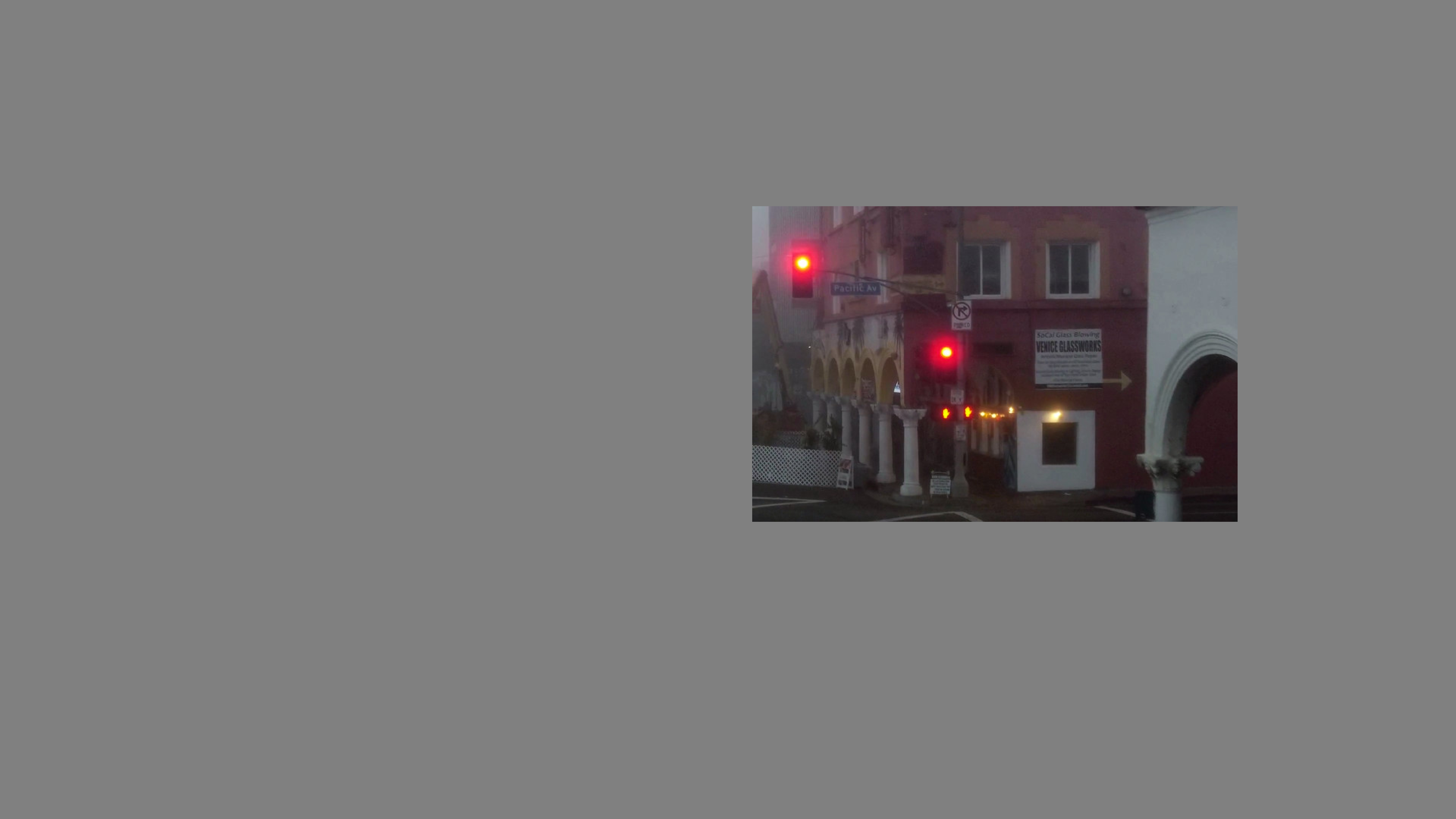}{2074 864 653 639}} &
      \colorbox{ourgreen}{\zoomcrop{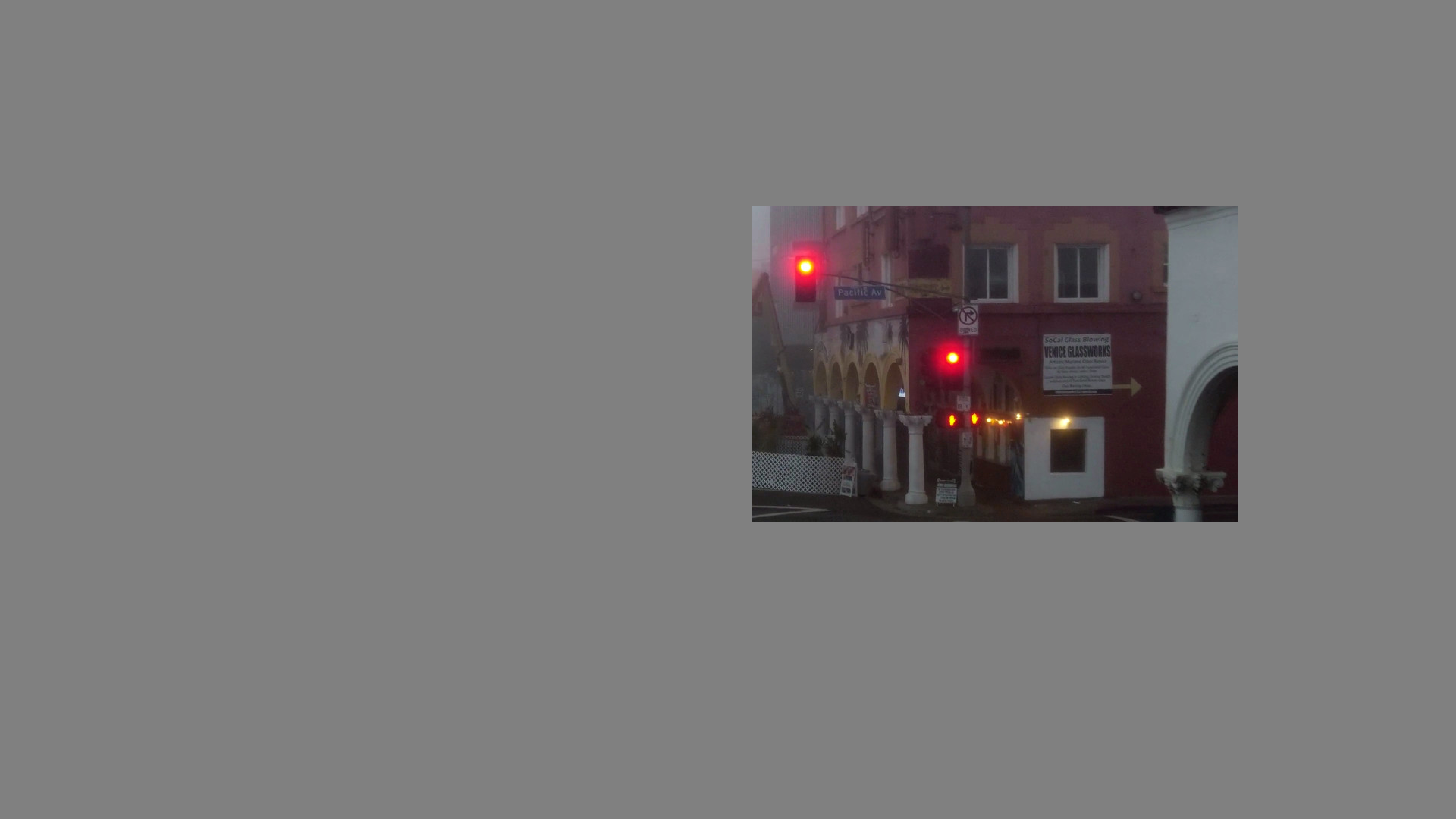}{2074 864 653 639}} &
      \colorbox{ourgreen}{\zoomcrop{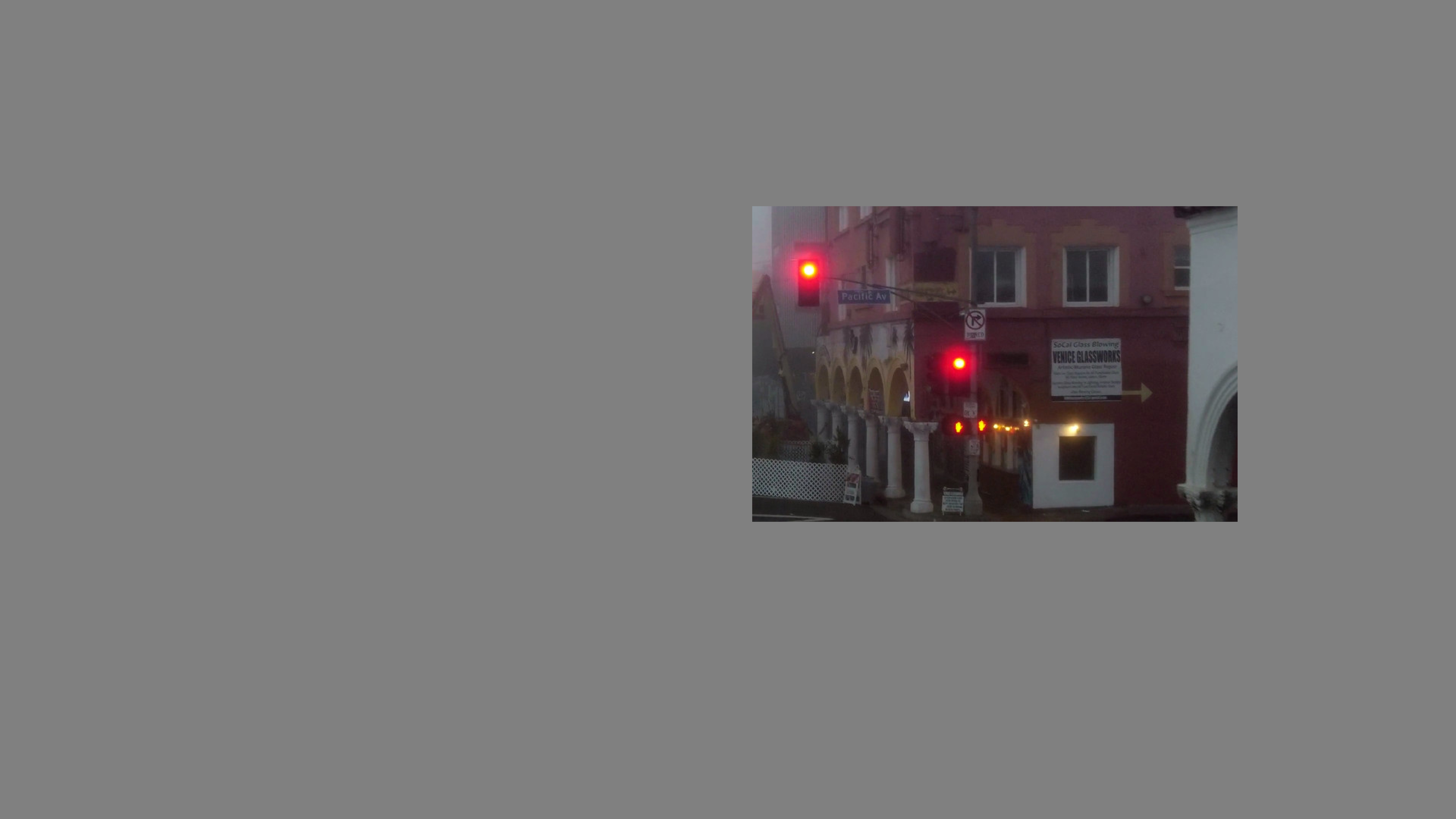}{2074 864 653 639}} &
      \colorbox{ourgreen}{\zoomcrop{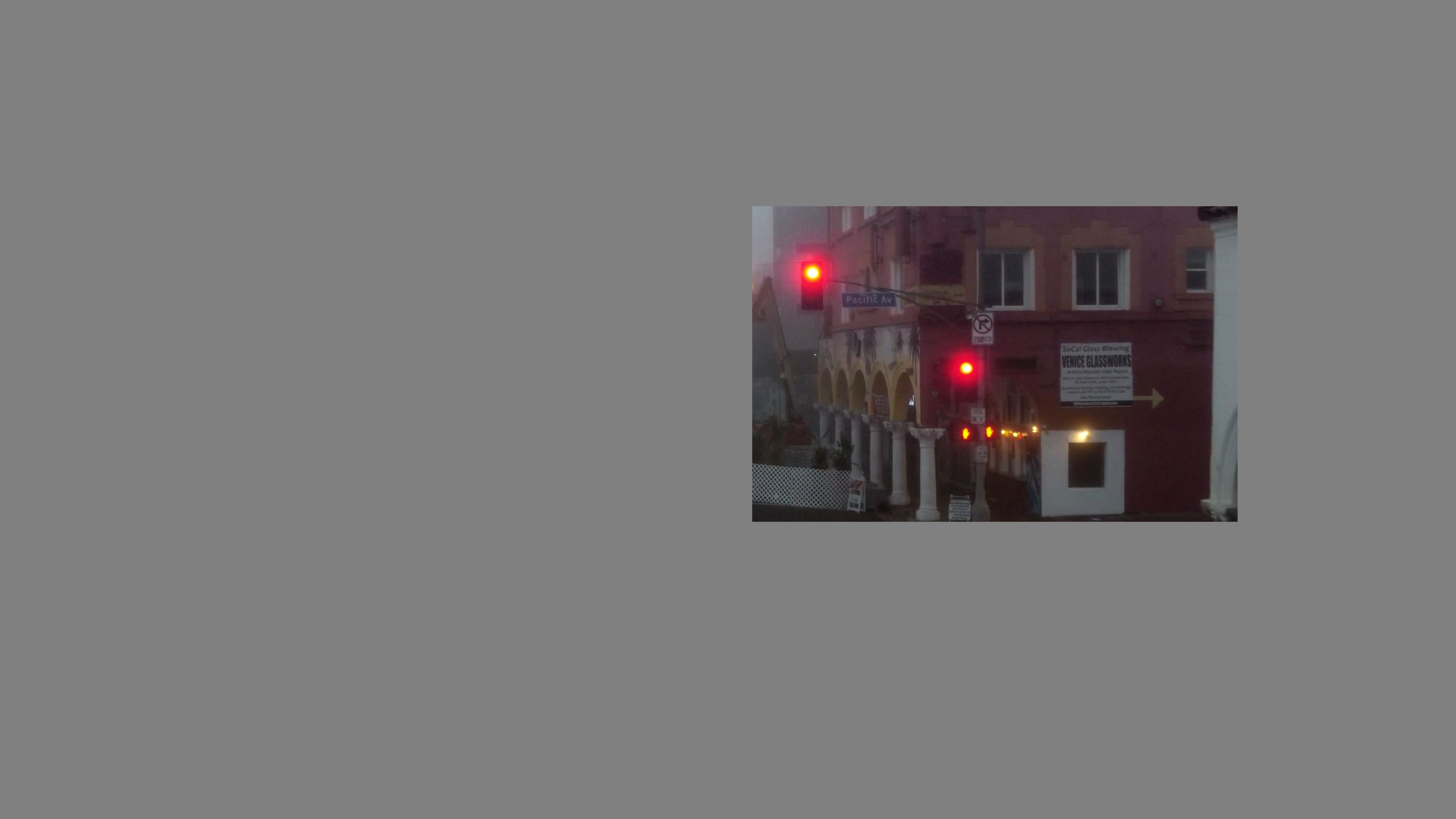}{2074 864 653 639}} &
      \colorbox{ourgreen}{\zoomcrop{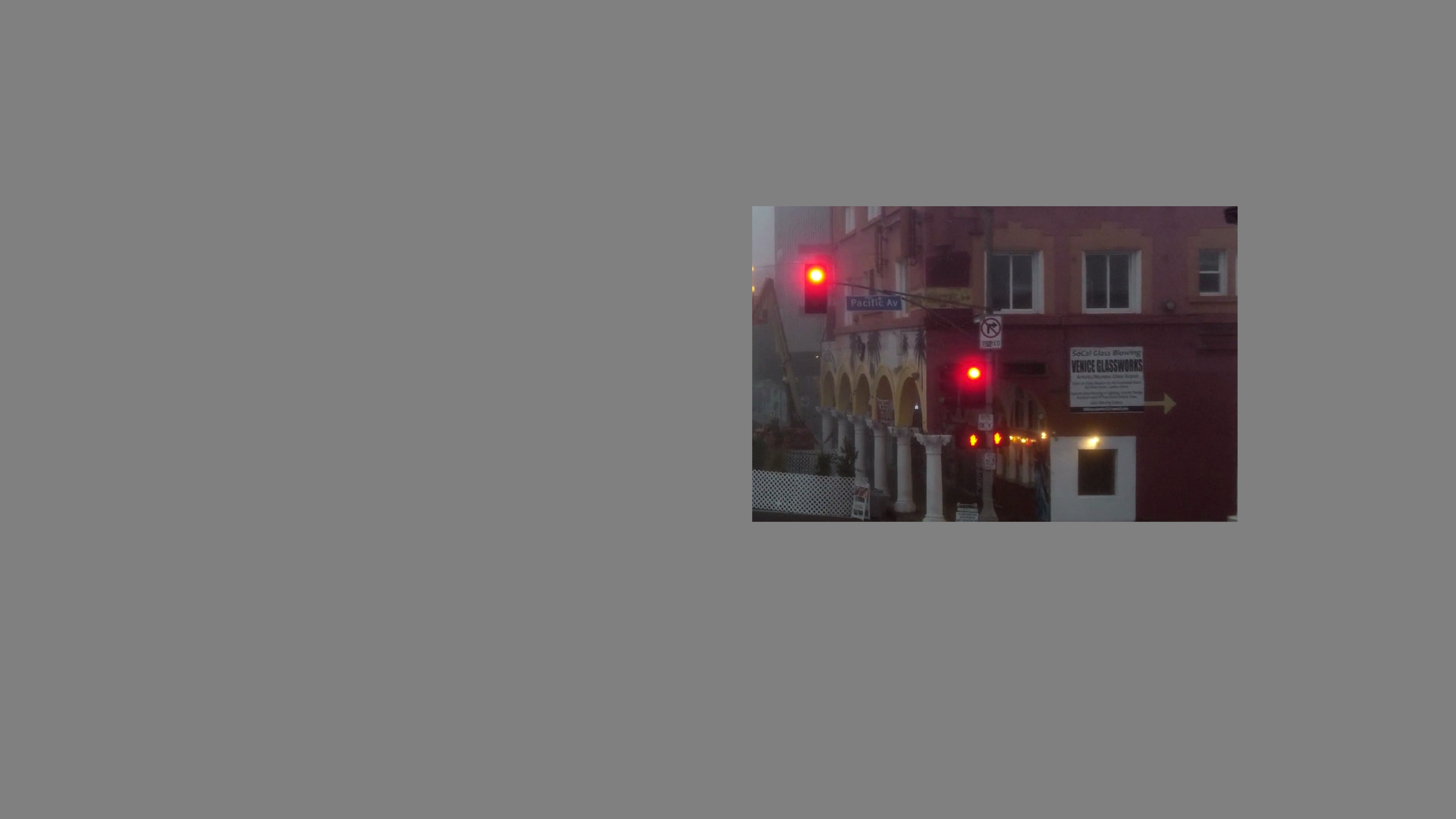}{2074 864 653 639}} &
      \colorbox{ourgreen}{\zoomcrop{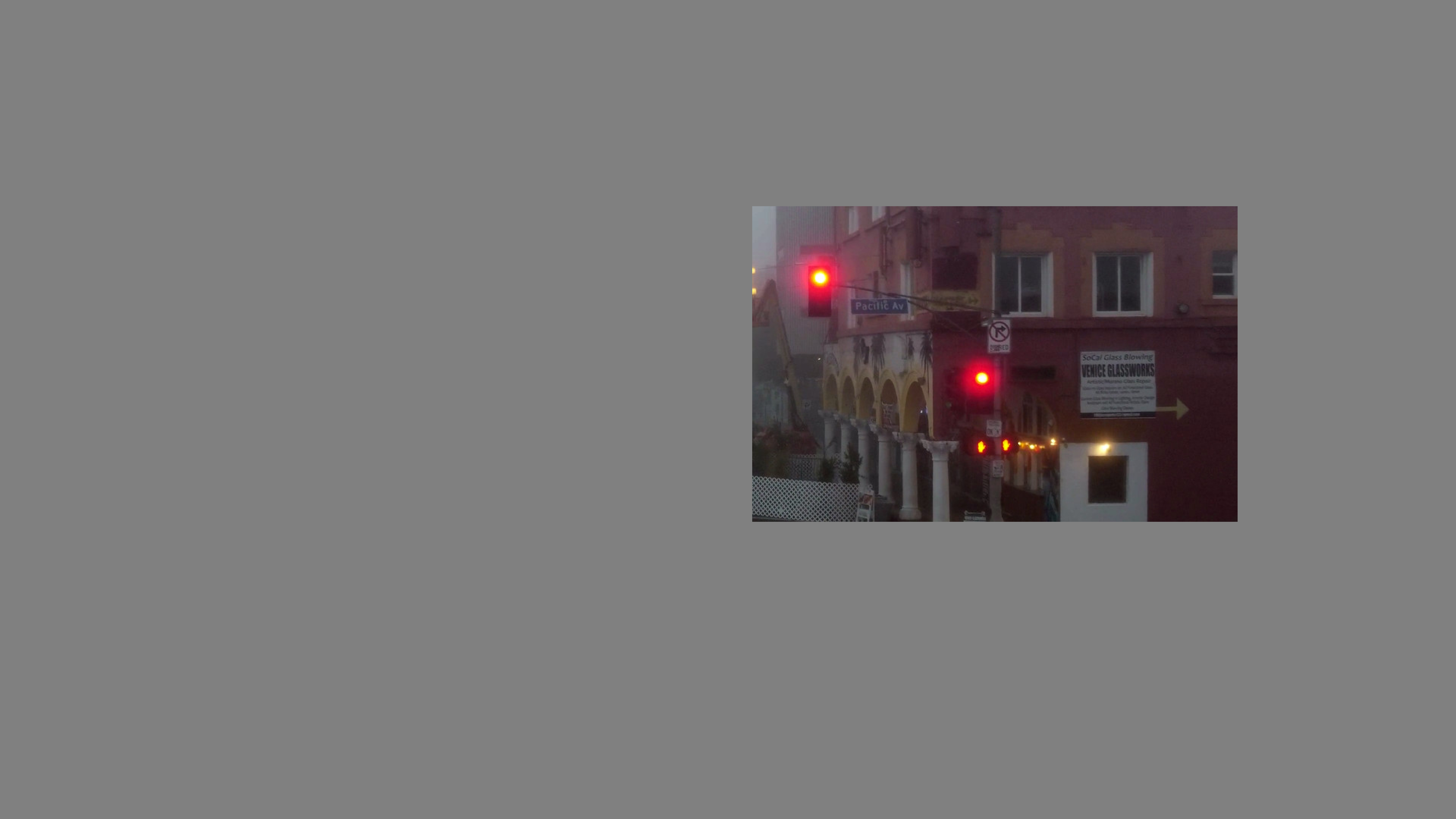}{2074 864 653 639}} \\
  \end{tabular}
  \caption{\textbf{Real-world 4K dehazing across eight consecutive frames.} Top: lake-and-skyline scene; bottom: Venice street. Row labels: (a)~CG-IDN, (b)~DVD, (c)~ViWS-Net, (d)~MAP-Net, \textbf{(e)}~Ours. Red boxes in the hazy row indicate the cropped region shown in rows (a)--(e). LiBrA-Net preserves thin structures and avoids color casts across both scenes; competing methods introduce halos, residual haze, or saturated shifts.}
  \label{fig:realworld_qualitative}
\end{figure}

\section{Grid Representation Structure}
\label{app:grid_anatomy}

The spatial--color grid $\mathcal{G}_\chi$ does not reduce to a depth-only template: it organizes affine coefficients across the full chromatic axis. Figure~\ref{fig:grid_anatomy_appendix} visualizes this structure.
Panel~(a) plots per-bin $\|\mathcal{G}_\chi\|_2$ across all eight color bins and 12 affine channels; every bin carries substantial magnitude with distinct per-channel profiles, confirming that the grid encodes color-dependent correction rather than a single depth-driven gain. Panel~(b) applies t-SNE to 20K randomly sampled grid cells and colors each point by its color-bin index; the clear cluster separation shows that the learned representation self-organizes along the chromatic guide without explicit supervision on the color axis.
These observations complement the Lie-algebraic composition analysis in \S\ref{sec:analysis_chapter}: the grid is both structurally interpretable (color-organized affine descriptor) and mathematically well-behaved (near-identity regime where additive fusion approximates multiplicative composition).

\begin{figure}[h]
  \centering
  \begin{subfigure}[b]{0.42\linewidth}
    \centering
    \includegraphics[width=\linewidth]{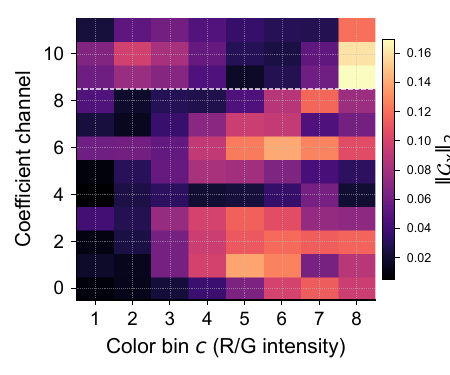}
    \caption{Per-bin $\|\mathcal{G}_\chi\|_2$ across eight color bins and 12 affine channels.}
  \end{subfigure}\hfill
  \begin{subfigure}[b]{0.42\linewidth}
    \centering
    \includegraphics[width=\linewidth]{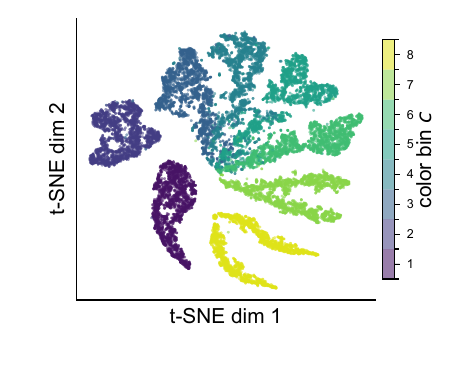}
    \caption{t-SNE of 20K grid cells colored by color-bin index.}
  \end{subfigure}
  \caption{\textbf{Anatomy of the spatial--color grid.}
    The grid exploits all eight color bins with distinct affine profiles~(a) and self-organizes cells into color-coherent clusters~(b), confirming a structured chromatic representation.}
  \label{fig:grid_anatomy_appendix}
\end{figure}

\section{Resolution Invariance of the Bilateral Grid}
\label{app:resolution_invariance}

This section provides the full quantitative analysis of resolution invariance summarized in \S\ref{sec:analysis_chapter}.
We probe the learned grid by predicting it at four resolutions below the native size and comparing against the native-resolution grid via relative Frobenius distance and Pearson~$r$.
Across a ${\geq}\,10{\times}$ resolution range on UHV-4K and ${\geq}\,7{\times}$ on REVIDE, both statistics converge monotonically toward the native anchor (Table~\ref{tab:resolution_invariance}).
Bilinear downsampling discards high-frequency texture but preserves the low-frequency, depth-driven affine statistic the grid encodes, leaving the trilinear slice as the only resolution-dependent operation in the forward pass.
The temporal grid $\mathcal{G}_\tau$ tracks the spatial--color grid $\mathcal{G}_\chi$ closely at every resolution, and the per-frame grid heads preserve frame-specific variation that frame-wise averaging would destroy.

\begin{table}[h]
  \centering
  \caption{\textbf{Resolution invariance of the learned bilateral grid.}
    Each row compares the grid predicted at a downsampled input against the grid at native resolution ($\text{mean}{\pm}\text{std}$ over the test split).}
  \label{tab:resolution_invariance}
  \small
  \setlength{\tabcolsep}{6pt}
  \renewcommand{\arraystretch}{1}
  \begin{tabular}{lcccc}
    \toprule
    \multirow{2}{*}{Input long-edge} & \multicolumn{2}{c}{Relative Frobenius $\downarrow$} & \multicolumn{2}{c}{Pearson $r$ $\uparrow$} \\
    \cmidrule(lr){2-3} \cmidrule(l){4-5}
     & $\mathcal{G}_\chi$ & $\mathcal{G}_\tau$ & $\mathcal{G}_\chi$ & $\mathcal{G}_\tau$ \\
    \midrule
    \tablegroup{5}{(a) UHV-4K ($3840{\times}2160$)}
    \midrule
    360p  & 0.078\,$\pm$\,0.035 & 0.086\,$\pm$\,0.072 & 0.997\,$\pm$\,0.002 & 0.995\,$\pm$\,0.008 \\
    540p  & 0.061\,$\pm$\,0.034 & 0.062\,$\pm$\,0.052 & 0.998\,$\pm$\,0.002 & 0.997\,$\pm$\,0.005 \\
    1080p & 0.032\,$\pm$\,0.018 & 0.031\,$\pm$\,0.017 & 0.999\,$\pm$\,0.001 & 0.999\,$\pm$\,0.001 \\
    1440p & 0.025\,$\pm$\,0.013 & 0.023\,$\pm$\,0.013 & 1.000\,$\pm$\,0.000 & 1.000\,$\pm$\,0.000 \\
    \rowcolor{ourgreen} \textbf{4K \textit{(native anchor)}} & \textbf{0.000} & \textbf{0.000} & \textbf{1.000} & \textbf{1.000} \\
    \midrule
    \tablegroup{5}{(b) REVIDE ($2708{\times}1800$)}
    \midrule
    360p  & 0.028\,$\pm$\,0.018 & 0.048\,$\pm$\,0.019 & 0.999\,$\pm$\,0.001 & 0.998\,$\pm$\,0.001 \\
    540p  & 0.019\,$\pm$\,0.013 & 0.032\,$\pm$\,0.012 & 1.000\,$\pm$\,0.000 & 0.999\,$\pm$\,0.001 \\
    1080p & 0.009\,$\pm$\,0.007 & 0.013\,$\pm$\,0.005 & 1.000\,$\pm$\,0.000 & 1.000\,$\pm$\,0.000 \\
    \rowcolor{ourgreen} \textbf{2708p \textit{(native anchor)}} & \textbf{0.000} & \textbf{0.000} & \textbf{1.000} & \textbf{1.000} \\
    \bottomrule
  \end{tabular}
\end{table}

\section{Parameterization and Fusion Alternatives}
\label{app:param_choice}

This section reports the alternatives to additive Lie-algebra fusion and to the Cayley map cited in the method.

Sequential composition applies the two branch transforms one after the other. It needs two Cayley solves and a matrix product per pixel, and the gradient reaching each grid head is multiplied by the other branch's current transform. Trained from scratch on UHV-4K, additive fusion reaches 2.3\,dB higher PSNR and runs at 25.3\,FPS rather than 22.1\,FPS.

Replacing Cayley by the matrix exponential lowers native 4K throughput to 20.4\,FPS. Projecting $I{+}M$ onto invertible matrices by an SVD lowers it to 3.8\,FPS. Both alternatives also train to lower PSNR than Cayley. Cayley remains the fastest of the maps we tested from $\mathfrak{gl}(3)$ to invertible matrices, and it needs no matrix logarithm because the grids predict $M$ directly.

\section{Notation}
\label{app:notation}

Table~\ref{tab:notation} collects the symbols used in the method. Bold symbols are per-pixel color quantities: lowercase bold for vectors in $\mathbb{R}^3$, uppercase bold for $3{\times}3$ matrices. Non-bold italic symbols are scalars. $M$ is an element of $\mathfrak{gl}(3)$, $I$ is the $3{\times}3$ identity, and $x$ is a pixel location.

\begin{table}[h]
  \centering
  \caption{Symbols used in the method.}
  \label{tab:notation}
  \small
  \setlength{\tabcolsep}{8pt}
  \begin{tabular}{@{}ll@{}}
    \toprule
    Symbol & Meaning \\
    \midrule
    $\mathbf{I}(x)$ & Hazy radiance, $\mathbb{R}^{3}$ \\
    $\mathbf{J}(x)$ & Clean radiance, $\mathbb{R}^{3}$ \\
    $\mathbf{A}_\infty$ & Global atmospheric light, $\mathbb{R}^{3}$ \\
    $t(x)$ & Transmission, scalar in $(0,1]$ \\
    $d(x)$ & Normalized depth \\
    $\beta$ & Scattering coefficient \\
    $a(x)$ & Scalar gain, $1/t(x)$ \\
    $\mathbf{b}(x)$ & Per-pixel translation, $\mathbb{R}^{3}$ \\
    $\mathbf{A}(x)$ & Per-pixel affine matrix, $\mathbf{A}(x)\in GL(3)$ \\
    $\mathbf{c}(x)$ & Twelve affine coefficients at pixel $x$ \\
    $M$ & Matrix part of $\mathbf{c}$, an element of $\mathfrak{gl}(3)$ \\
    $I$ & $3{\times}3$ identity \\
    $x$ & Pixel location \\
    $GL(3)$ & Real invertible $3{\times}3$ matrices \\
    $\mathfrak{gl}(3)$ & All real $3{\times}3$ matrices; the Lie algebra of $GL(3)$ \\
    $\mathfrak{aff}(3)$ & Lie algebra of the affine group, $\mathfrak{gl}(3)\oplus\mathbb{R}^{3}$ \\
    $\mathcal{G}_\chi$ & Spatial--color bilateral grid \\
    $\mathcal{G}_\tau$ & Temporal bilateral grid \\
    $G_c, G_h, G_w$ & Color bins and spatial grid size; fixed at $8$ and $16{\times}16$ \\
    $T$ & Frames in the window; $T{=}5$ \\
    $S$ & Encoder spatial size; $S{=}256$ \\
    $\operatorname{Cay}$ & Cayley map, Eq.~\eqref{eq:cayley} \\
    $H, W$ & Output height and width \\
    $p$ & Slice stride; $p{=}4$ \\
    \bottomrule
  \end{tabular}
\end{table}

\section{What the Temporal Field Uses}
\label{app:temporal_role}

The Temporal Field does not track motion. On UHV-4K test clips, reversing the order of the four neighbour frames changes PSNR by $-0.004$\,dB, while replacing those neighbours with frames from another clip costs $2.13$\,dB. Order is irrelevant and identity is not, which is what repeated observations of one slowly varying transmission field look like, and not what a motion model looks like. A linear probe read from the temporal grid recovers transmission at $R^{2}{\approx}0.53$.

The gain is also not extra capacity. A control keeps the parameter count at $6.12$\,M and the cost at $282$\,GFLOPs, and only stops the center query from attending to its neighbours. Under that protocol the self-only variant reaches $22.17$\,dB on UHV-4K, against $22.46$\,dB and $22.40$\,dB for the same architecture with cross-frame attention. On REVIDE, where consecutive frames barely move, the same control is flat ($20.50$ against $20.46$ and $20.71$). The temporal term in Table~\ref{tab:ablation}(a) is therefore carried by cross-frame attention.


\clearpage
\section*{NeurIPS Paper Checklist}

\begin{enumerate}

\item {\bf Claims}
    \item[] Question: Do the main claims made in the abstract and introduction accurately reflect the paper's contributions and scope?
    \item[] Answer: \answerYes{} 
    \item[] Justification: The Abstract and Introduction state the paper's two main contributions: UHV-4K and LiBrA-Net. These claims are supported by the dataset construction, method description, and experiments in Sections~\ref{sec:dataset}--\ref{sec:analysis_grid_geometry}.
    \item[] Guidelines: 
    \begin{itemize}
        \item The answer \answerNA{} means that the abstract and introduction do not include the claims made in the paper.
        \item The abstract and/or introduction should clearly state the claims made, including the contributions made in the paper and important assumptions and limitations. A \answerNo{} or \answerNA{} answer to this question will not be perceived well by the reviewers. 
        \item The claims made should match theoretical and experimental results, and reflect how much the results can be expected to generalize to other settings. 
        \item It is fine to include aspirational goals as motivation as long as it is clear that these goals are not attained by the paper. 
    \end{itemize}

\item {\bf Limitations}
    \item[] Question: Does the paper discuss the limitations of the work performed by the authors?
    \item[] Answer: \answerYes{} 
    \item[] Justification: The paper includes a dedicated Limitations and Future Work section, discussing the finite temporal window, the ASM assumptions, mixed-weather cases, and the lack of paired real 4K hazy/clean videos.
    \item[] Guidelines:
    \begin{itemize}
        \item The answer \answerNA{} means that the paper has no limitation while the answer \answerNo{} means that the paper has limitations, but those are not discussed in the paper. 
        \item The authors are encouraged to create a separate ``Limitations'' section in their paper.
        \item The paper should point out any strong assumptions and how robust the results are to violations of these assumptions (e.g., independence assumptions, noiseless settings, model well-specification, asymptotic approximations only holding locally). The authors should reflect on how these assumptions might be violated in practice and what the implications would be.
        \item The authors should reflect on the scope of the claims made, e.g., if the approach was only tested on a few datasets or with a few runs. In general, empirical results often depend on implicit assumptions, which should be articulated.
        \item The authors should reflect on the factors that influence the performance of the approach. For example, a facial recognition algorithm may perform poorly when image resolution is low or images are taken in low lighting. Or a speech-to-text system might not be used reliably to provide closed captions for online lectures because it fails to handle technical jargon.
        \item The authors should discuss the computational efficiency of the proposed algorithms and how they scale with dataset size.
        \item If applicable, the authors should discuss possible limitations of their approach to address problems of privacy and fairness.
        \item While the authors might fear that complete honesty about limitations might be used by reviewers as grounds for rejection, a worse outcome might be that reviewers discover limitations that aren't acknowledged in the paper. The authors should use their best judgment and recognize that individual actions in favor of transparency play an important role in developing norms that preserve the integrity of the community. Reviewers will be specifically instructed to not penalize honesty concerning limitations.
    \end{itemize}

\item {\bf Theory assumptions and proofs}
    \item[] Question: For each theoretical result, does the paper provide the full set of assumptions and a complete (and correct) proof?
    \item[] Answer: \answerYes{} 
    \item[] Justification: The paper's formal theoretical statement is Proposition~\ref{prop:lie_composition}. A proof sketch with intuition is provided in the method section, and the full formal proof is given in Appendix~\ref{app:proof_lie}.
    \item[] Guidelines:
    \begin{itemize}
        \item The answer \answerNA{} means that the paper does not include theoretical results. 
        \item All the theorems, formulas, and proofs in the paper should be numbered and cross-referenced.
        \item All assumptions should be clearly stated or referenced in the statement of any theorems.
        \item The proofs can either appear in the main paper or the supplemental material, but if they appear in the supplemental material, the authors are encouraged to provide a short proof sketch to provide intuition. 
        \item Inversely, any informal proof provided in the core of the paper should be complemented by formal proofs provided in appendix or supplemental material.
        \item Theorems and Lemmas that the proof relies upon should be properly referenced. 
    \end{itemize}

    \item {\bf Experimental result reproducibility}
    \item[] Question: Does the paper fully disclose all the information needed to reproduce the main experimental results of the paper to the extent that it affects the main claims and/or conclusions of the paper (regardless of whether the code and data are provided or not)?
    \item[] Answer: \answerYes{} 
    \item[] Justification: The paper specifies the model architecture, datasets, splits, metrics, baselines, training schedule, optimizer, crop size, temporal window, and inference protocol in the dataset, method, and experiment sections.
    \item[] Guidelines:
    \begin{itemize}
        \item The answer \answerNA{} means that the paper does not include experiments.
        \item If the paper includes experiments, a \answerNo{} answer to this question will not be perceived well by the reviewers: Making the paper reproducible is important, regardless of whether the code and data are provided or not.
        \item If the contribution is a dataset and\slash or model, the authors should describe the steps taken to make their results reproducible or verifiable. 
        \item Depending on the contribution, reproducibility can be accomplished in various ways. For example, if the contribution is a novel architecture, describing the architecture fully might suffice, or if the contribution is a specific model and empirical evaluation, it may be necessary to either make it possible for others to replicate the model with the same dataset, or provide access to the model. In general. releasing code and data is often one good way to accomplish this, but reproducibility can also be provided via detailed instructions for how to replicate the results, access to a hosted model (e.g., in the case of a large language model), releasing of a model checkpoint, or other means that are appropriate to the research performed.
        \item While NeurIPS does not require releasing code, the conference does require all submissions to provide some reasonable avenue for reproducibility, which may depend on the nature of the contribution. For example
        \begin{enumerate}
            \item If the contribution is primarily a new algorithm, the paper should make it clear how to reproduce that algorithm.
            \item If the contribution is primarily a new model architecture, the paper should describe the architecture clearly and fully.
            \item If the contribution is a new model (e.g., a large language model), then there should either be a way to access this model for reproducing the results or a way to reproduce the model (e.g., with an open-source dataset or instructions for how to construct the dataset).
            \item We recognize that reproducibility may be tricky in some cases, in which case authors are welcome to describe the particular way they provide for reproducibility. In the case of closed-source models, it may be that access to the model is limited in some way (e.g., to registered users), but it should be possible for other researchers to have some path to reproducing or verifying the results.
        \end{enumerate}
    \end{itemize}

\item {\bf Open access to data and code}
    \item[] Question: Does the paper provide open access to the data and code, with sufficient instructions to faithfully reproduce the main experimental results, as described in supplemental material?
    \item[] Answer: \answerYes{} 
    \item[] Justification: We submit our codes as supplemental material and provide instructions for reproducing our results. The data is public available.
    \item[] Guidelines:
    \begin{itemize}
        \item The answer \answerNA{} means that paper does not include experiments requiring code.
        \item Please see the NeurIPS code and data submission guidelines (\url{https://neurips.cc/public/guides/CodeSubmissionPolicy}) for more details.
        \item While we encourage the release of code and data, we understand that this might not be possible, so \answerNo{} is an acceptable answer. Papers cannot be rejected simply for not including code, unless this is central to the contribution (e.g., for a new open-source benchmark).
        \item The instructions should contain the exact command and environment needed to run to reproduce the results. See the NeurIPS code and data submission guidelines (\url{https://neurips.cc/public/guides/CodeSubmissionPolicy}) for more details.
        \item The authors should provide instructions on data access and preparation, including how to access the raw data, preprocessed data, intermediate data, and generated data, etc.
        \item The authors should provide scripts to reproduce all experimental results for the new proposed method and baselines. If only a subset of experiments are reproducible, they should state which ones are omitted from the script and why.
        \item At submission time, to preserve anonymity, the authors should release anonymized versions (if applicable).
        \item Providing as much information as possible in supplemental material (appended to the paper) is recommended, but including URLs to data and code is permitted.
    \end{itemize}

\item {\bf Experimental setting/details}
    \item[] Question: Does the paper specify all the training and test details (e.g., data splits, hyperparameters, how they were chosen, type of optimizer) necessary to understand the results?
    \item[] Answer: \answerYes{} 
    \item[] Justification: Section~\ref{sec:exp_settings} specifies the datasets, official or standard splits, evaluation metrics, optimizer, learning rate, weight decay, warmup, epochs, batch size, crop size, mixed precision, and native-resolution inference protocol.
    \item[] Guidelines:
    \begin{itemize}
        \item The answer \answerNA{} means that the paper does not include experiments.
        \item The experimental setting should be presented in the core of the paper to a level of detail that is necessary to appreciate the results and make sense of them.
        \item The full details can be provided either with the code, in appendix, or as supplemental material.
    \end{itemize}

\item {\bf Experiment statistical significance}
    \item[] Question: Does the paper report error bars suitably and correctly defined or other appropriate information about the statistical significance of the experiments?
    \item[] Answer: \answerNo{} 
    \item[] Justification: The perceptual study reports 95\% confidence intervals, but the main restoration and ablation tables report deterministic benchmark averages under fixed splits rather than repeated-training error bars. Repeating all sixteen compared models across multiple seeds would be computationally prohibitive.
    \item[] Guidelines:
    \begin{itemize}
        \item The answer \answerNA{} means that the paper does not include experiments.
        \item The authors should answer \answerYes{} if the results are accompanied by error bars, confidence intervals, or statistical significance tests, at least for the experiments that support the main claims of the paper.
        \item The factors of variability that the error bars are capturing should be clearly stated (for example, train/test split, initialization, random drawing of some parameter, or overall run with given experimental conditions).
        \item The method for calculating the error bars should be explained (closed form formula, call to a library function, bootstrap, etc.)
        \item The assumptions made should be given (e.g., Normally distributed errors).
        \item It should be clear whether the error bar is the standard deviation or the standard error of the mean.
        \item It is OK to report 1-sigma error bars, but one should state it. The authors should preferably report a 2-sigma error bar than state that they have a 96\% CI, if the hypothesis of Normality of errors is not verified.
        \item For asymmetric distributions, the authors should be careful not to show in tables or figures symmetric error bars that would yield results that are out of range (e.g., negative error rates).
        \item If error bars are reported in tables or plots, the authors should explain in the text how they were calculated and reference the corresponding figures or tables in the text.
    \end{itemize}

\item {\bf Experiments compute resources}
    \item[] Question: For each experiment, does the paper provide sufficient information on the computer resources (type of compute workers, memory, time of execution) needed to reproduce the experiments?
    \item[] Answer: \answerNo{} 
    \item[] Justification: The paper reports the main GPU setup for training and native-4K inference, but does not yet provide per-experiment wall-clock time, memory usage, or total compute estimates for all reported runs.
    \item[] Guidelines:
    \begin{itemize}
        \item The answer \answerNA{} means that the paper does not include experiments.
        \item The paper should indicate the type of compute workers CPU or GPU, internal cluster, or cloud provider, including relevant memory and storage.
        \item The paper should provide the amount of compute required for each of the individual experimental runs as well as estimate the total compute. 
        \item The paper should disclose whether the full research project required more compute than the experiments reported in the paper (e.g., preliminary or failed experiments that didn't make it into the paper). 
    \end{itemize}
    
\item {\bf Code of ethics}
    \item[] Question: Does the research conducted in the paper conform, in every respect, with the NeurIPS Code of Ethics \url{https://neurips.cc/public/EthicsGuidelines}?
    \item[] Answer: \answerYes{} 
    \item[] Justification: The work uses public video sources, synthetic haze generation, restoration benchmarks, and a low-risk perceptual rating protocol. We reviewed the NeurIPS Code of Ethics and did not identify deviations from its requirements.
    \item[] Guidelines:
    \begin{itemize}
        \item The answer \answerNA{} means that the authors have not reviewed the NeurIPS Code of Ethics.
        \item If the authors answer \answerNo, they should explain the special circumstances that require a deviation from the Code of Ethics.
        \item The authors should make sure to preserve anonymity (e.g., if there is a special consideration due to laws or regulations in their jurisdiction).
    \end{itemize}

\item {\bf Broader impacts}
    \item[] Question: Does the paper discuss both potential positive societal impacts and negative societal impacts of the work performed?
    \item[] Answer: \answerNo{} 
    \item[] Justification: The current manuscript focuses on the technical and benchmark contributions and does not include a dedicated discussion of positive and negative societal impacts.
    \item[] Guidelines:
    \begin{itemize}
        \item The answer \answerNA{} means that there is no societal impact of the work performed.
        \item If the authors answer \answerNA{} or \answerNo, they should explain why their work has no societal impact or why the paper does not address societal impact.
        \item Examples of negative societal impacts include potential malicious or unintended uses (e.g., disinformation, generating fake profiles, surveillance), fairness considerations (e.g., deployment of technologies that could make decisions that unfairly impact specific groups), privacy considerations, and security considerations.
        \item The conference expects that many papers will be foundational research and not tied to particular applications, let alone deployments. However, if there is a direct path to any negative applications, the authors should point it out. For example, it is legitimate to point out that an improvement in the quality of generative models could be used to generate Deepfakes for disinformation. On the other hand, it is not needed to point out that a generic algorithm for optimizing neural networks could enable people to train models that generate Deepfakes faster.
        \item The authors should consider possible harms that could arise when the technology is being used as intended and functioning correctly, harms that could arise when the technology is being used as intended but gives incorrect results, and harms following from (intentional or unintentional) misuse of the technology.
        \item If there are negative societal impacts, the authors could also discuss possible mitigation strategies (e.g., gated release of models, providing defenses in addition to attacks, mechanisms for monitoring misuse, mechanisms to monitor how a system learns from feedback over time, improving the efficiency and accessibility of ML).
    \end{itemize}
    
\item {\bf Safeguards}
    \item[] Question: Does the paper describe safeguards that have been put in place for responsible release of data or models that have a high risk for misuse (e.g., pre-trained language models, image generators, or scraped datasets)?
    \item[] Answer: \answerNA{} 
    \item[] Justification: The paper does not release a high-risk generative model, language model, scraped personal dataset, or other artifact with a direct high-risk misuse profile. The released assets are intended for video restoration benchmarking.
    \item[] Guidelines:
    \begin{itemize}
        \item The answer \answerNA{} means that the paper poses no such risks.
        \item Released models that have a high risk for misuse or dual-use should be released with necessary safeguards to allow for controlled use of the model, for example by requiring that users adhere to usage guidelines or restrictions to access the model or implementing safety filters. 
        \item Datasets that have been scraped from the Internet could pose safety risks. The authors should describe how they avoided releasing unsafe images.
        \item We recognize that providing effective safeguards is challenging, and many papers do not require this, but we encourage authors to take this into account and make a best faith effort.
    \end{itemize}

\item {\bf Licenses for existing assets}
    \item[] Question: Are the creators or original owners of assets (e.g., code, data, models), used in the paper, properly credited and are the license and terms of use explicitly mentioned and properly respected?
    \item[] Answer: \answerYes{} 
    \item[] Justification: Proper citations are provided throughout the document and the licenses will be included with the code when it is released.
    \item[] Guidelines:
    \begin{itemize}
        \item The answer \answerNA{} means that the paper does not use existing assets.
        \item The authors should cite the original paper that produced the code package or dataset.
        \item The authors should state which version of the asset is used and, if possible, include a URL.
        \item The name of the license (e.g., CC-BY 4.0) should be included for each asset.
        \item For scraped data from a particular source (e.g., website), the copyright and terms of service of that source should be provided.
        \item If assets are released, the license, copyright information, and terms of use in the package should be provided. For popular datasets, \url{paperswithcode.com/datasets} has curated licenses for some datasets. Their licensing guide can help determine the license of a dataset.
        \item For existing datasets that are re-packaged, both the original license and the license of the derived asset (if it has changed) should be provided.
        \item If this information is not available online, the authors are encouraged to reach out to the asset's creators.
    \end{itemize}

\item {\bf New assets}
    \item[] Question: Are new assets introduced in the paper well documented and is the documentation provided alongside the assets?
    \item[] Answer: \answerYes{} 
    \item[] Justification: UHV-4K is documented in Section~\ref{sec:dataset}, including source videos, sampling, synthesis protocol, modalities, splits, and quality checks. Additional asset documentation and usage instructions are provided with the released dataset package.
    \item[] Guidelines:
    \begin{itemize}
        \item The answer \answerNA{} means that the paper does not release new assets.
        \item Researchers should communicate the details of the dataset\slash code\slash model as part of their submissions via structured templates. This includes details about training, license, limitations, etc. 
        \item The paper should discuss whether and how consent was obtained from people whose asset is used.
        \item At submission time, remember to anonymize your assets (if applicable). You can either create an anonymized URL or include an anonymized zip file.
    \end{itemize}

\item {\bf Crowdsourcing and research with human subjects}
    \item[] Question: For crowdsourcing experiments and research with human subjects, does the paper include the full text of instructions given to participants and screenshots, if applicable, as well as details about compensation (if any)? 
    \item[] Answer: \answerNo{} 
    \item[] Justification: The paper reports a perceptual rating study with blind raters, but the current manuscript does not include the full participant instructions, interface screenshots, or compensation details.
    \item[] Guidelines:
    \begin{itemize}
        \item The answer \answerNA{} means that the paper does not involve crowdsourcing nor research with human subjects.
        \item Including this information in the supplemental material is fine, but if the main contribution of the paper involves human subjects, then as much detail as possible should be included in the main paper. 
        \item According to the NeurIPS Code of Ethics, workers involved in data collection, curation, or other labor should be paid at least the minimum wage in the country of the data collector. 
    \end{itemize}

\item {\bf Institutional review board (IRB) approvals or equivalent for research with human subjects}
    \item[] Question: Does the paper describe potential risks incurred by study participants, whether such risks were disclosed to the subjects, and whether Institutional Review Board (IRB) approvals (or an equivalent approval/review based on the requirements of your country or institution) were obtained?
    \item[] Answer: \answerNo{} 
    \item[] Justification: The paper includes a perceptual rating study, but the current manuscript does not yet state whether IRB approval, exemption, or an equivalent institutional review was obtained.
    \item[] Guidelines:
    \begin{itemize}
        \item The answer \answerNA{} means that the paper does not involve crowdsourcing nor research with human subjects.
        \item Depending on the country in which research is conducted, IRB approval (or equivalent) may be required for any human subjects research. If you obtained IRB approval, you should clearly state this in the paper. 
        \item We recognize that the procedures for this may vary significantly between institutions and locations, and we expect authors to adhere to the NeurIPS Code of Ethics and the guidelines for their institution. 
        \item For initial submissions, do not include any information that would break anonymity (if applicable), such as the institution conducting the review.
    \end{itemize}

\item {\bf Declaration of LLM usage}
    \item[] Question: Does the paper describe the usage of LLMs if it is an important, original, or non-standard component of the core methods in this research? Note that if the LLM is used only for writing, editing, or formatting purposes and does \emph{not} impact the core methodology, scientific rigor, or originality of the research, declaration is not required.
    \item[] Answer: \answerNA{} 
    \item[] Justification: LLMs are not used as an original, important, or non-standard component of the proposed method, dataset generation pipeline, or experimental evaluation. Any language assistance, if used, is limited to writing or formatting and does not affect the scientific content.
    \item[] Guidelines:
    \begin{itemize}
        \item The answer \answerNA{} means that the core method development in this research does not involve LLMs as any important, original, or non-standard components.
        \item Please refer to our LLM policy in the NeurIPS handbook for what should or should not be described.
    \end{itemize}

\end{enumerate}

\end{document}